%% file: npbai.tex
\documentclass{article}

\PassOptionsToPackage{numbers, compress}{natbib}

\usepackage[final]{neurips_2022}

\usepackage[utf8]{inputenc} 
\usepackage[T1]{fontenc}    
\usepackage{hyperref}       
\usepackage{url}            
\usepackage{booktabs}       
\usepackage{amsfonts}       
\usepackage{nicefrac}       
\usepackage{microtype}      
\usepackage{xcolor}         

\usepackage{enumitem}

\usepackage{macros}
\usepackage{amsthm,amsmath,amssymb}
\newtheorem{lemma}{Lemma}
\newtheorem*{lemma*}{Lemma}
\newtheorem{theorem}{Theorem}
\newtheorem{corollary}{Corollary}
\newtheorem{assumption}{Assumption}
\newtheorem{property}{Property}
\newtheorem{remark}{Remark}
\newtheorem{definition}{Definition}

\title{Top Two Algorithms Revisited}

\author{%
	Marc Jourdan$^1$ \\
	\texttt{marc.jourdan@inria.fr}\\
	\And
	R\'emy Degenne$^1$ \\
	\texttt{remy.degenne@inria.fr}\\
	\And
 	Dorian Baudry$^1$ \\
	\texttt{dorian.baudry@inria.fr}\\
	\And
 	Rianne de Heide$^2$ \\
	\texttt{r.de.heide@vu.nl}\\
	\And
	Emilie Kaufmann$^1$ \\
	\texttt{emilie.kaufmann@univ-lille.fr}\\
	\\
	$^1$ Univ. Lille, CNRS, Inria, Centrale Lille, UMR 9198-CRIStAL, F-59000 Lille, France \\
	$^2$ Vrije Universiteit Amsterdam
}

\begin{document}

\maketitle

\begin{abstract}
  Top Two algorithms arose as an adaptation of Thompson sampling to best arm identification in multi-armed bandit models \cite{Russo2016TTTS}, for parametric families of arms. They select the next arm to sample from by randomizing among two candidate arms, a \emph{leader} and a \emph{challenger}.
  Despite their good empirical performance, theoretical guarantees for fixed-confidence best arm identification have only been obtained when the arms are Gaussian with known variances.
  In this paper, we provide a general analysis of Top Two methods, which identifies desirable properties of the leader, the challenger, and the (possibly non-parametric) distributions of the arms.
  As a result, we obtain theoretically supported Top Two algorithms for best arm identification with bounded distributions.
  Our proof method demonstrates in particular that the sampling step used to select the leader inherited from Thompson sampling can be replaced by other choices, like selecting the empirical best arm.
\end{abstract}

\input{sections/introduction}

\input{sections/algorithms}

\input{sections/sampler_bounded}

\input{sections/analysis}

\input{sections/experiments}

\section{Perspectives}
\label{sec:conclusion}

We provided a general analysis of Top Two algorithms, including new variants using the EB leader and TCI challenger, and proved their asymptotic $\beta$-optimality on the non-parametric class of bounded distributions. On experiments on distributions coming from a real world application, several Top Two variants (in particular $\beta$-TS-TC and $\beta$-EB-TCI) proved more effective than all baselines. Furthermore, $\beta$-EB-TCI is computationally not costlier than computing the stopping rule.

As in previous work on Top Two methods our result only characterizes the asymptotic performance of the algorithms, and obtaining bounds on the sample complexity for any $\delta$ that would reflect their good empirical performance is a most pressing open question.
Our work also hints at what is needed to obtain non-asymptotic guarantees: the only variant for which the empirical behavior does not reflect the asymptotic bound is $\beta$-EB-TC, which is also the most greedy variant. Algorithms using a sampler naturally explore, and the penalized version $\beta$-EB-TCI successfully corrects the shortcomings of $\beta$-EB-TC by penalizing over-sampling.
Quantifying the amount of exploration required by Top Two algorithms should also allow the removal of the hypothesis $\Delta_{\min} > 0$ from Theorem~\ref{thm:beta-opt}.

Finally, Top Two algorithms are promising algorithms to tackle the setting of fixed budget identification, in which the algorithms have to stop at a given time and should then make as few mistakes as possible.
As their sampling rule is anytime (i.e.\ independent of $\delta$), Top Two algorithms might also have theoretical guarantees for BAI in the fixed-budget setting or even the anytime one, in which guarantees on the error probability should be given at all time.

\begin{ack}
	Experiments presented in this paper were carried out using the Grid'5000 testbed, supported by a scientific interest group hosted by Inria and including CNRS, RENATER and several Universities as well as other organizations (see https://www.grid5000.fr). This work has been partially supported by the THIA ANR program ``AI\_PhD@Lille''. The authors acknowledge the funding of the French National Research Agency under the project BOLD (ANR-19-CE23-0026-04), and the Dutch Research Council (NWO) Rubicon grant number 019.202EN.004.
\end{ack}

\bibliographystyle{abbrvnat}
\bibliography{npbai}

\clearpage
\appendix

\section{Outline} \label{app:outline}

The appendices are organized as follows:
\begin{itemize}
	\item Notation are summarized in Appendix~\ref{app:notation}.
	\item We present a unified analysis of Top Two algorithms in Appendix~\ref{app:unified_analysis_top_two}, which highlights key properties on the leader and challenger mechanisms.
	\item In Appendix~\ref{app:top_two_instances}, we analyze several instances for the leader and the challenger mechanisms.
	\item In Appendix~\ref{app:concentration}, we show Lemma~\ref{lem:kinf_concentration} and derive results from concentration on sub-Gaussian random variables.
	\item Appendix~\ref{app:kinf_for_bounded_distributions} gathers key properties on $\Kinf^{\pm}$, including new ones which are required for BAI.
	\item In Appendix~\ref{app:boundary_crossing_probability_bounds}, we show lower and upper bounds on Boundary Crossing Probability (BCP) and on $\bP_{n}[\theta_i \geq \theta_j]$ for the Dirichlet sampler.
	\item The generalization to single-parameter exponential families is done in Appendix~\ref{app:spef}.
	\item Implementation details and additional experiments are presented in Appendix~\ref{app:additional_experiments}.
\end{itemize}

\begin{table}[H]
\caption{Notation for the setting.}
\label{tab:notation_table_setting}
\begin{center}
\begin{tabular}{c c l}
	\toprule
Notation & Type & Description \\
\midrule
$K$ & $\N$ & Number of arms \\
$B$ & $\R^{\star}_{+}$ & Upper bound for bounded distributions \\
$\mathcal P(\mathbb{R})$ & & Probability distributions over $\mathbb{R}$ \\
$\cF$ & & Set of distributions, e.g. bounded distributions on $[0, B]$ \\
$F_i$ & $\cF$ & CDF of the distribution of arm $i \in [K]$ \\
$\bm F$ & $\cF^K$ & $\bm F \eqdef (F_i)_{i \in [K]}$ \\
$m$ & $\cF \to \R$ & Mean operator, $m(F) \eqdef \bE_{X \sim F}[X]$ \\
$\cI \subseteq \R$ &  & Interval of means $\cI \eqdef \{m(F) \mid F \in \cF\}$, e.g. $[0, B]$ for bounded \\
$\mu_i$ & $\mathring \cI$ & Mean of arm $i \in [K]$, i.e. $\mu_i \eqdef m(F_i)$ \\
$\mu$ & $(\mathring \cI)^K$ & Vector of means, $\mu \eqdef (\mu_i)_{i \in [K]}$ \\
$i^\star$ & $\cF^K \to [K]$ & Best arm operator, $i^\star(\bm F) \in \argmax_{i \in [K]} \mu_i$ \\
$T^\star(\bm F)$ & $\R^{\star}_{+}$ & Asymptotic characteristic time \\
$T^\star_{\beta}(\bm F)$ & $\R^{\star}_{+}$ & Asymptotic $\beta$-characteristic time \\
$w^\star(\bm F)$ & $\simplex$ &  Asymptotic optimal allocation \\
$w^\star_{\beta}(\bm F)$ & $\simplex$ & Asymptotic $\beta$-optimal allocation \\
	\bottomrule
\end{tabular}
\end{center}
\end{table}

\section{Notation} \label{app:notation}

We recall some commonly used notation:
the set of integers $[n] \eqdef \{1, \cdots, n\}$,
the complement $\overline{X}$ and interior $\mathring X$ of a set $X$,
the Kullback-Leibler (KL) divergence $\KL(F,G)$ between two distributions $F$ and $G$,
the KL for Bernoulli distributions $\kl$,
the Kinf $\Kinf^{\pm}(F, u)$ between a distribution $F$ and a scalar $u$,
Landau's notation $o$ and $\cO$ and
the $K$-dimensional probability simplex $\simplex \eqdef \left\{w \in \R_{+}^{K} \mid w \geq 0 , \: \sum_{i \in [K]} w_i = 1 \right\}$,
the infinity norm $\| \cdot \|_{\infty}$, i.e. $\| f \|_{\infty} = \sup_{x \in \R} f(x)$.
For all $\cF_n$-measurable set $A$, we denote by $\bP_{\mid n}[A] \eqdef \bP[A \mid \cF_{n}]$ its probability.
For all $\cF_n$-measurable set $A_{\theta}$ depending on $\theta \sim \Pi_n$, we denote by $\bP_{n}[A_{\theta}] \eqdef \bP_{\theta \sim \Pi_n }[A_{\theta} \mid \cF_n]$.
In Table~\ref{tab:notation_table_setting}, we summarize problem-specific notation.
Table~\ref{tab:notation_table_algorithms} gathers notation for the algorithms.

\begin{table}[H]
\caption{Notation for algorithms.}
\label{tab:notation_table_algorithms}
\begin{center}
\begin{tabular}{c c l}
	\toprule
Notation & Type & Description \\
\midrule
$B_{n}$ & $[K]$ & Leader at time $n$ \\
$C_{n}$ & $[K]$ & Challenger at time $n$ \\
$I_{n}$ & $[K]$ & Arm sampled at time $n$ \\
$\beta$ & $(0,1)$ & Probability of sampling the leader instead of the challenger \\
$X_{n, I_{n}}$ & $\cI$ & Sample observed at the end of time $n$, i.e. $X_{n, I_{n}} \sim F_{I_{n}}$ \\
$U_{n}$ &  & Internal randomization at time $n$ \\
$\cF_n$ &  & History at time $n$, i.e. $\cF_{n} \eqdef \sigma(U_1, I_1, X_{1,I_1}, \cdots, I_{n}, X_{n,I_{n}}, U_{n+1})$ \\
$\hat \imath_n$ & $[K]$ & Arm recommended after time $n$, i.e. $\hat \imath_n \in \argmax_{i\in [K]} \mu_{n,i}$ \\
$\tau_{\delta}$ & $\N$ & Sample complexity (stopping time of the algorithm) \\
$\hat{\imath}$ & $[K]$ & Arm recommended by the algorithm \\
$c(n,\delta)$ & $\N \times (0,1) \to \R^{\star}_{+}$ & Stopping threshold function\\
$N_{n,i}$ & $\N$ & Number of pulls of arm $i$ at time $n$, i.e. $N_{n,i} \eqdef \sum_{t \in [n]} \indi{I_{t} = i}$ \\
$F_{n,i}$ & $\cF$ & Empirical distribution, i.e. $F_{n,i} \eqdef \frac{1}{N_{n,i}} \sum_{t \in [n]} \delta_{X_{t, I_{t}}} \indi{I_{t} = i}$ \\
$\mu_{n,i}$ & $\cI$ & Empirical mean, i.e. $\mu_{n,i} \eqdef m(F_{n,i})$ \\
$W_{n}(i,j)$ & $\R_{+}$ & Empirical transportation between arms $i$ and $j$, defined in (\ref{eq:def_transportation_cost}) \\
$\Pi_n$ &  & Sampler at time $n$, e.g. Dirichlet sampler for bounded \\
$\theta$ & $\cI^{K}$ & Observation from the sampler, i.e. $\theta \sim \Pi_n$ \\
\addlinespace
$a_{n,i}$ & $[0,1]$ & $a_{n,i} \eqdef \bP_{n-1}[i \in \argmax_{j\in [K]} \theta_j]$  \\
$\psi_{n,i}$ & $[0,1]$ & Probability of sampling arm $i$ at time $n$: $\psi_{n,i} \eqdef \bP_{\mid (n-1)}[I_{n} = i]$  \\
$\Psi_{n,i}$ & $\R^{\star}_{+}$ & Cumulative sampling probability: $\Psi_{n,i} \eqdef \sum_{t \in [n]} \psi_{t,i}$ \\
$\widehat B_{n}$ & $[K]$ & Effective leader at time $n$ \\
$\widehat C_{n}$ & $[K]$ & Effective challenger at time $n$ \\
	\bottomrule
\end{tabular}
\end{center}
\end{table}

\input{sections/appendix_analysis}

\input{sections/appendix_top_two_instances}

\input{sections/appendix_concentration}

\input{sections/appendix_kinf_bounded}

\input{sections/appendix_bcp}

\input{sections/spef}

\input{sections/appendix_implementation_add_experiments}

\end{document}

%% file: sections/introduction.tex

\section{Introduction}
\label{sec:introduction}

Finding the distribution that has the largest mean by sequentially collecting samples from a pool of candidate distributions (``arms'') has been extensively studied in the multi-armed bandit \cite{Bubeck10BestArm,Jamieson14Survey} and ranking and selection \cite{Hong21Survey} literature. While existing approaches often rely on parametric assumptions for the distributions, we are interested in (near) optimal and computationally efficient strategies when the distributions belong to an arbitrary class $\cF$ of distributions.


For applications to online marketing such as A/B testing \cite{COLT14,Russac21ABn} assuming Bernoulli or Gaussian arms is fine, but more sophisticated distributions arise in other fields such as agriculture. In Section~\ref{sec:experiments} we consider a crop-management problem: a group of farmers wants to identify the best planting date for a rainfed crop. The reward (crop yield) can be modeled as a complex distribution with multiple modes, but upper bounded by a known yield potential. Therefore, sequentially identifying the best planting date calls for efficient best arm identification algorithms for the class of bounded distributions with a known range. 


To tackle this problem, we build on Top Two algorithms \cite{Russo2016TTTS,Qin2017TTEI,Shang20TTTS}, originally proposed for specific parametric families. We propose a generic analysis of this type of algorithms, which puts forward new possibilities for the choice of leader and challenger used by the algorithm. In particular, this work leads to the first asymptotically $\beta$-optimal strategies for bounded distributions.

%
%
%


\subsection{Setting and related work}

A bandit problem is described by a finite number of probability distributions ($K$ many), called arms.
Let $\simplex$ be the $K$-dimensional probability simplex and $\mathcal P(\mathbb{R})$ the set of probability distributions over $\mathbb{R}$.
Let $\cF \subset \mathcal P(\mathbb{R})$ be a known family of distributions to which the arms to. We will refer to tuples of distributions in $\cF^K$ with bold letters, e.g. $\bm F = (F_1, \ldots, F_K) \in \cF^K$ where $F_i$ is the cdf of arm $i$. We suppose that all distributions in $\cF$ have finite first moment and we denote the mean of $F \in \cF$ by $m(F)$. We denote by $\mathcal I = \{m(F) \mid F \in \cF\}$ the set of possible means for the arms.

The goal of a best arm identification (BAI) algorithm is to identify an arm with highest mean in the set of available arms, i.e. an arm which belongs to the set $i^\star(\bm F) = \argmax_{k \in [K]} m(F_k)$.
At each time $n\in \N$, the algorithm interacts with the environment (the set of arms) by (1) choosing an arm $I_{n}$ based on previous observations, (2) observing a sample $X_{n, I_{n}} \sim F_{I_{n}}$, and (3) deciding whether to stop and return an arm $\hat \imath_n$ or to continue.
We study the \emph{fixed confidence} identification setting, in which we require algorithms to make mistakes with probability less than a given $\delta \in (0,1)$.
To compare such algorithms we consider their \emph{sample complexity} $\tau_\delta$, which is a stopping time counting the number of rounds before the algorithm terminates.
The goal is then to minimize $\mathbb{E}[\tau_\delta]$ among the class of $\delta$-correct algorithms.

\begin{definition}
An algorithm is $\delta$-correct\footnote{A stronger definition of $\delta$-correctness has also been studied by requiring the algorithm to stop almost surely. } on $\cF^K$ if
$\mathbb{P}_{\bm F}\left(\tau_\delta < + \infty, \: \hat \imath_{\tau_\delta} \notin i^\star(\bm F) \right) \le \delta$ for all $\bm F \in \cF^K$.
\end{definition}

In order to be $\delta$-correct on $\cF^K$, an algorithm has to be able to distinguish problems in $\cF^K$ with different best arms. This intuition is formalized by the lower bound provided in Lemma~\ref{lem:lower_bound}.
The characteristic time defined in the lower bound depends on two functions $\Kinf^+$ and $\Kinf^-$, mapping $\mathcal P(\mathbb{R})\times \mathbb{R}$ to $\mathbb{R}_+$, obtained by minimizing a Kullback-Leibler divergence ($\KL$) over $\cF$, 
\begin{align*}
\Kinf^+(F, u) &\eqdef \inf \{\KL(F, G) \mid G \in \mathcal F, \: \mathbb{E}_{X\sim G}[X] > u\}
\: , \\
\Kinf^-(F, u) &\eqdef \inf \{\KL(F, G) \mid G \in \mathcal F, \: \mathbb{E}_{X\sim G}[X] < u\}
\: .
\end{align*}

\begin{lemma}[From \cite{GK16,Agrawal20GeneBAI}]\label{lem:lower_bound}
Any algorithm which is $\delta$-correct on $\cF^K$ verifies, for any $\bm F \in \cF^K$,
\begin{align*}
\bE_{\bm F}[\tau_{\delta}] \ge T^\star(\bm F ) \log \left(1/(2.4 \delta) \right) \: ,
\end{align*}
where $T^\star(\bm F)^{-1} \eqdef \sup_{w \in \simplex} \min_{i \neq i^\star} \inf_{u \in \mathcal I} \left\{ w_{i^\star} \Kinf^{-}(F_{i^\star} , u) + w_i \Kinf^{+}(F_i , u) \right\}$.
\end{lemma}

We say that an algorithm is asymptotically optimal if its sample complexity matches that lower bound, that is if $\limsup_{\delta \to 0} \bE_{\bm F}[\tau_{\delta}]/\log \left(1/\delta \right) \le T^\star(\bm F )$.

A related, weaker notion of (asymptotic) optimality is (asymptotic) $\beta$-optimality \cite{Shang20TTTS}. An algorithm is called asymptotically $\beta$-optimal if it satisfies $\limsup_{\delta \to 0} \bE_{\bm F}[\tau_{\delta}]/\log \left(1/\delta \right) \le T_\beta^\star(\bm F )$, for the complexity term
\begin{align*}
T_\beta^\star(\bm F)^{-1} \eqdef \sup_{w \in \simplex, w_{i^\star} = \beta} \min_{i \neq i^\star} \inf_{u \in \mathcal I} \left\{ \beta \Kinf^{-}(F_{i^\star} , u) + w_i \Kinf^{+}(F_i , u) \right\}
\: .
\end{align*}
An asymptotically $\beta$-optimal algorithm is asymptotically minimizing the sample complexity among algorithms which allocate a $\beta$ fraction of samples to the best arm
and $T^\star(\bm F) = \min_{\beta \in (0,1)} T^\star_{\beta}(\bm F)$. As was first shown by \cite{Russo2016TTTS} when $\cF$ is an exponential family, an asymptotically $\beta$-optimal algorithm with $\beta=1/2$ also has an expected sample complexity which is asymptotically optimal, up to a multiplicative factor 2. That is, $T_{1/2}^\star(\bm F) \le 2 T^\star(\bm F)$.

We denote by $w^{\star}(\bm F)$ and $w^{\star}_{\beta}(\bm F)$ the allocations realizing the argmax in the definition of $T^\star(\bm F)$ and $T^\star_\beta(\bm F)$, respectively. We will show that for common choices of
 $\cF$ these allocations are unique when there is a unique best arm.

\paragraph{Distribution classes}

The characteristic time $T^\star(\bm F)$ depends on the class of distributions $\cF$, known to the algorithm in advance, to which $\bm F$ belongs to. For example, all arms could have Bernoulli distributions.
We strive to provide an analysis which could easily be applied to many classes $\cF$, but we specialize our results to two main cases:
\begin{enumerate}
	\item distributions with bounded support, $\cF = \{F \in \mathcal P(\mathbb{R}) \mid \supp(F) \subseteq [0,B]\}$ for $B>0$,
	\item single parameter exponential families (SPEF) of sub-exponential distributions.
\end{enumerate}

Given a distribution $\mathbb{P}^{(0)}$ with cumulant generating function $\phi$, defined on an interval $\mathcal I_\phi$, the SPEF defined by $\mathbb{P}^{(0)}$ is the set of distributions $\mathbb{P}^{(\lambda)}$ with density with respect to $\mathbb{P}^{(0)}$ given by $\frac{d \mathbb{P}^{(\lambda)}}{d \mathbb{P}^{(0)}}(x) = e^{\lambda x - \phi(\lambda)}$.
For example, Gaussian distributions with a known variance form a SPEF, as do Bernoulli distributions with means in $(0,1)$.
We consider SPEF of sub-exponential distributions 
to have a concentration property for the empirical mean estimator.


\paragraph{Related work} The first Best Arm Identification (BAI) algorithms \cite{EvenDaral06,Shivaramal12,Gabillon12UGapE,Jamiesonal14LILUCB} were proposed and analyzed for bounded rewards, but their sample complexity scales with a sum of inverse gaps between the means of arms instead of the quantity $T^\star(\bm F)$ prescribed by the lower bound. Asymptotically optimal BAI algorithm were first designed when the arms belong to the same single-parameter exponential family.
In this context, two families of asymptotically optimal algorithms have emerged. Tracking-based algorithms solve the optimization problem provided by the lower bound in every round, and track the corresponding allocation \cite{GK16}.
The gamification approach views the characteristic time as a min-max game between the learner and the nature, and apply a saddle-point algorithm to solve it sequentially at a lower computational cost \cite{Degenne19GameBAI}.

Some Bayesian algorithms arose as another computationally appealing alternative to Track-and-Stop. Russo notably proposed the Top Two Probability Sampling (TTPS) and Top Two Thompson Sampling (TTTS) algorithms \cite{Russo2016TTTS}, that may be seen as counterparts of the popular Thompson Sampling algorithm for regret minimization \cite{Thompson33}. Other Bayesian flavored Top Two algorithms have been proposed, Top Two Expected Improvement (TTEI, \cite{Qin2017TTEI}) and Top Two Transportation Cost (T3C, \cite{Shang20TTTS}). All these algorithms sample either a \textit{leader} with fixed probability $\beta$ or a \textit{challenger} with probability $1-\beta$. TTTS, TTEI and T3C were proved to be asymptotically $\beta$-optimal for Gaussian bandits and perform well in practice even against asymptotically optimal algorithms \cite{Qin2017TTEI,Shang20TTTS}. This motivates our investigation of Top Two algorithms to tackle bounded distributions, which led us to propose a new generic analysis of this kind of algorithms of independent interest. We prove the asymptotic $\beta$-optimality of several Top Two instances for bounded bandit models, some of which depart from their original Bayesian motivation as they don't need a sampler. An asymptotically optimal algorithm for a non-parametric class of distribution has been proposed by \cite{Agrawal20GeneBAI} for heavy-tailed rewards. It relies on the computationally prohibitive Track-and-Stop approach, and an adaptation to bounded distributions is mentioned, yet without an explicit calibration of the stopping rule.





\subsection{Contributions}

We present the first fixed-confidence analysis of Top Two algorithms for distribution classes other than Gaussian, including the non-parametric setting of bounded distributions. In Section~\ref{sec:generic_top_two_algorithms}, we introduce several variants of Top Two algorithms, including new ones which choose the empirical best arm as a leader instead of relying on (Thompson) sampling and/or use some penalization in the previously proposed Transportation Cost challenger.

For the class of bounded distributions, we propose in Section~\ref{sec:dirichlet_sampler} a calibration of the stopping rule and a concrete instantiation of the Top Two algorithms, based on a Dirichlet sampler for the randomized variants. We prove in Theorem~\ref{thm:beta-opt} that those algorithms are asymptotically $\beta$-optimal. This optimality can also be shown for deterministic instances in the case of sub-exponential single parameter exponential families (Appendix~\ref{app:spef}). Our generic analysis, sketched in Section~\ref{app:sample_complexity_analysis}, provides insight on what properties the leader and challenger in a Top Two algorithm should have in order to reach asymptotic $\beta$-optimality.
We show that the algorithm should ensure that all arms are explored sufficiently, and explain how to guarantee that the sampling proportions reach their optimal values once the sufficient exploration condition holds.

Finally, in Section~\ref{sec:experiments} we report results from numerical experiments on a challenging non-parametric task using real-world data from a crop-management problem  for various members of the Top Two family of algorithms. Most of them perform significantly better than the baselines. 

%% file: sections/algorithms.tex
\section{Generic Top Two algorithms}
\label{sec:generic_top_two_algorithms}

Let $\bm F \in \cF^{K}$ such that $|i^\star(\bm F)|=1$ and $\mu_i \eqdef m(F_i) \in \cI $ for all $i\in [K]$.
As for most BAI algorithms, each arm is pulled once for the initialization.
At time $n + 1$, the $\sigma$-algebra $\cF_{n} \eqdef \sigma(U_1, I_1, X_{1,I_1},  \cdots, I_{n}, X_{n,I_{n}}, U_{n+1})$, called history, encompasses all the information available to the agent and the internal randomization denoted by $(U_{t})_{t \in [n + 1]}$, which is independent of everything else.
For all $\cF_n$-measurable sets $A$, we denote by $\bP_{\mid n}[A] \eqdef \bP[A \mid \cF_{n}]$ its probability.
For an arm $i$, we denote its number of pulls by $N_{n,i} \eqdef \sum_{t \in [n]} \indi{I_{t} = i}$, its empirical distribution by $F_{n,i} \eqdef \frac{1}{N_{n,i}} \sum_{t \in [n]} \delta_{X_{t, I_{t}}} \indi{I_{t} = i}$ and its empirical mean by $\mu_{n,i} \eqdef m(F_{n,i})$.

\paragraph{Stopping and recommendation rules}
Our Top Two algorithms rely on the same stopping rule, which can be expressed using the (empirical) transportation cost between arm $i$ and arm $j$, defined as
\begin{equation} \label{eq:def_transportation_cost}
  W_n(i,j) = \inf_{x \in \cI} \left[N_{n,i}\Kinf^-(F_{n,i},x) + N_{n,j}\Kinf^+(F_{n,j},x)\right].
\end{equation}
In particular, using the definition of $\Kinf^{\pm}$, it can be noted that $W_n(i,j) = 0$ if $\mu_{n,i} \leq \mu_{n,j}$.
Given a threshold function $c(n,\delta)$, the stopping rule is
\begin{equation} \label{eq:def_stopping_time}
  \tau_{\delta} = \inf \{ n \in \N \mid \min_{j \neq \hat \imath_n} W_n(\hat \imath_n, j) > c(n,\delta) \} \:,
\end{equation}
and the recommendation rule is $\hat \imath_{n} = \argmax_{i}\mu_{n,i}$.
Up to the choice of threshold, this stopping rule coincides with the GLR-based stopping rule proposed when $\cF$ is an exponential family \citep{GK16} and by \cite{Agrawal20GeneBAI} for heavy-tailed distributions with an upper bound on a non-centered moment.
For a general class $\cF$ the stopping rule can be calibrated to ensure $\delta$-correctness under any sampling rule if the threshold is such that the following time-uniform concentration inequality holds for all $\bm F \in \cF^{K}$:
\begin{equation} \label{eq:TimeUniformToProve}
   \bP_{\bm F}\left(\exists n, \: \exists i \neq i^\star(\bm F) : \: N_{n, i} \Kinf^- (F_{n, i}, \mu_{i} ) + N_{n,i^\star(\bm F)} \Kinf^+ (F_{n,i^\star(\bm F)}, \mu_{i^\star(\bm F)})  > c(n,\delta)\right) \leq \delta \: .
\end{equation}
Lemma~\ref{lem:kinf_concentration} in the next section gives an explicit threshold for the class of bounded distribution.
For SPEF, we can use generic stopping thresholds derived in \cite{KK18Mixtures}.

\paragraph{Sampling rule}

\begin{figure}[h]
\begin{minipage}{0.49\linewidth}
\begin{algorithmic}[1]
   \State \textbf{Input:} $\beta$
   \State Choose a leader $B_n \in [K]$
   \State $U \sim \mathcal{U}([0,1])$
   \If{$U < \beta$}
        \State $I_{n}=B_n$
   \Else
        \State Choose a challenger $C_n \in [K] \setminus \{B_n\}$
        \State $I_{n}=C_n$
   \EndIf
   \State \textbf{Output}: next arm to sample $I_{n}$
\end{algorithmic}
\caption{Generic $\beta$-Top Two sampling rule \label{alg:Top Two}}
\end{minipage}
\begin{minipage}{0.49\linewidth}
Choice of the leader (two propositions):
\begin{itemize}
     \item [{\bf EB}] - $B_n^{\text{EB}} \in \argmax_{i} \mu_{n-1,i}$
     \item [{\bf TS}] - Sample $\theta \sim \Pi_{n-1}$ then set $B_n^{\text{TS}} \in \argmax_{i \in [K]} \theta_i$
\end{itemize}
\vspace{2mm}
Choice of the challenger (three propositions):
\begin{itemize}
     \item [{\bf TC}] - $C_n^{\text{TC}} \in \argmin_{j\neq B_n} W_{n-1}(B_n,j)$
     \item [{\bf TCI}] - $C_n^{\text{TCI}} \in \argmin_{j\neq B_n} W_{n-1}(B_n,j) + \log N_{n-1,j}$
     \item [{\bf RS}] - repeat $\theta \sim \Pi_{n-1}$ until \\ $C_n^{\text{RS}} \in \argmax_{i \in [K]} \theta_i \not \owns B_n$
\end{itemize}
\caption{Choices of leader and challenger (uniform tie-breaking). \label{fig:leader_challenger}}
\end{minipage}
\end{figure}

The sampling rule of a Top Two algorithm is shown in Figure~\ref{alg:Top Two}. The method chooses a first arm $B_n$ called leader which is then sampled with probability $\beta$. If $B_n$ is not sampled, then a second arm $C_n$ called challenger is chosen and sampled. Our analysis isolates properties that those two choices should fulfill in order for the Top Two algorithm to be asymptotically $\beta$-optimal.

The practical implementation of a Top Two method then requires subroutines for $B_n$ and $C_n$. Two possibilities for the leader and three possibilities for the challenger are presented in Figure~\ref{fig:leader_challenger}. Our analysis will apply to any combination of those and we will refer to the algorithms obtained by $\beta$-[leader]-[challenger]; for example $\beta$-EB-TCI or $\beta$-TS-TC.

We have two flavors of leaders and challengers: deterministic and randomized.
The deterministic choices (EB, for Empirical Best, leader, TC and TCI challengers) rely on the empirical Transportation Costs (TC) $W_{n}(i,j)$ used in the stopping rule: the TC and TCI challengers are the arms which minimize the transportation cost from the leader (up to a penalization for TCI, hence TC Improved).
The randomized choices (TS leader and RS challenger) rely on a \emph{sampler}, denoted by $\Pi_{n}$. $\Pi_{n}$ generates i.i.d.\ vectors $\theta = (\theta_1,\dots,\theta_K) \in \cI^K$ which are interpreted as possible means for the arms, under a distribution which depends on observations gathered in the first $n$ rounds.
The TS leader is the best arm in the sampled vector, which is inspired by Thompson Sampling.
The RS (for Re-Sampling) challenger is obtained by performing repeated calls to the sampler until the best arm in the sampled vector is not $B_n$, then taking the best arm.

\paragraph{Randomized instances}
The samplers suggested by prior work all have a Bayesian flavor.
For SPEF bandits, they use $\Pi_n=\Pi_{n,1}\times \dots \times \Pi_{n,K}$ where $\Pi_{n,i}$ is the posterior distribution on the mean of arm $i$ after $n$ rounds (given some prior distribution).
With this choice of sampler, $\beta$-TS-RS coincides with the TTTS algorithm \cite{Russo2016TTTS}, while $\beta$-TS-TC coincides with the T3C algorithm \cite{Shang20TTTS}.
TTTS and T3C were only proved to be asymptotically $\beta$-optimal for Gaussian bandits with improper priors, whereas a by-product of the general analysis that we propose in this work permits to establish the necessary properties on the sampler for it to hold for more general distributions.
Moreover, we extend these algorithms to bounded distributions by virtue of Dirichlet sampling and also analyze their sampler-free counterparts.
As will be apparent in our analysis, the crucial property needed from the sampler in a Top Two algorithm using the RS challenger is that for all arms $i,j$ such that $\mu_i > \mu_j$, $\bP_{\theta \sim \Pi_n} (\theta_j > \theta_{i}) \simeq \exp(-W_n(i,j))$.

\paragraph{Deterministic instances} Under the RS challenger, the probability to obtain as a challenger arm $j$ is proportional to the probability that $\bP_{\theta \sim \Pi_n} (\theta_j > \theta_{B_n})$. Therefore, if $\Pi_n$ is a good sampler satisfying the above property, the TC challenger can be seen as replacing the randomization in the RS challenger by a computation of the mode of the distribution of $C_n^{\text{RS}}$.
This was the motivation behind T3C \cite{Shang20TTTS} as Gaussian transportation costs have a simple closed form expression while re-sampling becomes more and more costly when the posterior distributions are concentrated.
While our asymptotic analysis holds for deterministic algorithms, the empirical performance of fully deterministic algorithms might suffer from unlucky draws.
In Section~\ref{sec:experiments}, we show that $\beta$-EB-TC is indeed the least robust of all our instances.
To cope for this pitfall, explicit or implicit exploration mechanisms can be added.
Inspired by IMED \cite{Honda15IMED}, the TCI challenger fosters exploration by penalizing over-sampled challengers.
Randomization and forced exploration are two other examples of implicit and explicit exploration mechanisms.

%% file: sections/sampler_bounded.tex

\section{Asymptotically $\beta$-optimal algorithms for bounded distributions}
\label{sec:dirichlet_sampler}

For bounded distribution, Lemma~\ref{lem:kinf_concentration} provides a calibration of the stopping rule.
Its proof, given in Appendix~\ref{app:ss_kinf_bounded_distributions}, relies on a martingale construction proposed by \cite{agrawal2021optimal}.

\begin{lemma} \label{lem:kinf_concentration}
	The stopping rule (\ref{eq:def_stopping_time}) with threshold
	\begin{equation}  \label{eq:def_kinf_threshold_glr}
		c(n,\delta)  = \ln \left( 1/\delta\right) + 2\ln\left(1+n/2\right) + 2 + \ln(K-1)
	\end{equation}
	is $\delta$-correct for the family of bounded distributions.
\end{lemma}

\paragraph{Transportation costs}
Both the stopping rule and the TC and TCI challengers of the sampling rule require the computation of $W_{n}(i,j)$ defined in \eqref{eq:def_transportation_cost}.
For single-parameter exponential families, this can be done easily since $\Kinf^{\pm}$ are KL divergences and the transportation cost has a closed form expression \cite{GK16, Russo2016TTTS}.
However, for bounded distributions, computing $\Kinf^{\pm}$ is more challenging and we rely on the dual formulation first obtained by \cite{HondaTakemura10} (see Theorem~\ref{thm:Kinf_duality}):
\[
		N_{n,i} \Kinf^+(F_{n,i},x) = \sup_{\lambda \in [0, 1]} \sum_{t \in [n]} \indi{I_{t} = i}\ln \left( 1 - \lambda \frac{X_{t,i} - x}{B - x}\right) \: .
\]
The minimization in $\lambda$ can be computed using a zero-order optimization algorithm (e.g.\ Brent's method \cite{brent2013algorithms}). The same optimizer can be used to compute the minimization in $x\in [0,B]$ featured in $W_n(i,j)$.
By nesting those optimizations of univariate functions on a bounded interval, the computation of $W_{n}(i,j)$ in the stopping rule dominates the computational cost of our Top Tow algorithms (except the RS challenger).
Our experiments suggest that using (\ref{eq:def_stopping_time}) is twice as computationally expensive as the LUCB-based stopping rule, which is a mild price to pay for the improvement in terms of empirical stopping time.
Algorithms for non-parametric distributions are bound to be computationally more expensive than their counterpart in SPEF, where a sufficient statistic can summarize $\cF_n$.

\paragraph{Sampler} The TS leader and RS challenger require a sampler.
Our proposed sampler for bounded distributions in $[0,B]$ has a product form: $\Pi_n = \Pi_{n,1}\times \dots \times \Pi_{n,K}$ where $\Pi_{n,i}$ leverages $\cH_{n,i} \eqdef (X_{1,i},\dots,X_{N_{n,i},i})$, which is the history of samples from arm $i$ collected in the first $n$ rounds.
Let $\tilde F_{n,i}$ denote the empirical cdf of $\cH_{n,i}$ augmented by the known bounds on the support, $\{0,B\}$.
For each arm $i$, $\Pi_{n,i}$ outputs a random re-weighting of $\tilde F_{n,i}$.
Concretely, letting $\bm w = (w_1,\dots,w_{N_{n,i}+2})$ be drawn from a Dirichlet distribution $\Dir(\bm 1_{N_{n,i}+2})$, a call to the sampler $\Pi_{n,i}$ returns
\[
\sum_{t \in [N_{n,i}]}w_{t} X_{t,i} + Bw_{N_{n,i}+1} \: .
\]
This sampler is inspired by that used in the Non Parametric Thompson Sampling (NPTS) algorithm proposed by \cite{RiouHonda20} for regret minimization in bounded bandits, with the notable difference that we have to add both $0$ and $B$ in the support, while NPTS only adds the upper bound $B$.
We will see that this is only necessary to ensure that the re-sampling procedure stops.
Therefore, the TS leader could use a sampler $\tilde \Pi_n$ based directly on $\cH_{n,i}$.

\begin{theorem}\label{thm:beta-opt}
	Combining the stopping rule (\ref{eq:def_stopping_time}) with threshold (\ref{eq:def_kinf_threshold_glr}) and a Top Two algorithm with $\beta \in (0,1)$, instantiated with any pair of leader/challenger as in Figure~\ref{fig:leader_challenger}, yields a $\delta$-correct algorithm which is asymptotically $\beta$-optimal for all $\bm F \in \cF^{K}$ with $\mu_{\bm F} \in (0,B)^{K}$ and $\Delta_{\min}(\bm F) \eqdef \min_{i \neq j}|\mu_{F_i} - \mu_{F_j}| > 0$.
\end{theorem}

Theorem~\ref{thm:beta-opt} gives the asymptotic $\beta$-optimality for six algorithms (Figure~\ref{fig:leader_challenger}). Choosing our favorite Top Two instances therefore requires further empirical and computational considerations. Computing the EB leader has a constant computational cost, while the TS leader is computationally costly for large time $n$ since it requires to sample from a Dirichlet distribution with $N_{n,i} + 2$ parameters for each arm $i$. On the challenger side, the RS challenger is computationally very expensive for large time $n$ as the sampler becomes concentrated around the true mean vector. On the contrary, by leveraging computations done in the stopping rule \eqref{eq:def_stopping_time}, the TC and TCI challengers can be computed in constant time. Based on these computational considerations, the most appealing Top Two algorithm for bounded distribution appears to be the fully deterministic $\beta$-EB-TC. But experiments performed in Section~\ref{sec:experiments} reveal its lack of robustness, and for bounded distributions the best trade-off between robustness and computational complexity is $\beta$-EB-TCI. More generally, $\beta$-TS-TC can also be a good choice provided that we have access to an efficient sampler.

%

\paragraph{Distinct means} Restricting to instances such that $\Delta_{\min}(\bm F) > 0$ (which implies $|i^\star(\bm F)|=1$) is an uncommon assumption in BAI.
However, known Top Two algorithms \cite{Russo2016TTTS, Qin2017TTEI, Shang20TTTS} only have guarantees on those instances.
Our generic analysis reveals that it is solely used to prove sufficient exploration, characterized by \eqref{eq:sufficient_exploration_cdt} (Appendix~\ref{app:ss_how_to_explore}).
Experiments highlights that all our Top Two algorithms except $\beta$-EB-TC perform well on instances where $|i^\star(\bm F)|=1$ and $\Delta_{\min}(\bm F) = 0$ (Figure~\ref{fig:bernoulli_instances}(b)).
Proving theoretical guarantees in this situation is an interesting problem for future work (see Appendix~\ref{app:ss_beyond_all_distinct_means} for a discussion).

%% file: sections/analysis.tex

\section{Sample complexity analysis}
\label{app:sample_complexity_analysis}

In this section, we sketch the proof of Theorem~\ref{thm:beta-opt}, which follows from the generic sample complexity analysis of Top Two algorithms presented in Appendix~\ref{app:unified_analysis_top_two}. Our proof strategy is the same as that first introduced by \cite{Qin2017TTEI} for the analysis of TTEI and also used by \cite{Shang20TTTS} for TTTS and T3C. It consists in upper bounding the expectation of the \emph{convergence time}, defined as
\begin{equation} \label{eq:rv_T_eps_beta}
	T^{\epsilon}_{\beta} \eqdef \inf \left\{ T \ge 1 \mid \forall n \geq T, \: \max_{i \in [K]} \left| \frac{N_{n,i}}{n} -w_{i}^{\beta} \right| \leq \epsilon \right\}  \: ,
\end{equation}
for $\varepsilon$ small enough. Indeed, we prove in Appendix~\ref{app:ss_asymptotic_optimality} that for any sampling rule
\begin{equation} \label{eq:finite_Tepsbeta_optimality}
		\exists \epsilon_0(\bm F)> 0, \: \forall \epsilon \in (0,\epsilon_0(\bm F)], \: \bE_{\bm F}[T^{\epsilon}_{\beta}] < + \infty \quad \implies  \quad \limsup_{\delta \to 0} \frac{\bE_{\bm F}[\tau_{\delta}]}{\log \left(1/\delta \right)} \leq T_{\beta}^\star(\bm F ) \: .
\end{equation}
This implication only leverages the expression of the stopping rule and the threshold.
It was previously established for Gaussian bandits by \cite{Qin2017TTEI} and we extend this property to bounded distributions and SPEF of sub-exponential distributions.
Up to technicalities ($\Kinf$ continuity and second order terms), this implication is shown by using that if $\tau_{\delta} \ge n$, then
\[
	\ln \left( 1/\delta \right) \approx_{\delta \to 0} c(n,\delta) \ge \min_{j \neq \hat \imath_n} W_n(\hat \imath_n, j) \approx_{n \ge T^{\epsilon}_{\beta}} n T^\star_{\beta}(\bm F)^{-1} \: .
\]

To upper bound the expected convergence time, as prior work we first establish \emph{sufficient exploration}:
\begin{equation} \label{eq:sufficient_exploration_cdt}
    \exists N_1 \text{ s.t. }   \bE_{\bm F}[N_1] < + \infty, \: \forall n \ge N_1, \quad \min_{i\in [K]} N_{n,i} \ge \sqrt{n/K} \: .
\end{equation}
By generalizing \cite{Shang20TTTS} which considered Gaussian, we identify two generic properties for the leader and the challenger under which \eqref{eq:sufficient_exploration_cdt} hold (Appendix~\ref{app:ss_how_to_explore}), provided that we assume $\Delta_{\min}>0$.

We proceed similarly to prove convergence by identifying in Appendix~\ref{app:unified_analysis_top_two} desired properties for the leader and challenger, which are satisfied by all our leaders and challengers for bounded distributions (Appendix~\ref{app:top_two_instances}).
We sketch these conditions below. Let $i^\star$ be the unique element of $i^\star(\bm F)$.

The requirements on the leader and the challenger to ensure $\bE_{\bm F}[T^{\epsilon}_{\beta}] < + \infty$ become apparent when looking at generic properties of Top Two algorithms.
Under any Top Two algorithm, the probability to select arm $i$ at round $n$, $\psi_{n,i} \eqdef \bP_{\mid (n-1)}[I_{n} = i]$, can be written as
\begin{equation}\label{eq:selection-proba}
		\psi_{n,i} = \beta \bP_{\mid (n-1)}[B_n = i] + (1-\beta) \sum_{j \neq i} \bP_{\mid (n-1)}[B_n = j] \bP_{\mid (n-1)}[C_n = i| B_n = j] \: .
\end{equation}
We let $\Psi_{n,i} \eqdef \sum_{t \in [n]} \psi_{t,i}$. For the leader, we can prove using \eqref{eq:selection-proba} that
\[
\forall M \in \N, \quad \left|\frac{\Psi_{n,i^\star}}{n} - \beta \right| \leq   \frac{M-1}{n} + \frac{1}{n} \sum_{t=M}^{n} \bP_{\mid (t-1)}[B_t \neq i^\star] \: .
\]
This suggests that a \textit{good} leader should satisfy that there exists $N_2$ with $ \bE_{\bm F}[N_2] < + \infty$ s.t.
\begin{equation} \label{eq:good_leader_convergence}
	 \forall n \ge N_2, \quad  \bP_{\mid n}[ B_{n+1} \neq i^\star] \leq g(n)	\: ,
\end{equation}
where $g(n) =_{+\infty} o(n^{-\alpha})$ for some $\alpha >0$. For the challenger, noticing that
\[
\forall M \in \N, \: \forall i \ne i^\star, \quad \frac{\Psi_{n,i}}{n} \leq \frac{M-1}{n} + \frac{1}{n} \sum_{t=M}^{n} \bP_{\mid (t-1)}[B_t \neq i^\star] + \frac{1}{n} \sum_{t=M}^{n} \bP_{\mid (t-1)}[C_t = i| B_t = i^\star] \: ,
\]
suggests that a \textit{good} challenger should satisfy that there exists $N_3$ with $ \bE_{\bm F}[N_3] < + \infty$ s.t.
\begin{equation} \label{eq:good_challenger_convergence}
			\forall n \ge N_3, \: \forall i \neq i^\star , \quad \frac{\Psi_{n,i}}{n} \geq w_{i}^{\beta} + \epsilon  \: \Rightarrow \: \bP_{\mid n}[C_{n+1} = i {\mid} B_{n+1} = i^\star] \leq h(n)	\: ,
\end{equation}
where $h(n) =_{+\infty} o(n^{-\alpha})$ for some $\alpha >0$. Then, Cesaro's theorem further
yields
\[
	\exists N_4 \text{ s.t. }   \bE_{\bm F}[N_4] < + \infty, \: \forall n \ge N_4, \quad \max_{i \in [K]} \left| \frac{\Psi_{n,i}}{n} -w_{i}^{\beta} \right| \leq \epsilon \: .
\]
Using that $(N_{n,i} - \Psi_{n,i})/\sqrt{n}$ are sub-Gaussian random variables, we obtain $\bE_{\bm F}[T^{\epsilon}_{\beta}] < + \infty$.

We now explain why \eqref{eq:good_leader_convergence} and \eqref{eq:good_challenger_convergence} are satisfied for the leaders and challengers in  Figure~\ref{fig:leader_challenger} when $\cF$ is the class of bounded distributions.
This follows from concentration properties.
Using the fact that $\sqrt{n}\| F_{n,i} - F\|_{\infty}$ is sub-Gaussian, which follows for the Dvoretzky–Kiefer–Wolfowitz inequality \cite{massart1990}, the continuity of the mean operator $m$ on $\cF$ and the sufficient exploration property \eqref{eq:sufficient_exploration_cdt}, we establish that for all $\alpha > 0$, there exists a random variable $N_{\alpha}$ with finite expectation such that
\begin{equation}\label{eq:concentration-eps}
	\forall n \geq N_{\alpha}, \ \  \max_{i \in [K]} \| F_{n,i} - F_i \|_{\infty} \leq \alpha \ \ \text{ and } \ \ \max_{i \in [K]} |\mu_{n,i} - \mu_i | \leq \alpha \: .
\end{equation}

\paragraph{Deterministic instances}
Recall that $B_{n+1}^{\text{EB}} \in \argmax_{i \in [K]} \mu_{n,i}$. Choosing $\alpha$ in \eqref{eq:concentration-eps} smaller than half the gap between the best and second best arm (which is possible as $|i^\star(\bm F)| = 1$) yields that for all $n \geq N_{\alpha}$, $B_{n+1}^{\text{EB}} = i^\star$.
This proves (\ref{eq:good_leader_convergence}) with $g(n) = 0$.
Using continuity and convexity properties of $\Kinf^{\pm}$, we then establish that there exists $\alpha>0$ and a problem-dependent constant $C_{\bm F} > 0$ such that for $n \geq N_{\alpha}$ and for all $i\neq i^\star$,
\[
	 \frac{\Psi_{n,i}}{n} \geq w_{i}^{\beta} + \epsilon  \quad \implies \quad \frac{1}{n}\left(W_{n}(i^\star,i) - \min_{j\neq i^\star} W_{n}(i^\star,j) \right) \ge C_{\bm F} \: .
\]
This implies that $i \notin \min_{j\neq i^\star} W_{n}(i^\star,j)$, hence $\bP_{\mid n}[C_{n+1}^{\text{TC}} = i \mid B_{n+1} = i^\star] = 0$ for $n\geq N_{\alpha}$. Therefore, (\ref{eq:good_challenger_convergence}) holds with $h(n) =0$.
A similar argument holds for $C_{n+1}^{\text{TCI}}$.

\paragraph{Randomized instances}
Let $a_{n+1,i} \eqdef \bP_{\theta \sim \Pi_n}( i \in \argmax_{j \in [K]} \theta_j)$ be the probability that arm $i$ is the best arm in a sampled model at round $n$.
Since
\[
\bP_{\mid n}[ B_{n+1}^{\text{TS}} \neq i^\star] \le (K-1)\max_{i \neq i^\star} a_{n+1,i} \le (K-1) \max_{i \neq i^\star} \bP_{\theta \sim \Pi_n}(\theta_i \ge \theta_{i^\star}) \: ,
\]
an upper bound on $\bP_{\theta \sim \Pi_n}[\theta_i \ge \theta_{i^\star}]$ is sufficient to prove \eqref{eq:good_leader_convergence}.
We show in Lemma~\ref{lem:from_bcp_one_to_bcp_two} that this can be obtained by leveraging upper bound on the Boundary Crossing Probability (BCP) of the Dirichlet sampler, $\bP_{\theta \sim \Pi_n}[\theta_i \ge u]$ for a fixed threshold $u \in (0,B)$. An upper bound on the BCP can be obtained using the work of \cite{RiouHonda20} and is given in Theorem~\ref{thm:upper_bound_one_arm_bcp_bounded} for the sake of completeness. Putting things together yields that, for all $n$,
\begin{align*}
		\bP_{\theta \sim \Pi_n} [\theta_i \geq \theta_{i^\star}]
		&\leq f\left( \inf_{u \in [0,B]} [ (N_{n,i^\star} + 2) \Kinf^-(\tilde F_{n,i^\star},u)  + (N_{n,i} + 2) \Kinf^+(\tilde F_{n,i},u) ] \right) \: ,
\end{align*}
where $f(x) = (1+x)e^{-x}$. Using again continuity and concentration \eqref{eq:concentration-eps}, we conclude that (\ref{eq:good_leader_convergence}) holds with $g(n) = (K-1)f\left( \left(\sqrt{\frac{n}{K}} + 2 \right) D_{\bm F}\right)$, where $D_{\bm F} > 0$ is a problem dependent constant.

For the challenger, we first observe that
\[
\bP_{\mid n}[ C_{n+1}^{\text{RS}} = i \mid B_{n+1} = i^\star] = \frac{a_{n+1,i}}{1 - a_{n+1,i^\star}} \le \frac{\bP_{\theta \sim \Pi_n} [\theta_i \geq \theta_{i^\star}]}{\max_{j \neq i^\star} \bP_{\theta \sim \Pi_n} [\theta_j \geq \theta_{i^\star}]} \: .
\]
Further upper bounding this quantity to prove (\ref{eq:good_challenger_convergence}) requires a lower bound on $\bP_{\theta \sim \Pi_n}[\theta_i \ge \theta_{i^\star}]$ which can again be obtained using a lower bound on the BCP.
In Appendix~\ref{sub:lower_bound} we provide a tight lower bound on $\bP_{\theta \sim \Pi_n}[\theta_i \ge \theta_{i^\star}]$ featuring the $\Kinf^\pm$ functions.
It permits to prove that (\ref{eq:good_challenger_convergence}) holds with $- \log (h(n))/n =_{+ \infty} \tilde C_{\bm F} + o(1) $ where $\tilde C_{\bm F}> 0$ is a problem dependent constant.

The above derivations all use the concentration property \eqref{eq:concentration-eps}, which requires the sufficient exploration property \eqref{eq:sufficient_exploration_cdt}.
For our deterministic challengers, sufficient exploration is obtained by noticing that $W_{n}(i,j)$ can be upper and lower bounded by linear functions of the number of samples.
Proving sufficient exploration is more challenging for a randomized challenger, and existing proofs were exploiting the symmetry of the Gaussian posterior.
In our analysis we show that a coarse lower bound on the BCP is sufficient to obtain \eqref{eq:concentration-eps}, and prove such lower bound for the Dirichlet sampler:
\begin{align*}
\bP_{\theta \sim \Pi_n}[\theta_i \ge u]
&\ge \left(1 - u/B \right)^{n+1} \quad \text{and} \quad
\bP_{\theta \sim \Pi_n}[\theta_i \le u]
\ge \left(u / B \right)^{n+1} \: .
\end{align*}
These lower bounds ensure that any arm has some (small) probability of being the challenger thanks to re-sampling.
Without adding $\{0,B\}$ to $\cH_{n,i}$, those probabilities could be equal to zero.

Our analysis is easily amenable to tackle different families of distributions $\cF$.
This requires continuity and convexity properties for the corresponding $\Kinf$ functions, an appropriate concentration result and further upper and lower bounds on the BCP of the sampler if one wish to analyze randomized algorithms.
As an illustration, we show asymptotic $\beta$-optimality of the $\beta$-EB-TC, $\beta$-EB-TCI algorithms for SPEF with sub-exponential distributions, see Appendix~\ref{app:spef}.

%% file: sections/experiments.tex
\section{Experiments}
\label{sec:experiments}

We assess the empirical performance of our Top Two algorithms on the DSSAT real-world data and on Bernoulli instances in the moderate regime ($\delta= 0.01$).
The stopping rule (\ref{eq:def_stopping_time}) is used with the threshold $c(n,\delta)$ defined in (\ref{eq:def_kinf_threshold_glr}).
As Top Two sampling rules, we present results for $\beta$-EB-TC, $\beta$-EB-TCI, $\beta$-TS-TC and $\beta$-TS-TCI with $\beta=0.5$.
Additional experiments are available in Appendix~\ref{app:ss_supplementary_experiments}: on the RS challenger whose computational cost prevent it to be evaluated with (\ref{eq:def_kinf_threshold_glr}) and on larger sets of arms (up to $K=1000$). 

As benchmarks for the sampling rule, we use KL-LUCB with Bernoulli divergence \citep{COLT13} (whose theoretical guarantees extend to any distribution bounded in $[0,1]$), ``fixed'' sampling which is an oracle playing with proportions $w^\star(\bm F)$ and uniform sampling.
We also propose a heuristic adaptation of the DKM algorithm \citep{Degenne19GameBAI} (which is asymptotically optimal for SPEF) to tackle bounded distributions, which we denote by $\Kinf$-DKM, and uses forced exploration instead of optimism.
Inspired by the regret minimization algorithm $\Kinf$-UCB \citep{Agrawal21Regret}, we propose its LUCB variant \cite{Shivaramal12}, named $\Kinf$-LUCB.
The upper/lower confidence indices are obtained by inverting of $\Kinf^{\pm}$, i.e.
\begin{align*}
	\forall i \neq \hat \imath_n, \quad & U_{n+1,i} = \max \left\{ u \in [\mu_{n,i}, B] \mid N_{n,i} \Kinf^+(F_{n,i}, u) \leq c(n,\delta) \right\}	\: , \\
	& L_{n+1,\hat \imath_n} = \min \left\{ u \in [0, \mu_{n,\hat \imath_n}] \mid N_{n,\hat \imath_n} \Kinf^-(F_{n,\hat \imath_n}, u) \leq c(n,\delta) \right\}	\: .
\end{align*}
LUCB-based algorithms \citep{Shivaramal12} use their own stopping rule, namely they stop when $L_{n+1,\hat \imath_n} \ge \max_{j\neq \hat \imath_n} U_{n+1,j}$.
For Bernoulli distributions, $\Kinf$-LUCB recovers KL-LUCB.
While being asymptotically optimal for heavy-tailed distributions \cite{Agrawal20GeneBAI} with an adequate stopping threshold, the Track-and-Stop algorithm is computationally intractable for bounded distributions as it requires to compute $w^\star(\bm F_{n})$ at each time $n$ (or on a geometric grid).
We hence omit it from our experiments.

\paragraph{Crop-management problem}
We benchmark our algorithms on the DSSAT simulator\footnote{DSSAT is an Open-Source project maintained by the DSSAT	Foundation, see https://dssat.net.} \cite{hoogenboom2019dssat}.
Each arm corresponds to a choice of planting date and fixed soil conditions (details in Appendix~\ref{app:additional_experiments}).
To illustrate the problem's difficulty we represent an empirical estimate (independent of the runs of our algorithms) of the yield distributions in Figure~\ref{fig:dsat_instances}(b).
Since the gaps between means are small, the identification problem is hard.
Moreover, $\Kinf$ computations for non-parametric distributions are costlier than Bernoulli ones (see Appendix~\ref{app:ss_implementation_details}), so we only present the results for $100$ runs.

\begin{figure}[ht]
	\centering
	\includegraphics[width=0.485\linewidth]{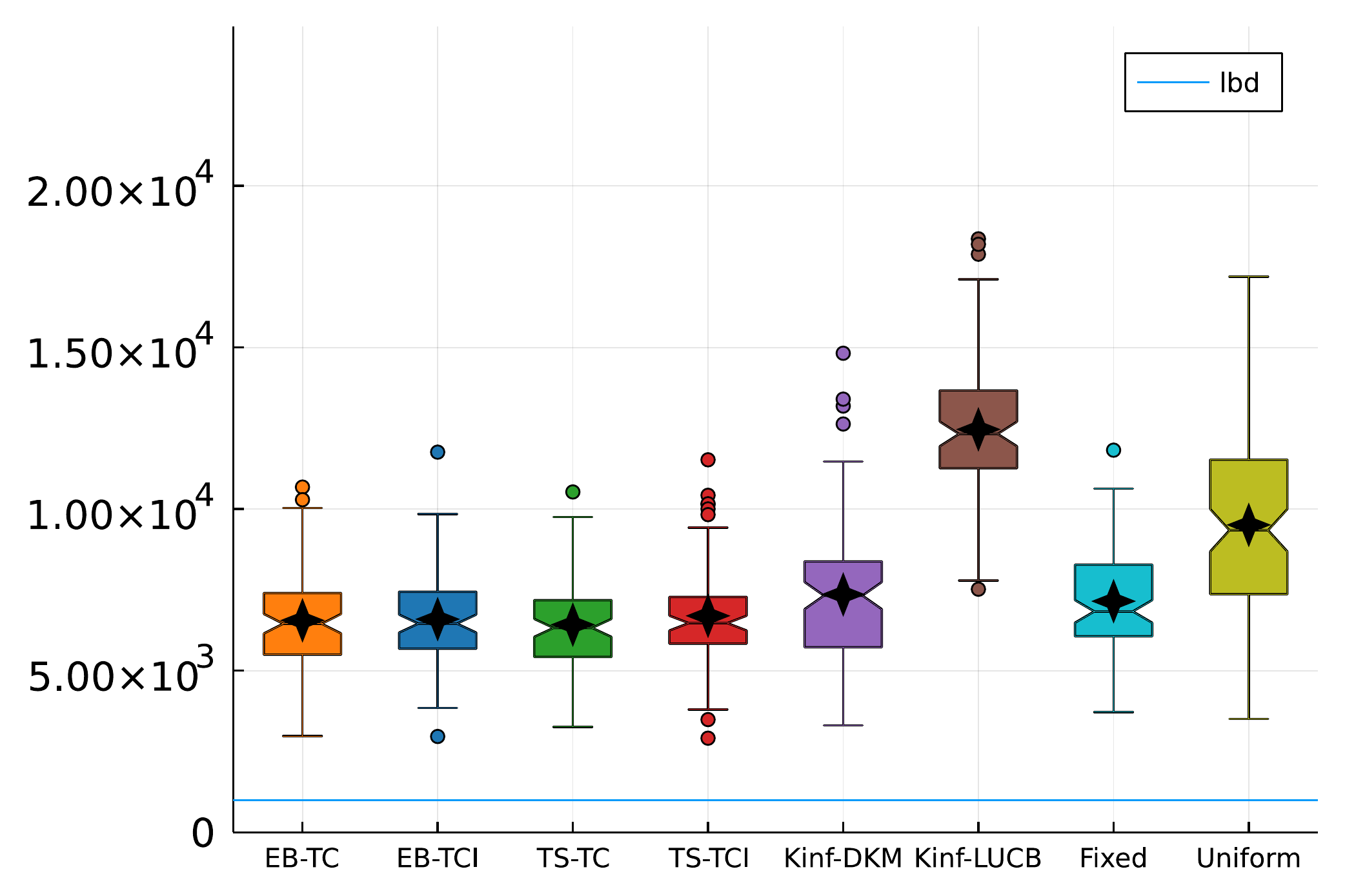}
	\includegraphics[width=0.485\linewidth]{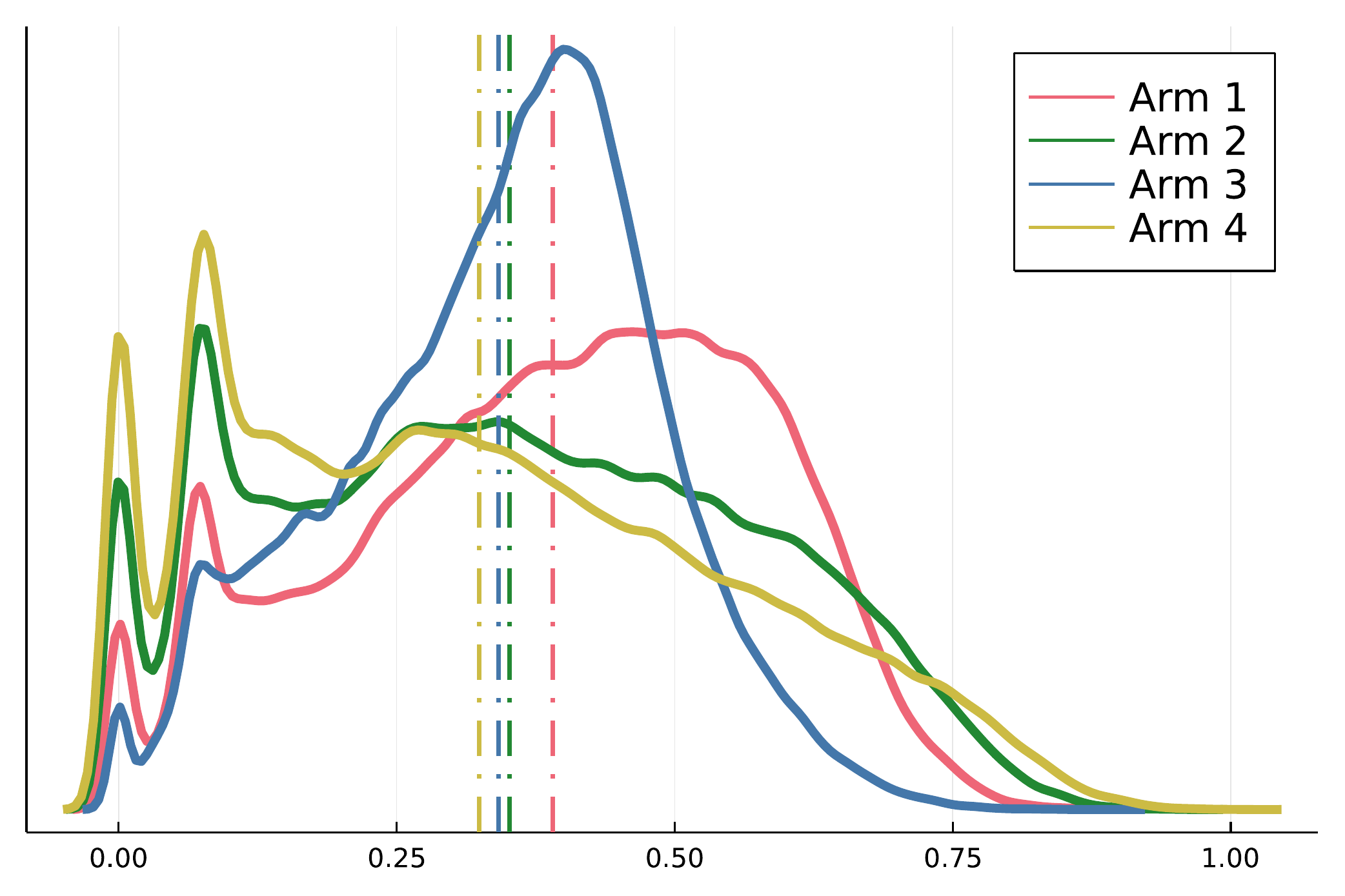}
	\caption{Empirical stopping time (a) on scaled DSSAT instances with their density and mean (b). Lower bound is $T^\star(\bm F) \ln(1/\delta)$. ``stars'' equal means.}
	\label{fig:dsat_instances}
\end{figure}

In Figure~\ref{fig:dsat_instances}, $\beta$-EB-TCI, $\beta$-TS-TC and $\beta$-TS-TCI slightly outperform $\Kinf$-DKM and the fixed (oracle) sampling rule.
Moreover, $\Kinf$-LUCB performs significantly worse than uniform sampling.
Due to the small number of runs, we don't observe large outliers for $\beta$-EB-TC (see Appendix~\ref{app:ss_supplementary_experiments}).
KL-LUCB performs ten times worse than $\Kinf$-LUCB, hence we omit it from Figure~\ref{fig:dsat_instances}.

\paragraph{Bernoulli instances}
Next we assess the performance on $1000$ random Bernoulli instances with $K=10$ such that $\mu_{1} = 0.6$ and $\mu_{i} \sim \mathcal U ([0.2, 0.5])$ for all $i \neq 1$, where we enforce that $\Delta_{\min} \ge 0.01$.
We also study the instance $\mu = (0.5, 0.45, 0.45)$, in which $\Delta_{\min} = 0$, and perform $1000$ runs.

\begin{figure}[ht]
	\centering
	\includegraphics[width=0.485\linewidth]{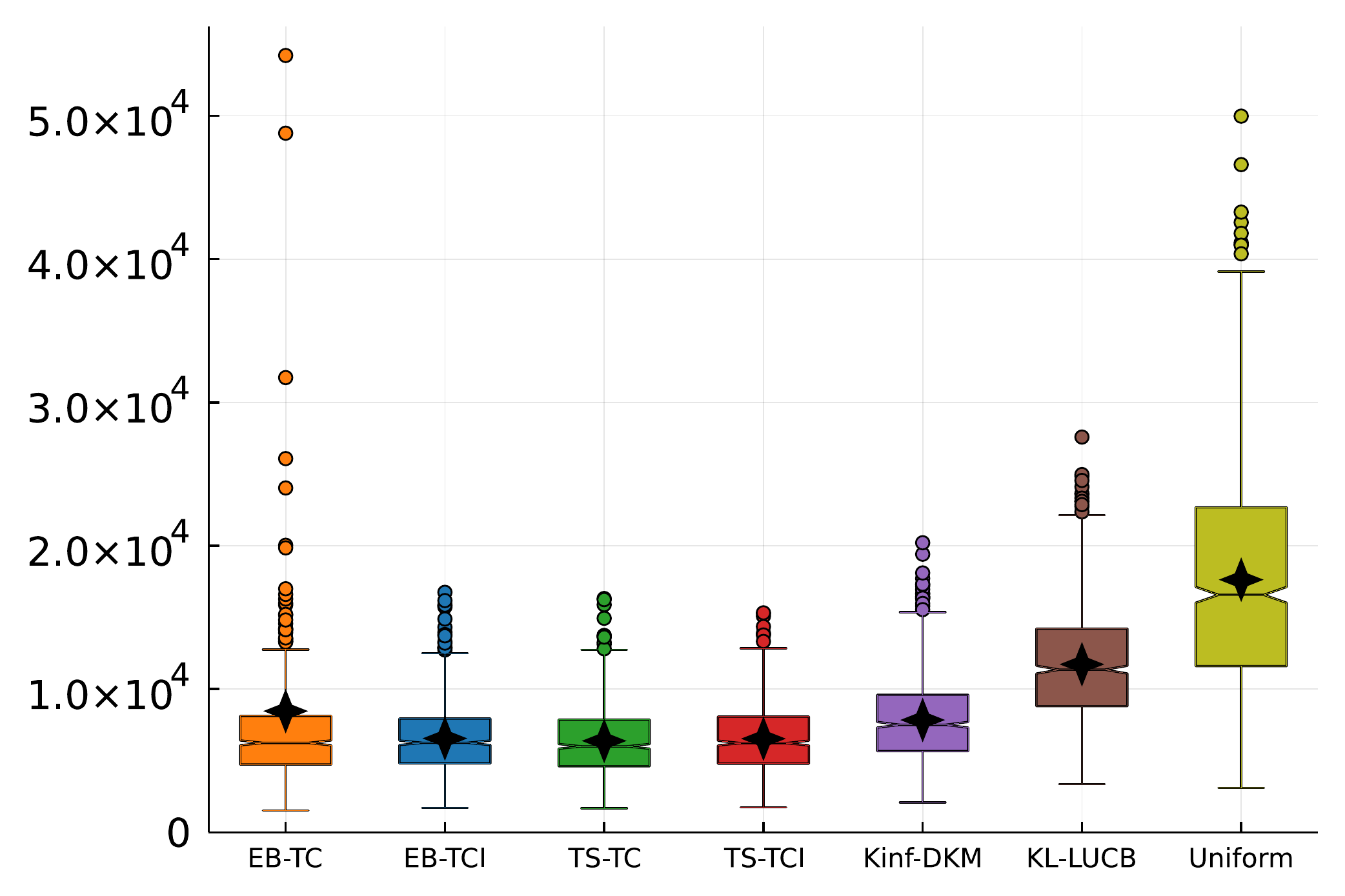}
	\includegraphics[width=0.485\linewidth]{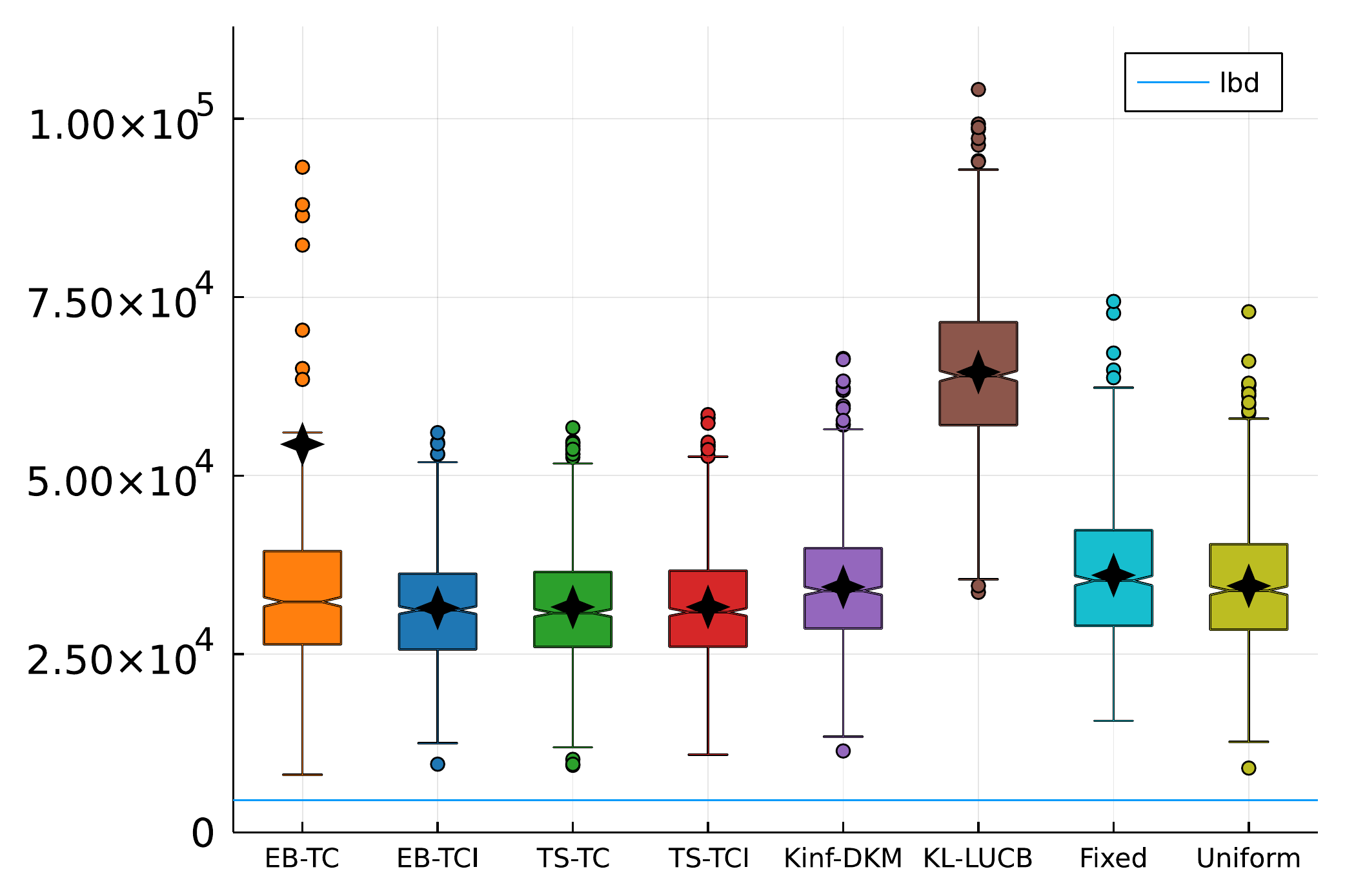}
	\caption{Empirical stopping time on Bernoulli (a) random instances with $K = 10$ and (b) instance $\mu = (0.5, 0.45, 0.45)$. }
	\label{fig:bernoulli_instances}
\end{figure}

In Figure~\ref{fig:bernoulli_instances}(a), we see that $\beta$-EB-TCI, $\beta$-TS-TC and $\beta$-TS-TCI outperform other algorithms.
While this gain is slim compared to $\Kinf$-DKM, the empirical stopping time is twice (resp.\ three times) as large for KL-LUCB (resp.\ uniform sampling).
Even when $\Delta_{\min} = 0$, Figure~\ref{fig:bernoulli_instances}(b) hints that their empirical performance might be preserved.
Figure~\ref{fig:bernoulli_instances} confirms the lack of robustness of $\beta$-EB-TC, which is prone to large outliers.
For the symmetric instance in Figure~\ref{fig:bernoulli_instances}(b), uniform sampling outperforms KL-LUCB and perform on par with the ``fixed'' sampling.

%% file: sections/appendix_analysis.tex

\section{Unified analysis of Top Two algorithms}
\label{app:unified_analysis_top_two}

In this section, we present a unified analysis of Top Two algorithms (Appendix~\ref{app:ss_generic_top_two_algorithms}).
The analysis is split into three parts that will highlight how to explore (Appendix~\ref{app:ss_how_to_explore}), how to converge towards the $\beta$-optimal allocation (Appendix~\ref{app:ss_how_to_converge}) and finally proving asymptotic optimality (Appendix~\ref{app:ss_asymptotic_optimality}).

In this section, we identify the required properties that the leader and the challenger should satisfy.
In Appendix~\ref{app:top_two_instances}, we prove that those properties are verified by the EB and TS leader and the TC, TCI and RS challenger for bounded distributions.
In Appendix~\ref{app:spef}, we discuss the proofs of those properties for single-parameter exponential families.

The general proof strategy follows that first proposed by \cite{Qin2017TTEI} for the TTEI algorithm and later also used by \cite{Shang20TTTS} for TTTS and T3C. However we contribute with a new, modular proof structure which furthermore get rids of several Gaussian-specific arguments.

Striving to tackle simultaneously bounded distributions (Appendix~\ref{app:kinf_for_bounded_distributions}) and single-parameter exponential families (Appendix~\ref{app:spef}), we need to introduce some notation to unify both formulation (Appendix~\ref{app:ss_generic_beta_optimality}).

\subsection{Generic asymptotic $\beta$-optimality}
\label{app:ss_generic_beta_optimality}

In the case of single-parameter exponential families, a distribution $F \in \cF$ is characterized by its mean parameter $m(F) \in \R$.
Therefore, convergence/continuity results can be formulated directly with the $| \cdot |$ norm.

For general bounded distributions, it is not possible to characterize them by using a scalar (or vector) parameter.
Therefore, we need to consider the space of probability measures on $(\R, \mathcal B(\R))$ with the topology of weak convergence of measures, denoted by $\mathcal P(\R)$.
Recall that weak convergence is equivalent to the convergence of the respective cdfs for the infinity norm $\| \cdot \|_{\infty}$, defined as $\| f \|_{\infty} = \sup_{x \in \R} f(x)$.

To unify both approach, we introduce an operator $\cT$ from $\cF$ to a topological space, which associates the distribution $F \in \cF$ with a transformation $\cT (F)$ that characterizes it.
When considering single-parameter exponential families, $\cT$ coincides with the mean operator $m: F \mapsto \bE_{F}[X]$.
For bounded distributions, $\cT$ will be the identity.
Moreover, we define $\cT(\bm F) \eqdef (\cT(F_i) )_{i \in [K]}$ for all $\bm F \in \cF^{K}$ and $\cT(\cF^{K}) \eqdef \left\{ \cT(\bm F) \mid \bm F \in \cF^{K} \right\}$.

Let $\cI \subseteq \R $ be the interval of means $\cI \eqdef \{m(F) \mid F \in \cF\}$.
The functions $(F,u) \mapsto \Kinf^{\pm}(F,u)$ are defined on $\cF \times \cI$.
Two archetypal examples for $\cI$ are $\cI = [0,B]$ (bounded, Bernoulli, Beta, etc) and $\cI =\R$ (Gaussian, etc).

\paragraph{Condition on the means}
We detail the assumptions on the means of $\bm F \in \cF^{K}$ under which Top Two algorithms can be studied.
\begin{assumption} \label{ass:standard_assumptions_BAI}
	There is a unique best arm denoted by $i^\star(\bm F)$ and the means are away from the boundary, i.e. $\mu_i \eqdef m(F_i) \in \mathring \cI$ for all $i \in [K]$.
\end{assumption}

The first part of Assumption~\ref{ass:standard_assumptions_BAI} is a standard assumption in BAI problem, where a unique best arm has to be identified.
Indeed, in the presence of two best arms existing algorithms would only stop with a small probability as they would try to statistically distinguish two identical distributions.
In order to circumvent this hurdle, one can relax the BAI problem in which the goal is to find one arm which is $\varepsilon$-close to the best arm, for some parameter $\varepsilon > 0$.
When it comes to asymptotic optimality, this setting is known to be much more complex than standard BAI \cite{Degenne19Multiple,jourdan2022choosing}.

The second part of Assumption~\ref{ass:standard_assumptions_BAI} is also standard, as we require the distribution to have mass away from the boundary.
When $\cI = \R$, the condition $\mu_{i} \in \mathring \cI = \R$ is always satisfied since we consider finite means.
When $\cI = [0,B]$, the assumption $\mu_i \in (0,B)$ is often made when studying Bernoulli or bounded distributions.
For the bounded setting, this excludes $\delta_{B}$ and $\delta_{0}$, where $\delta_{x}$ denotes the Dirac distributions in $x$.
Since those requirements are mild, we consider that Assumption~\ref{ass:standard_assumptions_BAI} holds in the following, without mentioning it further.

\begin{assumption} \label{ass:all_arms_distinct_bounded_mean}
	All the arms have distinct means, i.e. $\Delta_{\min}(\bm F) \eqdef \min_{i\neq j} |\mu_{F_i} - \mu_{F_j}| > 0$.
\end{assumption}

Requiring $\Delta_{\min} > 0$, is stronger than the unique best arm condition from Assumption~\ref{ass:standard_assumptions_BAI}.
While this is a unusual requirement to study BAI problem, previous works on Top Two algorithms \citep{Russo2016TTTS,Qin2017TTEI,Shang20TTTS} also supposed that $\Delta_{\min} > 0$.
Our unified analysis of Top Two algorithms highlights the role of this condition in the analysis.
It is solely used to prove sufficient exploration (Appendix~\ref{app:ss_how_to_explore}).
Provided enough exploration, the convergence towards the $\beta$-optimal allocation (Appendix~\ref{app:ss_how_to_converge}) only relies on Assumption~\ref{ass:standard_assumptions_BAI}.

As we aim to shed light on the role of Assumption~\ref{ass:all_arms_distinct_bounded_mean} in the analysis, we will explicitly highlight where it is used in the proof of sufficient exploration.
The empirical performance of Top Two algorithms on instances where $\Delta_{\min}= 0$ is assessed in Appendix~\ref{app:sss_expe_on_distinct_means}.
In Appendix~\ref{app:ss_beyond_all_distinct_means}, we discuss possible relaxations of this Assumption for some leaders and challengers.

\paragraph{Transportation costs and $\beta$-optimal allocation}
With the notation introduced above, the transportation cost between arms $(i,j) \in [K]^2$ for an allocation $w \in \simplex$ rewrites as
\begin{equation} \label{eq:def_transportation_cost_general_formulation}
	C_{i,j}(\cT(\bm F), w) \eqdef \inf_{u \in \cI} \left\{ w_i \Kinf^{-}(\cT(F_i), u) + w_j \Kinf^{+}(\cT(F_j), u)\right\} \: ,
\end{equation}
and the empirical transportation cost rewrites as
\begin{equation} \label{eq:def_emp_transportation_cost_general_formulation}
	\frac{1}{n} W_{n}(i,j) = C_{i,j}\left(\cT(\bm F_n), \frac{N_n}{n}\right) \: .
\end{equation}
Similarly, the $\beta$-characteristic time and $\beta$-optimal allocation
\begin{align*}
		T_{\beta}^\star(\bm F)^{-1} \eqdef \max_{w \in \simplex : w_{i^\star(\bm F)} = \beta} \min_{j \neq i^\star(\bm F)} C_{i^\star(\bm F),j}(\cT(\bm F), w) \: , \\
		w_{\beta}^\star(\bm F) \eqdef \argmax_{w \in \simplex : w_{i^\star(\bm F)} = \beta} \min_{j \neq i^\star(\bm F)} C_{i^\star(\bm F),j}(\cT(\bm F), w) \: .
\end{align*}

Property~\ref{prop:singleton_beta_allocation} requires $w_{\beta}^\star(\bm F)$ to be a singleton.
For single-parameter exponential families, it is well known that Property~\ref{prop:singleton_beta_allocation} holds \cite{Russo2016TTTS}.
For bounded distribution, we showed it in Lemma~\ref{lem:properties_characteristic_times}.
As Property~\ref{prop:singleton_beta_allocation} holds for the distributions of interest, we won't mention it further.

\begin{property} \label{prop:singleton_beta_allocation}
		For all $\bm F \in \cF^{K}$ satisfying Assumption~\ref{ass:standard_assumptions_BAI}, $w_{\beta}^\star(\bm F)$ is a singleton.
\end{property}

To ease the notation, in the sequel we denote the unique $\beta$-optimal allocation by $w^{\beta} = (w^{\beta}_i)_{i \in [K]}$. %
For an algorithm to be asymptotically $\beta$-optimal, its empirical allocation $(N_{n,i}/n)_{i\in [K]}$ should converge towards $w^{\beta}$.

\subsection{Generic Top Two algorithms}
\label{app:ss_generic_top_two_algorithms}

The $\sigma$-algebra $\cF_{n} \eqdef \sigma(U_1, I_1, X_{1,I_1},  \cdots, I_{n}, X_{n,I_{n}}, U_{n+1})$, called history, encompasses all the information available to the agent at time $n$ and the internal randomization denoted by $(U_{t})_{t \in [n+1]}$, which is independent of everything else.
For all $\cF_n$-measurable set $A$, we denote by $\bP_{\mid n}[A] \eqdef \bP[A \mid \cF_{n}]$ its probability.
As most BAI algorithms, out methods pull each arm once for the initialization.
At time $n +1$, a Top Two sampling rule outputs an arm $I_{n+1}$ which is $\cF_{n}$-measurable.
The choice $I_{n+1}$ is defined by two mechanisms: the choice of a leader $B_{n+1} \in [K]$ which is $\cF_{n}$-measurable and the choice of the challenger $C_{n+1} \in [K] \setminus \{ B_{n+1} \}$ is $\cF_{n}$-measurable.

Following the proof strategy first introduced by \cite{Qin2017TTEI}, our goal is to upper bound the expectation of the \emph{convergence time}.
For $\varepsilon> 0$, the random variable $T^{\varepsilon}_{\beta}$ (already defined in \eqref{eq:rv_T_eps_beta}) quantifies the number of samples required for the empirical allocations $\frac{N_{n}}{n}$ to be $\epsilon$-close to $w^{\beta}$:
\begin{equation*} 
	T^{\epsilon}_{\beta} \eqdef \inf \left\{ T \ge 1 \mid \forall n \geq T, \: \left\| \frac{N_{n}}{n} -w^{\beta} \right\|_{\infty} \leq \epsilon \right\}  \: .
\end{equation*}
To this end, we first leverage generic properties of Top Two algorithms to understand how the average probability to select an arm can converge to $w^{\beta}$. We denote by $\psi_{n,i} \eqdef \bP_{\mid (n-1)}[I_{n} = i]$, the probability that an arm is sampled at round $n$, and by $\Psi_{n,i} \eqdef \sum_{t \in [n]} \psi_{t,i}$ its cumulative version.
For a Top Two sampling rule, we have
\begin{equation} \label{eq:def_probability_of_being_sampled}
		\psi_{n,i} = \beta \bP_{\mid (n-1)}[B_n = i] + (1-\beta) \sum_{j \neq i} \bP_{\mid (n-1)}[B_n = j] \bP_{\mid (n-1)}[C_n = i| B_n = j] \: .
\end{equation}

\paragraph{Mean probability of being sampled}
Even before specifying the leader and the challenger mechanisms, we can study the general properties of Top Two algorithms, given in Lemmas~\ref{lem:deviation_wrt_the_optimal_arm} and \ref{lem:upper_bound_probability_sampling_other_arms}.
While being obtained by simple algebra, they highlight quite naturally the respective roles of the leader and the challenger mechanisms in order to achieve asymptotic $\beta$-optimality.

Lemma~\ref{lem:deviation_wrt_the_optimal_arm} upper bounds the deviation between the fixed allocation $\beta$ and the mean probability of sampling the optimal arm $\frac{\Psi_{n,i^\star(\bm F)}}{n}$.
\begin{lemma} \label{lem:deviation_wrt_the_optimal_arm}
For all $M \in \N^*$,
	\begin{equation} \label{eq:deviation_wrt_the_optimal_arm}
			\left|\frac{\Psi_{n,i^\star(\bm F)}}{n} - \beta \right| \leq \frac{M-1}{n} + \frac{1}{n} \sum_{t=M}^{n} \bP_{\mid (t-1)}[B_t \neq i^\star(\bm F)]	\: .
	\end{equation}
\end{lemma}
\begin{proof}
Let $i^\star = i^\star(\bm F)$ and $M \in \N^*$. Summing (\ref{eq:def_probability_of_being_sampled}) for $i^\star$ and using $\bP_{\mid (t-1)}[B_t = i^\star] = 1 - \bP_{\mid (t-1)}[B_t \neq i^\star]$ yields
\begin{align*}
	\frac{\Psi_{n,i^\star}}{n} - \beta = \frac{1-\beta}{n} \sum_{t \in [n]} \sum_{j \neq i^\star} \bP_{\mid (t-1)}[B_t = j] \bP_{\mid (t-1)}[C_t = i^\star| B_t = j] - \frac{\beta}{n} \sum_{t \in [n]} \bP_{\mid (t-1)}[B_t \neq i^\star]\: .
\end{align*}
Dropping the second negative term, splitting the sum into two and using that $\bP_{\mid (t-1)}[C_t = i^\star| B_t = j] \le 1$ and $\sum_{j \neq i^\star} \bP_{\mid (t-1)}[B_t = j] = \bP_{\mid (t-1)}[B_t \neq i^\star]$, we obtain the following upper bound
\begin{align*}
	\frac{\Psi_{n,i^\star}}{n} - \beta \le (1-\beta) \left( \frac{M-1}{n} + \frac{1}{n} \sum_{t = M}^{n}  \bP_{\mid (t-1)}[B_t \neq i^\star] \right) \: .
\end{align*}
Dropping the first positive term and splitting the sum into two, we obtain the following lower bound
\begin{align*}
	\frac{\Psi_{n,i^\star}}{n} - \beta \ge -\beta \left( \frac{M-1}{n} + \frac{1}{n} \sum_{t = M}^{n}  \bP_{\mid (t-1)}[B_t \neq i^\star] \right) \: .
\end{align*}
Combining the upper and the lower bound and using that $\max\{\beta, 1-\beta\} \leq 1$ yields the result.
\end{proof}
Since an asymptotically $\beta$-optimal algorithm should allocate a proportion $\beta$ of its samples to the best arm, the right-hand side of (\ref{eq:deviation_wrt_the_optimal_arm}) should vanish.
Cesaro's theorem yields the result when
\[
\lim_{t \to +\infty}\bP_{\mid (t-1)}[B_t \neq i^\star(\bm F)] = 0 \: .
\]
This means that a \textit{good} leader should asymptotically identify $i^\star(\bm F)$.
As we will see, the convergence almost surely won't be enough to obtain an upper bound on $\bE_{\bm F}[\tau_{\delta}]$.
To that end, we will need to specify the rate of convergence.

Lemma~\ref{lem:upper_bound_probability_sampling_other_arms} upper bounds the probability of sampling an arm different from the optimal one.
\begin{lemma} \label{lem:upper_bound_probability_sampling_other_arms}
For all $M \in \N^*$ and $i \neq i^\star(\bm F)$,
	\begin{equation} \label{eq:upp_proba_other_arms}
		\frac{\Psi_{n,i}}{n} \leq \frac{M-1}{n} + \frac{1}{n} \sum_{t=M}^{n} \bP_{\mid (t-1)}[B_t \neq i^\star(\bm F)] + \frac{1}{n} \sum_{t=M}^{n} \bP_{\mid (t-1)}[C_t = i| B_t = i^\star(\bm F)]	\: .
	\end{equation}
\end{lemma}
\begin{proof}
Let $i^\star = i^\star(\bm F)$. Using (\ref{eq:def_probability_of_being_sampled}) for $i \neq i^\star$ and $\bP[C_n = i| B_n = j] \leq 1$, we have
\begin{align*}
			\psi_{n,i} &\leq \beta \bP_{\mid (n-1)}[B_n = i] + (1-\beta) \sum_{j \notin \{i, i^\star\}} \bP_{\mid (n-1)}[B_n = j]\\
			&\qquad + (1-\beta) \bP_{\mid (n-1)}[B_n = i^\star] \bP_{\mid (n-1)}[C_n = i| B_t = i^\star] \\
			&\leq \max\{\beta, 1-\beta\} \bP_{\mid (n-1)}[B_n \neq i^\star] + (1-\beta) \bP_{\mid (n-1)}[B_n = i^\star] \bP_{\mid (n-1)}[C_n = i| B_n = i^\star] \\
			& \leq \bP_{\mid (n-1)}[B_n \neq i^\star] +  \bP_{\mid (n-1)}[C_n = i| B_n = i^\star]
\end{align*}
where we used that $\sum_{j \neq i^\star } \bP_{\mid (n-1)}[B_n = j] = \bP_{\mid (n-1)}[B_n \neq i^\star]$, $\max\{\beta, 1-\beta\} \leq 1$ and $\bP_{\mid (n-1)}[B_n = i^\star] \leq 1$.
Summing over $t \in [n]$ after splitting the sum into two, we obtain
\begin{align*}
	\frac{\Psi_{n,i}}{n} \leq \frac{M-1}{n} + \frac{1}{n} \sum_{t=M}^{n} \bP_{\mid (t-1)}[B_t \neq i^\star] + \frac{1}{n} \sum_{t=M}^{n} \bP_{\mid (t-1)}[C_t = i| B_t = i^\star]
\end{align*}
\end{proof}
Given a good leader, the first two terms on the right-hand side of (\ref{eq:upp_proba_other_arms}) vanish.
An asymptotically $\beta$-optimal algorithm should allocate a proportion $w^\beta_i$ of its samples to the sub-optimal arms.
As $\sum_{i \in [K]} \frac{\Psi_{n,i}}{n} = 1$, Cesaro's theorem yields the result when
\[
\lim_{t \to +\infty} \bP_{\mid (t-1)}[C_t = i| B_t = i^\star(\bm F)] = w^\beta_{i} \: .
\]
Given a good leader, a \textit{good} challenger should asymptotically have a probability $w^\beta_{i}$ of pulling a sub-optimal arm $i$.
Likewise, a rate of convergence will be necessary to upper bound $\bE_{\bm F}[\tau_{\delta}]$.

\paragraph{From $\Psi_{n,i}$ to $N_{n,i}$}
The above results feature $\frac{\Psi_{n,i}}{n}$, which is the mean probability of an arm to be sampled.
Thanks to Lemma~\ref{lem:subG_alloc}, it can be linked to the empirical allocation $\frac{N_{n,i}}{n}$.
Its proof, deferred to Appendix~\ref{ssub:sub_gaussian_random_variables}, is a direct consequence of concentration inequalities for sub-Gaussian random variables. A similar result was already derived in the work of \cite{Qin2017TTEI}, who first introduced this style of $W$-based concentration results, which we will use also in Appendix~\ref{app:top_two_instances}.

\begin{lemma} \label{lem:subG_alloc}
	There exists a sub-Gaussian random variable $W_1$ such that for all $(n,i) \in \N \times [K]$
	\begin{equation} \label{eq:subG_alloc}
		|N_{n,i}-\Psi_{n,i}| \leq W_1\sqrt{(n+1)\log(e+n)} \quad \text{a.s.} \: .
	\end{equation}
	In particular, $\bE \left[e^{\lambda W_1}\right] < + \infty$ for all $\lambda > 0$.
\end{lemma}

In the following, we take $W_1$ as in Lemma~\ref{lem:subG_alloc}. Since $\bE \left[e^{\lambda W_1}\right] < + \infty$ for all $\lambda > 0$, we have in particular that for all $N = Poly(W_1)$, $\bE \left[N\right] < + \infty$.

\paragraph{Convergence towards $\beta$-optimal allocation}
In Appendix~\ref{app:ss_asymptotic_optimality}, we show that if $\bE_{\bm F}\left[ T^{\epsilon}_{\beta} \right] < + \infty$ for all $\varepsilon$ small enough, then
\[
	\limsup_{\delta \to 0} \frac{\bE_{\bm F}[\tau_{\delta}]}{\log \left(\frac{1}{\delta} \right)} \leq T_{\beta}^\star(\bm F ) \: .
\]
The proof of $\bE_{\bm F}\left[ T^{\epsilon}_{\beta} \right] < + \infty$ can be naturally split into two distinct parts.
In Appendix~\ref{app:ss_how_to_explore}, under Assumption~\ref{ass:all_arms_distinct_bounded_mean}, we show that any Top Two algorithm ensures sufficient exploration provided its leader and challenger each satisfy one property.
In Appendix~\ref{app:ss_how_to_converge}, given sufficient exploration, the convergence of the empirical allocation towards the $\beta$-optimal one is proven for any Top Two algorithm provided its leader and challenger pairs each satisfy one property.

\subsection{How to explore}
\label{app:ss_how_to_explore}

In this section, we identify one property for the leader (Property \ref{prop:leader_cdt_suff_explo}) and one property for the challenger (Property \ref{prop:challenger_cdt_suff_explo}) under which we prove that the corresponding Top Two algorithm ensures sufficient exploration, when Assumption~\ref{ass:all_arms_distinct_bounded_mean} holds. More precisely, we prove in Lemma~\ref{lem:suff_exploration} that
\[\exists N_1 \text{s.t.} \ \  \bE_{\bm F}[N_1] < +\infty  : \ \ \forall n \geq N_1, \min_{i \in [K]} N_{n,i} \geq \sqrt{n/K}\]
We discuss other algorithmic choices that could ensure sufficient exploration without Assumption~\ref{ass:all_arms_distinct_bounded_mean} in Appendix~\ref{app:ss_beyond_all_distinct_means}. This section borrows several elements from existing proofs of sufficient exploration for Top Two algorithms in Gaussian bandits \cite{Qin2017TTEI,Shang20TTTS} but we managed to simplify the argument in order to put forward the key properties needed from a leader and a challenger. First, our generic analysis needs to define an appropriate notion of effective leader and challenger.

\paragraph{Effective leader and challenger}
For an algorithm to alleviate under-sampling some arms, it should have a strictly positive probability of sampling them.
In Top Two algorithms, the choice of the arm to pull $I_{n}$ is defined by the leader $B_n$ and the challenger $C_n$.
Due to possible randomization, it is not trivial to manipulate $B_{n}$ and $C_n$.
Therefore, we define the \textit{effective} leader $\widehat B_n$ and the \textit{effective} challenger $\widehat C_n$ as the arms maximizing the respective probability of being sampled:
\begin{equation} \label{eq:def_leader_challenger}
		\widehat B_n \in \argmax_{i \in [K]} \bP_{\mid (n-1)}[B_n = i] \quad \text{and} \quad \widehat C_n \in \argmax_{i \neq \widehat B_n} \bP_{\mid (n-1)}[C_n = i| B_n = \widehat B_n ] \: ,
\end{equation}
where $\widehat C_n$ is defined conditioned on the effective leader $\widehat B_n$.
We assume that ties are broken uniformly at random.
Note that they are fully determined by the leader and challenger mechanisms.

Lemma~\ref{lem:lower_bound_on_probability_of_sampling_leader_challenger} gives a strictly positive lower bound on the probability of sampling $\widehat B_n$ and $\widehat C_n$.
\begin{lemma} \label{lem:lower_bound_on_probability_of_sampling_leader_challenger}
		Let $\psi_{\min} \eqdef \frac{1}{K}\min \{ \beta, \frac{1-\beta}{K-1}\}$. Then, $\psi_{n,i} \geq \psi_{\min}$ \ for all $i \in \{\widehat B_n, \widehat C_n\}$.
\end{lemma}
\begin{proof}
		Since $\sum_{i \in [K]}\bP_{\mid (n-1)}[B_n = i] = 1$ and $\widehat B_n \in \argmax_{i \in [K]} \bP_{\mid (n-1)}[B_n = i]$, we have
		\[
		\psi_{n,\widehat B_n} \geq \beta \bP_{\mid (n-1)}[B_n = \widehat B_n] = \frac{\beta}{K} \geq \psi_{\min} \: .
		\]

		Similarly, $\sum_{i \in [K]} \bP_{\mid (n-1)}[C_n = i| B_n = \widehat B_n ] = 1$ and $\widehat C_n \in \argmax_{i \neq \widehat B_n} \bP_{\mid (n-1)}[C_n = i| B_n =  \widehat B_n ]$ yields that $\bP_{\mid (n-1)}[C_n = \widehat C_n| B_n = \widehat B_n ] \geq \frac{1}{K-1}$.
		Therefore,
		\[
		\psi_{n,\widehat C_n} \geq (1-\beta) \bP_{\mid (n-1)}[B_n = \widehat B_n ]\bP_{\mid (n-1)}[C_n = \widehat C_n| B_n = \widehat B_n ]  \geq \frac{1-\beta}{K(K-1)} \geq \psi_{\min} \: .
		\]
\end{proof}

In light of Lemma~\ref{lem:lower_bound_on_probability_of_sampling_leader_challenger}, the sufficient exploration can be proven if we show that either $\widehat B_n$ or $\widehat C_n$ is among the under-sampled arms if some still exists.
Before formalizing the properties required by the leader and challenger pair to ensure sufficient exploration, we introduce the relevant notation.

Given an arbitrary threshold $L \in \R_{+}^{*}$, we define the sampled enough set and its arms with highest 
mean (when not empty) as
\begin{equation} \label{eq:def_sampled_enough_sets}
	S_{n}^{L} \eqdef \{i \in [K] \mid N_{n,i} \ge L \} \quad \text{and} \quad \cI_n^\star \eqdef \argmax_{ i \in S_{n}^{L}} \mu_{i} \: .
\end{equation}
Assumption~\ref{ass:all_arms_distinct_bounded_mean} ensures that $\cI_n^\star$ is unique.
To highlight why it is necessary, we view $\cI_n^\star$ as a set with potentially multiple values, and derive properties without assuming that $\cI_n^\star$ is a singleton.
At time $n$, $S_{n}^{L}$ can only be non-empty for $L\le n$, hence it depends explicitly on $n$.

To prove sufficient exploration, we aim at finding a threshold $L(n)$ such that $S_{n}^{L(n)}=[K]$ for $n \ge \tilde N_0$ where $\bE_{\bm F}[\tilde N_0] < + \infty $.
We proceed by contradiction.
The idea is to show that if some arms are still highly under-sampled, then either $\widehat B_n$ or $\widehat C_n$ will be mildly under-sampled.
Since they have a strictly positive probability of being sampled (Lemma~\ref{lem:lower_bound_on_probability_of_sampling_leader_challenger}), this would yield a contradiction by the pigeonhole principle.
We define the highly and the mildly under-sampled sets
\begin{equation} \label{eq:def_undersampled_sets}
	U_n^L \eqdef \{i \in [K]\mid N_{n,i} < \sqrt{L} \} \quad \text{and} \quad V_n^L \eqdef \{i \in [K] \mid N_{n,i} < L^{3/4}\} \: .
\end{equation}
The choice of $\sqrt{L}$ and $L^{3/4}$ is arbitrary and we could consider instead $L^{\alpha_1}$ and $L^{\alpha_2}$ with $0<\alpha_1 < \alpha_2 < 1$.
Note that $U_n^L = \overline{S_{n}^{\sqrt{L}}}$ and $V_n^L = \overline{S_{n}^{L^{3/4}}}$ where $S_{n}^{L}$ is the set defined in (\ref{eq:def_sampled_enough_sets}). We are now ready to state the properties that the leader and the challenger should satisfy in order to show sufficient exploration under Assumption~\ref{ass:all_arms_distinct_bounded_mean}.

\paragraph{Exploring with leader and challenger pair}
We describe the properties that a good leader/challenger should have to ensure sufficient exploration.
In order for the challenger to explore, a \textit{good} leader should first identify the best arm among the arms that are sampled enough (Property~\ref{prop:leader_cdt_suff_explo}).
Then, given a good leader, a \textit{good} challenger should enforce exploration on the arms that are not sampled enough yet when the leader doesn't do it already (Property~\ref{prop:challenger_cdt_suff_explo}).

Property~\ref{prop:leader_cdt_suff_explo} states that if $\widehat B_{n+1}$ is sampled enough, then $\widehat B_{n+1}$ is an arm with highest mean among the sampled enough arms.
\begin{property} \label{prop:leader_cdt_suff_explo}
		There exists $L_0$ with $\bE_{\bm F}[(L_0)^{\alpha}] < +\infty$ for all $\alpha > 0$ such that if $L \ge L_0$, for all $n$ such that $S_{n}^{L} \neq \emptyset$, $\widehat B_{n+1} \in S_{n}^{L}$ implies $\widehat B_{n+1} \in \cI_n^\star$.
\end{property}

Property~\ref{prop:leader_cdt_suff_explo} holds for the EB leader (Lemma~\ref{lem:EB_ensures_suff_explo}) and the TS leader (Lemma~\ref{lem:TS_ensures_suff_explo}).

Property~\ref{prop:challenger_cdt_suff_explo} states that if some arms are still highly under-sampled, i.e. $U_n^L \neq \emptyset$, then having sampled $\widehat B_{n+1}$ enough implies that $\widehat C_{n+1}$ is mildly under-sampled or has highest true mean among the sampled enough arms.
\begin{property} \label{prop:challenger_cdt_suff_explo}
			Let $B_{n+1}$ be a leader satisfying Property~\ref{prop:leader_cdt_suff_explo} and $C_n$ the associated challenger.
			Let $\mathcal J_n^\star = \argmax_{ i \in \overline{V_{n}^{L}}} \mu_{i}$. There exists $L_1$ with $\bE_{\bm F}[L_1] < +\infty$ such that if $L \ge L_1$, for all $n$ such that $U_n^L \neq \emptyset$, $\widehat B_{n+1} \notin V_{n}^{L}$ implies $\widehat C_{n+1} \in V_{n}^{L} \cup \left(\mathcal J_n^\star \setminus \left\{\widehat B_{n+1} \right\}\right)$.
\end{property}
Property~\ref{prop:challenger_cdt_suff_explo} holds for the TC challenger (Lemma~\ref{lem:TC_ensures_suff_explo}), the TCI challenger (Lemma~\ref{lem:TCI_ensures_suff_explo}) and the RS challenger (Lemma~\ref{lem:RS_ensures_suff_explo}).

Provided Assumption~\ref{ass:all_arms_distinct_bounded_mean} holds, Lemma~\ref{lem:suff_exploration} shows that sufficient exploration is achieved for any Top Two algorithm satisfying Properties~\ref{prop:leader_cdt_suff_explo} and~\ref{prop:challenger_cdt_suff_explo}.
\begin{lemma} \label{lem:suff_exploration}
	Assume Assumption~\ref{ass:all_arms_distinct_bounded_mean} holds.
	Under a Top Two algorithm whose leader $B_{n+1}$ and challenger $C_{n+1}$ satisfy Properties~\ref{prop:leader_cdt_suff_explo} and~\ref{prop:challenger_cdt_suff_explo}, there exist $N_0$ with $\bE_{\bm F}[N_0] < + \infty$ such that for all $ n \geq N_0$ and all $i\in [K]$, $N_{n,i} \geq \sqrt{\frac{n}{K}}$.
\end{lemma}

\paragraph{Proof of Lemma~\ref{lem:suff_exploration}}
When Assumption~\ref{ass:all_arms_distinct_bounded_mean} holds, combining Properties~\ref{prop:leader_cdt_suff_explo} and~\ref{prop:challenger_cdt_suff_explo} and Lemma~\ref{lem:lower_bound_on_probability_of_sampling_leader_challenger} yields Lemma~\ref{lem:mildly_under_sampled_arms_with_strict_positive_probability_of_being_sampled}.
\begin{lemma} \label{lem:mildly_under_sampled_arms_with_strict_positive_probability_of_being_sampled}
	Assume Assumption~\ref{ass:all_arms_distinct_bounded_mean} holds.
	Under a Top Two algorithm whose leader $B_{n+1}$ and challenger $C_{n+1}$ satisfy Properties \ref{prop:leader_cdt_suff_explo} and \ref{prop:challenger_cdt_suff_explo}, there exists $L_2$ with $\bE_{\bm F}[L_2] < +\infty$ such that if $L \geq L_2$, for all $n$, $U_n^L \neq \emptyset$ implies that there exists $J_{n+1} \in V_{n}^{L}$ such that $\psi_{n+1,J_{n+1}} \geq \psi_{\min}$.
\end{lemma}
\begin{proof}
	Let $\mathcal J_n^\star = \argmax_{ i \in \overline{V_{n}^{L}}} \mu_{i}$.
	Under Assumption~\ref{ass:all_arms_distinct_bounded_mean}, we know that $|\mathcal J_n^\star| = 1$.
	Let $L_0$ as in Property~\ref{prop:leader_cdt_suff_explo}.
	If $L \ge L_0^{4/3}$, for all $n$, $\widehat B_{n+1} \in \overline{V_{n}^{L}}$ implies $\mathcal J_n^\star = \{\widehat B_{n+1} \}$.
	Let $L_1$ as in Property~\ref{prop:challenger_cdt_suff_explo}.
	Therefore, we have if $L \ge L_{2} \eqdef \max\{L_1, L_0^{4/3}\}$, for all $n$ such that $U_n^L \neq \emptyset$, $\widehat B_{n+1} \notin V_{n}^{L}$ implies $\widehat C_{n+1} \in V_{n}^{L}$.
	By Lemma~\ref{lem:lower_bound_on_probability_of_sampling_leader_challenger}, we know that $\psi_{n+1,i} \geq \psi_{\min}$ for all $i \in \{\widehat B_{n+1}, \widehat C_{n+1}\}$.
	Since $\bE_{\bm F}[L_2] \leq \bE_{\bm F}[L_1] + \bE_{\bm F}[L_0^{4/3}] < + \infty$, this concludes the proof.
\end{proof}

Using concentration on $\|T_{n}-\Psi_{n}\|_{\infty}$ (Lemma~\ref{lem:subG_alloc}) and the pigeonhole principle yield a contradiction for a large enough $L$.
Therefore, the set of highly under-sampled arms is empty (Lemma~\ref{lem:lemma_11_Shang20TTTS}).
This technical result was proven in \cite{Shang20TTTS} for TTTS (Lemma 11) and T3C (Lemma 18).
For the sake of completeness, we include the proof.
\begin{lemma} \label{lem:lemma_11_Shang20TTTS}
	Assume Assumption~\ref{ass:all_arms_distinct_bounded_mean} holds.
	Under a Top Two algorithm whose leader $B_{n+1}$ and challenger $C_{n+1}$ satisfy Properties~\ref{prop:leader_cdt_suff_explo} and~\ref{prop:challenger_cdt_suff_explo}, there exists $L_3$ with $\bE_{\bm F}[L_3] < +\infty$ such that for all $L \ge L_3, U_{\lfloor KL \rfloor}^L$ is empty.
\end{lemma}
\begin{proof}
	Assume Assumption~\ref{ass:all_arms_distinct_bounded_mean} holds and we are given a Top Two algorithm whose leader $B_{n+1}$ and challenger $C_{n+1}$ satisfy Properties~\ref{prop:leader_cdt_suff_explo} and \ref{prop:challenger_cdt_suff_explo}.
	Let $L_2$ as in Lemma~\ref{lem:mildly_under_sampled_arms_with_strict_positive_probability_of_being_sampled}, with $\bE_{\bm F}[L_2] < + \infty$.
	We proceed by contradiction, and we assume that $U_{\lfloor K L\rfloor}^{L}$ is not empty.
	Then for any $1 \leq \ell \leq\lfloor K L\rfloor, U_{\ell}^{L}$ and $V_{\ell}^{L}$ are non empty as well.
There exists a deterministic $L_{4}$ such that for all $L \ge L_{4}$, $\lfloor L\rfloor \geq K L^{3 / 4}$.
In particular, $\bE_{\bm F}[L_4] = L_4 < + \infty$. In the following, we consider $L \ge \max \{L_2, L_4\}$.

Using the pigeonhole principle, there exists some $i \in [K]$ such that $N_{\lfloor L\rfloor, i} \geq L^{3 / 4}$.
Thus, we have $\left|V_{\lfloor L\rfloor}^{L}\right| \leq K-1$.
Next, we prove $\left|V_{\lfloor 2 L\rfloor}^{L}\right| \leq K-2$.
Otherwise, since $U_{\ell}^{L}$ is non-empty for any $\lfloor L\rfloor+1 \leq \ell \leq\lfloor 2 L\rfloor$, thus by Lemma~\ref{lem:mildly_under_sampled_arms_with_strict_positive_probability_of_being_sampled}, there exists $J_{\ell +1} \in V_{\ell}^{L}$ such that $\psi_{\ell+1, J_{\ell+1}} \geq \psi_{\min }$. Since $V_{\ell}^{L} \subset V_{\lfloor L\rfloor}^{L}$, we have
\[
\sum_{i \in V_{\ell}^{L}} \psi_{\ell+1, i} \geq \psi_{\min } \quad \text{and} \quad  \sum_{i \in V_{\lfloor L\rfloor}^{L}} \psi_{\ell+1, i} \geq \psi_{\min } \: .
\]
Therefore,
\[
\sum_{i \in V_{\lfloor L\rfloor}^{L}}\left(\Psi_{\lfloor 2 L\rfloor +1, i}-\Psi_{\lfloor L\rfloor +1, i}\right)=\sum_{\ell=\lfloor L\rfloor+1}^{\lfloor 2 L\rfloor} \sum_{i \in V_{\lfloor L\rfloor}^{L}} \psi_{\ell + 1, i} \geq \psi_{\min }\lfloor L\rfloor
\]
Then, using Lemma~\ref{lem:subG_alloc}, there exists $L_{5}=Poly\left(W_{1}\right)$ such that for all $ L \ge  \max \{L_2, L_4, L_5\}$, we have
\begin{align*}
&\sum_{i \in V_{\lfloor L\rfloor}^{L}}\left(N_{\lfloor 2 L\rfloor + 1 , i}-N_{\lfloor L\rfloor + 1, i}\right) \\
&\geq \sum_{i \in V_{\lfloor L\rfloor}^{L}}\left(\Psi_{\lfloor 2 L\rfloor + 1, i}-\Psi_{\lfloor L\rfloor +1, i}-2 W_{1} \sqrt{(\lfloor 2 L\rfloor +1)\log \left(e+\lfloor 2 L\rfloor +1\right)}\right) \\
&\geq \psi_{\min }\lfloor L\rfloor-2 K W_{1} \sqrt{(\lfloor 2 L\rfloor +1)\log \left(e+\lfloor 2 L\rfloor + 1\right)}
\: .
\end{align*}
Then, there exists $L_{3}=Poly\left(W_{1}\right)$ such that for all $ L \ge L_3 \eqdef \max \{L_2, L_4, L_5, L_6\}$,
\[
\sum_{i \in V_{\lfloor L\rfloor}^{L}}\left(N_{\lfloor 2 L\rfloor+1, i}-N_{\lfloor L\rfloor+1, i}\right) \geq K L^{3 / 4} \: ,
\]
which implies that we have one arm in $V_{\lfloor L\rfloor}^{L}$ that is pulled at least $L^{3 / 4}$ times between $\lfloor L\rfloor+1$ and $\lfloor 2 L\rfloor$, thus $\left|V_{\lfloor 2 L\rfloor}^{L}\right| \leq K-2$.

By induction, for any $1 \leq k \leq K$, we have $\left|V_{\lfloor k L\rfloor}^{L}\right| \leq K-k$, and finally $U_{\lfloor K L\rfloor}^{L}=\emptyset$ for all $ L \ge L_{3}$.
Since $\bE[e^{\lambda W_1}] < + \infty$ for all $\lambda > 0$, we have in particular that $\bE_{\bm F }[Poly(W_1)] < + \infty$. Since
\[
	\bE_{\bm F}[L_3] \leq \sum_{i \in \{2,4,5,6\}} \bE_{\bm F}[ L_i ] < + \infty \: ,
\]
this concludes the proof.
\end{proof}


	 Let $L_3$ as Lemma~\ref{lem:lemma_11_Shang20TTTS}. Defining $N_0 = K L_3$, we have $\bE_{\bm F}[ N_0 ]  = K \bE_{\bm F}[ L_3 ] < + \infty $. For all $n \ge N_0$, we let $L = \frac{n}{K}$, then by Lemma~\ref{lem:lemma_11_Shang20TTTS}, we have $U_{\lfloor KL \rfloor}^L = U_n^{n/K}$ is empty, which concludes the proof of Lemma~\ref{lem:suff_exploration}.

\qed

\subsection{How to converge}
\label{app:ss_how_to_converge}

In this section, we identify one property for the leader (Property~\ref{prop:leader_cdt_convergence}) and one property for the challenger (Property~\ref{prop:challenger_cdt_convergence}) under which we prove that the convergence time of the corresponding Top Two algorithm has finite expectation expectation (Lemma~\ref{lem:finite_mean_time_eps_convergence_beta_opti_alloc}), provided sufficient exploration occurs. Sufficient exploration is formalized by Property~\ref{prop:suff_exploration}.
\begin{property} \label{prop:suff_exploration}
	There exist $N_1$ with $\bE_{\bm F}[N_1] < +\infty$ such that for all $ n \geq N_1$ and all $i\in [K]$, $N_{n,i} \geq \sqrt{\frac{n}{K}}$.
\end{property}
For a Top Two algorithm whose leader $B_n$ and challenger $C_n$ satisfy Properties~\ref{prop:leader_cdt_suff_explo} and~\ref{prop:challenger_cdt_suff_explo}, Lemma~\ref{lem:suff_exploration} shows that Property~\ref{prop:suff_exploration} holds provided that Assumption~\ref{ass:all_arms_distinct_bounded_mean} holds.
While Assumption~\ref{ass:all_arms_distinct_bounded_mean} allows to ensure sufficient exploration, other algorithmic choices could ensure this property without having the constraint that all means are distinct.
We discuss algorithmic fixes in Appendix~\ref{app:ss_beyond_all_distinct_means}.
As mentioned above the dependency $\sqrt{n}$ is arbitrary and we could consider $n^{\alpha}$ with $\alpha \in (0,1)$.
A similar Lemma~\ref{lem:suff_exploration} could be obtained for this choice.

Let $S_{n}^{L} \eqdef \{i \in [K] \mid N_{n,i} \ge L \}$ as in (\ref{eq:def_sampled_enough_sets}) and $N_1$ as in Property~\ref{prop:suff_exploration}.
Using Assumption~\ref{ass:standard_assumptions_BAI}, we have that $\argmax_{ i \in S_{n}^{\sqrt{n/K}}} \mu_i = i^\star(\bm F)$ for all $n \ge N_1$.

\paragraph{Converging with leader and challenger pair}
We describe the properties that a good leader/challenger should have to ensure convergence towards the $\beta$-optimal allocation, assuming Property~\ref{prop:suff_exploration} holds.
In order for the challenger to converge, a \textit{good} leader should first identify the best arm $i^\star(\bm F)$ (Property~\ref{prop:leader_cdt_convergence}).
Then, given a good leader, a \textit{good} challenger should sample each sub-optimal arm with probability $w_{i}^{\beta}$: in particular, when the mean probability of sampling a sub-optimal arm $i$ exceeds $w_{i}^{\beta}$, this arm should have a small probability of being sampled again (Property~\ref{prop:challenger_cdt_convergence}).


\begin{property} \label{prop:leader_cdt_convergence}
	Assume Property~\ref{prop:suff_exploration} holds.
	There exists $N_2$ with $\bE_{\bm F}[N_2] < +\infty$ such that for all $n \ge N_2$,
	\[
	\bP_{\mid n}[ B_{n+1} \neq i^\star (\bm F)] \leq g(n) \: ,
	\]
	where $g : \N^\star \to (0, +\infty)$ such that $g(n) =_{+\infty} o(n^{-\alpha})$ with $\alpha >0$.
\end{property}
Property~\ref{prop:leader_cdt_convergence} holds for the EB leader (Lemma~\ref{lem:EB_ensures_convergence}) and the TS leader (Lemma~\ref{lem:TS_ensures_convergence}).


\begin{property} \label{prop:challenger_cdt_convergence}
	Assume Property~\ref{prop:suff_exploration} holds.
	Let $B_{n+1}$ be a leader satisfying Property~\ref{prop:leader_cdt_convergence} and $C_{n+1}$ the associated challenger.
	Let $\epsilon \in (0, \epsilon_{0}(\bm F)]$ where $\epsilon_0(\bm F) > 0$ is a problem dependent constant. There exists $N_3$ with $\bE_{\bm F}[N_3] < +\infty$ such that for all $n \geq N_3$ and all $i \neq i^\star(\bm F)$,
	\begin{equation} \label{eq:overshooting_implies_not_sampled_anymore}
	\frac{\Psi_{n,i}}{n} \geq w_{i}^{\beta} + \epsilon  \quad \implies \quad \bP_{\mid n}[C_{n+1} = i \mid B_{n+1} = i^\star(\bm F)] \leq h(n) \: ,
	\end{equation}
	where $h: \N^\star \to (0,+ \infty)$ such that $h(n) =_{+\infty} o(n^{-\alpha})$ with $\alpha >0$.
\end{property}
Property~\ref{prop:challenger_cdt_convergence} holds for the TC challenger (Lemma~\ref{lem:TC_ensures_convergence}), the TCI challenger (Lemma~\ref{lem:TCI_ensures_convergence}) and the RS challenger (Lemma~\ref{lem:RS_ensures_convergence}).

Provided that Property~\ref{prop:suff_exploration} holds, Lemma~\ref{lem:finite_mean_time_eps_convergence_beta_opti_alloc} shows that $\bE_{\bm F}[T_{\beta}^{\epsilon}] < +\infty$ for any Top Two algorithm satisfying Properties~\ref{prop:leader_cdt_convergence} and~\ref{prop:challenger_cdt_convergence}.
\begin{lemma} \label{lem:finite_mean_time_eps_convergence_beta_opti_alloc}
	Assume Property~\ref{prop:suff_exploration} holds.
	Let $\epsilon \in (0, \epsilon_1(\bm F)]$ where $\epsilon_1(\bm F) > 0$ is a problem dependent constant.
	Let $T_{\beta}^{\epsilon}$ as in (\ref{eq:rv_T_eps_beta}).
	Under a Top Two algorithm whose leader $B_{n+1}$ and challenger $C_{n+1}$ satisfy Properties~\ref{prop:leader_cdt_convergence} and~\ref{prop:challenger_cdt_convergence}, we have $\bE_{\bm F}[T_{\beta}^{\epsilon}] < +\infty$.
\end{lemma}

\paragraph{Proof of Lemma~\ref{lem:finite_mean_time_eps_convergence_beta_opti_alloc}}

We first establish in Lemma~\ref{lem:convergence_towards_optimal_allocation_best_arm} the convergence towards the optimal allocation for the best arm,  $w_{i^\star(\bm F)}^{\beta}=\beta$.
\begin{lemma} \label{lem:convergence_towards_optimal_allocation_best_arm}
		Let $\varepsilon>0$.
		Assume Property~\ref{prop:suff_exploration} holds.
		Under a Top Two algorithm whose leader $B_{n}$ satisfies Property~\ref{prop:leader_cdt_convergence}, there exists $N_{4}$ with $\bE_{\bm F}[N_4] < +\infty$ such that for all $n \geq N_{4}$,
\[
 \left| \frac{N_{n, i^\star(\bm F)}}{n} - \beta \right| \leq  \varepsilon \: .
\]
\end{lemma}
\begin{proof}
Let $i^\star = i^\star(\bm F) $ and $\varepsilon > 0 $.
Let $N_1$ as in Property~\ref{prop:suff_exploration} and $N_{2}$ as in Property~\ref{prop:leader_cdt_convergence}.

For all $n \geq \max \{ N_1, N_2 \}$, we have $\bP_{\mid (n-1)}[ B_n \neq i^\star] \leq g(n-1)$.
Let $M \ge \max \{ N_1, N_2 \}$.
Using Lemma~\ref{lem:deviation_wrt_the_optimal_arm}, we have
	\begin{align*}
			\left|\frac{\Psi_{n,i^\star}}{n} - \beta \right| &\leq \frac{M-1}{n} + \frac{1}{n} \sum_{t=M}^{n} \bP_{\mid (t-1)}[B_t \neq i^\star] \leq \frac{M-1}{n} + \frac{1}{n} \sum_{t=M}^{n} g(t-1)
	\end{align*}

	By Property~\ref{prop:leader_cdt_convergence}, $g(n) = o(n^{-\alpha})$ with $\alpha > 0$. Using Cesaro's theorem, there exists $N_5 = Poly(W_1) $ such that for all $n \ge \max \{ N_1, N_2, N_5 \}$,
	\[
	\frac{M-1}{n} \leq \frac{\epsilon}{4} \quad \text{and} \quad \frac{1}{n} \sum_{t=M}^{n} g(t-1) \leq \frac{\epsilon}{4} \: .
	\]

	Therefore, we have shown that $\left|\frac{\Psi_{n,i^\star}}{n} - \beta \right| \leq \frac{\epsilon}{2}$ for all $n \ge \max \{ N_1, N_2, N_5 \}$. Using Lemma~\ref{lem:subG_alloc}, we have
\[
\forall n \in \N , \quad \left| \frac{N_{n,i^\star}}{n} - \frac{\Psi_{n,i^\star}}{n} \right| \leq W_1\sqrt{\frac{n+1}{n^2}\log(e+n)} \: .
\]
Therefore, there exists $N_6 = Poly\left(W_{1}\right)$ such that for all $n \geq N_4 \eqdef \max \{ N_1, N_2, N_5, N_6 \}$,
\[
\left| \frac{N_{n,i^\star}}{n} - \frac{\Psi_{n,i^\star}}{n} \right| \leq \frac{\epsilon}{2} \: .
\]
Using the triangular inequality, we obtain
\[
\left| \frac{N_{n,i^\star}}{n}  - \beta \right| \leq \epsilon \: .
\]
Finally $N_4$ verifies the condition $\bE_{\bm F}[N_4] < \infty$ since
\[
	\bE_{\bm F}[N_4] \leq \sum_{i \in \{1,2,5,6\}} \bE_{\bm F}[ N_i ] < + \infty \: .
\]
\end{proof}

We then prove in Lemma~\ref{lem:convergence_towards_optimal_allocation_other_arms} the convergence towards the optimal allocation for all arms. We notably use the previous convergence result established for the optimal arm.

\begin{lemma} \label{lem:convergence_towards_optimal_allocation_other_arms}
	Assume Property~\ref{prop:suff_exploration} holds.
	Let $\epsilon \in (0, \epsilon_1(\bm F)]$ where $\epsilon_1(\bm F) > 0$ is a problem dependent constant.
	Under a Top Two algorithm whose leader $B_n$ and challenger $C_n$ satisfy Properties~\ref{prop:leader_cdt_convergence} and~\ref{prop:challenger_cdt_convergence}, there exists $N_{5}$ with $\bE_{\bm F}[N_5] < +\infty$ such that for all $n \geq N_{5}$,
	\[
		\forall i \in [K], \quad \left| \frac{N_{n,i}}{n} - w_{i}^{\beta} \right| \leq  \varepsilon \: .
	\]
\end{lemma}
\begin{proof}
	Let $i^\star = i^\star(\bm F)$.
Let $\epsilon_0 = \epsilon_0 (\bm F)$ and $N_3$ as in Property~\ref{prop:challenger_cdt_convergence} and $\epsilon \in (0, \epsilon_0]$.
Let $N_1$ and $N_{2}$ as in Properties~\ref{prop:suff_exploration} and~\ref{prop:leader_cdt_convergence}.
Let $N_{4}$ as in Lemma~\ref{lem:convergence_towards_optimal_allocation_best_arm}.
For all $n \geq \max_{i \in [4]} N_i$, we have $\left| \frac{N_{n, i^\star}}{n} - \beta \right| \leq  \varepsilon$ and for all $i \neq i^\star$,
\[
 	\frac{\Psi_{n-1,i}}{n-1} \geq w_{i}^{\beta} + \epsilon \: \implies \: \bP_{\mid (n-1)}[C_n = i \mid B_n = i^\star] \leq h(n-1)  \: .
\]

	Let $M \ge \max_{i \in [4]} N_i$. By Properties~\ref{prop:leader_cdt_convergence} and~\ref{prop:challenger_cdt_convergence}, $g(n) = o(n^{-\alpha})$ with $\alpha > 0$ and $h(n) = o(n^{-\alpha})$ with $\alpha > 0$.
	Using Cesaro's theorem, there exists a deterministic $N_6$ such that for all $n \ge N_6$,
	\[
\frac{M-1}{n} \leq \epsilon \quad \text{,} \quad \frac{1}{n} \sum_{t=M}^{n} g(t-1) \leq \epsilon \quad \text{and} \quad \frac{1}{n} \sum_{t=M}^{n} h(t-1) \leq \epsilon \: .
	\]
	In particular, $\bE_{\bm F}[N_6] = N_6 < +\infty$. Let $t_{n-1,i}(\epsilon)=\max \left\{t \leq n \mid \frac{\Psi_{t-1, i}}{n-1} \leq w_{i}^{\beta}+ \epsilon \right\}$.
Using Lemma~\ref{lem:upper_bound_probability_sampling_other_arms} and $\frac{\Psi_{t-1, i}}{n-1} \le \frac{\Psi_{t-1, i}}{t-1}$ for $t\leq n$, we obtain for all $n \ge \max_{i \in [4] \cup \{6\}} N_i$
\begin{align*}
		\frac{\Psi_{n,i}}{n} &\leq \frac{M-1}{n} + \frac{1}{n} \sum_{t=M}^{n} \bP_{\mid (t-1)}[B_t \neq i^\star] + \frac{1}{n} \sum_{t=M}^{n} \bP_{\mid (t-1)}[C_t = i| B_t = i^\star] \\
		&\leq  \epsilon + \frac{1}{n} \sum_{t=M}^{n} g(t-1) + \frac{1}{n}\sum_{t=M}^{n} \bP_{\mid (t-1)}[C_t = i| B_t = i^\star] \indi{\frac{\Psi_{t-1, i}}{n-1} \geq w_{i}^{\beta}+\epsilon}\\
		&\quad +\frac{1}{n}\sum_{t=M}^{n} \bP_{\mid (t-1)}[C_t = i| B_t = i^\star] \indi{\frac{\Psi_{t, i}}{n-1} \leq w_{i}^{\beta}+\epsilon} \\
		&\leq 2\epsilon + \frac{1}{n}\sum_{t=M}^{n} h(t-1)  +\frac{1}{n}\sum_{t=M}^{t_{n,i}(\epsilon)} \bP_{\mid (t-1)}[C_t = i| B_t = i^\star] \indi{\frac{\Psi_{t-1, i}}{n-1} \leq w_{i}^{\beta}+\epsilon} \\
		&\leq 3\epsilon + \frac{\Psi_{t_{n-1,i}(\epsilon),i}}{n-1} \\
		&\leq w_{i}^{\beta} + 4\epsilon
\end{align*}
As a similar upper bound was already shown in the proof of Lemma~\ref{lem:convergence_towards_optimal_allocation_best_arm}, we obtain $\frac{\Psi_{n, i}}{n}  \leq w_{i}^{\beta} + 4\epsilon$ for all $i \in [K]$ and all $n \ge \max_{i \in [4] \cup \{6\}} N_i$.

Since $\frac{ \Psi_{n,i}}{n}$ and $ w_i^\beta$ sum to $1$, we obtain for all $n \geq \max_{i \in [4] \cup \{6\}} N_i$ and all $i \in [K]$,
\begin{align*}
\frac{ \Psi_{n,i}}{ n} &= 1 - \sum_{j \neq i} \frac{ \Psi_{n,j}}{ n}
\geq 1 - \sum_{j \neq i} \left(  w_j^\beta + 4\epsilon \right)
=  w_i^\beta - 4(K-1)\epsilon \: .
\end{align*}
Therefore, for all $ n \geq \max_{i \in [4] \cup \{6\}} N_i$ and all $i \in [K]$,
$
\left|\frac{\Psi_{n, i}}{n} - w_{i}^{\beta}\right| \leq 4(K-1)\epsilon \: .
$

Using Lemma~\ref{lem:subG_alloc}, we have for all $n \in \N$ and all $i \in [K]$,
$
\left| \frac{N_{n,i}}{n} - \frac{\Psi_{n,i}}{n} \right| \leq W_1\sqrt{\frac{n+1}{n^2}\log(e+n)} \: ,
$
 hence there exist $N_7 = Poly\left(W_{1}\right) $ such that for all $n \geq N_5 \eqdef \max_{i \in [4] \cup \{6, 7\}} N_i$ and all $i \in [K]$,
\[
\left| \frac{N_{n,i}}{n} - \frac{\Psi_{n,i}}{n} \right| \leq \epsilon \: .
\]

Using the triangular inequality, we obtain that for all $n \geq N_5 $ and all $i \in [K]$,
\[
\left| \frac{N_{n,i}}{n}  - w_{i}^{\beta} \right| \leq (4K-3)\epsilon \: .
\]
Since
\[
	\bE_{\bm F}[N_5] \leq \sum_{i \in [4] \cup \{6, 7\}} \bE_{\bm F}[ N_i ] < + \infty \: ,
\]
taking $\epsilon_1 = \frac{\epsilon_0}{4K-3}$ yields the result for all $\epsilon \in (0, \epsilon_1]$.
\end{proof}

Let $N_{5}$ as in Lemma~\ref{lem:convergence_towards_optimal_allocation_other_arms}.
By definition of $T_{\beta}^{\epsilon}$ in (\ref{eq:rv_T_eps_beta}), we have $\bE_{\bm F}[T_{\beta}^{\epsilon}] \le \bE_{\bm F}[N_{5}] < + \infty$.
This concludes the proof of Lemma~\ref{lem:finite_mean_time_eps_convergence_beta_opti_alloc}.

\subsection{Asymptotic $\beta$-optimality}
\label{app:ss_asymptotic_optimality}

Provided some regularity assumption on the class of distribution $\cF$, we show that asymptotic $\beta$-optimality is a direct consequence of $\bE_{\bm F}[T^{\epsilon}_{\beta}] < + \infty$. More precisely, we show \eqref{eq:finite_Tepsbeta_optimality}:
\[
		\exists \epsilon_1(\bm F)> 0, \: \forall \epsilon \in (0,\epsilon_1(\bm F)], \: \bE_{\bm F}[T^{\epsilon}_{\beta}] < + \infty \quad \implies  \quad \limsup_{\delta \to 0} \frac{\bE_{\bm F}[\tau_{\delta}]}{\log \left(1/\delta \right)} \leq T_{\beta}^\star(\bm F ) \: .
\]

In \cite{Qin2017TTEI}, \eqref{eq:finite_Tepsbeta_optimality} was proven for Gaussian. We generalize the proof from \cite{Qin2017TTEI} to hold provided we have joint continuity of the minimal transportation cost (Property~\ref{prop:joint_continuity_true_tc}) and rate of convergence for $\cT(F_{n,i})$ (Property~\ref{prop:rate_convergence_mean_and_distribution}).
Those properties hold for bounded distributions and SPEF with sub-exponential distributions, and potentially many other distributions.

Before stating the adequate properties, we recall the notation introduced in Appendix~\ref{app:ss_generic_beta_optimality}.
Let $C_{i,j}(\cT(\bm F), w)$ be transportation costs defined in (\ref{eq:def_transportation_cost_general_formulation}) as
\[
	C_{i,j}(\cT(\bm F), w) \eqdef \inf_{u \in \cI} \left\{ w_i \Kinf^{-}(\cT(F_i), u) + w_j \Kinf^{+}(\cT(F_j), u)\right\} \: ,
\]
where $\cI \subseteq \R$, and their empirical version in (\ref{eq:def_emp_transportation_cost_general_formulation}) as
\[
 \frac{1}{n} W_{n}(i,j) = C_{i,j}\left(\cT(\bm F_n), \frac{N_n}{n}\right) \: .
\]
Similarly, the $\beta$-characteristic time and $\beta$-optimal allocation
\begin{align*}
		T_{\beta}^\star(\bm F)^{-1} \eqdef \max_{w \in \simplex : w_{i^\star(\bm F)} = \beta} \min_{j \neq i^\star(\bm F)} C_{i^\star(\bm F),j}(\cT(\bm F), w) \: , \\
		w_{\beta}^\star(\bm F) \eqdef \argmax_{w \in \simplex : w_{i^\star(\bm F)} = \beta} \min_{j \neq i^\star(\bm F)} C_{i^\star(\bm F),j}(\cT(\bm F), w) \: .
\end{align*}

\begin{property} \label{prop:joint_continuity_true_tc}
	$\cT(F) \mapsto m(F)$ is continuous on $\cT(\cF)$ and
$
 \left( \cT(\bm F), w  \right) \mapsto \min_{j \neq i^\star(\bm F)} C_{i^\star(\bm F),j}(\cT(\bm F), w)
$
is continuous on $\cT(\cF^{K}) \times \simplex$.
If $|i^\star(\bm F)|=1$, then $w_{\beta}^\star(\bm F) = \{w^{\beta}\}$ is a singleton such that $\min_{i \in [K]}w^{\beta}_i > 0$.
\end{property}

For single-parameter exponential families, Property~\ref{prop:joint_continuity_true_tc} is a known result from the literature \cite{Russo2016TTTS} as $\cT(F) = m(F)$.
Property~\ref{prop:joint_continuity_true_tc} holds for bounded distributions: $F \mapsto m(F)$ continuous (bounded), using proof of Lemma~\ref{lem:T_star_and_w_star_continuous} (consequence of Lemma~\ref{lem:unique_continuous_mu_star}) and by Lemmas~\ref{lem:properties_characteristic_times} and~\ref{lem:w_star_positive}.

\begin{property} \label{prop:rate_convergence_mean_and_distribution}
	For all $\epsilon > 0$, there exists $N_{\epsilon}$ with $\bE_{\bm F}[N_{\epsilon}] < + \infty$ such that
	\[
	\forall i \in [K], \: \forall N_{n,i} \ge N_{\epsilon}, \quad  \| \cT(F_{n,i}) - \cT(F_i)\| \le \epsilon  \: ,
	\]
	where $\|\cdot\|$ is the norm on $\cT(\cF)$.
\end{property}
For SPEF, we have $\cT(F_{n,i}) = \mu_{n,i}$, hence Property~\ref{prop:rate_convergence_mean_and_distribution} holds for any SPEF with sub-exponential distributions (see Lemma~\ref{lem:W_concentration_spef}).
For bounded distributions, Property~\ref{prop:rate_convergence_mean_and_distribution} is a direct corollary of Lemma~\ref{lem:subG_cdf}.
Since $N_{\epsilon} = Poly(\frac{1}{\epsilon}, W_2)$ and $\bE_{\bm F} [e^{\lambda W_2}] < + \infty$ for all $\lambda > 0$, we have directly that $\bE_{\bm F}[N_{\epsilon}] < + \infty$.

Using the empirical transportation defined in (\ref{eq:def_emp_transportation_cost_general_formulation}), generalizing the stopping time in (\ref{eq:def_stopping_time}) yields
\begin{equation} \label{eq:def_stopping_time_general_formulation}
  \tau_{\delta} = \inf \left\{ n \in \N \mid \min_{j \neq \hat \imath_n} W_n(\hat \imath_n, j) > c(n,\delta) \right\} \: .
\end{equation}

Calibrating the stopping threshold to obtain $\delta$-correctness of the stopping rule (\ref{eq:def_stopping_time_general_formulation}) highly depends on the considered $\cF$.
Definition~\ref{def:asymptotically_tight_threshold} introduces \textit{asymptotically tight} thresholds, whose $(n,\delta)$ dependencies ensure asymptotic ($\beta$-)optimality.
\begin{definition}[Asymptotically tight threshold] \label{def:asymptotically_tight_threshold}
A threshold $c : \N \times (0,1] \to \R_+$ is said to be asymptotically tight if there exists $\alpha \in [0,1)$, $\delta_0 \in (0,1]$, functions $f,\bar{T} : (0,1] \to \R_+$ and $C$ independent of $\delta$ satisfying:
(1) for all $\delta \in  (0,\delta_0]$ and $n \ge \bar{T}(\delta)$, then $c(n,\delta) \le f(\delta) + C n^\alpha$,
(2) $\limsup_{\delta \to 0} f(\delta)/\log(1/\delta) \le 1$ and
(3) $\limsup_{\delta \to 0} \bar{T}(\delta)/\log(1/\delta) = 0$.
\end{definition}
For bounded distributions, the stopping threshold defined in (\ref{eq:def_kinf_threshold_glr}) is asymptotically tight, e.g. take $(\alpha, \delta_0, C) = (1/2,1,1)$, $f(\delta) = \ln \left( \frac{K-1}{\delta}\right) + 2$ and $\bar{T}(\delta) = 1$.
Lemma~\ref{lem:kinf_concentration} shows that it is also $\delta$-correct for bounded distributions.

For single-parameter exponential families, the thresholds for which $\delta$-correctness has been proved are also asymptotically tight, e.g. the ones derived in \cite{KK18Mixtures}. Those thresholds are all upper bounded by some threshold of the form $c(n,\delta) = \ln\left( \frac{D n^{\kappa}}{\delta}\right)$.
This stylized threshold is asymptotically tight, e.g. by taking $(\alpha, \delta_0, C) = (1/2,1,\kappa)$, $f(\delta) = \ln \left( \frac{D}{\delta}\right) $ and $\bar{T}(\delta) = 1$.

Theorem~\ref{thm:asymptotic_optimality_top_two_algorithms} shows \eqref{eq:finite_Tepsbeta_optimality} when using the stopping rule (\ref{eq:def_stopping_time_general_formulation}) with an asymptotically tight threshold.
\begin{theorem} \label{thm:asymptotic_optimality_top_two_algorithms}
	Assume that Properties~\ref{prop:joint_continuity_true_tc} and~\ref{prop:rate_convergence_mean_and_distribution} hold on $\cF^K$.
	Let $(\delta, \beta) \in (0,1)^2$.
	Assume that there exists $\epsilon_1(\bm F)> 0$ such that for all $\epsilon \in (0,\epsilon_1(\bm F)]$, $\bE_{\bm F}[T^{\epsilon}_{\beta}] < + \infty$.
	Combining the stopping rule (\ref{eq:def_stopping_time_general_formulation}) with an asymptotically tight threshold
	yields an algorithm such that for all $\bm F \in \cF^{K}$, with $|i^\star(\bm F)|=1$ and $\mu_{\bm F} \in \left(\mathring \cI \right)^{K}$,
	\[
		\limsup_{\delta \to 0} \frac{\bE_{\bm F}[\tau_{\delta}]}{\log \left(\frac{1}{\delta} \right)} \leq T_{\beta}^\star(\bm F ) \: .
	\]
\end{theorem}
\begin{proof}
	Let $i^\star = i^\star(\bm F)$ and $\epsilon_1 = \epsilon_1(\bm F)$.
	Let $c_{\beta} = \frac{1}{2}\min_{i \in [K]}w^{\beta}_i > 0$ and $\Delta = \min_{j \neq i^\star}|\mu_{i^\star} - \mu_i| > 0$.
	Let $\zeta > 0$. Using Property~\ref{prop:joint_continuity_true_tc}, the continuity of
	\[
	 \left( \cT(\bm F), w  \right) \mapsto \min_{j \neq i^\star(\bm F)} C_{i^\star(\bm F),j}(\cT(\bm F), w) \quad \text{and} \quad \cT(F) \mapsto m(F)
	\]
	yields that there exists $\epsilon_2 > 0$ such that
\begin{align*}
			&\max_{i \in [K]} \left| \frac{N_{n,i}}{n} -w^{\beta}_i \right| \leq \epsilon_2 \quad \text{and} \quad \max_{i \in [K]} \left\| \cT (F_{n,i}) - \cT(F_i) \right\| \leq \epsilon_2 \\
			 \implies \quad &  \max_{i \in [K]}|\mu_{n,i} - \mu_{i}| \le \frac{\Delta}{4} \quad  \text{and} \quad \frac{1}{n}\min_{j\neq i^\star} W_{n}(i^\star, j) \geq \frac{1 - \zeta}{T_{\beta}^\star(\bm F )} \: .
\end{align*}
	Choosing such a $\epsilon_2$, we take $\epsilon \in (0, \min\{\epsilon_1, \epsilon_2, c_{\beta}\})$.
	By assumption, we have $\bE_{\bm F}[T^{\epsilon}_{\beta}] < + \infty$, hence $\frac{N_{n,i}}{n} \ge w^{\beta}_{i} - \epsilon \ge c_{\beta}$ for all $i \in [K]$.

	Let $N_{\epsilon}$ as in Property~\ref{prop:rate_convergence_mean_and_distribution}. Using Property~\ref{prop:rate_convergence_mean_and_distribution}, for all $n \ge  c_{\beta}^{-1}N_{\epsilon}$, we have $\max_{i \in [K]}\| \cT(F_{n,i}) - \cT(F_i)\| \le \epsilon \le \epsilon_2$ as $\min_{i \in [K]} N_{n,i} \ge N_{\epsilon}$.
	Therefore, we have $\hat \imath_n \in \argmax_{i \in [K]} \mu_{n,i} = \argmax_{i \in [K]} \mu_i$ as $\max_{i \in [K]}|\mu_{n,i} - \mu_{i}| \le \frac{\Delta}{4}$.

	Let $\alpha \in [0,1)$, $\delta_0 \in (0,1]$, functions $f,\bar{T} : (0,1] \to \mathbb{R}_+$ and $C$ as in the definition of an asymptotically tight family of thresholds.
		In the following, we consider $\delta \leq \delta_{0}$.
		Let $\kappa > 0 $.
		Let $T \ge \frac{1}{\kappa}\max\{T^{\epsilon}_{\beta}, c_{\beta}^{-1}N_{\epsilon}, \bar{T}(\delta) \}$.
		Using the definition of the stopping rule (\ref{eq:def_stopping_time}) with a family of asymptotically tight threshold, we have
		\begin{align*}
		\min \left\{\tau_{\delta}, T\right\} \leq \kappa T + \sum_{n=\kappa T}^{T} \indi{\tau_{\delta}>n}  &\leq \kappa T + \sum_{n=\kappa T}^{T} \indi{ \min_{j\neq i^\star} W_{n}(i^\star, j) \leq c(n,\delta)}\\
		&\leq \kappa T + \sum_{n=\kappa T}^{T} \indi{ n \frac{1 - \zeta}{T_{\beta}^\star(\bm F )}	 \leq  f(\delta) + CT^\alpha}\\
		&\leq \kappa T  + \frac{T_{\beta}^\star(\bm F )}{1 - \zeta}  \left( f(\delta) + CT^\alpha\right) \: .
		\end{align*}

	Let $T_{\zeta}(\delta)$ defined as
	\[
				T_{\zeta}(\delta) \eqdef \inf \left\{ T \geq 1 \mid  \frac{T_{\beta}^\star(\bm F )}{(1 - \zeta)(1 - \kappa)} \left( f(\delta) + CT^\alpha\right) \leq  T \right\} \: .
	\]
	 For every $T \ge \max \{ T_{\zeta}(\delta), \frac{1}{\kappa}\max \{T^{\epsilon}_{\beta}, c_{\beta}^{-1}N_{\epsilon}, \bar{T}(\delta) \} \}$, we have $\tau_{\delta} \le T$, hence
\[
	\bE_{\bm F}[\tau_{\delta}] \leq \frac{1}{\kappa} \bE_{\bm F}[T^{\epsilon}_{\beta}] + \frac{1}{\kappa c_{\beta}} \bE_{\bm F}[N_{\epsilon}] + \frac{1}{\kappa} \bar{T}(\delta) + T_{\zeta}(\delta) \: .
\]
	As $\bE_{\bm F}[T^{\epsilon}_{\beta}] + c_{\beta}^{-1} \bE_{\bm F}[N_{\epsilon}] < + \infty$ and $\lim_{\delta \to 0} \frac{\bar{T}(\delta)}{\log (1 / \delta)}$, we obtain for all $\zeta, \kappa > 0$
	 \begin{align*}
		 \limsup _{\delta \to 0} \frac{\bE_{\bm F}\left[\tau_{\delta}\right]}{\log (1 / \delta)}
		 \leq \limsup_{\delta \to 0} \frac{T_{\zeta}(\delta)}{\log (1 / \delta)}
		 \le \frac{T_{\beta}^\star(\bm F )}{(1 - \zeta)(1 - \kappa)}
		 \: ,
	 \end{align*}
	 where the last inequality uses Lemma~\ref{lem:inversion_lemma_technicality}, which is an inversion result.
	 Letting $\zeta$ and $\kappa$ go to zero yields that
 \begin{align*}
	 \limsup _{\delta \to 0} \frac{\bE_{\bm F}\left[\tau_{\delta}\right]}{\log (1 / \delta)} \leq T_{\beta}^\star(\bm F )
	 \: .
 \end{align*}
\end{proof}

Corollary~\ref{cor:asymptotic_optimality_top_two_algorithms} is a direct consequence of Lemma~\ref{lem:suff_exploration}, Lemma~\ref{lem:finite_mean_time_eps_convergence_beta_opti_alloc} and Theorem~\ref{thm:asymptotic_optimality_top_two_algorithms}.

\begin{corollary} \label{cor:asymptotic_optimality_top_two_algorithms}
	Assume that Properties~\ref{prop:joint_continuity_true_tc} and~\ref{prop:rate_convergence_mean_and_distribution} hold on $\cF^K$.
	Let $(\delta, \beta) \in (0,1)^2$.
	Combining the stopping rule (\ref{eq:def_stopping_time_general_formulation}) with an asymptotically tight threshold and a Top Two algorithm,
	whose leader $B_{n}$ and challenger $C_{n}$ satisfy Properties~(\ref{prop:leader_cdt_suff_explo},~\ref{prop:leader_cdt_convergence}) and (\ref{prop:challenger_cdt_suff_explo},~\ref{prop:challenger_cdt_convergence}),
	yields an algorithm such that for all $\bm F \in \cF^{K}$, with $\Delta_{\min}(\bm F) \eqdef \min_{j\neq i}|\mu_{i} - \mu_{j}| > 0$ and $\mu_{\bm F} \in \left(\mathring \cI \right)^{K}$,
	\[
		\limsup_{\delta \to 0} \frac{\bE_{\bm F}[\tau_{\delta}]}{\log \left(\frac{1}{\delta} \right)} \leq T_{\beta}^\star(\bm F ) \: .
	\]
\end{corollary}
\begin{proof}
	Having $\Delta_{\min}(\bm F) > 0$ yields that $i^\star(\bm F)$ is a singleton.
	Since $\Delta_{\min}(\bm F) > 0$ and leader $B_{n}$ and challenger $C_{n}$ satisfies Properties~\ref{prop:leader_cdt_suff_explo} and~\ref{prop:challenger_cdt_suff_explo}, Lemma~\ref{lem:suff_exploration} shows that Property~\ref{prop:suff_exploration} holds.
	As leader $B_{n}$ and challenger $C_{n}$ satisfies Properties~\ref{prop:leader_cdt_convergence} and~\ref{prop:challenger_cdt_convergence}, we can use Lemma~\ref{lem:finite_mean_time_eps_convergence_beta_opti_alloc}.
	Directly applying Theorem~\ref{thm:asymptotic_optimality_top_two_algorithms} yields the result.
\end{proof}

\begin{lemma} \label{lem:inversion_lemma_technicality}
	Let $C, D \in \R^{\star}_{+}$, $\alpha \in [0,1)$, $f: (0,1] \to \R_{+}$ such that $\lim_{\delta \to 0} \frac{f(\delta)}{\log (1/\delta)} \le 1$ and
	\begin{equation} \label{eq:def_t0_delta}
				T_{D}(\delta) \eqdef \inf \left\{ T \geq 1 \mid D \left( f(\delta) + CT^\alpha\right) \leq  T \right\} \: .
	\end{equation}
	Then,
	\[
	\limsup_{\delta \to 0} \frac{T_{D}(\delta)}{\log (1 / \delta)} \le D \: .
	\]
\end{lemma}
\begin{proof}
	Let $\gamma > 0$. Since $\alpha \in [0,1)$, there exists $T_{\gamma}$ (depending on $D$) such that for all $T \ge T_{\gamma}$,
		\[
			T \frac{1}{D} - C T^{\alpha} \ge T \frac{1}{D(1+\gamma)} \: .
		\]
	 Then,
	 \begin{align*}
	 T_{D}(\delta) &\le  T_{\gamma} + \inf \left\{T \ge 1 \mid f(\delta) \le T \frac{1}{D(1+\gamma)} \right\} \le   T_{\gamma}   + D(1+\gamma) f(\delta) + 1
	 \: .
	 \end{align*}
	 Since $\limsup_{\delta \to 0} \frac{f(\delta)}{\log(1/\delta)} \le 1$, we obtain for all $\gamma > 0$
	 \begin{align*}
	 \limsup_{\delta \to 0}\frac{T_{D}(\delta)}{\log(1/\delta)}
	 &\le D(1+\gamma)
	 \: .
	 \end{align*}
	  Letting $\gamma$ go to zero yields the result.
\end{proof}

%% file: sections/appendix_top_two_instances.tex

\section{Top Two instances for bounded distributions}
\label{app:top_two_instances}

While we provided a unified analysis of Top Two algorithms in Appendix~\ref{app:unified_analysis_top_two}, we are interested in specific instances.
We distinguish between the deterministic mechanisms in Appendix~\ref{app:ss_deterministic_mechanisms} and the randomized mechanisms in Appendix~\ref{app:ss_randomized_mechanisms}.
After introducing them, we will show they each satisfies the properties required on the leader and the challenger to ensure sufficient exploration (Appendix~\ref{app:ss_how_to_explore}) and to converge towards the $\beta$-optimal allocation (Appendix~\ref{app:ss_how_to_converge}).

As deterministic mechanisms, we study the EB leader (Appendix~\ref{app:sss_eb_leader}), the TC challenger (Appendix~\ref{app:sss_tc_challenger}) and the TCI challenger (Appendix~\ref{app:sss_tci_challenger}).
For the randomized mechanisms which are based on a sampler $\Pi_n$ (Appendix~\ref{app:sss_how_to_sample}), we consider the TS leader (Appendix~\ref{app:sss_ts_leader}) and the RS challenger (Appendix~\ref{app:sss_rs_challenger}).
While those leaders and challengers are defined and analyzed for bounded distributions, we will also discuss why the analysis still hold for single-parameter exponential families (Appendix~\ref{app:spef}).
This is especially simple for deterministic mechanisms.
For randomized mechanisms, a natural sampler is the posterior distribution.
However, proving the properties on $\Pi_n$ (Appendix~\ref{app:sss_how_to_sample}) in all generality is more cumbersome.

By the end of Appendix~\ref{app:top_two_instances}, we will have shown that Properties~\ref{prop:leader_cdt_suff_explo} and~\ref{prop:leader_cdt_convergence} hold for the EB and TS leaders, and that Properties~\ref{prop:challenger_cdt_suff_explo} and~\ref{prop:challenger_cdt_convergence} hold for the TC, TCI and RS challengers, which leads to Theorem~\ref{thm:beta-opt}.

\paragraph{Proof of Theorem~\ref{thm:beta-opt}}
The threshold (\ref{eq:def_kinf_threshold_glr}) is asymptotically tight (Definition~\ref{def:asymptotically_tight_threshold}), e.g. $(\alpha, \delta_0, C) = (1/2,1,1)$, $f(\delta) = \ln \left( \frac{K-1}{\delta}\right) + 2$ and $\bar{T}(\delta) = 1$.
Lemma~\ref{lem:kinf_concentration} shows that it is also $\delta$-correct for bounded distributions.
Therefore, Theorem~\ref{thm:beta-opt} is obtained by applying Corollary~\ref{cor:asymptotic_optimality_top_two_algorithms}.

\paragraph{Proof of Property~\ref{prop:rate_convergence_mean_and_distribution}}
Lemma~\ref{lem:subG_cdf} gives the convergence rate of empirical cdfs $(F_{n,i})_{i \in [K]}$ towards the true cdfs $(F_{i})_{i \in [K]}$.
Deferred to Appendix~\ref{ssub:sub_gaussian_random_variables}, its proof is a direct consequence of concentration inequalities for sub-Gaussian random variables.
\begin{lemma} \label{lem:subG_cdf}
	There exists a sub-Gaussian random variable $W_2$ such that for all $(n,i) \in \N \times [K]$
	\begin{equation} \label{eq:subG_cdf}
		\left\|{F}_{n,i} - F_{i}\right\|_{\infty} \leq W_2 \sqrt{\frac{\log(e+N_{n,i})}{1+N_{n,i}}}  \quad \text{a.s.} \: .
	\end{equation}
	In particular, $\bE \left[e^{\lambda W_2}\right] < + \infty$ for all $\lambda > 0$.
\end{lemma}
In the following, we take $W_2$ as in Lemma~\ref{lem:subG_cdf}.
Let $\epsilon > 0$. Using Lemma~\ref{lem:subG_cdf}, there exists $N_{\epsilon} = Poly(\frac{1}{\epsilon}, W_2)$ such that for all $i \in [K]$ and all $N_{n,i} \ge N_{\epsilon}$,
\[
 \max_{i \in [K]} \left\|{F}_{n,i} - F_{i}\right\|_{\infty} \le \epsilon \: .
\]
As $\bE_{\bm F} [e^{\lambda W_2}] < + \infty$ for all $\lambda > 0$, we have $\bE_{\bm F}[N_{\epsilon}] < + \infty$.
Therefore, Property~\ref{prop:rate_convergence_mean_and_distribution} holds for bounded distributions.

Property~\ref{prop:joint_continuity_true_tc} holds for bounded distributions: $F \mapsto m(F)$ continuous (bounded), using proof of Lemma~\ref{lem:T_star_and_w_star_continuous} (consequence of Lemma~\ref{lem:unique_continuous_mu_star}) and by Lemmas~\ref{lem:properties_characteristic_times} and~\ref{lem:w_star_positive}.

\subsection{Deterministic mechanisms}
\label{app:ss_deterministic_mechanisms}

Conditioned on the history $\cF_n$, deterministic mechanisms for the leader and the challenger don't depend on a sampler $\Pi_n$.
The sole randomness in those mechanisms occurs in case of ties, which are broken uniformly at random.
In Appendix~\ref{app:sss_eb_leader}, we define the EB leader and shows that it satisfies Properties~\ref{prop:leader_cdt_suff_explo} and~\ref{prop:leader_cdt_convergence}.
In Appendix~\ref{app:sss_tc_challenger}, we define the TC challenger and proves that Properties~\ref{prop:challenger_cdt_suff_explo} and~\ref{prop:challenger_cdt_convergence} hold.
In Appendix~\ref{app:sss_tci_challenger}, we define the TCI challenger and proves that Properties~\ref{prop:challenger_cdt_suff_explo} and~\ref{prop:challenger_cdt_convergence} hold.

\paragraph{Rates for empirical transportation costs}
Analyzing deterministic mechanisms heavily relies on properties of the empirical transportation costs.
Given two arms having distinct mean, Lemma~\ref{lem:fast_rate_emp_tc} shows that the transportation cost is strictly positive and increases linearly.
\begin{lemma} \label{lem:fast_rate_emp_tc}
		Let $S_{n}^{L}$ and $\cI_n^\star$ as in (\ref{eq:def_sampled_enough_sets}).
		There exists $L_4$ with $\bE_{\bm F}[(L_4)^{\alpha}] < +\infty$ for all $\alpha > 0$ such that if $L \ge L_4$, for all $n$ such that $S_{n}^{L} \neq \emptyset$,
		\[
		\forall (i, j) \in \cI_n^\star \times \left(S_{n}^{L} \setminus  \cI_n^\star \right), \quad	W_{n}(i, j) \geq  L D_{\bm F} \: ,
		\]
		where $ D_{\bm F} > 0$ is a problem dependent constant.
\end{lemma}
\begin{proof}
		Let $S_{n}^{L}$ and $\cI_n^\star$ as in (\ref{eq:def_sampled_enough_sets}).
		Assume that $S_{n}^{L} \neq \emptyset$.
		If $S_{n}^{L} \setminus  \cI_n^\star $ is empty, then the statement is not informative.
		Assume $S_{n}^{L} \setminus  \cI_n^\star $ is not empty.
		Let $(i, j) \in \cI_n^\star \times \left(S_{n}^{L} \setminus  \cI_n^\star \right)$.

		By definition of $W_{n}$ in (\ref{eq:def_transportation_cost}) and using $\{i, j\} \subseteq S_{n}^{L}$, we obtain
	\begin{align*}
	W_{n}(i, j) &= \inf_{u \in [0,B]} \left\{ N_{n,i} \Kinf^{-}( F_{n,i}, u) + N_{n,j} \Kinf^{+}( F_{n,j}, u) \right\}	\\
		&\geq L \inf_{u \in [0,B]} \left\{ \Kinf^{-}( F_{n,i}, u) + \Kinf^{+}( F_{n,j}, u)  \right\}  \: .
	\end{align*}
	Using Lemma~\ref{lem:positive_strict_geometric_cst}, there exists $\alpha > 0$ such that
	\begin{equation*}
		D_{\bm F} = \min_{(i, j) : m(F_i) > m(F_j) } \inf_{\substack{G_i,G_j : \\ \forall k \in \{i,j\}, \|G_k - F_k\|_{\infty}\leq \alpha}} \inf_{u \in [0,B]} \left\{ \Kinf^{-}(G_i , u) + \Kinf^{+}(G_j , u) \right\} > 0 \: .
	\end{equation*}

	Using Lemma~\ref{lem:subG_cdf}, there exists $L_4 = Poly(W_2) $ such that for all $L\geq L_4$ and all $i \in S_{n}^{L}$,
	\[
		 \left\| F_{n,i} - F_i \right\|_{\infty} \leq \alpha \: .
	\]

	Further lower bounding by using that $\mu_i > \mu_j$, we obtain
	\begin{align*}
		W_{n}(i, j) &\geq  L \inf_{\substack{G_i,G_{j} : \\ \forall k \in \{i,j\}, \|G_k - F_k\|_{\infty}\leq \alpha}} \inf_{u \in [0,B]} \left\{ \Kinf^{-}( G_{i}, u) + \Kinf^{+}( G_{j}, u)  \right\} \geq L D_{\bm F} \: .
	\end{align*}
	As $\bE_{\bm F} [e^{\lambda W_2}] < + \infty$ for all $\lambda > 0$, we have $\bE_{\bm F}[(L_4)^{\alpha}] < + \infty$ for all $\alpha > 0$ since $(L_4)^{\alpha} = Poly(W_2) $.
	This concludes the proof.
\end{proof}

Lemma~\ref{lem:small_transportation_cost_undersampled_arms} gives an upper bound on the transportation costs between a sampled enough arm and an under-sampled one.
\begin{lemma} \label{lem:small_transportation_cost_undersampled_arms}
	Let $S_{n}^{L}$ as in (\ref{eq:def_sampled_enough_sets}). There exists $L_5$ with $\bE_{\bm F}[(L_5)^{\alpha}] < +\infty$ for all $\alpha > 0$ such that for all $L \geq L_5$ and all $n \in \N$,
	\[
	\forall  (i,j) \in  S_{n}^{L} \times \overline{S_{n}^{L}} , \quad 	W_{n}( i, j) \leq  L D_1  \: ,
	\]
where $D_1 > 0$ is a problem dependent constant.
\end{lemma}
\begin{proof}
For bounded distributions, $F \mapsto m(F)$ is continuous on $\cF$ for the weak convergence.
Since $\mu_i \in (0,B)$ for all $i \in [K]$ (Assumption~\ref{ass:standard_assumptions_BAI}), Lemma~\ref{lem:subG_cdf} yields that there exists $L_6 = Poly(W_1)$ such that for all $L \ge L_6$ and all $i \in S_{n}^{L}$, we have $\mu_{n,i} \in (0,B)$. In the following, we consider $L \ge L_6$.

Let $(i, j) \in S_{n}^{L} \times \overline{S_{n}^{L}}$. By definition and taking $u = \mu_{n,i} \in (0,B)$ yields
\begin{align*}
	W_{n}(i, j) &= \inf_{u \in [0,B]} \left\{ N_{n,i} \Kinf^{-}( F_{n,i}, u) + N_{n,j} \Kinf^{+}( F_{n,j}, u) \right\}	\\
	&\leq N_{n,j} \Kinf^{+}( F_{n,j}, \mu_{n,i}) \leq L \Kinf^{+}( F_{n,j}, \mu_{n,i}) \leq -L \ln \left( 1 - \frac{ \mu_{n,i}}{B}\right) \: ,
\end{align*}
where we used that $j \in \overline{S_{n}^{L}}$ and Lemma~\ref{lem:Kinf_finite_distribution_coarse_upper_bound}. By continuity of $F \mapsto m(F)$, Lemma~\ref{lem:subG_cdf} yields that there exists $L_7 = Poly(W_1)$ such that for all $L \ge L_5 \eqdef \max \{L_6, L_7\}$ and all $i \in S_{n}^{L}$
\[
 - \ln \left( 1 - \frac{ \mu_{n,i}}{B}\right) \leq -2  \ln \left( 1 - \frac{ \mu_{i}}{B}\right) \leq D_1 \: ,
\]
where $D_1 = -2\log\left(1 - \frac{\max_{k \in [K]} \mu_{k}}{B}\right)$.
As $\bE_{\bm F} [e^{\lambda W_2}] < + \infty$ for all $\lambda > 0$, we have $\bE_{\bm F}[(L_5)^{\alpha}] \leq \bE_{\bm F}[(L_6)^{\alpha}] + \bE_{\bm F}[(L_7)^{\alpha}] < + \infty$ since $(L_6)^{\alpha} = Poly(W_2)$ and $(L_7)^{\alpha} = Poly(W_2)$.
This concludes the proof.
\end{proof}

\subsubsection{EB leader}
\label{app:sss_eb_leader}

Conditioned on $\cF_n$, the Empirical Best (EB) leader is defined as an arm with highest empirical mean
\begin{equation} \label{eq:def_eb_based_leader}
	B_{n+1}^{\text{EB}} \in \argmax_{i \in [K]} \mu_{n,i} \quad \text{,} \quad \bP_{\mid n}[B_{n+1}^{\text{EB}} = i] = \frac{\indi{i \in \argmax_{i \in [K]} \mu_{n,i}}}{|\argmax_{i \in [K]} \mu_{n,i}|} \: .
\end{equation}
and $\widehat B_{n+1}^{\text{EB}} = B_{n+1}^{\text{EB}}$.

\paragraph{Property~\ref{prop:leader_cdt_suff_explo}}
Lemma~\ref{lem:EB_ensures_suff_explo} shows that Property~\ref{prop:leader_cdt_suff_explo} is satisfied by $B_{n+1}^{\text{EB}}$.
\begin{lemma} \label{lem:EB_ensures_suff_explo}
	Let $S_{n}^{L}$ and $\cI_n^\star$ as in (\ref{eq:def_sampled_enough_sets}).
	Let $L_4$ in Lemma~\ref{lem:fast_rate_emp_tc}.
	Then, for all $L \ge L_4$, for all $n$ such that $S_{n}^{L} \neq \emptyset$, $\widehat B_{n+1}^{\text{EB}} \in S_{n}^{L}$ implies $\widehat B_{n+1}^{\text{EB}} \in \cI_n^\star$.
\end{lemma}
\begin{proof}
	Let $S_{n}^{L}$ and $\cI_n^\star$ as in (\ref{eq:def_sampled_enough_sets}).
	Assume that $S_{n}^{L} \neq \emptyset$.
	If $S_{n}^{L} \setminus  \cI_n^\star $ is empty, then the result is true.
	Assume $S_{n}^{L} \setminus  \cI_n^\star $ is not empty.
	Let $L_4$ in Lemma~\ref{lem:fast_rate_emp_tc}. Then,
	\[
	\forall (i, j) \in \cI_n^\star \times \left(S_{n}^{L} \setminus  \cI_n^\star \right), \quad	W_{n}(i, j) \geq  L D_{\bm F} \: ,
	\]
Assume that $\widehat B_{n+1}^{\text{EB}} \in S_{n}^{L}$. Suppose towards contradiction that $\widehat B_{n+1}^{\text{EB}} \notin \cI_n^\star$.
Therefore, $W_{n}(i, \widehat B_{n+1}^{\text{EB}}) \geq L D_{\bm F} > 0$ for all $i \in \cI_n^\star$.
Since the choice of the leader is deterministic, we have $B_{n+1}^{\text{EB}} = \widehat B_{n+1}^{\text{EB}}$.
Since $B_{n+1}^{\text{EB}} \in \argmax_{i \in [K]} \mu_{n,i}$, we have $W_{n}(i, \widehat B_{n+1}^{\text{EB}}) = 0$.
This is a contradiction, hence $\widehat B_{n+1}^{\text{EB}} \in \cI_n^\star$.
\end{proof}

\paragraph{Property~\ref{prop:leader_cdt_convergence}}
Lemma~\ref{lem:EB_ensures_convergence} shows that Property~\ref{prop:leader_cdt_convergence} is satisfied by $B_{n+1}^{\text{EB}}$.
More precisely, we show that after enough time, the leader is the best arm almost surely.
\begin{lemma} \label{lem:EB_ensures_convergence}
		Assume Property~\ref{prop:suff_exploration} holds.
		There exists $N_6$ with $\bE_{\bm F}[N_6] < + \infty$ such that for all $n \ge N_6$,
		\[
		\bP_{\mid n}[ B_{n+1}^{\text{EB}} \neq i^\star (\bm F)]  = 0 \: .
		\]
\end{lemma}
\begin{proof}
	Let $i^\star = i^\star(\bm F)$.
	Let $N_1$ as in Property~\ref{prop:suff_exploration}, then $N_{n,i} \geq \sqrt{\frac{n}{K}}$ for all $n \ge N_1$.
	Since $i^\star$ is unique, we have $\Delta \eqdef \min_{j\neq i^\star}|\mu_{i^\star} - \mu_{j}| > 0$.
	For bounded distributions, $F \mapsto m(F)$ is continuous on $\cF$ for the weak convergence.
	Lemma~\ref{lem:subG_cdf} yields that there exists $N_7 = Poly(W_2) $ such that for all $n \ge N_6 \eqdef \max\{N_1,N_7\}$ and all $i \in [K]$, we have $|\mu_{n,i} - \mu_{i}| \leq \frac{\Delta}{4}$.
	Therefore, for all $n \ge N_6$, $\argmax_{i \in [K]} \mu_{n,i} = \argmax_{i \in [K]} \mu_{i} = i^\star$ and
	\begin{align*}
		\bP_{\mid n}[B_{n+1}^{\text{EB}} \neq i^\star] = 1 -  \bP_{\mid n}[B_{n+1}^{\text{EB}} = i^\star] = 1 - \frac{\indi{i^\star \in \argmax_{i \in [K]} \mu_{n,i}}}{|\argmax_{i \in [K]} \mu_{n,i}|} = 0  \: .
	\end{align*}
	As $\bE_{\bm F} [e^{\lambda W_2}] < + \infty$ for all $\lambda > 0$, we have $\bE_{\bm F}[N_7] < + \infty$.
	Therefore, $\bE_{\bm F}[N_6] \le \bE_{\bm F}[N_1] + \bE_{\bm F}[N_7] < + \infty$ yields the result.
\end{proof}

\subsubsection{TC challenger}
\label{app:sss_tc_challenger}

Conditioned on $\cF_n$ and given a leader $B_{n+1}$, the Transportation Cost (TC) challenger is defined as the arm with smallest transportation cost compared to the leader
\begin{equation} \label{eq:def_tc_based_challenger}
	C_{n+1}^{\text{TC}} \in \argmin_{j \neq   B_{n+1}} W_{n}(  B_{n+1},j)  \quad \text{,} \quad  \bP_{\mid n}[C_{n+1}^{\text{TC}} = j| B_{n+1} = i] = \frac{\indi{j \in \argmin_{k \neq i} W_{n}(i,k)}}{|\argmin_{k \neq i} W_{n}(i,k)|}  \: ,
\end{equation}
and $\widehat C_{n+1}^{\text{TC}} \in \argmin_{j \neq \widehat  B_{n+1}} W_{n}(\widehat  B_{n+1},j)$.

\paragraph{Property~\ref{prop:challenger_cdt_suff_explo}}
We prove Property~\ref{prop:challenger_cdt_suff_explo} for $C_{n+1}^{\text{TC}}$ in Lemma~\ref{lem:TC_ensures_suff_explo} by comparing the rates at which $W_n$ increases (Lemmas~\ref{lem:fast_rate_emp_tc} and~\ref{lem:small_transportation_cost_undersampled_arms}).
The effective challenger $\widehat C_{n+1}^{\text{TC}}$ is taken as an arm minimizing the transportation cost compared to the leader $\widehat B_{n+1}$.
Therefore, it is sufficient to show that the sampled enough arms have higher transportation costs than the mildly under-sampled ones.
This implies that $\widehat C_{n+1}^{\text{TC}}$ has to be mildly under-sampled or be an arm with highest mean among the sampled enough arms.
\begin{lemma} \label{lem:TC_ensures_suff_explo}
	Let $B_{n+1}$ be a leader satisfying Property~\ref{prop:leader_cdt_suff_explo}.
	Given $(B_{n+1}, \widehat B_{n+1})$, let $(C_{n+1}^{\text{TC}}, \widehat C_{n+1}^{\text{TC}})$ as in (\ref{eq:def_tc_based_challenger}).
	Let $U_n^L$ and $V_n^L$ as in (\ref{eq:def_undersampled_sets}) and $\mathcal J_n^\star = \argmax_{ i \in \overline{V_{n}^{L}}} \mu_{i}$.
	There exists $L_6$ with $\bE_{\bm F}[L_6] < + \infty$ such that if $L \ge L_6$, for all $n$ such that $U_n^L \neq \emptyset$, $\widehat B_{n+1} \notin V_{n}^{L}$ implies $\widehat C_{n+1}^{\text{TC}} \in V_{n}^{L} \cup \left( \mathcal J_n^\star \setminus \left\{\widehat B_{n+1} \right\} \right)$.
\end{lemma}
\begin{proof}
	Let $\mathcal J_n^\star = \argmax_{ i \in \overline{V_{n}^{L}}} \mu_{i}$.
	In the following, we consider $U_n^L \neq \emptyset$ (hence $V_n^L \neq \emptyset$) and $\widehat B_{n+1} \in V_n^L$.
	Let $B_{n+1}$ be a leader satisfying Property~\ref{prop:leader_cdt_suff_explo}, and $L_0$ defined therein.
	Then, for $L \geq L_{0}^{4/3}$, we have $\widehat B_{n+1} \in \mathcal J_n^\star$.
	If $\widehat C_{n+1}^{\text{TC}} \in \mathcal J_n^\star \setminus \left\{\widehat B_{n+1} \right\}$, we are done.
	Assume that $\widehat C_{n+1}^{\text{TC}} \notin \mathcal J_n^\star \setminus \left\{\widehat B_{n+1} \right\}$.

	Let $(L_4, D_{\bm F})$ and $(L_5, D_1)$ as in Lemmas~\ref{lem:fast_rate_emp_tc} and \ref{lem:small_transportation_cost_undersampled_arms}.
	Then, for all $L \geq \max\{L_{0}^{4/3}, L_4^{4/3}, L_5^2\}$,
		\begin{align*}
			&\widehat B_{n+1} \in \mathcal J_n^\star \: , \\
			\forall (i,j) \in \mathcal J_n^\star \times \left(\overline{V_n^L} \setminus \mathcal J_n^\star\right), \quad &W_{n}(i, j) \geq  L^{3/4} D_{\bm F}   \: , \\
			\forall (i,j) \in \overline{U_n^L} \times U_n^L, \quad 	&W_{n}(i,j) \leq  \sqrt{L} D_1 \: .
		\end{align*}
	There exists a deterministic $L_7$ such that for all $L \ge L_7$,
	\[
	L^{3/4} D_{\bm F} >  \sqrt{L} D_1 \: .
	\]

	Since $\mathcal J_n^\star \subseteq \overline{V_n^L} \subseteq \overline{U_n^L}$, for all $L \geq L_6 \eqdef \max\{L_{0}^{4/3}, L_4^{4/3}, L_5^2, L_7\}$ we have
	\[
	\forall  (i,k,j) \in \mathcal J_n^\star \times U_n^L \times \left(\overline{V_n^L}  \setminus \mathcal J_n^\star \right), \quad  W_{n}(i, j) > W_{n}(i, k) \: .
	\]
	As $\widehat B_{n+1} \in \mathcal J_n^\star$ and $\widehat C_{n+1}^{\text{TC}} \notin \mathcal J_n^\star \setminus \left\{\widehat B_{n+1} \right\}$, the definition $\widehat C_{n+1}^{\text{TC}}  \in \argmax_{j \neq \widehat B_{n+1}} W_{n}(\widehat B_{n+1}, j)$ yields that $\widehat C_{n+1}^{\text{TC}} \in V_{n}^{L}$.
	Otherwise the above strict inequality would wield a contradiction.
	Since
	\[
		\bE_{\bm F}[L_6] \le L_7 + \bE_{\bm F}[(L_{0})^{4/3}] + \bE_{\bm F}[(L_4)^{4/3}] +  \bE_{\bm F}[(L_5)^2] < + \infty	\: ,
	\]
	this concludes the proof.
\end{proof}

\paragraph{Property~\ref{prop:challenger_cdt_convergence}}
Lemma~\ref{lem:TC_ensures_convergence} shows that the Property~\ref{prop:challenger_cdt_convergence} is satisfied by $C_{n+1}^{\text{TC}}$.
More precisely, it shows that if the mean probability of sampling a sub-optimal arm overshoots its $\beta$-optimal allocation, then it won't be sampled almost surely if the leader is the best arm.
\begin{lemma} \label{lem:TC_ensures_convergence}
		Assume Property~\ref{prop:suff_exploration} holds.
		Let $\epsilon > 0$.
		Let $B_{n+1}$ be a leader satisfying Property~\ref{prop:leader_cdt_convergence} and $C_{n+1}^{\text{TC}}$ as in (\ref{eq:def_tc_based_challenger}).
		There exists $N_7$ with $\bE_{\bm F}[N_7] < + \infty$ such that for all $n \geq N_7$ and all $i \neq i^\star(\bm F)$,
		\begin{equation} \label{eq:overshooting_implies_not_sampled_anymore_tc}
			\frac{\Psi_{n,i}}{n} \geq w_{i}^{\beta} + \epsilon  \quad \implies \quad \bP_{\mid n}[C_{n+1}^{\text{TC}} = i \mid B_{n+1} = i^\star(\bm F)] = 0 \: .
		\end{equation}
\end{lemma}
\begin{proof}
Let $\epsilon > 0$ and $i^\star = i^\star(\bm F)$.
Let $N_1$ as in Property~\ref{prop:suff_exploration}, then $N_{n,i} \geq \sqrt{\frac{n}{K}}$ for all $n \ge N_1$.
Since $i^\star$ is unique, we have $\Delta \eqdef \min_{j\neq i^\star}|\mu_{i^\star} - \mu_{j}| > 0$.
For bounded distributions, $F \mapsto m(F)$ is continuous on $\cF$ for the weak convergence.
Lemma~\ref{lem:subG_cdf} yields that there exists $N_8 = Poly(W_2)$ such that for all $n \ge \max\{N_1, N_8\}$ and all $i \in [K]$, we have $|\mu_{n,i} - \mu_{i}| \leq \frac{\Delta}{4}$.
Therefore, for all $n \ge \max\{N_1, N_8\}$, $\argmax_{i \in [K]} \mu_{n,i} = \argmax_{i \in [K]} \mu_{i} = i^\star$.

Let $\xi >0$. Since Property~\ref{prop:suff_exploration} holds and $B_{n+1}$ satisfies Property~\ref{prop:leader_cdt_convergence}, we can use the results from Lemma~\ref{lem:convergence_towards_optimal_allocation_best_arm}. Let $N_4$ defined in Lemma~\ref{lem:convergence_towards_optimal_allocation_best_arm}, we have $\left| \frac{N_{n,i^\star}}{n} - \beta \right| \leq \xi$ for all $n \ge \max \{ N_1, N_4\}$.

Using the definition of $C_{n+1}^{\text{TC}}$ in (\ref{eq:def_tc_based_challenger}), we have
\begin{align*}
	\bP_{\mid n}[C_{n+1}^{\text{TC}} = i \mid B_{n+1} = i^\star] = 0 \quad &\iff \quad  i \notin \argmin_{k \neq i^\star} W_{n}(i^\star,k) \\
	& \iff \frac{1}{n}\left(W_{n}(i^\star,i) - \min_{j\neq i^\star}W_{n}(i^\star,j) \right) > 0 \: .
\end{align*}

Let $i \neq i^\star$ such that $\frac{\Psi_{n,i}}{n} \geq w_{i}^{\beta} + \epsilon$.
Using Lemma~\ref{lem:subG_alloc}, there exists $N_9 = Poly(W_1) $, such that for all $n \ge \max\{N_1, N_9\}$, we have $\frac{N_{n,i}}{n} \geq w_{i}^{\beta} + \frac{\epsilon}{2}$.
Therefore, for all $n \ge \max \{N_1, N_4, N_8, N_9\}$,
\begin{align*}
	 &\frac{1}{n}\left(W_{n}(i^\star,i) - \min_{j\neq i^\star}W_{n}(i^\star,j) \right)\\
	 &\geq \inf_{u \in  [0,B]} \left\{  \frac{N_{n,i^\star}}{n} \Kinf^{-}( F_{n,i^\star}, u)  + \left( w_{i}^{\beta} + \frac{\epsilon}{2} \right) \Kinf^{+}( F_{n,i}, u)  \right\}  \\
	 &\quad - \min_{j\neq i^\star}\inf_{u \in  [0,B]} \left\{  \frac{N_{n,i^\star}}{n} \Kinf^{-}( F_{n,i^\star}, u)  +  \frac{N_{n,j}}{n} \Kinf^{+}( F_{n,j}, u)  \right\} \\
	 &\geq \inf_{u \in  [0,B]} \left\{  \frac{N_{n,i^\star}}{n} \Kinf^{-}( F_{n,i^\star}, u)  + \left( w_{i}^{\beta} + \frac{\epsilon}{2} \right) \Kinf^{+}( F_{n,i}, u)  \right\}  \\
	 &\quad - \sup_{w \in \simplex: w_{i^\star} = \frac{N_{n,i^\star}}{n}} \min_{j\neq i^\star}\inf_{u \in  [0,B]} \left\{  w_{i^\star} \Kinf^{-}( F_{n,i^\star}, u)  +  w_j \Kinf^{+}( F_{n,j}, u)  \right\} \\
	 &\geq \inf_{\tilde \beta : \left|\tilde \beta - \beta \right| \leq \xi} G_{i}(\bm F_{n}, \tilde \beta)
\end{align*}
where
\begin{align*}
	G_{i}(\bm F, \tilde \beta) & = \inf_{u \in [0,B]} \left\{  \tilde \beta \Kinf^{-}( F_{i^\star}, u)  + \left( w_{i}^{\beta} + \frac{\epsilon}{2} \right) \Kinf^{+}( F_{i}, u)  \right\}  \\
	&\quad - \sup_{w \in \simplex: w_{i^\star} = \tilde \beta} \min_{j\neq i^\star} \inf_{u \in  \cI} \left\{  w_{i^\star}  \Kinf^{-}( F_{i^\star}, u)  + w_{j} \Kinf^{+}( F_{j}, u)  \right\} \: ,
\end{align*}
where we lower bounded by considering the best possible allocation such that $w_{i^\star} = \frac{N_{n,i^\star}}{n}$.

Using Lemma~\ref{lem:continuity_results_for_analysis}, the functions $(\bm F, \tilde \beta) \mapsto G_{i}(\bm F, \tilde \beta)$ and $\bm F \mapsto \inf_{\tilde \beta : \left|\tilde \beta - \beta \right| \leq \xi} G_{i}(\bm F, \tilde \beta)$ are continuous.
Therefore, there exists $N_{10} = Poly(W_2)$ and $\xi_0$ such that for $n \ge N_{7} \eqdef \{N_1, N_4, N_8, N_9, N_{10}\}$ and all $\xi \leq \xi_0$,
\begin{align*}
	\inf_{\tilde \beta : \left|\tilde \beta - \beta \right| \leq \xi} G_{i}(\bm F_{n}, \tilde \beta) \geq \frac{1}{2}  \inf_{\tilde \beta : \left|\tilde \beta - \beta \right| \leq \xi} G_{i}(\bm F, \tilde \beta) \geq  \frac{1}{4} G_{i}(\bm F, \beta) \: .
\end{align*}

At the $\beta$-equilibrium all transportation costs are equal (Lemma~\ref{lem:properties_characteristic_times}).
Therefore, by definition of $w^{\beta}$,
\begin{align*}
	&\sup_{w \in \simplex : w_{i^\star} = \beta} \min_{j \neq i^\star}  \inf_{u  \in [0,B]}\left\{w_{i^\star}\Kinf^-( F_{i^\star}, u) + w_j \Kinf^+(F_j, u) \right\} \\
	&\qquad= \min_{j \neq i^\star}  \inf_{u  \in [0,B]}\left\{\beta \Kinf^-(F_{i^\star}, u) + w^\beta_j \Kinf^+(F_j, u) \right\} \\
	&\qquad= \inf_{u  \in [0,B]}\left\{\beta \Kinf^-( F_{i^\star}, u) + w^\beta_i \Kinf^+(F_i, u) \right\} \\
	&\qquad< \inf_{u  \in [0,B]}\left\{\beta \Kinf^-( F_{i^\star}, u) + \left( w^\beta_i + \frac{\epsilon}{2} \right) \Kinf^+(F_i, u) \right\}
\end{align*}
where the strict inequality is obtained because the transportation costs are increasing in their allocation arguments (Lemma~\ref{lem:inf_Kinf_increasing_in_w}).
Therefore, we have $G_{i}(\bm F, \beta) > 0$. This yields that $W_{n}(i^\star,i) > \min_{j\neq i^\star} W_{n}(i^\star,j)$.
As $\bE_{\bm F} [e^{\lambda W_1}] < + \infty$ and $\bE_{\bm F} [e^{\lambda W_2}] < + \infty$ for all $\lambda > 0$, we have $\bE_{\bm F}[N_i] < + \infty$ for $i \in \{8,9,10\}$.
Since
\[
	\bE_{\bm F}[N_7] \le \sum_{i \in \{1,4,8,9,10\}}\bE_{\bm F}[N_i] < + \infty	\: ,
\]
this concludes the proof.
\end{proof}

\subsubsection{TCI challenger}
\label{app:sss_tci_challenger}

Conditioned on $\cF_n$ and given a leader $B_{n+1}$, the Transportation Cost Improved (TCI) challenger is defined as the arm with smallest penalized transportation cost compared to the leader
\begin{equation} \label{eq:def_tci_based_challenger}
	C_{n+1}^{\text{TCI}} \in \argmin_{j \neq   B_{n+1}} \left\{ W_{n}(B_{n+1}, j) + \ln(N_{n,j}) \right\} \quad \text{,} \quad  \widehat C_{n+1}^{\text{TCI}} \in \argmin_{j \neq \widehat  B_{n+1}} \left\{ W_{n}(\widehat  B_{n+1},j)+ \ln(N_{n,j}) \right\}  \: ,
\end{equation}
and
\[
\bP_{\mid n}[C_{n+1}^{\text{TCI}} = j| B_{n+1} = i] = \frac{\indi{j \in \argmin_{k \neq i} \left\{ W_{n}(i,k) + \ln(N_{n,k}) \right\}  }}{|\argmin_{k \neq i} \left\{ W_{n}(i,k) + \ln(N_{n,k}) \right\}|}
\]
The TCI challenger is inspired by IMED \cite{Honda15IMED}.
As we will see in Appendix~\ref{app:ss_supplementary_experiments}, it is more stable than the TC challenger.
In Appendix~\ref{app:ss_beyond_all_distinct_means}, we explain intuitively why.
The analysis of the TCI challenger is very close to the one of the TC challenger.

\paragraph{Property~\ref{prop:challenger_cdt_suff_explo}}
With similar arguments as in Lemma~\ref{lem:TC_ensures_suff_explo}, Lemma~\ref{lem:TCI_ensures_suff_explo} shows that the Property~\ref{prop:challenger_cdt_suff_explo} is satisfied by $C_{n+1}^{\text{TCI}}$.
\begin{lemma} \label{lem:TCI_ensures_suff_explo}
	Let $B_{n+1}$ be a leader satisfying Property~\ref{prop:leader_cdt_suff_explo}.
	Let $(C_{n+1}^{\text{TCI}}, \widehat C_{n+1}^{\text{TCI}})$ as in (\ref{eq:def_tci_based_challenger}).
	Let $U_n^L$ and $V_n^L$ as in (\ref{eq:def_undersampled_sets}) and $\mathcal J_n^\star = \argmax_{ i \in \overline{V_{n}^{L}}} \mu_{i}$.
	There exists $\tilde L_6$ with $\bE_{\bm F}[\tilde L_6] < + \infty$ such that if $L \ge \tilde L_6$, for all $n$ such that $U_n^L \neq \emptyset$, $\widehat B_{n+1} \notin V_{n}^{L}$ implies $\widehat C_{n+1}^{\text{TCI}} \in V_{n}^{L} \cup \left( \mathcal J_n^\star \setminus \left\{\widehat B_{n+1} \right\} \right)$.
\end{lemma}
\begin{proof}
	In the following, we consider $U_n^L \neq \emptyset$ (hence $V_n^L \neq \emptyset$) and $\widehat B_{n+1} \in V_n^L$.
	Let $L_0$ be defined as in Property~\ref{prop:leader_cdt_suff_explo}.
	Then, for $L \geq L_{0}^{4/3}$, we have $\widehat B_{n+1} \in \mathcal J_n^\star$.
	If $\widehat C_{n+1}^{\text{TCI}} \in \mathcal J_n^\star \setminus \left\{\widehat B_{n+1} \right\}$, we are done.
	Assume that $\widehat C_{n+1}^{\text{TCI}} \notin \mathcal J_n^\star \setminus \left\{\widehat B_{n+1} \right\}$.

	Let $(L_4, D_{\bm F})$ and $(L_5, D_1)$ as in Lemmas~\ref{lem:fast_rate_emp_tc} and \ref{lem:small_transportation_cost_undersampled_arms}.
		Then, for all $L \geq \max\{L_{0}^{4/3}, L_4^{4/3}, L_5^2\}$,
			\begin{align*}
				&\widehat B_{n+1} \in \mathcal J_n^\star \: , \\
				\forall (i,j) \in \mathcal J_n^\star \times \left(\overline{V_n^L} \setminus \mathcal J_n^\star\right), \quad &W_{n}(i, j) + \ln(N_{n,j}) \geq  L^{3/4} D_{\bm F}  + \frac{3}{4} \ln L  \: , \\
				\forall (i,j) \in \overline{U_n^L} \times U_n^L, \quad 	&W_{n}(i,j)  + \ln(N_{n,j})  \leq  \sqrt{L} D_1 + \frac{1}{2}\ln L  \: .
			\end{align*}
		There exists a deterministic $\tilde L_7$ such that for all $L \ge \tilde L_7$,
		\[
		L^{3/4} D_{\bm F} + \frac{3}{4}  >  \sqrt{L} D_1 + \frac{1}{2}\ln L \: .
		\]
		Since $\mathcal J_n^\star \subseteq \overline{V_n^L} \subseteq \overline{U_n^L}$, for all $L \geq \tilde L_6 \eqdef \max\{L_{0}^{4/3}, L_4^{4/3}, L_5^2, \tilde L_7\}$ we have
	\[
	\forall  (i,k,j) \in \mathcal J_n^\star \times U_n^L \times \left(\overline{V_n^L}  \setminus \mathcal J_n^\star \right), \quad  W_{n}(i, j)+ \ln(N_{n,j}) > W_{n}(i, k)+ \ln(N_{n,k}) \: .
	\]
	As $\widehat B_{n+1} \in \mathcal J_n^\star$ and $\widehat C_{n+1}^{\text{TCI}} \notin \mathcal J_n^\star \setminus \left\{\widehat B_{n+1} \right\}$, the definition $\widehat C_{n+1}^{\text{TCI}}  \in \argmax_{j \neq \widehat B_{n+1}} \left\{ W_{n}(\widehat B_{n+1}, j) + \ln(N_{n,j}) \right\}$ yields that $\widehat C_{n+1}^{\text{TCI}} \in V_{n}^{L}$.
	Otherwise the above strict inequality would wield a contradiction.
	Since
	\[
		\bE_{\bm F}[\tilde L_6] \le \tilde L_7 + \bE_{\bm F}[L_{0}^{4/3}] + \bE_{\bm F}[L_4^{4/3}] +  \bE_{\bm F}[L_5^2] < + \infty	\: ,
	\]
	this concludes the proof.
\end{proof}

\paragraph{Property~\ref{prop:challenger_cdt_convergence}}
With similar arguments as in Lemma~\ref{lem:TC_ensures_convergence}, Lemma~\ref{lem:TCI_ensures_convergence} shows that the Property~\ref{prop:challenger_cdt_convergence} is satisfied by $C_{n+1}^{\text{TCI}}$.
\begin{lemma} \label{lem:TCI_ensures_convergence}
		Assume Property~\ref{prop:suff_exploration} holds.
		Let $\epsilon > 0$.
		Let $B_{n+1}$ be a leader satisfying Property~\ref{prop:leader_cdt_convergence} and $C_{n+1}^{\text{TCI}}$ as in (\ref{eq:def_tci_based_challenger}).
		There exists $\tilde N_7$ with $\bE_{\bm F}[\tilde N_7] < + \infty$ such that for all $n \geq \tilde N_7$ and all $i \neq i^\star(\bm F)$,
		\begin{equation} \label{eq:overshooting_implies_not_sampled_anymore_tci}
			\frac{\Psi_{n,i}}{n} \geq w_{i}^{\beta} + \epsilon  \quad \implies \quad \bP_{\mid n}[C_{n+1}^{\text{TCI}} = i \mid B_{n+1} = i^\star(\bm F)] = 0 \: .
		\end{equation}
\end{lemma}
\begin{proof}
	Let $\epsilon > 0$ and $i^\star = i^\star(\bm F)$.
	Using the definition of $C_{n+1}^{\text{TCI}}$ in (\ref{eq:def_tci_based_challenger}), we have
	\begin{align*}
		&\bP_{\mid n}[C_{n+1}^{\text{TCI}} = i \mid B_{n+1} = i^\star] = 0 \\
		\iff & \:  i \notin \argmin_{k \neq i^\star} \left\{ W_{n}(i^\star,k) + \ln (N_{n,k})\right\} \\
		\iff & \frac{1}{n}\left(W_{n}(i^\star,i) + \ln (N_{n,i}) - \min_{j\neq i^\star} \left\{ W_{n}(i^\star,j) + \ln (N_{n,j})\right\} \right) > 0 \\
		\impliedby & \frac{1}{n}\left(W_{n}(i^\star,i) - \min_{j\neq i^\star} W_{n}(i^\star,j) \right) > \frac{\ln(nK)}{2n} \: ,
	\end{align*}
	where we used that $N_{n,i} \ge \sqrt{\frac{n}{K}}$ and $N_{n,j} \leq n$.

	Let $N_7$ as in Lemma~\ref{lem:TCI_ensures_convergence}. In the proof of Lemma~\ref{lem:TCI_ensures_convergence}, we showed that there exists $C_{\bm F} > 0$ such that for all $n \geq N_7$ and all $i \neq i^\star$,
	\[
	\frac{\Psi_{n,i}}{n} \geq w_{i}^{\beta} + \epsilon  \quad \implies \quad \frac{1}{n}\left(W_{n}(i^\star,i) - \min_{j\neq i^\star} W_{n}(i^\star,j) \right) \ge C_{\bm F} \: .
	\]
	Since $\frac{\ln(nK)}{2n} \to_{\infty} 0$, there exists a deterministic $N_8$ such that for all $n \ge \tilde N_8$,
	\[
	\frac{\ln(nK)}{2n} < C_{\bm F} \: .
	\]
	Therefore, for all $n \geq \tilde N_7 \eqdef \max \{ N_8, N_7\}$ and all $i \neq i^\star$,
	\[
	\frac{\Psi_{n,i}}{n} \geq w_{i}^{\beta} + \epsilon  \quad \implies \quad \bP_{\mid n}[C_{n+1}^{\text{TCI}} = i \mid B_{n+1} = i^\star] = 0 \: .
	\]
	Since $\bE_{\bm F}[\tilde N_7] = N_8 + \bE_{\bm F}[ N_7] < + \infty$, this concludes the proof.
\end{proof}

\subsection{Randomized mechanisms}
\label{app:ss_randomized_mechanisms}

Conditioned on the history $\cF_n$, randomized mechanisms for the leader and the challenger depend on a sampler $\Pi_n$.
In addition, ties will be broken uniformly at random. In Appendix~\ref{app:sss_how_to_sample}, we introduce the general properties that a \textit{good} sampler should verify. Depending on whether we are using the sampler for the leader or for the challenger different properties are necessary.
In Appendix~\ref{app:sss_ts_leader}, we define the TS leader and shows that it satisfies Properties~\ref{prop:leader_cdt_suff_explo} and~\ref{prop:leader_cdt_convergence}.
In Appendix~\ref{app:sss_rs_challenger}, we define the RS challenger and proves that Properties~\ref{prop:challenger_cdt_suff_explo} and~\ref{prop:challenger_cdt_convergence} hold.

\subsubsection{How to sample}
\label{app:sss_how_to_sample}

As both TS leader and RS challenger rely on a sampler $\Pi_n$, it is crucial for this sampler to be tailored to the considered set of distributions $\cF^{K}$. For bounded distributions, the sampler $\Pi_n$ under scrutiny is the Dirichlet sampler introduced in Section~\ref{sec:dirichlet_sampler}, which produces for each arm a random re-weighting of the current history of rewards, augmented with $\{0,B\}$. Yet, aiming for a unified analysis of Top Two algorithms relying on a sampler, we put forward some general properties that the sampler should satisfy. Those properties are only expressed in terms of the Boundary Crossing Probability (BCP) associated to each arm, i.e. the probabilities $\bP_{n}[\theta_i \geq u]$ and $\bP_{n}[\theta_i \leq u]$ where $u$ is a fixed threshold.
For all measurable sets $A_{\theta}$, we denote by $\bP_{n}[A_{\theta}] \eqdef \bP_{\theta \sim \Pi_n }[A_{\theta} \mid \cF_n]$.


\paragraph{From $\bm{a_{n+1,i}}$ to BCP}
Let $a_{n,i} \eqdef \bP_{n-1}[i \in \argmax_{j \in [K]} \theta_j]$.
From the definition of the TS leader in (\ref{eq:def_ts_based_leader}) and the RS challenger in (\ref{eq:def_rs_based_challenger}), it becomes apparent that we should control the quantity $a_{n,i}$. We will need both upper and lower bounds on $a_{n+1,i}$.
To derive those, we can write
\begin{align*}
	\forall j \neq i , \quad a_{n+1,i} &\leq \bP_n[\theta_i \geq \theta_j] \: , \\
	a_{n+1,i} &\leq 1 - \max_{j\neq i}\bP_n[\theta_j \geq \theta_i] \: , \\
		\forall u \in (0,B), \quad a_{n+1,i} &\geq
		\bP_{n} \left[ \theta_{i} \geq u \right] \prod_{j \neq i } \bP_{n} \left[ \theta_{j} \leq u \right]  \: .
\end{align*}
The lower bound is already expressed using BCPs, however the upper bound requires to control $\bP_{n}[\theta_i \geq \theta_j]$. In Lemma~\ref{lem:from_bcp_one_to_bcp_two} stated in Appendix~\ref{app:boundary_crossing_probability_bounds}, we provide upper bounds on this probability featuring only BCPs (for some well-chosen threshold).

As we will see, a sampler $\Pi_n$ is tailored to the considered set of distributions when the upper and lower bounds on the BCP involve $\Kinf^{\pm}$.
Those bounds will be referred to as \textit{tight}, the ones without $\Kinf^{\pm}$ will be referred as \textit{coarse}.
To show that the TS leader satisfies Properties~\ref{prop:leader_cdt_suff_explo} and~\ref{prop:leader_cdt_convergence}, we need a tight upper bound on the BCP. Proving Property~\ref{prop:challenger_cdt_suff_explo} for the RS challenger requires a tight upper bound and a coarse lower bound on the BCP.
However, the proof of Property~\ref{prop:challenger_cdt_convergence} for the RS challenger relies on a tight upper \emph{and} lower bound on the BCP.

In Appendix~\ref{app:boundary_crossing_probability_bounds}, we prove the corresponding bounds on the BCP for the Dirichlet sampler. These bounds use some ingredients from BPCs bounds obtained for variants of Non-Parametric Thompson Sampling \cite{RiouHonda20,baudry21a} in the regret minimization literature. Our new Lemma~\ref{lem:from_bcp_one_to_bcp_two} is instrumental to bring those to the best arm identification literature.

\paragraph{Coarse lower bound on $\bm{a_{n+1,i}}$}
Recall that $ F_{n,i}$ is the empirical cdf of arms $i$.
To ensure that the sampling stops, the sampler $\Pi_n$ should rely on modified cdfs instead of simply using $\bm F_n$.
Those probability measures are denoted by $\tilde F_{n,i}$ and their means by $\tilde \mu_{n,i}$.
For single-parameter exponential family, this modification corresponds to the posterior update based on the prior.
For bounded distributions (and Bernoulli), this modification amounts to adding $\{0,B\}$ to the support.
This ensures that $\tilde \mu_{n,i} \in (0,B)$, hence it is not a Dirac in $0$ or $B$.
Alternatively, we can view this step as mixing $F_{n,i}$ with the distribution $G = \frac{1}{2}\left( \delta_{0} + \delta_{B}\right)$ which is in the interior of the domain, i.e.
\begin{equation} \label{eq:def_transformed_cdf_for_bounded_distributions}
		\tilde F_{n,i} = \left(1 - \frac{2}{n+2} \right) F_{n,i}  + \frac{1}{n+2} \left( \delta_{0} + \delta_{B}\right) \: .
\end{equation}
The necessity of adding $B$ to the support was already known \citep{Baudry21DS} to obtain a lower bound on $\bP_{n}[\theta_i \geq u]$.
Since we also need to control $\bP_{n}[\theta_i \leq u]$, we should add $0$ in the support (by symmetry).

In all generality, the considered modified cdfs should be close to the empirical cdf, i.e.
\begin{equation} \label{eq:def_closeness_modified_cdfs_true_cdfs}
	\forall n \in \N , \quad \max_{i \in [K]} \left\| \tilde F_{n,i} -  F_{n,i} \right\|_{\infty} \le d_0(n) \; ,
\end{equation}
where the function $d_0: \N^* \to \N^*$ satisfies $d_0(n) =_{+\infty} o(n^{-\alpha})$ with $\alpha > 0$.
For bounded distributions (and Bernoulli), $ \bm{\tilde{F}}_{n}$ defined in (\ref{eq:def_transformed_cdf_for_bounded_distributions}) verifies Property~(\ref{eq:def_closeness_modified_cdfs_true_cdfs}) for $d_{0}(n) = \frac{3}{n+2}$.

Property~\ref{prop:one_arm_bcp_coarse_lower_bound} states that the sampler $\Pi_n$ based on $ \bm{\tilde{F}}_{n}$ yields an exponential lower bound on the BCP.
It is coarse as it doesn't involve $\Kinf^{\pm}$.
\begin{property} \label{prop:one_arm_bcp_coarse_lower_bound}
	There exists functions $\kappa^{+}, \kappa^{-}: (0,B) \to \R_{+}^{*}$ such that for all $u \in (0,B)$, all $n \ge 1$ and all $i \in [K]$
\begin{align*}
	\bP_n [\theta_i \ge u ] \ge e^{-c_0(N_{n,i}) \kappa^+(u) } \qquad \text{and} \qquad \bP_n [\theta_i \le u] \ge e^{-c_0(N_{n,i}) \kappa^-(u) }  \: .
\end{align*}
The function $c_0: \N^* \to \N^*$ is increasing with $c_0(x) \sim_{+\infty} x$. Moreover, $\kappa^-(B) = \kappa^+(0) = 0$ and $\lim_{u \to B} \kappa^+(u) = \lim_{u \to 0} \kappa^-(u) = + \infty$.
\end{property}

Property~\ref{prop:one_arm_bcp_coarse_lower_bound} plays a role in the proof of sufficient exploration for the RS challenger. For bounded distributions (and Bernoulli), Lemma~\ref{lem:coarse_lower_bound_BCP} in Appendix~\ref{app:boundary_crossing_probability_bounds} shows that Property~\ref{prop:one_arm_bcp_coarse_lower_bound} holds with
\begin{align*}
	\kappa^{-}(u) = - \ln \left(\frac{u}{B}\right) \quad \text{,} \quad
	\kappa^{+}(u) = - \ln \left( 1 - \frac{u}{B}\right)  \quad \text{,} \quad
	c_{0}(n) = n+1 \: .
\end{align*}

\paragraph{Tight upper bound on $\bm{a_{n+1,i}}$}
Property~\ref{prop:one_arm_bcp_upper_bound} states that the sampler $\Pi_n$ based on $ \bm{\tilde{F}}_{n}$ yields an exponential upper bound on the BCP.
Importantly, this upper bound is tailored to the considered family of distributions $\cF$ as it involves the $\Kinf^{\pm}$ for the set of distributions $\cF$.
\begin{property} \label{prop:one_arm_bcp_upper_bound}
	For all $u \in (0,B)$, all $n \ge 1$ and all $i \in [K]$,
\begin{align*}
	\bP_n [\theta_i \ge u] \le e^{-c_1(N_{n,i}) \Kinf^+(\tilde F_{n,i},u)} \qquad \text{and} \qquad \bP_n [\theta_i \le u] \le e^{-c_1(N_{n,i}) \Kinf^-(\tilde F_{n,i},u)}  \: ,
\end{align*}
where $c_1: \N^* \to \N^*$ is an increasing function such that $c_1(x) \sim_{+\infty} x$.
\end{property}

Property~\ref{prop:one_arm_bcp_upper_bound} plays an important role for the TS leader and the RS challenger, both in the proof of sufficient exploration and convergence towards the $\beta$-optimal allocation.
For bounded distributions (and Bernoulli), Theorem~\ref{thm:upper_bound_one_arm_bcp_bounded} in Appendix~\ref{app:boundary_crossing_probability_bounds} shows that Property~\ref{prop:one_arm_bcp_upper_bound} holds with $c_{1}(n) = n+2$.

Lemma~\ref{lem:two_arm_BCP_tight_upper_bound} is a direct corollary of Property~\ref{prop:one_arm_bcp_upper_bound} by using Lemma~\ref{lem:from_bcp_one_to_bcp_two}.
For bounded distributions, it is exactly Corollary~\ref{cor:upper_bound_two_arms_bcp_bounded}.
\begin{lemma} \label{lem:two_arm_BCP_tight_upper_bound}
	Let $\Pi_{n}$ satisfying Property~\ref{prop:one_arm_bcp_upper_bound}.
	Then, for all $n \ge 1$ and all $(i,j) \in [K]^2$
	\begin{align*}
		\bP_n [\theta_j \geq \theta_i]
		&\leq f\left( \inf_{u \in [0,B]} \left\{ c_1(N_{n,i}) \Kinf^-(\tilde F_{n,i},u)  + c_1(N_{n,j}) \Kinf^+(\tilde F_{n,j},u) \right\} \right) \: ,
	\end{align*}
	where $f : x \mapsto (1+x)e^{-x}$ is decreasing on $\R_{+}$ with values in $(0,1]$.
\end{lemma}
\begin{proof}
	Using Property~\ref{prop:one_arm_bcp_upper_bound} and Lemma~\ref{lem:from_bcp_one_to_bcp_two}, we obtain for all $n \ge 1$ and all $(i,j) \in [K]^2$
	\begin{align*}
		\bP_n [\theta_j \geq \theta_i] &\leq f\left( c_1(N_{n,i}) \Kinf^-(\tilde F_{n,i},u_{i,j})  + c_1(N_{n,j}) \Kinf^+(\tilde F_{n,j},u_{i,j}) \right)\\
		&\leq f\left( \inf_{u \in [0,B]} \left\{ c_1(N_{n,i}) \Kinf^-(\tilde F_{n,i},u)  + c_1(N_{n,j}) \Kinf^+(\tilde F_{n,j},u) \right\} \right) \: ,
	\end{align*}
	where $u_{i,j} = \argmax_{c \in [0,B]}\bP_{n}[\theta_j \geq c]\bP_{n}[\theta_i \leq c]$. When $\tilde \mu_{n,i} \leq \tilde \mu_{n,j}$, this result is non informative as $f(0) = 1$.
\end{proof}

While the proof of Property~\ref{prop:one_arm_bcp_coarse_lower_bound} heavily relies on the transformed cdfs, Property~\ref{prop:one_arm_bcp_upper_bound} also holds for the empirical cdfs $\bm F_{n}$.
There is no need to add $\{0,B\}$ in the support for this property.
We aim at presenting a unified sampler $\Pi_n$, which could be used both for the TS leader and the RS challenger.
Therefore, we present all the results with the modified cdfs $\bm{\tilde{F}}_{n}$ instead of differentiating between a sampler $\Pi_n$ for the TS leader and a sampler $\tilde \Pi_{n}$ for the RS challenger.

\paragraph{Tight lower bound on $\bm{a_{n+1,i}}$}
Property~\ref{prop:tight_time_uniform_lower_bound_2arm_BCP} states that the sampler $\Pi_n$ based on $ \bm{\tilde{F}}_{n}$ yields an exponential lower bound on $\bP_{n}[\theta_i \geq \theta_{i^\star}]$.
Importantly, this lower bound is tailored to the considered family of distributions $\cF$ as it involves the $\Kinf^{\pm}$ for the set of distributions $\cF$.
\begin{property} \label{prop:tight_time_uniform_lower_bound_2arm_BCP}
	Let $\varepsilon>0$ and $i^\star = i^\star(\bm F)$. There exists $N_8$, with $\bE_{\bm F}[(N_8)^{\alpha}] < + \infty$ for all $\alpha > 0$, such that for all $n$ with $\min_{i\in [K]} N_{n,i} \ge N_8$ and all $i \neq i^\star$,
	\[
	\bP_{n}\left[\theta_i \geq \theta_{i^\star} \right] \ge \frac{e^{-\varepsilon (N_{n,i^\star} + N_{n,i})}}{h_{\epsilon}(N_{n,i^\star}, N_{n,i})} \exp\left(- \inf_{x\in [0,B]}\left\{N_{n,i^\star}\Kinf^-(F_{i^\star}, x) + N_{n,i} \Kinf^+(F_{i}, x) \right\} \right) \: ,
	\]
	where $h_{\epsilon} : (\N^\star)^2 \to (0, +\infty)$ is an increasing function of both arguments, such that $h_{\epsilon}(n, m) =_{+\infty} o\left( e^{(n+m)^{\alpha}} \right)$ where $\alpha < 1$.
\end{property}
For bounded distributions (and Bernoulli), Property~\ref{prop:tight_time_uniform_lower_bound_2arm_BCP} is a direct corollary of Theorem~\ref{thm:two_arms_BCP_tight_lower_bound} given in Appendix~\ref{app:boundary_crossing_probability_bounds}.
Let $\epsilon > 0$ and $\eta > 0$ as in Theorem~\ref{thm:two_arms_BCP_tight_lower_bound}.
Using Lemma~\ref{lem:subG_cdf}, there exists $N_8 = Poly(W_2)$ such that for all $n$ such that $\min_{i\in [K]} N_{n,i} \ge N_8$, we have $\max_{i \in [K]} \|F_{n,i} - F_{i}\|_{\infty} \le \eta$.
As $\bE_{\bm F} [e^{\lambda W_2}] < + \infty$ for all $\lambda > 0$, we have $\bE_{\bm F}[(N_8)^{\alpha}] < + \infty$ for all $\alpha > 0$ since $(N_8)^{\alpha} = Poly(W_2)$.
Therefore, Property~\ref{prop:tight_time_uniform_lower_bound_2arm_BCP} holds for bounded distributions with
\[
h_{\epsilon}(n, m) = (n m)^{\frac{M_{\epsilon}+1}{2}}C_{\epsilon} \quad \text{and} \quad C_{\epsilon} =  \frac{4 (8 \pi)^{M_{\epsilon}-1}}{M_{\epsilon}^{M_{\epsilon}}} \: ,
\]
where $h_{\epsilon}$ satisfies the conditions from Property~\ref{prop:tight_time_uniform_lower_bound_2arm_BCP}.

Using Lemma~\ref{lem:from_bcp_one_to_bcp_two}, Theorem~\ref{thm:two_arms_BCP_tight_lower_bound} was shown thanks to tight lower bound on the BCP for bounded distributions (Lemma~\ref{lem:bcp_lower_bound_discretized_one_arm}).
Given a family of distribution for which a tight lower bound on the BCP exists, similar manipulations would yield a tight lower bound on $\bP_{n}[\theta_i \geq \theta_j]$, hence showing Property~\ref{prop:tight_time_uniform_lower_bound_2arm_BCP}. The proof of Theorem~\ref{thm:two_arms_BCP_tight_lower_bound} heavily relies on the modified empirical cdf featured in the tight BCP lower bound.
Indeed, for bounded distributions Lemma~\ref{lem:bcp_lower_bound_discretized_one_arm} follows from a discretization of the empirical cdf, which allows to use results on multinomial distributions. Since the technicalities depend on the considered distribution, we don't provide a general proof of Property~\ref{prop:tight_time_uniform_lower_bound_2arm_BCP} given a tight lower bound on the BCP.

\paragraph{Rates for $\bm{a_{n+1,i}}$}
Analyzing randomized mechanisms heavily relies on properties of $a_{n+1,i}$.
Lemma~\ref{lem:fast_rate_posterior} shows that $a_{n+1,i}$ decreases exponentially with a linear rate for the arms not having highest means.
\begin{lemma} \label{lem:fast_rate_posterior}
		Let $\Pi_n$ satisfying Property~\ref{prop:one_arm_bcp_upper_bound}, and $c_1$ therein.
		Let $S_{n}^{L}$ and $\cI_n^\star$ as in (\ref{eq:def_sampled_enough_sets}).
		There exists $L_7$ with $\bE_{\bm F}[(L_7)^{\alpha}] < + \infty$ for all $\alpha > 0$ such that if $L \ge L_7$, for all $n$ such that $S_{n}^{L} \neq \emptyset$,
		\[
		\forall i \in S_{n}^{L} \setminus  \cI_n^\star, \quad	a_{n+1,i} \leq  f(c_{1}(L) D_{\bm F}) \: ,
		\]
		where $f(x) = (1+x)e^{-x}$ and $ D_{\bm F} > 0$ is the problem dependent constant from Lemma~\ref{lem:fast_rate_emp_tc}.
\end{lemma}
\begin{proof}
		Let $S_{n}^{L}$ and $\cI_n^\star$ as in (\ref{eq:def_sampled_enough_sets}).
		Assume that $S_{n}^{L} \neq \emptyset$.
		If $S_{n}^{L} \setminus  \cI_n^\star $ is empty, then the statement is not informative.
		Assume $S_{n}^{L} \setminus  \cI_n^\star $ is not empty.
		Let $(i, j) \in \cI_n^\star \times \left(S_{n}^{L} \setminus  \cI_n^\star \right)$.

		Since $\Pi_n$ satisfies Property~\ref{prop:one_arm_bcp_upper_bound}, using Lemma~\ref{lem:two_arm_BCP_tight_upper_bound} yields
	\begin{align*}
		a_{n+1,j} \le  \bP_n [\theta_j \geq \theta_{i}]
		&\leq f\left( \inf_{u \in [0,B]} \left\{ c_1(N_{n,i}) \Kinf^-(\tilde F_{n,i},u)  + c_1(N_{n,j}) \Kinf^+(\tilde F_{n,j},u) \right\} \right) \\
		&\leq f\left( c_1(L) \inf_{u \in [0,B]} \left\{ \Kinf^-(\tilde F_{n,i},u)  + \Kinf^+(\tilde F_{n,j},u) \right\} \right) \: ,
	\end{align*}
	where we used that $\{i,j\} \subset S_{n}^{L}$, $c_1$ increasing and $f$ decreasing.

	Using Lemma~\ref{lem:positive_strict_geometric_cst}, there exists $\alpha > 0$ such that
	\begin{equation*}
		D_{\bm F} = \min_{(i, j) : m(F_i) > m(F_j) } \inf_{\substack{G_i,G_j : \\ \forall k \in \{i,j\}, \|G_k - F_k\|_{\infty}\leq \alpha}} \inf_{u \in [0,B]} \left\{ \Kinf^{-}(G_i , u) + \Kinf^{+}(G_j , u) \right\} > 0 \: .
	\end{equation*}

	Using Lemma~\ref{lem:subG_cdf} and (\ref{eq:def_closeness_modified_cdfs_true_cdfs}), i.e. $\max_{i \in [K]} \left\| \tilde F_{n,i} -  F_{n,i} \right\|_{\infty} \le d_0(n)$ where $d_0(n) =_{+\infty} o(n^{-\alpha})$, there exists $L_7 = Poly(W_2) $ such that for all $L\geq L_7$ and all $i \in S_{n}^{L}$,
	$
		 \left\| \tilde F_{n,i} - F_i \right\|_{\infty} \leq \alpha \: .
	$

	Since $f$ is decreasing, further upper bounding yields directly that for all $L\geq L_7$ and all $j \in S_{n}^{L} \setminus  \cI_n^\star$, we have
	\[
	a_{n+1,j} \le f\left( c_1(L) D_{\bm F} \right) \: .
	\]
	As $\bE_{\bm F} [e^{\lambda W_2}] < + \infty$ for all $\lambda > 0$, we have $\bE_{\bm F}[(L_7)^{\alpha}] < + \infty$ for all $\alpha > 0$ since $(L_7)^{\alpha} = Poly(W_2)$.
	This concludes the proof.
\end{proof}

Lemma~\ref{lem:slow_posterior_rate_undersampled_arms} gives a lower bound on $a_{n,i}$ for under-sampled arms.
	\begin{lemma} \label{lem:slow_posterior_rate_undersampled_arms}
		Let $\Pi_n$ satisfying Properties~\ref{prop:one_arm_bcp_coarse_lower_bound} and~\ref{prop:one_arm_bcp_upper_bound}.
		Let $S_{n}^{L}$ as in (\ref{eq:def_sampled_enough_sets}).
		There exists $L_8$ with $\bE_{\bm F}[(L_8)^{\alpha}] <  + \infty$ for all $\alpha > 0$ such that for all $L \geq L_8$ and all $n \in \N$,
		\[
		\forall i \in \overline{S_{n}^{L}}, \quad a_{n+1, i} \geq \frac{e^{-D_0 c_0(L)}}{2^{K-1}}	\: ,
		\]
	where $c_0$ is defined in Property~\ref{prop:one_arm_bcp_coarse_lower_bound} and $D_0 > 0$ is a problem dependent constant.
	\end{lemma}
	\begin{proof}
	Let $i \in \overline{S_{n}^{L}}$ and $u \in (0,B)$. As explained above, we have
	\[
		a_{n+1,i} \geq
		\bP_{n} \left[ \theta_{i} \geq u \right] \prod_{j \in S_{n}^{L} } \bP_{n} \left[ \theta_{j} \leq u \right] \prod_{j \in \overline{S_{n}^{L}} \setminus\{i\}} \bP_{n} \left[ \theta_{j} \leq u \right]  \: .
	\]
	For all $j \in S_{n}^{L}$, Property~\ref{prop:one_arm_bcp_upper_bound} yields
	\begin{align*}
		 \bP_{n} \left[ \theta_{j} \leq u \right]
		 = 1 -  \bP_{n} \left[ \theta_{j} \geq u \right]
		 \geq 1 - e^{-c_1(N_{n,j}) \Kinf^{+}(\tilde F_{n,j}, u)}
		 \geq 1 - e^{-c_1(L) \Kinf^{+}(\tilde F_{n,j}, u)} \: ,
	\end{align*}
	where we used that $N_{n,j}\geq L $ for all $j \in S_{n}^{L}$ and $c_1$ increasing.

	By Theorem~\ref{thm:Kinf_continuous}, the function $F \mapsto \Kinf^{+}(F, u)$ is continuous on $\cF$.
	Using Lemma~\ref{lem:subG_cdf} and (\ref{eq:def_closeness_modified_cdfs_true_cdfs}), i.e. $\max_{i \in [K]} \left\| \tilde F_{n,i} -  F_{n,i} \right\|_{\infty} \le d_0(n)$ where $d_0(n) =_{+\infty} o(n^{-\alpha})$, there exists $L_9 = Poly(W_2) $ such that for all $L\geq L_9$ and all $j \in S_{n}^{L}$,
	\[
	\Kinf^{+}(\tilde F_{n,j}, u) \geq \frac{1}{2} \Kinf^{+}( F_{j}, u) \geq \frac{1}{2} \min_{j \in [K]} \Kinf^{+}( F_{j}, u) \: .
	\]
	Since $\mu_{j} \in (0,B)$ for all $j \in [K]$, there exists $u \in (0,B)$ such that $\min_{j\in [K]} \Kinf^{+}(F_{j}, u) > 0$.
	Choosing such a $u$, there exists a deterministic $L_{10}$ such that for all $L \geq L_8 \eqdef \max\{L_9, L_{10}\}$
	\[
	\forall j \in S_{n}^{L}, \quad \bP_n \left[ \theta_{j} \leq u \right] \geq \frac{1}{2} \: .
	\]
	For the under-sampled arms $j \in \overline{S_{n}^{L}}$, Property~\ref{prop:one_arm_bcp_coarse_lower_bound} yields that
	\begin{align*}
	\forall j \in \overline{S_{n}^{L}} \setminus\{i\}, \qquad
	&\bP_{n} \left[ \theta_j \le u \right]
	\ge e^{- c_0(N_{n,j}) \kappa^-(u)}
	\ge e^{- c_0(L) \kappa^-(u)}
	\: , \\
	&\bP_{n}\left[ \theta_{i} \ge u \right]
	\ge e^{- c_0(N_{n,i})\kappa^+(u)}
	\ge e^{- c_0(L)\kappa^+(u)}
	\: ,
	\end{align*}
	where we used that $N_{n,i} < L$ for $j \in \overline{S_{n}^{L}}$, $c_0$ increasing and $\kappa^-$, $\kappa^+$ strictly positive.

	Combining the above and further lower bounding, we have shown that for $L \geq L_8$,
	\begin{align*}
		\forall i \in \overline{S_{n}^{L}}, \quad a_{n+1,i} &\geq \frac{e^{-D_0 c_0(L)}}{2^{K-1}} \: ,
	\end{align*}
	where $D_0 = \kappa^+(u) + (K-1)\kappa^-(u) $.
	As $\bE_{\bm F} [e^{\lambda W_2}] < + \infty$ for all $\lambda > 0$, we have $\bE_{\bm F}[(L_9)^{\alpha}] < + \infty$ for all $\alpha > 0$.
	Since $\bE_{\bm F}[(L_8)^{\alpha}] \le (L_{10})^{\alpha}+ \bE_{\bm F}[(L_9)^{\alpha}] < + \infty$, this concludes the proof.
	\end{proof}

\subsubsection{TS leader}
\label{app:sss_ts_leader}

Conditioned on $\cF_n$, the internal randomness of the Thompson Sampling (TS) leader is parameterized by a sampler $\Pi_n$, where $a_{n+1,i} = \bP_{n}[i \in \argmax_{j \in [K]} \theta_j]$.
Given an observation $\theta \sim \Pi_n$, the TS leader is defined as an arm with highest mean for $\theta$,
\begin{equation} \label{eq:def_ts_based_leader}
	B_{n+1}^{\text{TS}} \in \argmax_{i \in [K]} \theta_{i} \quad \text{,} \quad \bP_{\mid n}[B_{n+1}^{\text{TS}} = i] = a_{n+1,i} \quad \text{and} \quad \widehat B_{n+1}^{\text{TS}} \in \argmax_{i \in [K]} a_{n+1,i} \: ,
\end{equation}
where $\widehat B_{n+1}^{\text{TS}}$ is defined as an arm with highest $a_{n,i}$.

\paragraph{Property~\ref{prop:leader_cdt_suff_explo}}
Lemma~\ref{lem:TS_ensures_suff_explo} shows that Property~\ref{prop:leader_cdt_suff_explo} is satisfied by $B_{n+1}^{\text{TS}}$.
\begin{lemma} \label{lem:TS_ensures_suff_explo}
	Let $\Pi_n$ satisfying Property~\ref{prop:one_arm_bcp_upper_bound}.
	Let $S_{n}^{L}$ and $\cI_n^\star$ as in (\ref{eq:def_sampled_enough_sets}).
	There exists $\tilde L_7$ with $\bE_{\bm F}[(\tilde L_7)^{\alpha}] < + \infty$ for all $\alpha > 0$ such that if $L \ge \tilde L_7$, for all $n$ such that $S_{n}^{L} \neq \emptyset$, $\widehat B_{n+1}^{\text{TS}} \in S_{n}^{L}$ implies $\widehat B_{n+1}^{\text{TS}} \in \cI_n^\star$.
\end{lemma}
\begin{proof}
	Let $S_{n}^{L}$ and $\cI_n^\star$ as in (\ref{eq:def_sampled_enough_sets}).
	Assume that $S_{n}^{L} \neq \emptyset$.
	If $S_{n}^{L} \setminus  \cI_n^\star $ is empty, then the result is true.
	Assume $S_{n}^{L} \setminus  \cI_n^\star $ is not empty.
	Let $j \in S_{n}^{L} \setminus  \cI_n^\star$ and $L_7$ as in Lemma~\ref{lem:fast_rate_posterior}, hence
\[
a_{n+1,j} \le f\left( c_1(L) D_{\bm F} \right) \: .
\]
As $c_1(x) \sim_{+ \infty} x$ and $\lim_{+ \infty} f(x) = 0$, there exists a deterministic $L_8 $ such that for all $L\geq L_8$,
\[
f\left( c_1(L) D_{\bm F} \right) < \frac{1}{K} \: .
\]
Therefore, for all $L \ge \tilde L_7 \eqdef \max\{ L_7, L_8\}$ and all $j \in S_{n}^{L} \setminus  \cI_n^\star$, we have $a_{n+1,j} < \frac{1}{K}$.

Assume that $\widehat B_{n+1}^{\text{TS}} \in S_{n}^{L}$. Suppose towards contradiction that $\widehat B_{n+1}^{\text{TS}} \notin \cI_n^\star$.
Then, the above shows that $a_{n+1,\widehat B_{n+1}^{\text{TS}}} < \frac{1}{K}$.
This is a contradiction with $\widehat B_{n+1}^{\text{TS}} \in \argmax_{i \in [K]} a_{n+1,i}$, hence $\widehat B_{n+1}^{\text{TS}} \in \cI_n^\star$.
Since $\bE_{\bm F}[(\tilde L_7)^{\alpha}] \le (L_8)^{\alpha} + \bE_{\bm F}[( L_7)^{\alpha}]  $, this concludes the proof.
\end{proof}

\paragraph{Property~\ref{prop:leader_cdt_convergence}}
Lemma~\ref{lem:TS_ensures_convergence} shows that Property~\ref{prop:leader_cdt_convergence} is satisfied by $B_{n+1}^{\text{TS}}$.
More precisely, we show that after enough time, the probability for the leader to not be the best arm is decreasing exponentially fast.
\begin{lemma} \label{lem:TS_ensures_convergence}
		Assume Property~\ref{prop:suff_exploration} holds.
		Let $\Pi_n$ satisfying Property~\ref{prop:one_arm_bcp_upper_bound}, and $c_1$ therein.
		There exists $N_{9}$ with $\bE_{\bm F}[N_{9}] < + \infty$ such that for all $n \ge N_{9}$,
		\[
		\bP_{\mid n}[ B_{n+1}^{\text{TS}} \neq i^\star (\bm F)] \leq (K-1)f\left(c_{1}\left( \sqrt{\frac{n}{K}}\right) D_{\bm F}\right)\: ,
		\]
		where $f(x) = (1+x)e^{-x}$ and $ D_{\bm F} > 0$ is the problem dependent constant from Lemma~\ref{lem:fast_rate_emp_tc}.
\end{lemma}
\begin{proof}
	Let $i^\star = i^\star(\bm F)$.
	Let $N_1$ as in Property~\ref{prop:suff_exploration}, then $N_{n,i} \geq \sqrt{\frac{n}{K}}$ for all $n \ge N_1$.
	Let $L_{7}$ as in Lemma~\ref{lem:fast_rate_posterior}.
	For all $n \geq N_{9} = \max\{N_1, K (L_{7})^2\}$, Lemma~\ref{lem:TS_ensures_suff_explo} and Property~\ref{prop:suff_exploration} yields that
	\[
		\forall i \neq i^\star, \quad a_{n+1,i} \leq  f\left(c_{1}\left( \sqrt{\frac{n}{K}}\right) D_{\bm F}\right) \: .
	\]

	Using the definition of $B_{n+1}^{\text{TS}}$ in (\ref{eq:def_ts_based_leader}), we obtain
	\begin{align*}
			\bP_{\mid n}[ B_{n+1}^{\text{TS}} \neq i^\star ]
			= \sum_{i \neq i^\star} \bP_{\mid n}[ B_{n+1}^{\text{TS}} = i ]
			\leq (K-1) \max_{i \neq i^\star} a_{n+1,i}
			\leq (K-1)f\left(c_{1}\left( \sqrt{\frac{n}{K}}\right) D_{\bm F}\right) \: .
	\end{align*}
	Since $\bE_{\bm F}[ N_{9}] \le \bE_{\bm F}[ N_{1}] + K \bE_{\bm F}[ (L_{7})^2 ]< + \infty$, this concludes the result.
\end{proof}

\subsubsection{RS challenger}
\label{app:sss_rs_challenger}

Conditioned on $\cF_n$, the internal randomness of the Re-Sampling (RS) challenger is parameterized by a sampler $\Pi_n$, where $a_{n+1,i} \eqdef \bP_{n}[i \in \argmax_{j \in [K]} \theta_j]$.
Given a leader $B_{n+1}$, the RS challenger is defined by repeatedly sampling $\tilde \theta \sim \Pi_n$ until $B_{n+1} \notin \argmax_{i \in [K]} \tilde \theta_{i}$ and by taking an arm with highest mean for this $\tilde \theta$
\begin{equation} \label{eq:def_rs_based_challenger}
	C_{n+1}^{\text{RS}} \in \argmax_{i \in [K]} \tilde \theta_{i} \not\owns B_{n+1} \quad \text{and} \quad \widehat C_{n+1}^{\text{RS}} \in \argmax_{j \neq \widehat B_{n+1}} a_{n+1,j} \: ,
\end{equation}
where
\begin{align*}
	\bP_{\mid n}[C_{n+1}^{\text{RS}} = j| B_{n+1} = i] = \sum_{k = 0}^{+\infty} a_{n+1,i}^{k} a_{n+1,j} = \frac{a_{n+1,j}}{1-a_{n+1,i}} \: .
\end{align*}

\paragraph{Property~\ref{prop:challenger_cdt_suff_explo}}
We prove Property~\ref{prop:challenger_cdt_suff_explo} for $C_{n+1}^{\text{RS}}$ in Lemma~\ref{lem:RS_ensures_suff_explo} by comparing the rates with which $a_{n+1,i}$ decreases.
The effective challenger $\widehat C_{n+1}^{\text{RS}} $ is taken as an arm different from $\widehat B_{n+1}$ which maximizes $a_{n+1,i}$.
Therefore, it is sufficient to show that the sampled enough arms have lower $a_{n+1,i}$ than the mildly under-sampled ones.
This will imply that $\widehat C_{n+1}^{\text{RS}} $ has to be mildly under-sampled or be an arm with highest true mean among the sampled enough arms.
\begin{lemma} \label{lem:RS_ensures_suff_explo}
	Let $\Pi_n$ satisfying Properties~\ref{prop:one_arm_bcp_coarse_lower_bound} and~\ref{prop:one_arm_bcp_upper_bound}.
	Let $B_{n+1}$ be a leader satisfying Property~\ref{prop:leader_cdt_suff_explo}.
	Given $(B_{n+1}, \widehat B_{n+1})$, let $(C_{n+1}^{\text{RS}}, \widehat C_{n+1}^{\text{RS}})$ as in (\ref{eq:def_rs_based_challenger}).
	Let $U_n^L$ and $V_n^L$ as in (\ref{eq:def_undersampled_sets}) and $\mathcal J_n^\star = \argmax_{ i \in \overline{V_{n}^{L}}} \mu_{i}$.
	There exists $L_9$ with $\bE_{\bm F}[L_9] < + \infty$ such that if $L \ge L_9$, for all $n$ such that $U_n^L \neq \emptyset$, $\widehat B_{n+1} \notin V_{n}^{L}$ implies $\widehat C_{n+1}^{\text{RS}} \in V_{n}^{L} \cup \left(\mathcal J_n^\star \setminus \left\{\widehat B_{n+1} \right\}\right)$.
\end{lemma}
\begin{proof}
		Let $\mathcal J_n^\star = \argmax_{ i \in \overline{V_{n}^{L}}} \mu_{i}$.
		In the following, we consider $U_n^L \neq \emptyset$ (hence $V_n^L \neq \emptyset$) and $\widehat B_{n+1} \in V_n^L$.
		Let $B_{n+1}$ be a leader satisfying Property~\ref{prop:leader_cdt_suff_explo}, and $L_0$ defined therein.
		Then, for $L \geq L_{0}^{4/3}$, we have $\widehat B_{n+1} \in \mathcal J_n^\star$.
		If $\widehat C_{n+1}^{\text{RS}} \in \mathcal J_n^\star \setminus \left\{\widehat B_{n+1} \right\}$, we are done.
		Assume that $\widehat C_{n+1}^{\text{RS}} \notin \mathcal J_n^\star \setminus \left\{\widehat B_{n+1} \right\}$.

		Since $\Pi_n$ satisfies Properties~\ref{prop:one_arm_bcp_coarse_lower_bound} and~\ref{prop:one_arm_bcp_upper_bound}, let $L_7$ and $L_8$ as in Lemmas~\ref{lem:fast_rate_posterior} and~\ref{lem:slow_posterior_rate_undersampled_arms}.
		Then, for all $L \geq \max\{L_{0}^{4/3}, L_7^{4/3}, L_8^2\}$,
		\begin{align*}
			&\widehat B_{n+1} \in \mathcal J_n^\star \: , \\
			\forall i \in \overline{V_n^L} \setminus \mathcal J_n^\star, \quad & a_{n+1,i} \leq f(c_{1}(L^{3/4}) D_{\bm F})     \: , \\
			\forall j \in U_n^L, \quad 	& a_{n+1,j} \geq  \frac{e^{-D_0 c_0(\sqrt{L})}}{2^{K-1}} \: .
		\end{align*}

		Since $f(x) = (1+x)e^{-x}$, $c_0(x) \sim_{+\infty} x$ and $c_1(x) \sim_{+\infty} x$, there exists a deterministic $L_{10}$ such that for all $L \ge L_{10}$,
		\[
			f(c_{1}(L^{3/4}) D_{\bm F}) < \frac{e^{-D_0 c_0(\sqrt{L})}}{2^{K-1}}  \: .
		\]
		Therefore, for all $L \ge L_{9} \eqdef \max \{L_{0}^{4/3}, L_7^{4/3}, L_8^2, L_{10}\}$,
		\[
		\forall  (j, i) \in U_n^L \times \left(\overline{V_n^L}  \setminus \mathcal J_n^\star \right), \quad  a_{n+1,j} > a_{n+1,i} \: .
		\]
		As $\widehat B_{n+1} \in \mathcal J_n^\star$ and $\widehat C_{n+1}^{\text{RS}} \notin \mathcal J_n^\star \setminus \left\{\widehat B_{n+1} \right\}$, the definition $\widehat C_{n+1}^{\text{RS}}  \in \argmax_{j \neq \widehat B_{n+1}} a_{n+1,j}$ yields that $\widehat C_{n+1}^{\text{RS}} \in V_{n}^{L}$.
		Otherwise the above strict inequality would wield a contradiction.
		Since
		\[
		\bE_{\bm F}[L_9] \le L_{10} + \bE_{\bm F}[(L_{0})^{4/3}] + \bE_{\bm F}[(L_7)^{4/3}] + \bE_{\bm F}[(L_8)^2] < + \infty		\: ,
		\]
		this concludes the proof.
\end{proof}

\paragraph{Property~\ref{prop:challenger_cdt_convergence}}
Lemma~\ref{lem:RS_ensures_convergence} shows that Property~\ref{prop:challenger_cdt_convergence} is satisfied by $C_{n+1}^{\text{RS}}$.
\begin{lemma} \label{lem:RS_ensures_convergence}
		Assume Property~\ref{prop:suff_exploration} holds.
		Let $\Pi_n$ satisfying Properties~\ref{prop:one_arm_bcp_upper_bound} and~\ref{prop:tight_time_uniform_lower_bound_2arm_BCP}.
		Let $B_{n+1}$ be a leader satisfying Property~\ref{prop:leader_cdt_convergence}.
		Let $\epsilon \in (0, \epsilon_0]$ where $\epsilon_0$ is a problem dependent constant.
		Given $B_{n+1}$, let $C_{n+1}^{\text{RS}}$ as in (\ref{eq:def_rs_based_challenger}).
		There exists $N_{10}$ with $\bE_{\bm F}[N_{10}] < + \infty$ such that for all $n \geq N_{10}$ and all $i \neq i^\star(\bm F)$,
		\begin{equation} \label{eq:overshooting_implies_not_sampled_anymore_rs}
				\frac{\Psi_{n,i}}{n} \geq w_{i}^{\beta} + \epsilon  \quad \implies \quad \bP_{\mid n}[C_{n+1}^{\text{RS}} = i \mid B_{n+1} = i^\star(\bm F)] \leq h(n) \: ,
		\end{equation}
		where $h : \N^\star \to (0, + \infty)$ such that $h(n) =_{+\infty} o(n^{-\alpha})$ with $\alpha > 0$.
\end{lemma}
\begin{proof}
	Let $\epsilon > 0$ and $i^\star = i^\star(\bm F)$.
	Let $N_1$ as in Property~\ref{prop:suff_exploration}, then $N_{n,i} \geq \sqrt{\frac{n}{K}}$ for all $n \ge N_1$.
	Since $i^\star$ is unique, we have $\Delta \eqdef \min_{j\neq i^\star}|\mu_{i^\star} - \mu_{j}| > 0$.
	For bounded distributions, $F \mapsto m(F)$ is continuous on $\cF$ for the weak convergence.
	Lemma~\ref{lem:subG_cdf} yields that there exists $N_{11} = Poly(W_2)$ such that for all $n \ge \max\{N_1, N_{11}\}$ and all $i \in [K]$, we have $|\mu_{n,i} - \mu_{i}| \leq \frac{\Delta}{4}$.
	Therefore, for all $n \ge \max\{N_1, N_8\}$, $\argmax_{i \in [K]} \mu_{n,i} = \argmax_{i \in [K]} \mu_{i} = i^\star$.

	Let $\xi >0$. Since Property~\ref{prop:suff_exploration} holds and $B_{n+1}$ satisfies Property~\ref{prop:leader_cdt_convergence}, we can use the results from Lemma~\ref{lem:convergence_towards_optimal_allocation_best_arm}.
	Let $N_4$ defined in Lemma~\ref{lem:convergence_towards_optimal_allocation_best_arm}, we have $\left| \frac{N_{n,i^\star}}{n} - \beta \right| \leq \xi$ for all $n \ge \max \{ N_1, N_4\}$.

	Using the definition of $C_{n+1}^{\text{RS}}$ in (\ref{eq:def_rs_based_challenger}), we have
	\begin{align*}
		\bP_{\mid n}[C_{n+1}^{\text{RS}} = i \mid B_{n+1} = i^\star] = \frac{a_{n+1,i}}{1-a_{n+1,i^\star}} \leq \frac{\bP_{n}[\theta_i \ge \theta_{i^\star}]}{\max_{j \neq i^\star} \bP_{n}[\theta_j \ge \theta_{i^\star}]} \: ,
	\end{align*}
	where we used that $\{\theta_j > \theta_{i^\star}\} \subseteq \bigcup_{j\neq i^\star} \{\theta_j > \theta_{i^\star}\} = \{i^\star \notin \argmax_{j \in [K]} \theta_j\}$.

	Let $i \neq i^\star$ such that $\frac{\Psi_{n,i}}{n} \geq w_{i}^{\beta} + \epsilon$.
	Using Lemma~\ref{lem:subG_alloc}, there exists $N_{12} = Poly(W_1) $, such that for all $n \ge \max\{N_1, N_{12}\}$, we have $\frac{N_{n,i}}{n} \geq w_{i}^{\beta} + \frac{\epsilon}{2}$.
	Therefore, for all $n \ge  \max_{i \in \{1, 4, 11, 12\}} N_i$,

	Let $f(x) = (1+x)e^{-x}$. Since $\Pi_n$ satisfies Property~\ref{prop:one_arm_bcp_upper_bound}, Lemma~\ref{lem:two_arm_BCP_tight_upper_bound} yields
	\begin{align*}
		\bP_n [\theta_i \geq \theta_{i^\star}] \leq f\left( n \inf_{u \in [0,B]} \left\{ \frac{c_1(N_{n,i^\star})}{n} \Kinf^-(\tilde F_{n,i^\star},u)  + \frac{c_1(N_{n,i})}{n} \Kinf^+(\tilde F_{n,i},u) \right\} \right) \: .
	\end{align*}

	Let $\tilde \epsilon > 0 $. Since $f(x) =_{+ \infty} \cO \left( e^{-(1-\tilde \epsilon)x} \right)$ and $c_{1}(x) \sim_{+\infty} x$, there exists deterministic $ C_{\tilde \epsilon}$ and $N_{13}$ such that for all $n \geq \max_{i \in \{1, 4, 11, 12, 13\}} N_i$,
	\begin{align*}
		\bP_n [\theta_i \geq \theta_{i^\star}] &\leq C_{\tilde \epsilon} \exp \left(  - n(1-\tilde \epsilon) \inf_{u \in [0,B]} \left\{ \frac{N_{n,i^\star}}{n} \Kinf^-(\tilde F_{n,i^\star},u)  + \frac{N_{n,i}}{n} \Kinf^+(\tilde F_{n,i},u) \right\} \right) \\
		&\leq C_{\tilde \epsilon} \exp \left(  - n(1-\tilde \epsilon) \inf_{u \in [0,B]} \left\{ \frac{N_{n,i^\star}}{n} \Kinf^-(\tilde F_{n,i^\star},u)  + \left( w_{i}^{\beta} + \frac{\epsilon}{2} \right) \Kinf^+(\tilde F_{n,i},u) \right\} \right) \: .
	\end{align*}

Let $(h_{\epsilon}, N_8)$ as in Property~\ref{prop:tight_time_uniform_lower_bound_2arm_BCP}.
Since $h_{\epsilon}$ is increasing in both its arguments, we have $h_{\epsilon}(N_{n,i^\star}, N_{n,i}) \leq h_{\epsilon}(n,n)$ and $N_{n,i^\star} + N_{n,i} \leq n$.
Therefore, for all $n \ge \max\{K N_{8}^2,\max_{i \in \{1, 4, 11, 12\}} N_i\}$,
	\begin{align*}
	&\max_{j \neq i^\star} \bP_{n} [\theta_j \geq \theta_{i^\star}] \\
	&\ge \frac{e^{-\varepsilon n}}{h_{\epsilon}(n,n)} \exp\left(- n \min_{j \neq i^\star} \inf_{x\in [0,B]}\left\{  \frac{N_{n,i^\star}}{n}\Kinf^-(F_{i^\star}, x) + \frac{N_{n,j}}{n} \Kinf^+(F_{j}, x) \right\} \right) \\
	&\ge \frac{e^{-\varepsilon n}}{h_{\epsilon}(n,n)} \exp\left(- n \sup_{w\in \simplex: w_{i^\star} = \frac{N_{n,i^\star}}{n}}\min_{j \neq i^\star} \inf_{x\in [0,B]}\left\{ w_{i^\star} \Kinf^-(F_{i^\star}, x) + w_j \Kinf^+(F_{j}, x) \right\} \right)  \: ,
	\end{align*}
	where we lower bounded by considering the best possible allocation such that $w_{i^\star} = \frac{N_{n,i^\star}}{n}$.
	For $( \bm G,\tilde \beta ) \in \cF^{2}\times [0,1]$, let
	\begin{align*}
		H_{\tilde \epsilon}(\bm G, \tilde \beta) &= (1-\tilde \epsilon) \inf_{u \in  [0,B]} \left\{ \tilde \beta \Kinf^{-}( G_{1}, u)  + \left( w_{i}^{\beta} + \frac{\epsilon}{2} \right) \Kinf^{+}( G_{2}, u)  \right\} \\
		&\quad - \sup_{w \in \simplex : w_{i^\star} = \tilde \beta} \min_{j \neq i^\star}  \inf_{u  \in [0,B]}\left\{w_{i^\star}\Kinf^-( F_{i^\star}, u) + w_j \Kinf^+(F_j, u) \right\} \:.
	\end{align*}

	Let $\bm{\tilde{G}}_{n,i^\star, i} = (\tilde F_{n,i^\star}, \tilde F_{n,i})$. Combining the upper and the lower bound, we obtain for all $n \geq \max\{K N_{8}^2,\max_{i \in \{1, 4, 11, 12, 13\}} N_i\}$,
	\begin{align*}
		\bP_{\mid n}[C_{n+1}^{\text{RS}} = i \mid B_{n+1} = i^\star] &\leq C_{\tilde \epsilon} h_{\epsilon}(n,n)e^{\varepsilon n} \exp \left( - n H_{\tilde \epsilon}\left(\bm{\tilde{G}}_{n,i^\star, i}, \frac{N_{n,i^\star}}{n} \right) \right) \\
		&\leq  C_{\tilde \epsilon} h_{\epsilon}(n,n)e^{\varepsilon n} \exp \left( - n \inf_{\tilde \beta: |\beta - \tilde \beta| \leq \xi } H_{\tilde \epsilon}\left(\bm{\tilde{G}}_{n,i^\star, i}, \tilde \beta \right) \right) \: .
	\end{align*}

	Using Lemma~\ref{lem:continuity_results_for_analysis}, the functions $(\bm G, \tilde \beta) \mapsto H_{\tilde \epsilon}(\bm G, \tilde \beta)$ and $\bm G \mapsto\inf_{\tilde \beta: |\beta - \tilde \beta| \leq \xi }   H_{\tilde \epsilon}(\bm G, \tilde \beta)$ are continuous.
	Let $\bm G_{i^\star, i} = (F_{i^\star}, F_{i})$.
	Therefore, there exists $N_{14} = Poly(W_1) $, $\xi_0 > 0$ and $\tilde \epsilon_0 > 0$ such that for all $n\ge N_{10} \eqdef \max\{K N_{8}^2,\max_{i \in \{1, 4, 11, 12, 13, 14\}} N_i\}$, all $\xi \in(0, \xi_0]$ and $\tilde \epsilon \in (0, \tilde \epsilon_0]$, we have
	\begin{align*}
		\inf_{\tilde \beta: |\beta - \tilde \beta| \leq \xi } H_{\tilde \epsilon}\left(\bm{\tilde{G}}_{n,i^\star, i}, \tilde \beta \right) \geq \frac{1}{2} \inf_{\tilde \beta: |\beta - \tilde \beta| \leq \xi }  H_{\tilde \epsilon}(\bm G_{i^\star, i}, \tilde \beta) \ge \frac{1}{4} H_{\tilde \epsilon}(\bm G_{i^\star, i}, \beta) \ge \frac{1}{8} H_{0}(\bm G_{i^\star, i}, \beta) \: .
	\end{align*}
	In the following, we take such $\xi_0 > 0$ and $\tilde \epsilon_0 > 0$ and $\epsilon \in (0,\epsilon_0]$ where $\epsilon_{0} = \frac{1}{16} H_{0}(\bm G_{i^\star, i}, \beta)$ is a problem dependent constant.

	At the $\beta$-equilibrium all transportation costs are equal (Lemma~\ref{lem:properties_characteristic_times}). Therefore, by definition of $w^{\beta}$,
	\begin{align*}
		&\sup_{w \in \simplex : w_{i^\star} = \beta} \min_{j \neq i^\star}  \inf_{u  \in [0,B]}\left\{w_{i^\star}\Kinf^-( F_{i^\star}, u) + w_j \Kinf^+(F_j, u) \right\} \\
		&\qquad = \min_{j \neq i^\star}  \inf_{u  \in [0,B]}\left\{\beta \Kinf^-( F_{i^\star}, u) + w^\beta_j \Kinf^+(F_j, u) \right\} \\
		&\qquad = \inf_{u  \in [0,B]}\left\{\beta \Kinf^-( F_{i^\star}, u) + w^\beta_i \Kinf^+(F_i, u) \right\} \\
		&\qquad < \inf_{u  \in [0,B]}\left\{\beta \Kinf^-( F_{i^\star}, u) + \left( w^\beta_i + \frac{\epsilon}{2} \right) \Kinf^+(F_i, u) \right\}
	\end{align*}
	where the strict inequality is obtained because the transportation costs are strictly increasing in their allocation arguments (Lemma~\ref{lem:inf_Kinf_increasing_in_w}). Therefore, we have $H_{0}(\bm G_{i^\star, i}, \beta) > 0$.

	As $\bE_{\bm F} [e^{\lambda W_1}] < + \infty$ and $\bE_{\bm F} [e^{\lambda W_2}] < + \infty$ for all $\lambda > 0$, we have $\bE_{\bm F}[N_i] < + \infty$ for $i \in \{11, 12, 14\}$ and
	\[
	\bE_{\bm F}[N_{10}] \le N_{13} + K \bE_{\bm F}[(N_8)^2] + \sum_{i \in \{1, 4, 11, 12, 14\}} \bE_{\bm F}[N_i]  < + \infty		\: .
	\]

	Summarizing, we have shown that for all $\epsilon \in (0,\epsilon_0]$, there exists $N_{10}$ with $\bE_{\bm F}[N_{10}] < + \infty$ such that for all $n \ge N_{10}$,
	\begin{align*}
				\frac{\Psi_{n,i}}{n} \geq w_{i}^{\beta} + \epsilon  \quad \implies \quad \bP_{\mid n}[C_{n+1}^{\text{RS}} = i \mid B_{n+1} = i^\star(\bm F)] \leq h(n) \: ,
	\end{align*}
	where
	\begin{align*}
		h(n) \eqdef C_{\tilde \epsilon_0} h_{\epsilon}(n,n) \exp \left( - \frac{n}{16} H_{0}(\bm G_{i^\star, i}, \beta) \right) \: .
	\end{align*}
	Since $H_{0}(\bm G_{i^\star, i}, \beta) > 0$, $n \mapsto h_{\epsilon}(n, n)$ is decreasing and $h_{\epsilon}(n, n) =_{+\infty} o\left( e^{(2n)^{\alpha}} \right)$ where $\alpha <1$, we obtain that $h(n) =_{+\infty} o(n^{-\alpha})$ with $\alpha > 0$. It is obvious by definition that $h(n) \in (0, +\infty)$ for all $n \in \N^\star$.
\end{proof}

\subsection{Relaxing the distinct means assumption}
\label{app:ss_beyond_all_distinct_means}

In Appendix~\ref{app:unified_analysis_top_two}, we highlighted that Assumption~\ref{ass:all_arms_distinct_bounded_mean} ($\Delta_{\min}(\bm F) > 0$) was only used to show sufficient exploration (see Appendix~\ref{app:ss_how_to_explore}).
We also remarked that the proofs in Appendices~\ref{app:ss_how_to_explore} and~\ref{app:ss_how_to_converge} work similarly when the amount of exploration $\sqrt{\frac{n}{K}}$ in Lemma~\ref{lem:suff_exploration} and Property~\ref{prop:suff_exploration} is replaced by $\left(\frac{n}{K}\right)^{\alpha}$ for some arbitrary $\alpha \in (0,1)$.
We conjecture that, besides $\beta$-EB-TC, all the Top Two algorithms studied in this paper are also asymptotically $\beta$-optimal when $\Delta_{\min}(\bm F)=0$, as detailed below. Let $\Delta_{\min} \eqdef \Delta_{\min}(\bm F)$.

\paragraph{Lack of robustness of $\beta$-EB-TC for $\Delta_{\min}=0$}
For the EB-TC sampling rule, a simple explanation hints that it can dramatically fail empirically, which is confirmed experimentally in Appendix~\ref{app:ss_supplementary_experiments}. Let $\bm F$ be a bandit instance in which there are two arms with equal mean that are closest to $\mu_{i^\star}$.
At small time, it can happen that the best arm is under-estimated (e.g. when under-sampled) and the two second-best arms have higher empirical mean.
In that case, is is very hard for $\beta$-EB-TC to recover as it will mostly sample the two second-best arms instead of the best arm.
The EB leader will alternate between one of the two second-best arms, depending on the collected samples.
Then, given the EB leader, the TC challenger will output the arm with smallest transportation cost.
When both second-best arms have higher empirical mean and the best arm is under-estimated, the transportation cost will be smaller between the two second-best arms.
Therefore, the TC challenger will propose the second of the two second-best arms.
As neither the leader nor the challenger propose to sample the true best arm, it is very hard for $\beta$-EB-TC to recover from unlucky first draws.

The condition $\Delta_{\min} > 0$ asymptotically prevents the above situation.
When $\mu_i > \mu_j$, the transportation cost between $(i,j)$ grows linearly with $N_{n,i} + N_{n,j}$.
Therefore, the transportation cost between the over-sampled arms will become larger than between the current leader and the best arm, even if it is under-estimated.
This ensures that the challenger will propose to sample the best arm, hence allowing the algorithm to eventually recover from unlucky first draws.
Based on our analysis, the number of samples required by $\beta$-EB-TC to recover from unlucky first draws is a function of $(D_{\bm F})^{-1}$, where $D_{\bm F}$ is a problem dependent constant defined in (\ref{eq:def_geometric_constant}).
Extrapolating from results on Gaussian, it is intuitive to expect that small $\Delta_{\min}$ yields small $D_{\bm F}$.
Therefore, for small $\Delta_{\min}$, $\beta$-EB-TC can need a large number of samples before recovering from unlucky first draws.
This undesirable behavior in moderate confidence regime is hidden in the asymptotic analysis.
Therefore, we expect $\beta$-EB-TC to also suffer from large outliers in the moderate regime, even when $\Delta_{\min} > 0$.

\paragraph{On asymptotic $\beta$-optimality for $\Delta_{\min}=0$} Experiments reported in Appendix~\ref{app:sss_expe_on_distinct_means} reveal that on some instance with $\Delta_{\min}=0$, the other Top Two instances still have a good performance. We conjecture that using either regularization in the TCI challenger or randomization in the TS leader or RS challenger is adding the right amount of exploration to avoid the undesirable behavior of $\beta$-EB-TC described above, and ensure asymptotic $\beta$-optimality. More precisely, we conjecture that this amount of exploration is actually logarithmic, and that logarithmic exploration is sufficient to prove $\beta$-optimality (which is currently not supported by our analysis).

In particular, for the TS leader it is known from the literature on regret minimization that Thompson Sampling is selecting sub-optimal arms a logarithmic amount of time (at least in expectation) \cite{AGCOLT12}.
As for the TCI challenger, we observe that it is designed to avoid the situation described above in which $\beta$-EB-TC fails when there are two equal second best arms.
When choosing the challenger, we penalize the highly over-sampled arms by adding $\log(N_{n,j})$.
While the transportation cost can be very small for two highly sampled arms having similar means, the penalization makes sure that the under-sampled best arm will be selected as the challenger.
We conjecture that the TCI challenger ensures an implicit logarithmic exploration.

\paragraph{On forced exploration}
Another natural idea to prove asymptotic $\beta$-optimality when $\Delta_{\min}=0$ is to add some small amount of forced exploration to the algorithm.
A round $n$, if there exists an arm $i$ such that $N_{n,i} < {n}^{\alpha}$ (for some small value of $\alpha$), we draw this arm.
This will make Property~\ref{prop:suff_exploration} hold for an exploration level $(n/K)^\alpha$.
However, forced exploration can be wasteful as it is agnostic to $\cF_{n}$ and all under-sampled arms should not be drawn equally.
Our experiments confirm that it is actually not needed for most Top Two algorithms.

Concurrently to our work, \cite{mukherjee_2022_SPRTBAI} introduces and studies the TT-SPRT algorithm for general SPEF.
In our terminology, it corresponds to the $\beta$-EB-TC algorithm with an added forced exploration in $\sqrt{n/K}$.
As expected, adding forced exploration allows to obtain asymptotic $\beta$-optimality even for instances where $\Delta_{\min} = 0$.
By adding forced exploration, their result also holds for SPEF which are not sub-exponential distributions.
In our work, the sub-exponential assumption is made to control the concentration towards the mean parameter.
Controlling the concentration rate is of the upmost importance to prove sufficient exploration.
Therefore, while this fact is not a direct consequence of our unified analysis, it is not surprising.

\subsection{Technicalities}
\label{app:ss_technicalities}

We present some technical results used in the above proofs.
Those technicalities are direct corollaries of properties on $\Kinf^{\pm}$ obtained in the Appendix~\ref{app:kinf_for_bounded_distributions}.

\begin{lemma} \label{lem:positive_strict_geometric_cst}
	There exists $\alpha > 0$ such that
	\begin{equation} \label{eq:def_geometric_constant}
		D_{\bm F} = \min_{(i, j) : m(F_i) > m(F_j) } \inf_{\substack{G_i,G_j : \\ \forall k \in \{i,j\}, \|G_k - F_k\|_{\infty}\leq \alpha}} \inf_{u \in [0,B]}\left\{ \Kinf^{-}(G_i , u) + \Kinf^{+}(G_j , u) \right\} > 0 \: .
	\end{equation}
\end{lemma}
\begin{proof}
Using Lemma~\ref{lem:unique_continuous_mu_star} for $w_1 = w_2 =1$, we have that
\[
\bm F \mapsto \inf_{u \in [0,B]} \left\{ \Kinf^{-}(F_i , u) + \Kinf^{+}(F_j , u) \right\}
\]
is continuous on $\cF^{K}$. Since it has strictly positive values when $m(F_i) > m(F_j)$ (Lemma~\ref{lem:inf_Kinf_restricted_to_open_mean_interval}), there exists $\alpha$ such that
\[
\inf_{\substack{G_i,G_j : \\ \forall k \in \{i,j\}, \|G_k - F_k\|_{\infty}\leq \alpha}} \inf_{u \in [0,B]}\left\{ \Kinf^{-}(G_i , u) + \Kinf^{+}(G_j , u) \right\}  > 0 \: .
\]
Further lower bounding by a finite number of strictly positive constants yields the result.
\end{proof}

For all $i \in [K]$, we define the distributions for which $i$ is among the best arm
\[
\cF^{K}_{i} \eqdef \left\{ \bm F \in \cF^{K} \mid i \in i^\star(\bm F) \right\} \: .
\]

\begin{lemma} \label{lem:continuity_results_for_analysis}
	Let $i^\star \in [K]$, $\bm F \in \cF^{K}_{i^\star}$, $i \neq i^\star$ and $\phi \in [0,1]$. Define for $ \beta \in [0,1]$,
\begin{align*}
	G_{i}(\bm F,  \beta) & = \inf_{u \in  [0,B]} \left\{   \beta \Kinf^{-}( F_{i^\star}, u)  + \phi \Kinf^{+}( F_{i}, u)  \right\}  \\
	&\quad- \sup_{w \in \simplex: w_{i^\star} =  \beta} \min_{j\neq i^\star} \inf_{u \in  [0,B]} \left\{  w_{i^\star}  \Kinf^{-}( F_{i^\star}, u)  + w_{j} \Kinf^{+}( F_{j}, u)  \right\} \: .
\end{align*}
	Then, $(\bm F,  \beta) \mapsto G_{i}(\bm F, \beta)$ is continuous on $\cF^{K} \times [0,1]$. Moreover, the function $\bm F \mapsto \inf_{\tilde \beta: |\beta - \tilde \beta| \leq \xi } G_{i}(\bm F, \tilde \beta) $ is continuous on $\cF^{K}$.

	Let $\bm \nu \in \cF^{K}_{1}$, $\bm F \in \cF^{2}$ such that $m(F_1) > m(F_2)$, $\alpha > 0$ and $\phi \in [0,1]$. Define for $ \beta \in [0,1]$,
	\begin{align*}
	H(\bm F, \beta) &= \alpha \inf_{u \in [0,B]} \left\{ \beta \Kinf^{-}( F_{1}, u)  + \phi \Kinf^{+}( F_{2}, u)  \right\}  \\
	&\quad- \sup_{w \in \simplex : w_{1} = \beta} \min_{i \neq 1} \inf_{u  \in [0,B]}\left[w_{1}\Kinf^-( \nu_{1}, u) + w_i \Kinf^+(\nu_i, u) \right] \:.
	\end{align*}
	Then, $(\bm F,  \beta) \mapsto H(\bm F, \beta)$ is continuous on $\cF^{2} \times [0,1]$. Moreover, the function $\bm F \mapsto \inf_{\tilde \beta: |\beta - \tilde \beta| \leq \xi }   H(\bm F, \tilde \beta)$ is continuous on $\cF^{2}$.
\end{lemma}
\begin{proof}
Since $\bigcup_{i \in [K]}\cF^{K}_{i} = \cF^{K}$, it is enough to show the property for all $i \in [K]$. Let $i^\star \in [K]$ and $i \neq i^\star$. In the proof of Lemma~\ref{lem:T_star_and_w_star_continuous}, we have obtained that
\begin{align*}
&(\bm F, w) \mapsto \inf_{u \in [0,B]} \left\{ w_{i^\star} \Kinf^{-}(F_{i^\star} , u) + w_i \Kinf^{+}(F_i , u) \right\}  \\
\text{and}\quad &(\bm F, w) \mapsto \min_{i \neq i^\star} \inf_{u \in [0,B]} \left\{ w_{i^\star} \Kinf^{-}(F_{i^\star} , u) + w_i \Kinf^{+}(F_i , u) \right\}
\end{align*}
are continuous on $\cF^{K}_{i^\star} \times \simplex$.

Let $\Phi_{i^\star} : (\bm F, \beta) \mapsto \{w \in \simplex \mid w_{i^\star} = \beta\}$, it is compact valued and non-empty for all $\beta \in [0,1]$. It is also continuous (both lower and upper hemicontinuous). Using the continuity proven above, Berge's theorem yields that
\[
	(\bm F, \beta) \mapsto \sup_{w \in \simplex : w_{1} = \beta} \min_{i \neq 1} \inf_{u  \in [0,B]}\left[w_{1}\Kinf^-( \nu_{1}, u) + w_i \Kinf^+(\nu_i, u) \right] \: ,
\]
is continuous on $\cF^{K}_{i^\star} \times [0,1]$. Combining with the above continuity results, we obtain that $(\bm F,  \beta) \mapsto G_{i}(\bm F, \beta)$ is continuous on $\cF^{K}_{i^\star} \times [0,1]$ for all $i^\star \in [K]$, hence on $\cF^{K}  \times [0,1]$.

Let $\Phi: \bm F \mapsto \left\{\tilde \beta: |\beta - \tilde \beta| \leq \xi \right\}$, it is a continuous (constant), compact valued and non-empty correspondence.
Using the continuity proven above, Berge's theorem yields that $\bm F \mapsto \inf_{\tilde \beta: |\beta - \tilde \beta| \leq \xi } G_{i}(\bm F, \tilde \beta) $ is continuous on $\cF^{K}$.

Using exactly the same arguments, we obtain that $(\bm F,  \beta) \mapsto H(\bm F, \beta)$ is continuous on $\cF^{2} \times [0,1]$ and $\bm F \mapsto \inf_{\tilde \beta: |\beta - \tilde \beta| \leq \xi } H(\bm F, \tilde \beta)$ is continuous on $\cF^{2}$.
\end{proof}

%% file: sections/appendix_concentration.tex

\section{Concentration}
\label{app:concentration}

In Appendix~\ref{app:ss_kinf_bounded_distributions}, we leverage results on martingales \cite{agrawal2021optimal} to prove $\delta$-correctness of the threshold (\ref{eq:def_kinf_threshold_glr}) from Lemma~\ref{lem:kinf_concentration}.
In Appendix~\ref{ssub:sub_gaussian_random_variables}, we derive technical results needed in our analysis of Top Two algorithms based on concentration for sub-Gaussian random variables.

\subsection{Calibration for bounded distributions}
\label{app:ss_kinf_bounded_distributions}

After proving Lemma~\ref{lem:calibrating_by_concentration}, we give a threshold for bounded distributions (Lemma~\ref{lem:kinf_concentration}).
The concentration is obtained as a direct corollary of recent work on martingales \citep{agrawal2021optimal}.
We apply their technical result to the case of bounded distributions.

\paragraph{Calibration by concentration}
Lemma~\ref{lem:calibrating_by_concentration} states that $\delta$-correct thresholds can be obtained by concentration results.

\begin{lemma} \label{lem:calibrating_by_concentration}
If with probability $1 - \delta$, for all $n \in \N$ and all $i \ne i^\star(\bm F)$,
\begin{equation} \label{eq:kinf_based_concentration}
N_{n, i} \Kinf^- (F_{n, i}, \mu_{i} ) + N_{n,i^\star(\bm F)} \Kinf^+ (F_{n,i^\star(\bm F)}, \mu_{i^\star(\bm F)}) )
\le c(n,\delta)
\: ,
\end{equation}
then the stopping rule (\ref{eq:def_stopping_time}) using $c(n,\delta)$ is $\delta$-correct on $\cF^K$.
\end{lemma}
\begin{proof}
Let $i^\star = i^\star(\bm F)$ and $\hat \imath_n = i^\star(\bm F_n)$. The empirical transportation costs in (\ref{eq:def_transportation_cost}) can be rewritten as
\[
	W_n(\hat \imath_n, j) = \inf_{x \leq y} \left[N_{n,\hat \imath_n}\Kinf^-(F_{n,\hat \imath_n}, x) + N_{n,j}\Kinf^+(F_{n,j},y)\right] \: .
\]
Using $j = i^\star$, $x = \mu_{i}$ and $y = \mu_{i^\star}$, we obtain
\begin{align*}
	&\bP\left(\tau_{\delta} < + \infty , \hat{\imath}_{\tau_{\delta}} \neq i^\star \right) \\
	&\leq \bP\left(\exists n \in \N, \: \exists i \neq i^\star ,\: i=\hat{\imath}_n, \: \min_{j \neq i} W_n(i, j)  > c(n,\delta)\right) \\
	&\leq \bP\left(\exists n \in \N, \: \exists i \neq i^\star , \: N_{n,i} \Kinf^- (F_{n,i}, \mu_{i} ) + N_{n,i^\star} \Kinf^+ (F_{n,i^\star}, \mu_{i^\star} )  > c(n,\delta)\right) \: .
\end{align*}
\end{proof}

\paragraph{Concentration of $\Kinf$}
The key technical result, which was extracted from \cite{agrawal2021optimal}, is reproduced in Lemma~\ref{lem:lemma_E1_agrawal2021optimal}.
\begin{lemma}[Lemma E.1 in \cite{agrawal2021optimal}] \label{lem:lemma_E1_agrawal2021optimal}
	Let a compact and convex set $\Lambda \subseteq \R^d$, and $q$ be the uniform distribution on $\Lambda$. Let $g_t : \Lambda \mapsto \R$ be any series of exp-concave functions. Then,
	\begin{align*}
		\max_{\bm\lambda \in \Lambda} \sum_{k = 1}^{n} g_k(\bm \lambda) \leq \ln \bE_{\bm \lambda \sim q} \left[ e^{\sum_{k = 1}^{n} g_t(\bm \lambda) }\right] + d \log(n+1) +1
	\end{align*}
\end{lemma}

We are now ready to prove Lemma~\ref{lem:kinf_concentration}.
\begin{proof}
For all $(n,i) \in \N \times [K]$, we denote by $(X_{k,i})_{k \in [N_{n,i}]}$ the samples collected on arm $i$.
Let $i^\star = i^\star(\bm F)$ and $i \in [K] \setminus \{i^\star\}$.
Using Theorem~\ref{thm:Kinf_duality}, we obtain
	\begin{align*}
		N_{n,i^\star} \Kinf^+ (F_{n,i^\star}, \mu_{i^\star}) ) &= \max_{\lambda \in \left[ 0, \frac{1}{B-\mu_{i^\star} }\right]} \sum_{k \in [N_{n,i^\star}]}  \ln(1 - \lambda(X_{k,i^\star} -\mu_{i^\star})) \: , \\
		N_{n,i} \Kinf^- (F_{n,i}, \mu_{i}) ) &= \max_{\lambda \in \left[ 0, \frac{1}{ \mu_{i}}\right]} \sum_{k \in [N_{n,i}]} \ln(1 + \lambda(X_{k,i} -\mu_{i})) \: .
	\end{align*}
	Let $q_{i}^+$ and $q_i^-$ be the uniform distributions over $\left[ 0, \frac{1}{B-\mu_{i} }\right]$ and $\left[ 0, \frac{1}{ \mu_{i}}\right]$, which are compact and convex sets of $\R$. Define
	\begin{align*}
		L_{n,i} &=  \bE_{\lambda \sim q_i^-} \left[ \prod_{k \in [N_{n,i}]} (1 + \lambda(X_{k,i} -\mu_{i}))  \mid X_{1,i}, \cdots, X_{N_{n,i},i} \right] \: , \\
		U_{n,i} &=  \bE_{\lambda \sim q_{i}^+}  \left[ \prod_{k \in [N_{n,i}]} (1 - \lambda(X_{i,k} -\mu_{i}))  \mid  X_{1,i}, \cdots, X_{N_{n,i},i}  \right]  \: , \\
		Y^{-}_{n,i} &= N_{n,i} \Kinf^- (F_{n,i}, \mu_{i} ) - \ln(N_{n,i} + 1) - 1 \: , \\
		Y^{+}_{n,i} &= N_{n,i} \Kinf^+ (F_{n,i}, \mu_{i} ) - \ln(N_{n,i} + 1) - 1 \: .
	\end{align*}
	With $d=1$, using Lemma~\ref{lem:lemma_E1_agrawal2021optimal} with the exp-concave functions $g_{k,i}^{+}(\lambda) = \ln(1 - \lambda(X_{k,i} -\mu_{i}))$ for $k \in [N_{n,i}]$, and $g_{k,i}^{-}(\lambda) = \ln(1 + \lambda(X_{k,i} -\mu_{i}))$  for $k \in [N_{n,i}]$, yields
	\begin{align*}
		e^{Y^{-}_{n,i}} \leq L_{n,i} \qquad \text{and} \qquad e^{Y^{+}_{n,i}} \leq U_{n,i} \qquad \text{a.s.}
	\end{align*}
	Furthermore, it is easy to verify that for each arm $i \in [K]$, $L_{n,i}$ and $U_{n,i}$ are non-negative martingales with unit initial value $L_{0,i} = 1$ and $U_{0,i} = 1$ almost surely.
	The martingale property is shown directly by the tower rule (conditioned on the arm sampled at time $n$) and $\bE \left[ 1 \pm \lambda(X_{N_{n,i},i} -\mu_{i}) \right] = 1$.
	Furthermore, they satisfy $\bE [U_{n,i}] \leq 1$ and $\bE [L_{n,i}] \leq 1$.
	Thus, $U_{n,i^\star} L_{n,i}$ is a non-negative martingale with unit initial value.

 By concavity of $\log$ and using $\sum_{j \in \{i, i^\star\}} N_{n,j} \leq n$, we have
 \[
			c(n,\delta) \geq \ln \left( \frac{K-1}{\delta}\right) + 2 + \sum_{j \in \{i, i^\star\}}\ln\left(N_{n,j} + 1\right) \: .
 \]
	Taking a union bound over $i \neq i^\star$ and using Ville's inequality, we obtain
	\begin{align*}
		&\bP\left(\exists t \in \N, \: \exists i \neq i^\star , \: N_{n,i} \Kinf^- (F_{n,i}, \mu_{i} ) + N_{n,i^\star} \Kinf^+ (F_{n,i^\star}, \mu_{i^\star} )  > c(n,\delta)\right) \\
		&\leq \sum_{i \neq i^\star} \bP\left(\exists t \in \N, \:  Y^{-}_{n,i} + Y^{+}_{n,i^\star} > \ln \left( \frac{K-1}{\delta}\right) \right)\\
		&\leq \sum_{i \neq i^\star} \bP\left(\exists t \in \N, \:  U_{n,i^\star} L_{n,i} > \frac{K-1}{\delta}\right) \leq \delta \: .
	\end{align*}
	Combining the above concentration with Lemma~\ref{lem:calibrating_by_concentration} yields the result.
\end{proof}

\subsection{Sub-Gaussian random variables}
\label{ssub:sub_gaussian_random_variables}

We want to exhibit a sub-Gaussian random variables which controls the deviation of various random variables to their means. More precisely, we will prove the existence of $W_1$ in Lemma~\ref{lem:subG_alloc} and $W_2$ in Lemma~\ref{lem:subG_cdf}.

\begin{definition}
A random variable $X$ is said to be sub-Gaussian with constant $c$ if for all $x \ge 0$, $\mathbb{P}(X \ge x) \le e^{-c x^2/2}$ and for all $x \le 0$, $\mathbb{P}(X \le x) \le e^{-c x^2/2}$.
\end{definition}

We are interested in sub-Gaussian random variable mainly due to the following property.

\begin{lemma}
If $X$ is sub-Gaussian, then for all $\lambda \in \mathbb{R}$, $\mathbb{E}[e^{\lambda X}] < \infty$.
\end{lemma}

The proof can be found in any textbook dealing with sub-Gaussian random variables, e.g. \cite{vershynin18HDP}. We will furthermore use the following classical properties:
\begin{itemize}
	\item If $X$ and $Y$ are sub-Gaussian and $\alpha \in \mathbb{R}$ then $X + Y$ is sub-Gaussian and $\alpha X$ is sub-Gaussian.
	\item Bounded random variables are sub-Gaussian.
	\item If $X$ verifies that for all $x \ge x_1 \ge 0$, $\mathbb{P}(X \ge x) \le a_1 e^{- c_1 x^2/2}$ and for all $x \le x_2 \le 0$, $\mathbb{P}(X \le x) \le a_2 e^{- c_2 x^2/2}$, then $X$ is sub-Gaussian.
	\item The maximum (or minimum) of a finite number of sub-Gaussian random variables is sub-Gaussian.
\end{itemize}

\begin{lemma}\label{lem:subG_sup_N}
If $(X_n)_{n\in \mathbb{N}, n\ge 1}$ are sub-Gaussian random variables with constants $(c_n)$ such that $\inf_n c_n > 0$, then $\sup_n \frac{X_n}{\sqrt{n\log(e+n)/(1+n)}}$ is sub-Gaussian.
\end{lemma}
\begin{proof}
For $x \ge \sqrt{\frac{8}{\inf_n c_n}}$,
\begin{align*}
\mathbb{P}(\xi \geq x) &\leq  \sum_{n=1}^{\infty} \mathbb{P}\left(X_n  \geq x \sqrt{\frac{n\log(e+n)}{n+1}}\right) \\
& \leq  \sum_{n=1}^{\infty} \exp\left(- (\inf_n c_n) x^2 {\frac{n\log(e+n)}{n+1}}\right) \\
& \leq  \sum_{n=1}^{\infty} \exp\left(- \left[2\log(e + n) + \frac{\inf_n c_n}{2}\frac{x^2n}{n+1}\right] \right) \\
& \leq  \sum_{n=1}^{\infty} \exp\left(- \left[2\log(e + n) + \frac{\inf_n c_n}{4}x^2\right] \right) \\
& =  \left[\sum_{n=1}^{\infty} \frac{1}{(e+n)^2}\right] e^{-\frac{\inf_n c_n}{4}x^2} \: ,
\end{align*}
where we have used that $\frac{n}{n+1} \geq \tfrac{1}{2}$ and that $\alpha \beta \geq \alpha + \beta$ for $\alpha,\beta \geq 2$. Now for the lower tail, for $x \le 0$,
\begin{align*}
\mathbb{P}(\xi \le x)
\le \mathbb{P}\left(X_1 \le x \sqrt{\frac{\log(1+e)}{2}}\right)
\le \exp\left( - \frac{\log(1+e)}{2}c_1 x^2/2\right)
\: .
\end{align*}
\end{proof}

\paragraph{Application to our work} We will use the following two examples of sub-Gaussian variables. Lemma~\ref{lem:subG_DKW} is a consequence of the Dvoretzky-Kiefer-Wolfowitz (DKW) inequality \cite{massart1990} while Lemma~\ref{lem:subG_martingale} follows from Azuma's inequality.

\begin{lemma}\label{lem:subG_DKW}
Let $(X_n)_{n\ge 1}$ be i.i.d.\ random variables with cdf $F$ and let $F_n$ be the empirical distribution function of $(X_i)_{i \in [n]}$. Then for all $n$, $\sqrt{n}\Vert F_n - F\Vert_\infty$ is sub-Gaussian with a constant which does not depend on $n$.
\end{lemma}


\begin{lemma}\label{lem:subG_martingale}
Let $(X_n)_{n\ge 1}$ be a martingale with mean 0 and $c$-sub-Gaussian increments and $\mu_n = X_n / n$. Then for all $n$, $\sqrt{n}\mu_n$ is sub-Gaussian with constant $c$.
\end{lemma}

Theses results permit to establish the concentration results that are used in Appendix~\ref{app:unified_analysis_top_two} and~\ref{app:top_two_instances}.

\begin{lemma*}[Lemma~\ref{lem:subG_alloc} and~\ref{lem:subG_cdf}]
	There exists sub-Gaussian random variables $W_1$ and $W_2$ such that for all $(n,i) \in \N \times [K]$ with $n\ge K+1$ (such that all arms are pulled at least once)
\begin{align*}
	&|N_{n,i}-\Psi_{n,i}| \leq W_1\sqrt{(n+1)\log(e+n)} \quad \text{a.s.} \: ,\\
	&\left\|{F}_{n,i} - F_{i}\right\|_{\infty} \leq W_2 \sqrt{\frac{\log(e+N_{n,i})}{1+N_{n,i}}} \quad \text{a.s.} \:.
\end{align*}
	In particular, $\bE \left[e^{\lambda W_i}\right] < + \infty$ for all $\lambda > 0$ and $i\in \{1,2\}$.
\end{lemma*}

\begin{proof} For the first inequality, we use that $N_{n,i} - \Psi_{n,i}$ is a martingale and combine Lemma~\ref{lem:subG_martingale} with Lemma~\ref{lem:subG_sup_N} to get that for all $i \in [K]$\[W_{1,i} \eqdef \sup_{n \ge 1} \frac{|N_{n,i}-\Psi_{n,i}| / \sqrt{n}}{\sqrt{n\log(e+n) / (1 + n)}}\]
is sub-Gaussian. Therefore $W_1 = \max_{i \in [K]} W_{1,i}$ is sub-Gaussian and we have, for all $(n,i)$,
\[|N_{n,i}-\Psi_{n,i}| \leq W_1 {\sqrt{\frac{n^2}{n+1}\log(e+n)}} \leq W_1 \sqrt{(n+1)\log(e+n)}.\]
For the second inequality, we define $W_{2,i} \eqdef \sup_{n\ge K+1} \left\|{F}_{n,i} - F_{i}\right\|_{\infty} \sqrt{\frac{1+N_{n,i}}{\log(e+N_{n,i})}}$. Letting $\hat{F}_{n,i}$ be the empirical distribution fo the first $n$ samples from arm $i$ (while $F_{n,i}$ is the empirical distribution of the samples collected up to time $n$), one can rewrite
\[W_{2,i} = \sup_{n\ge 1} \left\|\hat{F}_{n,i} - F_{i}\right\|_{\infty} \sqrt{\frac{1+n}{\log(e+n)}} = \sup_{n\ge 1}  \frac{\sqrt{n}\left\|\hat{F}_{n,i} - F_{i}\right\|_{\infty}}{\sqrt{n \log(e+n) / (1+n)}}.\]
Combining Lemma~\ref{lem:subG_DKW} with Lemma~\ref{lem:subG_sup_N} yields that $W_{2,i}$ is sub-Gaussian for all $i\in [K]$. Then $W_2 = \max_{i \in [K]} W_{2,i}$ is sub-Gaussian and we have, for all $(n,i)$,
\[ \left\|{F}_{n,i} - F_{i}\right\|_{\infty} \leq W_1\sqrt{\frac{1+N_{n,i}}{\log(e+N_{n,i})}}.\]
\end{proof}

%% file: sections/appendix_kinf_bounded.tex

\section{Kinf for bounded distributions}
\label{app:kinf_for_bounded_distributions}

Here, $\cF$ is the set of probability distributions with support in the interval $[0,B]$.
The goal of this section is to study the properties of $\Kinf^+$ and $\Kinf^-$, which are functions $\cF \times [0,B] \to \mathbb{R}_+$ defined by
\begin{align*}
\Kinf^+(F, \mu)
&= \inf\{\KL(F,G) \mid G \in \cF, \mathbb{E}_G[X] > \mu\}
\:, \\
\Kinf^-(F, \mu)
&= \inf\{\KL(F,G) \mid G \in \cF, \mathbb{E}_G[X] < \mu\}
\: .
\end{align*}

As a first step, we remark that for $\mu \in (0,B)$ we can rewrite the $\Kinf$ functions using non-strict inequalities, which will be more convenient \cite{HondaTakemura10,garivier2018kl}. We do so and will work in this section on
\begin{align*}
\Kinf^+(F, \mu)
&= \inf\{\KL(F,G) \mid G \in \cF, \mathbb{E}_G[X] \ge \mu\}
\:, \\
\Kinf^-(F, \mu)
&= \inf\{\KL(F,G) \mid G \in \cF, \mathbb{E}_G[X] \le \mu\}
\: .
\end{align*}
There is a strong link between these two definitions, which we will use to transport results from one function to the other.

\begin{lemma}\label{lem:Kinf+_symm_eq_Kinf-}
Let $F \in \cF$, $\mu \in [0,B]$ and let $f:[0,B] \to [0,B]$ be defined by $f(x) = B - x$.
Let $F^{B-X}$ be the pushforward measure of $F$ through $f$. Then
\begin{align*}
\Kinf^+(F^{B-X}, B - \mu) = \Kinf^-(F, \mu) \quad \text{ and } \quad \Kinf^-(F^{B-X}, B - \mu) = \Kinf^+(F, \mu) \: .
\end{align*}
\end{lemma}
\begin{proof}
The function $f$ is measurable, bijective and involutive. We have $\KL(F, G) = \KL(F^{B-X}, G^{B-X})$ for all $F,G \in \cF$.
\begin{align*}
\Kinf^+(F^{B-X}, B - \mu)
&= \inf\{\KL(F^{B-X},G) \mid G \in \cF, \mathbb{E}_G[X] \ge B - \mu\}
\\
&= \inf\{\KL(F^{B-X},G) \mid G \in \cF, \mathbb{E}_{G^{B-X}}[X] \le \mu\}
\\
&= \inf\{\KL(F^{B-X},(G^{B-X})^{B-X}) \mid G \in \cF, \mathbb{E}_{G^{B-X}}[X] \le \mu\}
\\
&= \inf\{\KL(F, G^{B-X}) \mid G \in \cF, \mathbb{E}_{G^{B-X}}[X] \le \mu\}
\\
&= \inf\{\KL(F, G) \mid G \in \cF, \mathbb{E}_{G}[X] \le \mu\}
\\
&= \Kinf^-(F, \mu)
\: .
\end{align*}
\end{proof}

Let $\mathcal F^+(\mu) = \{G \in \cF \mid \mathbb{E}_G[X] \ge \mu\}$ and define $\mathcal F^-(\mu)$ similarly.

\begin{lemma}\label{lem:F+_nonempty_compact_convex}
For all $\mu \in [0,B]$, $\mathcal F^+(\mu)$ is a nonempty compact convex set (for the weak convergence of measures).
\end{lemma}
\begin{proof}
It is nonempty since the Dirac distribution at $B$ belongs to the set.

The set $\cF$ is compact, hence a sequence of distributions in $\mathcal F^+(\mu)$ admits a convergent subsequence. Suppose then that we have a convergent sequence $(F_n)_{n\in \mathbb{N}}$, converging to $F$, and let's show that $F \in \mathcal F^+(\mu)$. We can rewrite
$
\mathcal F^+(\mu) = \{G \in \cF \mid \mathbb{E}_G[\max\{X, B\}] \ge \mu\} \: .
$
This is useful since the function $x \mapsto \max\{x, B\}$ is bounded from above and continuous. We can thus apply the Portmanteau theorem to write
\begin{align*}
\mathbb{E}_F[\max\{X, B\}] \ge \limsup_n \mathbb{E}_{F_n}[\max\{X, B\}] \ge \mu \: .
\end{align*}
We conclude that $F \in \mathcal F^+(\mu)$, which is then compact.

To prove convexity, let $F,G \in \mathcal F^+(\mu)$ and let $\alpha \in [0,1]$: $\alpha F + (1 - \alpha)G \in \cF$ and
\begin{align*}
\mathbb{E}_{\alpha F + (1 - \alpha)G}[X] = \alpha \mathbb{E}_F[X] + (1 - \alpha) \mathbb{E}_G[X] \ge \alpha \mu + (1 - \alpha)\mu = \mu \: .
\end{align*}
\end{proof}

\begin{lemma}\label{lem:F+_upper_hemicontinuous}
$\mu \mapsto \mathcal F^+(\mu)$ is an upper hemicontinuous correspondence.
\end{lemma}
\begin{proof}
Since $\mathcal F^+$ is a compact-valued correspondence (Lemma~\ref{lem:F+_nonempty_compact_convex}), it suffices to show that for all sequences $(\mu_n)$ and $(Q_n)$ with $Q_n \in \mathcal F^+(\mu_n)$, if $\mu_n \to \mu$ then there exists a convergent subsequence of $(Q_n)$, which converges to $Q \in \mathcal F^+(\mu)$ \cite[Proposition 9.8, p. 231]{sundaram1996first}. The existence of a convergent subsequence comes from the compactness of $\cF$. The limit $Q$ then belongs to $\cF$ since $\cF$ is closed. We need to show that $Q$ belongs to $\{G \mid \mathbb{E}_G[X] \ge \mu\}$.

We can rewrite $\mathcal F^+(\mu) = \{G \in \cF \mid \mathbb{E}_G[\max\{X, B\}] \ge \mu\}$ and prove that $Q$ belongs to $\{G \mid \mathbb{E}_G[\max\{X, B\}] \ge \mu\}$.
This is useful since the function $x \mapsto \max\{x, B\}$ is bounded from above and continuous. We can thus apply the Portmanteau theorem to write
\begin{align*}
\mathbb{E}_Q[\max\{X, B\}] \ge \limsup_n \mathbb{E}_{Q_n}[\max\{X, B\}] \ge \limsup_n \mu_n = \mu \: .
\end{align*}
We conclude that $Q \in \mathcal F^+(\mu)$. We have proved upper hemicontinuity.
\end{proof}

\begin{lemma}\label{lem:inf_attained_in_Kinf}
The infimum in the definition of $\Kinf^+$ is attained at a distribution in $\mathcal F^+(\mu)$.
\end{lemma}
\begin{proof}
$G \mapsto \KL(F,G)$ is lower semicontinuous wrt the topology of weak convergence of measures and the set $\mathcal F^+(\mu)$ over which the minimization is performed is compact, hence the functions attains its infimum at a point in $\mathcal F^+(\mu)$.
\end{proof}

\subsection{Duality}
\label{sub:duality}

For $(\lambda, F, u) \in \mathbb{R}_+ \times \cF \times [0,B]$, let $H^+(\lambda, F, u) = \mathbb{E}_F[\log (1 - \lambda (X - u))]$, where we define $\log(x) = -\infty$ for $x \le 0$. Let $H^-(\lambda, F, u) = \mathbb{E}_F[\log (1 + \lambda (X - u))]$.

\begin{theorem}\label{thm:Kinf_duality}
For all $F \in \cF$ and $u \in [0,B]$,
\begin{align*}
\Kinf^+(F, u)
&= \sup_{\lambda \in [0, (B-u)^{-1}]} H^+(\lambda, F, u)
\:, \\
\Kinf^-(F, u)
&= \sup_{\lambda \in [0, u^{-1}]} H^-(\lambda, F, u)
\:.
\end{align*}
\end{theorem}
\begin{proof}
A proof of this statement for $\Kinf^+$ can be found in any one of \cite{HondaTakemura10,garivier2018kl}. The result for $\Kinf^-$ then follows from Lemma~\ref{lem:Kinf+_symm_eq_Kinf-}.
\end{proof}

\begin{lemma}[\cite{HondaTakemura10}, Lemma 14]\label{lem:Kinf_finite_distribution_coarse_upper_bound}
	For all $(F,u) \in \cF \times [0, B)$, $\Kinf^{+}(F, u) \leq -\log\left(1 - \frac{u}{B}\right)$.
\end{lemma}
\begin{proof}
Let $(F,u) \in \cF \times [0, B)$. The proof relies on Theorem~\ref{thm:Kinf_duality} and $X \geq 0$ for all $X \in \supp(F)$. Using that $\log$ is increasing on $(0,+\infty)$ and $\bE_{F}[1]=1$,
\begin{align*}
	\Kinf^{+}(F, u) &= \sup_{\lambda \in [0, (B-u)^{-1}]} H^+(\lambda, F, u) \leq \sup_{\lambda \in [0, (B-u)^{-1}]}  \log \left(1 + \lambda u \right) = -\log\left(1 - \frac{u}{B}\right) \: .
\end{align*}
\end{proof}

\subsection{Continuity and differentiability}
\label{sub:continuity}

\begin{lemma}\label{lem:H_usc}
The function $\lambda, F, u \mapsto H^+(\lambda, F, u)$ is upper semicontinuous (jointly in all arguments) on $\mathbb{R}_+ \times \cF \times [0,B]$.
\end{lemma}

\begin{proof}
Let $(\lambda_n, F_n, u_n) \in \mathbb{R}_+ \times \cF \times [0,B]$ be a sequence converging to $(\lambda, F, u)$. We want to prove that
\begin{align*}
\limsup \mathbb{E}_{F_n}[\log (1 - \lambda_n (X - u_n))] \le \mathbb{E}_{F}[\log (1 - \lambda (X - u))] \: .
\end{align*}

By Skorokhod's representation theorem, there exists real random variables $(X_n)_{n\in \mathbb{N}}, X$ defined on a common probability space $(\Omega, \mathcal A, \mathbb{P})$ such that the law of $X_n$ is $F_n$ for all $n \in \mathbb{N}$, the law of $X$ is $F$ and $(X_n)$ converges to $X$ almost surely (hence also in probability).

The family $(X_n)$ has compact support, hence it is uniformly integrable.
By Vitali's theorem, since $X_n$ is uniformly integrable and converges in probability, it also converges in $L^1$.

We get that $\lambda_n(X_n - u_n) \xrightarrow{a.s.} \lambda(X - u)$ and $\mathbb{E}[\lambda_n(X_n - u_n)] \to \mathbb{E}[\lambda(X - u)]$. Now we want to translate this into a statement about the log.

Since $-\lambda_n(X_n - u_n) - \log(1 - \lambda_n(X_n - u_n)) \ge 0$, by Fatou's lemma,
\begin{align*}
& \mathbb{E}\left[-\lambda(X - u)\right] - \mathbb{E}\left[\limsup \log(1 -\lambda_n(X_n - u_n))\right]
\\
&= \mathbb{E}\left[\liminf \left( -\lambda_n(X_n - u_n) - \log(1 - \lambda_n(X_n - u_n)) \right)\right]
\\
&\le \liminf \left( \mathbb{E}\left[-\lambda_n(X_n - u_n) - \log(1 - \lambda_n(X_n - u_n)) \right]\right)
\\
&= \mathbb{E}\left[-\lambda(X - u)\right] - \limsup\mathbb{E}\left[\log(1 - \lambda_n(X_n - u_n)) \right]
\: .
\end{align*}
Canceling the first term, we get the inequality we were after.
\end{proof}

\begin{theorem}\label{thm:Kinf_continuous}
The function $\Kinf^+$ (resp. $\Kinf^-$) is continuous on $\cF \times [0,B)$ (resp. $\cF \times (0,B]$).
\end{theorem}

\begin{proof}
We follow the proof method of \cite{agrawal2021optimal} (which applied to a slightly different setting).

We first prove lower semicontinuity. We want to apply Berge's Maximum Theorem \cite[Theorem 2, p. 116]{berge1997topological} to the correspondence $C(F,u) = \mathcal F^+(u)$ and the function $f((F,u), G) = -\KL(F,G)$. We will obtain the upper semicontinuity of $f^*(F,u) := \inf\{f((F,u),G) \mid G \in C(F,u)\} = -\inf\{\KL(F, G) \mid G \in \mathcal F^+(u)\}$, which gives us the lower semicontinuity we are after. We need to show that
\begin{itemize}
	\item $F,u \mapsto C(F,u) = \mathcal F^+(u)$ is upper hemicontinuous: this is proved in Lemma~\ref{lem:F+_upper_hemicontinuous},
	\item $F,u,G \mapsto f(F,u,G) = -\KL(F,G)$ is upper semicontinuous (jointly in all arguments): this is true since $\KL$ is jointly lower semicontinuous \cite{posner1975random}.
\end{itemize}
We have lower semicontinuity on $\cF \times [0,B]$.

To prove upper semicontinuity, we first use duality (Theorem~\ref{thm:Kinf_duality}) to write $\Kinf^+(F, u) = \sup_{\lambda \in [0, (B-u)^{-1}]} H^+(\lambda, F, u)$. Since we want to prove semicontinuity on $\cF \times [0,B)$, we can take any $\varepsilon > 0$ and prove it on $\cF \times [0,B-\varepsilon]$.

We want to apply Berge's Maximum Theorem \cite[Theorem 2, p. 116]{berge1997topological} to the correspondence $C(F,u) = \mathcal [0, (B-u)^{-1}]$ and the function $f((F,u), \lambda) = \mathbb{E}_F[\log (1 - \lambda (X - u))]$. We will obtain the upper semicontinuity of $f^*(F,u) = \sup\{f((F,u),G) \mid G \in C(F,u)\}$, which is exactly what we are after. We need to show that
\begin{itemize}
	\item $F,u \mapsto C(F,u) = [0, (B-u)^{-1}]$ is upper hemicontinuous, nonempty and compact,
	\item $F,u,\lambda \mapsto f(F,u,\lambda) = \mathbb{E}_F[\log (1 - \lambda (X - u))]$ is upper semicontinuous (jointly in all arguments): this is true by Lemma~\ref{lem:H_usc}.
\end{itemize}
For the first point, nonempty compact values are obvious for $u \le B-\varepsilon$. The upper hemicontinuity comes from the continuity of the upper bound of the interval.
\end{proof}

\begin{lemma}\label{lem:H_strict_concave}
$\lambda \mapsto H^+(\lambda, F, u)$ is strictly concave on $[0, (B-u)^{-1})$.
\end{lemma}
\begin{proof}
For $\alpha \in (0,1)$, by strict concavity of the logarithm,
\begin{align*}
H^+(\alpha\lambda + (1 - \alpha) \eta, F, u)
&= \mathbb{E}_F[\log (\alpha (1 - \lambda (X - u)) + (1 - \alpha) (1 - \eta (X - u)))]
\\
&> \mathbb{E}_F[\alpha \log (1 - \lambda (X - u)) + (1 - \alpha) \log (1 - \eta (X - u))]
\\
&= \alpha H^+(\lambda, F, u) + (1 - \alpha) H^+(\eta, F, u)
\: .
\end{align*}
The restriction to the interval $[0, (B-u)^{-1})$ guarantees that all quantities appearing in logarithms above are finite.
\end{proof}

\begin{lemma}\label{lem:H_continuous}
$(\lambda, u) \mapsto H^+(\lambda, F, u)$ is continuous on $\{(\lambda, u) \in \mathbb{R}_+ \times [0,B] \mid \lambda < (B-u)^{-1}\}$.
\end{lemma}
\begin{proof}
We already have upper semicontinuity by Lemma~\ref{lem:H_usc}. We only need lower semicontinuity. That is, we need that for $\lambda_n \to \lambda$ and $u_n \to u$ in that set, we have
\begin{align*}
\liminf_n \mathbb{E}_F[\log (1 - \lambda_n (X - u_n))] \ge \mathbb{E}_F[\log (1 - \lambda (X - u))] \: .
\end{align*}
There exists $\varepsilon > 0$ such that $\lambda \le (B-u)^{-1}(1 - \varepsilon)$. Then for $n$ big enough, $\lambda_n \le (B-u_n)^{-1}(1 - \varepsilon/2)$.
For all $n$ large enough, we get $\log (1 - \lambda_n (X - u_n)) - \log (\varepsilon/2) \ge 0$. By Fatou's lemma,
\begin{align*}
\liminf_n \mathbb{E}_F[\log (1 - \lambda_n (X - u_n)) - \log (\varepsilon/2)] \ge \mathbb{E}_F[\log (1 - \lambda (X - u)) - \log (\varepsilon/2)] \: .
\end{align*}
We cancel the $\log (\varepsilon/2)$ term and get the lower semicontinuity.
\end{proof}

Let $\lambda_\star^+(F, u) = \argmax_{\lambda \in [0, (B-u)^{-1}]} \mathbb{E}_F[\log (1 - \lambda (X - u))] = \argmax_{\lambda \in [0, (B-u)^{-1}]} H^+(\lambda, F, u)$ and $\lambda_\star^-(F, u) = \argmax_{\lambda \in [0, u^{-1}]} H^-(\lambda, F, u)$.

\begin{lemma}\label{lem:lambda_star_properties}
$u \mapsto \lambda_\star^+(F, u)$ is continuous over the set $\{u \in [0,B) \mid \lambda_\star^+(F, u) < (B-u)^{-1}\}$.
\end{lemma}

\begin{proof}
We first show that for any $\varepsilon \in (0, B^{-1}]$, the function $u \mapsto \argmax_{\lambda \in [0, (B-u)^{-1} - \varepsilon]} H^+(\lambda, F, u)$ is continuous on $[0,B)$ and the argmax is unique. This is not exactly continuity of $u \mapsto \lambda_\star^+(F, u)$ because of the $[0, (B-u)^{-1} - \varepsilon]$ interval instead of $[0, (B-u)^{-1}]$.

We will apply Berge's Maximum theorem \cite[page 116]{berge1997topological}. For $\varepsilon \in (0, B^{-1}]$, let
\begin{align*}
\phi(\lambda, u)
&= H(\lambda, F, u)
\:, \\
\Gamma(u)
&= [0, (B-u)^{-1} - \varepsilon]
\:, \\
M(u)
&= \max \{H(\lambda, u) \mid \lambda \in \Gamma(u)\}
\:, \\
\Phi(u)
&= \argmax \{\phi(\lambda, u) \mid \lambda \in \Gamma(u)\}
\: .
\end{align*}
We verify the hypotheses of the theorem for any $\varepsilon'>0$ and $u < B-\varepsilon'$:
\begin{itemize}
	\item $H$ is continuous on $\{(\lambda, u) \in \mathbb{R}_+ \times [0,B] \mid \lambda < (B-u)^{-1}\}$, by Lemma~\ref{lem:H_continuous}, which is a domain containing $\Gamma(u) \times \{u\}$ for all $u$.
	\item $\Gamma$ is nonempty (since $(B-u)^{-1} - \varepsilon \ge 0$), compact-valued (since $u \le B - \varepsilon'$) and continuous.
\end{itemize}
We obtain that $M$ is continuous on $[0,B)$ and that $\Phi$ is upper hemicontinuous.

Now since $\phi$ is a strictly concave function of $\lambda$ (by Lemma~\ref{lem:H_strict_concave}) and $\Gamma$ is convex, we can argue as in \cite[Theorem 9.17]{sundaram1996first} to prove that $\Phi$ is a single-valued upper hemicontinuous correspondence, hence a continuous function.

Now that we have proved the continuity of the argmax restricted to the interval $[0, (B-u)^{-1} - \varepsilon]$, let's prove the continuity of $u \mapsto \lambda_\star^+(F, u)$ over the set $\{u \in [0,B) \mid \lambda_\star^+(F, u) < (B-u)^{-1}\}$.

let $u \in [0,B)$ such that $\lambda_\star^+(F, u) < (B-u)^{-1}$. Then there exists $\varepsilon > 0$ such that
\begin{align*}
\lambda_\star^+(F, u) &= \argmax_{\lambda \in [0, (B-u)^{-1} - \varepsilon]} H^+(\lambda, F, u) \: ,
\\
\text{and }\lambda_\star^+(F, u) &\le (B-u)^{-1} - 3\varepsilon \: .
\end{align*}

Remark that by concavity of $H$ in $\lambda$, for all $u$, if $\argmax_{\lambda \in [0, (B-u)^{-1} - \varepsilon]} H^+(\lambda, F, u) \ne (B-u)^{-1} - \varepsilon$ then $\lambda_\star^+(F, u) = \argmax_{\lambda \in [0, (B-u)^{-1} - \varepsilon]} H^+(\lambda, F, u)$.

For $v$ in some neighborhood of $u$, we have both $(B-v)^{-1} - \varepsilon > (B-u)^{-1} - 2\varepsilon$ and $\argmax_{\lambda \in [0, (B-v)^{-1} - \varepsilon]} H^+(\lambda, F, v) < (B-u)^{-1} - 2\varepsilon$. This means that for all $v$ in that neighborhood, $\lambda_\star^+(F, v) = \argmax_{\lambda \in [0, (B-v)^{-1} - \varepsilon]} H^+(\lambda, F, v)$. The continuity of the $\epsilon$ version then gives continuity of $\lambda_\star^+$ at $u$.
\end{proof}

\begin{lemma}[Theorem 6 in \cite{HondaTakemura10}] \label{lem:differentiability_Kinf}
	For all $F \in \cF$ and $u \in (m(F),B]$, $u \mapsto \Kinf^+(F,u)$ is differentiable and
	\begin{equation} \label{eq:differentiating_Kinf+}
		\frac{\partial \Kinf^{+}(F,u)}{\partial u } = \lambda_{\star}^+(F,u) \: .
	\end{equation}
	For all $F \in \cF$ and $u \in [0,m(F))$, $u \mapsto \Kinf^-(F,u)$ is differentiable and
	\begin{equation} \label{eq:differentiating_Kinf-}
		\frac{\partial \Kinf^{-}(F,u)}{\partial u } = - \lambda_{\star}^{-}(F,u) \: .
	\end{equation}
\end{lemma}

\subsection{Convexity}
\label{sub:convexity}

\begin{lemma}
The functions $\Kinf^+$ and $\Kinf^-$ are jointly convex on $\cF \times [0,B]$.
\end{lemma}

\begin{proof}
We prove the result for $\Kinf^+$, but the proof for $\Kinf^-$ is identical.
Let $F_1, F_2 \in \cF$, $u_1, u_2 \in [0,B]$ and let $G_1, G_2 \in \cF$ be distributions at which the infimum is attained in $\Kinf^+(F_1, u_1)$ and $\Kinf^+(F_2, u_2)$ respectively (which exist by Lemma~\ref{lem:inf_attained_in_Kinf}). For all $\alpha \in [0,1]$, $\alpha G_1 + (1 - \alpha)G_2$ has expectation $\alpha u_1 + (1 - \alpha) u_2$. Hence for all $\alpha \in [0,1]$,
\begin{align*}
\Kinf^+(\alpha F_1 + (1 - \alpha) F_2, \alpha u_1 + (1 - \alpha) u_2)
&\le \KL(\alpha F_1 + (1 - \alpha) F_2, \alpha G_1 + (1 - \alpha) G_2)
\\
&\le \alpha \KL(F_1, G_1) + (1 - \alpha)\KL(F_2, G_2)
\\
&= \alpha \Kinf^+(F_1, u_1) + (1 - \alpha)\Kinf^+(F_2, u_2)
\end{align*}
The first inequality follows from the definition of $\Kinf^+$ as an infimum and the second inequality comes from the joint convexity of the Kullback-Leibler divergence. We have proved joint convexity.
\end{proof}

\begin{lemma}[Theorem~5 in \cite{HondaTakemura10}] \label{lem:explicit_implicit_solutions}
Let $F \in \cF$ and $u^{+}(F) = B - \frac{1}{\bE_{F}\left[ \frac{1}{B-X}\right]} \geq m(F)$. We have
\begin{equation} \label{eq:cdt_null_dual_variable}
	\lambda_{\star}^{+}(F,u) = 0 \: \iff \: u \leq m(F) \: ,
\end{equation}
\begin{equation} \label{eq:implicit_def_dual_variable}
	u \in (m(F), u^{+}(F)]  \: \implies \: \bE_{F} \left[ \frac{1}{1- \lambda_{\star}^{+}(F,u)(X - u)}\right] = 1 \: ,
\end{equation}
and
\begin{equation} \label{eq:explicit_def_dual_variable}
	\lambda_{\star}^{+}(F,u) = \frac{1}{B-u} \: \iff \: u \geq u^{+}(F) \: .
\end{equation}
\end{lemma}
\begin{proof}
First, we have
\begin{equation*}
	\lambda_{\star}^{+}(F,u) = 0 \: \iff \: u \leq m(F) \: .
\end{equation*}
If $\lambda_{\star}^{+}(F,u) = 0$, then $\Kinf^{+}(F, u) = H^{+}(0,F,u) = 0$, hence $u \leq m(F)$. The other direction is obtained in Theorem~5 in \cite{HondaTakemura10}.

Let $u^{+}(F) = B - \frac{1}{\bE_{F}\left[ \frac{1}{B-X}\right]}$ and $u^{-}(F) = \frac{1}{\bE_{F}\left[ \frac{1}{X}\right]}$. Then,
\begin{align*}
	u \in (m(F), u^{+}(F)]  \: \implies \: \bE_{F} \left[ \frac{1}{1- \lambda_{\star}^{+}(F,u)(X - u)}\right] = 1 \: ,
\end{align*}

Moreover,
\begin{equation*}
	\lambda_{\star}^{+}(F,u) = \frac{1}{B-u} \: \iff \: u \geq u^{+}(F) \: .
\end{equation*}
If $u \geq u^{+}(F)$, then $\bE_{F}\left[ \frac{B-u}{B-X}\right] \leq 1$, hence $\lambda_{\star}^{+}(F,u) = \frac{1}{B-u}$ by Theorem~5 in \cite{HondaTakemura10}. Assume that there exists $u < u^{+}(F)$, such that $\lambda_{\star}^{+}(F,u) = \frac{1}{B-u}$. Then, by equation (\ref{eq:implicit_def_dual_variable}), we obtain that $1 = \bE_{F} \left[ \frac{1}{1- \frac{X - u}{B-u}}\right] = \bE_{F}\left[ \frac{B-u}{B-X}\right] $. This condition can be rewritten as $ u = u(F)$, hence contradicting $u < u^{+}(F)$. Therefore, we have shown $\lambda_{\star}^{+}(F,u) = \frac{1}{B-u}$ implies that $u \geq u^{+}(F)$.
\end{proof}

\begin{lemma} \label{lem:strict_convexity_Kinf}
The function $u \mapsto \Kinf^+(F, u)$ is strictly convex on $(m(F), B]$. The function $u \mapsto \Kinf^-(F, u)$ is strictly convex on $[0,m(F))$.
\end{lemma}

\begin{proof}
We prove the result for $\Kinf^+$. The proof for $\Kinf^-$ is similar.

Using Lemma~\ref{lem:differentiability_Kinf}, $u \mapsto \Kinf^{+}(F,u)$ is strictly convex for $u > m(F)$ if and only if $\lambda_{\star}^{+}(F,u)$ is increasing for $u > m(F)$.

Using (\ref{eq:cdt_null_dual_variable}), $\lambda_{\star}^{+}(F,u)$ is null for $u \leq m(F)$. Therefore, $u \mapsto \Kinf^{\pm}(F,u)$ is not strictly convex on those intervals.

Using (\ref{eq:explicit_def_dual_variable}), we obtain directly that $\lambda_{\star}^{+}(F,u)$ is increasing for $u \geq u^{+}(F)$.

Suppose towards contradiction that $u \mapsto \lambda_{\star}^{+}(F,u)$ is not increasing for $(m(F), u^{+}(F))$. Therefore, there exists an open $\cO \subseteq (m(F), u^{+}(F))$, such that $u \mapsto \lambda_{\star}^{+}(F,u)$ is constant on $\cO$, i.e. there exists $c_{\cO} \in \left[0,\frac{1}{B-\inf_{u \in \cO} u}\right]$ such that $\lambda_{\star}^{+}(F,u) = c_{\cO}$. Using (\ref{eq:cdt_null_dual_variable}-\ref{eq:explicit_def_dual_variable}), we know that $c_{\cO} \in \left(0,\frac{1}{B-\inf_{u \in \cO} u}\right)$. On $\cO$, $u \mapsto \lambda_{\star}^{+}(F,u)$ is constant, hence it is continuously differentiable with null derivative. Since $\cO \subseteq (m(F), u^{+}(F))$, (\ref{eq:implicit_def_dual_variable}) defines implicitly $\lambda_{\star}^{+}(F,u)$ as satisfying
\[
	 \bE_{F} \left[ \frac{1}{1- \lambda_{\star}^{+}(F,u)(X - u)}\right] = 1 	\: .
\]
Since $\cO \subseteq (m(F), u^{+}(F))$, we have $\lambda_{\star}^{+}(F,u) \in \left(0,\frac{1}{B-\inf_{u \in \cO} u}\right)$. Therefore, the function $(u, x) \mapsto \frac{1}{1- \lambda_{\star}^{+}(F,u)(x - u)}$ is bounded on $[0,B] \times \cO$, hence integrable, and the function $u \mapsto \frac{1}{1- \lambda_{\star}^{+}(F,u)(x - u)}$ is continuously differentiable. Moreover, the function $x \mapsto \frac{1}{\left(1- \lambda_{\star}^{+}(F,u)(x - u)\right)^2}$ is strictly positive and bounded on $[0,B]$, hence integrable with strictly positive integrable. Having checked all the conditions to interchange the derivative with the expectation, differentiating the above yields
\begin{align*}
	0 = \bE_{F} \left[ - \frac{ \lambda_{\star}^{+}(F,u) + (u-X) \frac{\partial \lambda_{\star}^{+}(F,u)}{\partial u}}{\left(1 - (X - u)\lambda_{\star}^{+}(F,u) \right)^2}\right] = - c_{\cO} \bE_{F} \left[  \frac{1}{\left(1 - (X - u)c_{\cO} \right)^2}\right] < 0 \: ,
\end{align*}
where the strict inequality is obtained since we show that $c_{\cO} > 0$ and $\bE_{F} \left[  \frac{1}{\left(1 - (X - u)c_{\cO} \right)^2}\right] > 0$. This is a contradiction, hence such $\cO \subset (m(F), u^{+}(F))$ doesn't exist. Therefore, $u \mapsto \lambda_{\star}^{+}(F,u)$ is increasing on $(m(F), u^{+}(F))$.

Since the convexity already gave that $u \mapsto \lambda_{\star}^{+}(F,u)$ is increasing on $(m(F), B]$. The fact that $u \mapsto \lambda_{\star}^{+}(F,u)$ is increasing on $(m(F), u^{+}(F))$ and on $[u^{+}(F), B]$, yields that $u \mapsto \lambda_{\star}^{+}(F,u)$ is increasing on $(m(F), B]$.
\end{proof}

\begin{lemma}\label{lem:Kinf_increasing}
$u \mapsto \Kinf^+(F, u)$ is equal to zero on $[0, m(F)]$ and increasing on $(m(F), B]$.
\end{lemma}
\begin{proof}
We already proved that $\Kinf^+(F, u)$ is equal to zero on $[0, m(F)]$. Since $\Kinf^+(F, m(F)) = 0$, $\Kinf^+$ is nonnegative and strictly convex for $u > m(F)$, then $u \mapsto \Kinf^+(F, u)$ is increasing on $(m(F), B]$.
\end{proof}

\begin{lemma}\label{lem:strict_convex_sum_Kinf}
For all $F,G \in \cF$ with means $m(F) \le m(G)$ and for all $w \in \mathbb{R}_+^2$.
\begin{itemize}
	\item If $\max\{w_1, w_2\}>0$, then $\mu \mapsto w_1 \Kinf^+(F, \mu) + w_2 \Kinf^-(G, \mu)$ is strictly convex on $[m(F),m(G)]$.
	\item If $\min\{w_1, w_2\}>0$, then $\mu \mapsto w_1 \Kinf^+(F, \mu) + w_2 \Kinf^-(G, \mu)$ is strictly convex on $[0,B]$.
\end{itemize}

\end{lemma}

\begin{proof}
For $\mu \le m(F)$, the function is equal to $w_2 \Kinf^-(G, \mu)$, which is strictly convex unless $w_2 = 0$. For $\mu \ge m(G)$, the function is equal to $w_1 \Kinf^+(F, \mu)$, which is strictly convex unless $w_1 = 0$. In the interval $(m(F), m(G))$, it is the sum of two convex functions, one of which is strictly convex. Furthermore, the function is continuous at $m(F)$ and $m(G)$.
\end{proof}

\subsubsection{More continuity, using convexity}
\label{ssub:more_continuity_using_convexity}

\begin{lemma}\label{lem:inf_Kinf_restricted_to_closed_mean_interval}
Let $F,G \in \cF$ with means $m(F) \le m(G)$ in $(0,B)$ and $w \in \mathbb{R}_+^2$. Then
\begin{align*}
\inf_{\mu \in [0,B]}(w_1 \Kinf^+(F, \mu) + w_2 \Kinf^-(G, \mu))
= \inf_{\mu \in [m(F),m(G)]}(w_1 \Kinf^+(F, \mu) + w_2 \Kinf^-(G, \mu))
\: .
\end{align*}
\end{lemma}
\begin{proof}
On $[0, m(F)]$, $\Kinf^+(F, \mu)$ is constant equal to 0 and $\Kinf^-(G, \mu)$ is non-increasing, hence the minimum over that interval is attained at $m(F)$. We argue similarly for the interval $[m(G), B]$.
\end{proof}

\begin{lemma}\label{lem:unique_continuous_mu_star}
Let $F,G \in \cF$ with means in $(0,B)$ and $w \in \mathbb{R}_+^2$. Then
\begin{enumerate}
	\item $(F, G, w) \mapsto \inf_{\mu \in [0,B]}(w_1 \Kinf^+(F, \mu) + w_2 \Kinf^-(G, \mu))$ is continuous on $\cF \times \cF \times \mathbb{R}_+^2$.
	\item If $\max\{w_1, w_2\} > 0$, $\mu_\star(F, G, w) = \argmin_{\mu \in [0,B]}(w_1 \Kinf^+(F, \mu) + w_2 \Kinf^-(G, \mu))$ is unique and continuous on $\cF \times \cF \times \mathbb{R}_+^2$.
\end{enumerate}
\end{lemma}
\begin{proof}
We can restrict the inf to $[\mu_F, \mu_G]$ by Lemma~\ref{lem:inf_Kinf_restricted_to_closed_mean_interval}.

We will apply Berge's Maximum theorem \cite[page 116]{berge1997topological}. Let
\begin{align*}
\phi(\mu, F, G, w)
&= - w_1 \Kinf^+(F, \mu) - w_2 \Kinf^-(G, \mu)
\:, \\
\Gamma(F, G, w)
&= [\mu_F, \mu_G]
\:, \\
M(F, G, w)
&= \max \{\phi(\mu, F, G, w) \mid \mu \in \Gamma(F, G, w)\}
\:, \\
\Phi(F, G, w)
&= \argmax \{\phi(\mu, F, G, w) \mid \mu \in \Gamma(F, G, w)\}
\: .
\end{align*}
We verify the hypotheses of the theorem:
\begin{itemize}
	\item $\phi$ is continuous on $[\mu_F, \mu_G] \times \cF \times \cF \times C$, by Theorem~\ref{thm:Kinf_continuous} since $\mu_F, \mu_G \in (0,B)$.
	\item $\Gamma$ is nonempty, compact-valued and continuous (since constant).
\end{itemize}
We obtain that $M$ is continuous on $\cF \times \cF \times \mathbb{R}_+^2$ and that $\Phi$ is upper hemicontinuous.

Now since $\phi$ is a strictly concave function of $\mu$ (by Lemma~\ref{lem:strict_convex_sum_Kinf}) and $\Gamma$ is convex, we can argue as in \cite[Theorem 9.17]{sundaram1996first} to prove that $\Phi$ is a single-valued upper hemicontinuous correspondence, hence a continuous function.
\end{proof}

\begin{lemma}\label{lem:inf_Kinf_restricted_to_open_mean_interval}
Let $F,G \in \cF$ with means $m(F) < m(G)$ in $(0,B)$ and $w \in \mathbb{R}_+^2$ such that $\min\{w_1, w_2\} > 0$.
The value $\mu_\star(F, G, w) = \argmin_{\mu \in [0,B]}(w_1 \Kinf^+(F, \mu) + w_2 \Kinf^-(G, \mu))$ (unique by Lemma~\ref{lem:unique_continuous_mu_star}) belongs to the interval $(m(F), m(G))$ and is such that both $\Kinf$ are positive.
\end{lemma}
\begin{proof}
We know by Lemma~\ref{lem:inf_Kinf_restricted_to_closed_mean_interval} that the minimum is attained inside $[m(F), m(G)]$. We only need to exclude the boundaries.

The function $\mu \mapsto w_1 \Kinf^+(F, \mu) + w_2 \Kinf^-(G, \mu)$ is differentiable on $(m(F), m(G))$ with derivative $w_1 \lambda_\star^+(F, \mu) - w_2 \lambda_\star^-(G, \mu)$. If we show that this derivatives takes the value zero in the open interval, then we prove the result.

For $\mu > m(F)$ in a neighborhood of $m(F)$, $\lambda_\star^+(F, \mu)$ tends to 0 by continuity (Lemma~\ref{lem:lambda_star_properties}) and $\lambda_\star^-(G, \mu) > \lambda_\star^-(G, \frac{m(F)+m(G)}{2}) > 0$. We get that close to $m(F)$, the derivative is negative. Similarly, we get that the derivative is positive close to $m(G)$. We conclude that the infimum is indeed attained inside the open interval.
\end{proof}

\begin{lemma}\label{lem:inf_Kinf_increasing_in_w}
Let $F,G \in \cF$ with means $m(F) < m(G)$ in $(0,B)$ and $w_1 > 0$. Then $w_2 \mapsto \min_{\mu \in [0,B]}(w_1 \Kinf^+(F, \mu) + w_2 \Kinf^-(G, \mu))$ is increasing on $\mathbb{R}_+$.
\end{lemma}
\begin{proof}
Let $w_2' > w_2 \ge 0$. Then by Lemma~\ref{lem:inf_Kinf_restricted_to_open_mean_interval}, since $w_2'>0$, there exists $\mu' \in [0,B]$ with $\Kinf^-(G, \mu') > 0$ such that $\min_{\mu \in [0,B]}(w_1 \Kinf^+(F, \mu) + w_2' \Kinf^-(G, \mu)) = w_1 \Kinf^+(F, \mu') + w_2' \Kinf^-(G, \mu')$. Then we have
\begin{align*}
\min_{\mu \in [0,B]}(w_1 \Kinf^+(F, \mu) + w_2' \Kinf^-(G, \mu))
&= w_1 \Kinf^+(F, \mu') + w_2' \Kinf^-(G, \mu')
\\
&> w_1 \Kinf^+(F, \mu') + w_2 \Kinf^-(G, \mu')
\\
&\ge \min_{\mu \in [0,B]}(w_1 \Kinf^+(F, \mu) + w_2 \Kinf^-(G, \mu))
\: .
\end{align*}
\end{proof}

\begin{lemma}\label{lem:inf_Kinf_concave_in_w}
Let $F,G, F_1, \ldots, F_K \in \cF$ with means in $(0,B)$.

The function $w \mapsto \min_{\mu \in [0,B]}(w_1 \Kinf^+(F, \mu) + w_2 \Kinf^-(G, \mu))$ is concave on $\mathbb{R}_+^2$.

The function $w \mapsto \min_{j \ne 1}\min_{\mu \in [0,B]}(w_1 \Kinf^-(F_1, \mu) + w_j \Kinf^+(F_j, \mu))$ is concave on $\mathbb{R}_+^K$.
\end{lemma}
\begin{proof}
These functions are minimums of linear functions, hence concave.
\end{proof}

\subsection{Properties of the characteristic time}
\label{sub:properties_of_the_characteristic_time}

Let $\bm F \in \cF^K$ and $i^\star = i^\star(\bm F)$, supposed unique.
Recall that
\begin{align*}
T^\star(\bm F)^{-1}
&= \sup_{w \in \simplex} \min_{i \neq i^\star} \inf_{u \in [0,B]} \left\{ w_{i^\star} \Kinf^{-}(F_{i^\star} , u) + w_i \Kinf^{+}(F_i , u) \right\}
\: , \\
w^\star(\bm F)
&= \argmax_{w \in \simplex} \min_{i \neq i^\star} \inf_{u \in [0,B]} \left\{ w_{i^\star} \Kinf^{-}(F_{i^\star} , u) + w_i \Kinf^{+}(F_i , u) \right\}
\: , \\
T_{\beta}^\star(\bm F )^{-1}
&= \sup_{\substack{w \in \simplex \\ w_{i^\star} = \beta}}  \min_{i \neq i^\star} \inf_{u \in [0,B]} \left\{ w_{i^\star} \Kinf^{-}(F_{i^\star} , u) + w_i \Kinf^{+}(F_i , u) \right\}
\: , \\
w_{\beta}^\star(\bm F )
&= \argmax_{\substack{w \in \simplex \\ w_{i^\star} = \beta}}  \min_{i \neq i^\star} \inf_{u \in [0,B]} \left\{ w_{i^\star} \Kinf^{-}(F_{i^\star} , u) + w_i \Kinf^{+}(F_i , u) \right\}
\: .
\end{align*}

\begin{lemma}\label{lem:T_star_and_w_star_continuous}
The functions ${T^\star}^{-1}$ and ${T^\star_\beta}^{-1}$ are continuous on $\cF^{K}$. The correspondences $w^\star$ and $w^\star_{\beta}$ are upper hemicontinuous on $\cF^{K}$ with compact convex values.
\end{lemma}
\begin{proof}
Let $\cF^{K}_{i} = \left\{ \bm F \in \cF^{K} \mid i \in i^\star(\bm F) \right\}$. Since $\bigcup_{i \in [K]}\cF^{K}_{i} = \cF^{K}$, it is enough to show the property for all $\cF^{K}_{i}$ for $i \in [K]$. Let $i^\star \in [K]$.

First, the function $(w, \bm F) \mapsto \min_{i \neq i^\star} \inf_{u \in [0,B]} \left\{ w_{i^\star} \Kinf^{-}(F_{i^\star} , u) + w_i \Kinf^{+}(F_{i^\star} , u) \right\}$ is continuous on $\simplex \times \mathcal F^K$ by Lemma~\ref{lem:unique_continuous_mu_star} and the fact that a minimum of continuous functions is continuous. It is concave in $w$ by Lemma~\ref{lem:inf_Kinf_concave_in_w}.

The correspondence $(w, \bm F) \mapsto \simplex$ is nonempty compact-valued and continuous (since constant). By Berge's maximum theorem, we get that $T^\star(\bm F)^{-1}$ is continuous on $\cF^{K}_{i^\star}$ and that $w^\star(\bm F)$ is upper hemicontinuous with compact values. By \cite[Theorem 9.17]{sundaram1996first}, the concavity of the function being maximized implies that $w^\star(\bm F)$ is convex-valued.

The correspondence $(w, \bm F) \mapsto \simplex \cap \{w_{i^\star} = \beta\}$ is nonempty compact-valued and continuous (since constant). By Berge's maximum theorem, we get that $T_{\beta}^\star(\bm F)^{-1}$ is continuous on $\cF^{K}_{i^\star}$ and that $w^\star_\beta(\bm F)$ is upper hemicontinuous with compact values. By \cite[Theorem 9.17]{sundaram1996first}, the concavity of the function being maximized implies that $w_\beta^\star(\bm F)$ is convex-valued.
\end{proof}

\begin{lemma}\label{lem:complexity_ne_zero_of_unique_i_star}
If $i^\star(\bm F)$ is a singleton and $\beta \in (0,1)$, then $T^\star(\bm F)^{-1} > 0$ and $T_\beta^\star(\bm F)^{-1} > 0$.
\end{lemma}
\begin{proof}
\begin{align*}
T^\star(\bm F)^{-1}
&= \sup_{w \in \simplex} \min_{i \neq i^\star} \inf_{u \in [0,B]} \left\{ w_{i^\star} \Kinf^{-}(F_{i^\star} , u) + w_i \Kinf^{+}(F_i , u) \right\}
\\
&\ge \min_{i \neq i^\star} \inf_{u \in [0,B]} \left\{ \frac{1}{K} \Kinf^{-}(F_{i^\star} , u) + \frac{1}{K} \Kinf^{+}(F_i , u) \right\}
> 0
\: ,
\end{align*}
since we proved that the inner infimum is positive for nonzero coefficients and $\mu_i < \mu_{i^\star}$. The proof for $T_\beta^\star$ is similar.
\end{proof}

\begin{lemma}\label{lem:w_star_positive}
If $i^\star(\bm F)$ is a singleton and $\beta \in (0,1)$, then for all $i \in [K]$, $w^\star_i(\bm F) > 0$ and $w^\star_{\beta,i}(\bm F) > 0$.
\end{lemma}
\begin{proof}
We proceed by contradiction. If $i^\star(\bm F)$ is unique and there exists $j$ with $w^\star_{j}(\bm F) = 0$, then we show $T^\star(\bm F)^{-1} = 0$, which is absurd by Lemma~\ref{lem:complexity_ne_zero_of_unique_i_star}. If $j = i^\star$ we have
\begin{align*}
T^\star(\bm F)^{-1}
= \min_{i \neq i^\star} \inf_{u \in [0,B]} w_i^\star \Kinf^{+}(F_i , u)
\le \min_{i \neq i^\star} w_i^\star \Kinf^{+}(F_i , F_i)
= 0
\: .
\end{align*}
If $j\ne i^\star$,
\begin{align*}
T^\star(\bm F)^{-1}
&= \min_{i \neq i^\star} \inf_{u \in [0,B]} \left\{ w_{i^\star}^\star \Kinf^{-}(F_{i^\star} , u) + w_i^\star \Kinf^{+}(F_i , u) \right\} \\
&\le \inf_{u \in [0,B]} \left\{ w_{i^\star}^\star \Kinf^{-}(F_{i^\star} , u) + w_j^\star \Kinf^{+}(F_j , u) \right\}
\\
&= \inf_{u \in [0,B]}  w_{i^\star}^\star \Kinf^{-}(F_{i^\star} , u)
= 0
\: .
\end{align*}
A similar proof holds for $T_\beta^\star$.
\end{proof}

\begin{lemma}\label{lem:properties_characteristic_times}
If $i^\star(\bm F)$ is a singleton and $\beta \in (0,1)$, then
\begin{itemize}
	\item for all $i \ne i^\star(\bm F)$, $\inf_{u \in [0,B]} \left\{ w_{i^\star}^\star(\bm F) \Kinf^{-}(F_{i^\star} , u) + w_i^\star(\bm F) \Kinf^{+}(F_i , u) \right\} = T^\star(\bm F)^{-1}$~,
	\item for all $i \ne i^\star(\bm F)$, $\inf_{u \in [0,B]} \left\{ \beta \Kinf^{-}(F_{i^\star} , u) + w_{\beta,i}^\star(\bm F) \Kinf^{+}(F_i , u) \right\} = T_\beta^\star(\bm F)^{-1}$ ,
	\item $w^\star(\bm F)$ and $w_\beta^\star(\bm F)$ are singletons: the optimal allocations are unique.
\end{itemize}
\end{lemma}
\begin{proof}
At the optimal allocations, all $w_i$ are positive. Suppose w.l.o.g. that 1 is the best arm. By dividing by $w_1$ and defining
\[
G_{j}(x) = \inf_{u \in [0,B]} \left( \Kinf^{-}(F_{i^\star} , u) + x \Kinf^{+}(F_j , u) \right) \: ,
\]
we obtain directly that
\begin{align*}
	T^\star(\bm F)^{-1} = \max_{w \in \simplex, w_1 > 0} w_1 \min_{j \neq 1} G_{j}\left( \frac{w_j}{w_1}\right) \: .
\end{align*}

Let $w^\star \in w^\star(\bm F)$. Then, using the above result, we obtain
\begin{align*}
	w^\star \in \argmax_{w \in \simplex} w_1 \min_{j \neq 1} G_{j}\left( \frac{w_j}{w_1}\right)
\end{align*}
Introducing $x_{j}^{\star}=\frac{w^\star_{j}}{w_{1}^{\star}}$ for all $j \neq 1$, using that $\sum_{j\in [K]}w_j^\star = 1$, one has
\begin{align*}
w_{1}^{\star}=\frac{1}{1+\sum_{j=2}^{K} x_{j}^{\star}} \quad \text { and, for } j \geq 2, w^\star_j=\frac{x_{j}^{\star}}{1+\sum_{j=2}^{K} x_{j}^{\star}} \: .
\end{align*}
If $x^\star$ is unique, then so is $w^\star$.

Since it is optimal, $\{x_{j}^{\star}\}_{j=2}^{K} \in \mathbb{R}^{K-1}$ belongs to
\begin{equation} \label{eq:reformulation_optimization_with_xs}
	\argmax_{\{x_{j}\}_{j=2}^{K} \in \mathbb{R}^{K-1}} \frac{\min_{j \neq 1} G_{j}\left(x_{j}\right)}{1+\sum_{j=2}^{K} x_{j}}
\end{equation}

Let's show that all the $G_{j}\left(x_{j}^{\star}\right)$ have to be equal. Let $\cO =\left\{ i \in [K]\setminus \{1\} \mid G_{i}\left(x_{i}^{\star}\right)=\min _{j \neq 1} G_{j}\left(x_{j}^{\star}\right)\right\}$ and $\cA= [K]\setminus (\{1\}\cup \cO )$. Assume that $\cA \neq \emptyset$. For all $a \in \cA$ and $b \in \cO$, one has $G_{j}\left(x_{j}^{\star}\right)>G_{i}\left(x_{i}^{\star}\right)$. Using the continuity of the $G_j$ functions and the fact that they are increasing (Lemma~\ref{lem:inf_Kinf_increasing_in_w}), there exists $\epsilon>0$ such that
\begin{align*}
\forall j \in \cA, i \in \cO, \quad G_{j}\left(x_{j}^{\star}-\epsilon /|\cA|\right)>G_{i}\left(x_{i}^{\star}+\epsilon /|\cO|\right)>G_{i}\left(x_{i}^{\star}\right) \: .
\end{align*}
We introduce $\bar{x}_{j}=x_{j}^{\star}-\epsilon /|\cA|$ for all $j \in \cA$ and $\bar{x}_{i}=x_{i}^{\star}+\epsilon /|\cO|$ for all $i \in \cO$, hence $\sum_{j=2}^{K} \bar{x}_{j} = \sum_{j=2}^{K}x_{j}^\star$. There exists $i \in \cO$ such that $\min _{j \neq 1} G_{j}\left(\bar{x}_{j}\right) = G_{i}\left(x_{i}^{\star}+\epsilon /|\cO|\right)$, hence
\begin{align*}
\frac{\min_{j \neq 1} G_{j}\left(\bar{x}_{j}\right)}{1+\bar{x}_{2}+\ldots \bar{x}_{K}}=\frac{G_{i}\left(x_{i}^{\star}+\epsilon /|\cO|\right)}{1+x_{2}^{\star}+\cdots+x_{K}^{\star}}>\frac{G_{i}\left(x_{i}^{\star}\right)}{1+x_{2}^{\star}+\cdots+x_{K}^{\star}}=\frac{\min _{j \neq 1} G_{j}\left(x_{j}^{\star}\right)}{1+x_{2}^{\star}+\cdots+x_{K}^{\star}} \: .
\end{align*}
This is a contradiction with the fact that $x^\star$ belongs to (\ref{eq:reformulation_optimization_with_xs}). Therefore, we have $\cA = \emptyset$.

We have proved that there is a unique value by $y^\star \in \R_{+}$, such that for all $j \neq 1$, $G_{j}\left(x_{j}^{\star}\right) = y^\star$. Now since $G_j$ is increasing, this defines a unique value for $x_j^\star$, equal to $G_j^{-1}(y^\star)$.

For $y$ in the intersection of the ranges of all $G_j$, let $x_j(y) = G_j^{-1}(y)$. $y^\star$ belongs to
\begin{equation} \label{eq:reformulation_optimization_with_y}
	\argmax_{y \in \left[0 , \min_{j \neq 1} \lim_{+\infty} G_{j}(x) \right)} \frac{y}{1+\sum_{j \neq 1} x_{j}(y)} \: .
\end{equation}

For $\beta \in (0,1)$, the same results (and proof) hold for $w_{\beta}^\star(\bm F)$ by noting that
\begin{align*}
	T^\star_{\beta}(\bm F)^{-1} = \max_{w \in \simplex : w_1 = \beta} \beta \min_{j \neq 1} G_{j}\left( \frac{w_j}{\beta}\right) \: .
\end{align*}

Let $w^{\beta} \in w_{\beta}^\star(\bm F)$, since we have equality at the equilibrium, we obtain for all $j \neq 1$,
\begin{align*}
	\beta G_{j}\left( \frac{w^{\beta}_j}{\beta}\right) = T^\star_{\beta}(\bm F)^{-1} \: ,
\end{align*}
Using the inverse mapping $x_j$, we obtain for all $j \neq 1$,
\begin{align*}
	w^{\beta}_j = \beta x_{j} \left( \frac{1}{T^\star_{\beta}(\bm F) \beta} \right) \: .
\end{align*}

Therefore, we have shown that $w_{\beta}^\star(\bm F) = \{w^{\beta}\}$, where
\begin{align*}
	w^{\beta}_{i} = \begin{cases}
		\beta x_{i} \left( \frac{1}{T^\star_{\beta}(\bm F) \beta} \right) &\text{if } i\neq i^\star \\
		\beta &\text{else}
	\end{cases} \: .
\end{align*}
\end{proof}

\begin{lemma} \label{lem:robust_beta_optimality}
$T^\star_{1/2}(\bm F)  \leq 2 T^\star(\bm F)$ and with $\beta^\star = w^\star_{i^\star}(\bm F)$,
\begin{align*}
	\frac{T^\star(\bm F)^{-1}}{T^\star_{\beta}(\bm F)^{-1}} \leq \max \left\{\frac{\beta^\star}{\beta}, \frac{1-\beta^\star}{1-\beta} \right\}
	\: .
\end{align*}
\end{lemma}

\begin{proof}
Define for each non-negative vector $\bm\psi \in \mathbb{R}_+^K$,
\begin{align*}
f(\bm\psi)
\eqdef \min _{i \neq i^\star(\bm F)} \inf_{u \in [0,B]} \left\{ \psi_{i^\star} \Kinf^{-}(F_{i^\star} , u) + \psi_i \Kinf^{+}(F_i , u) \right\}
\: .
\end{align*}

$T^\star(\bm F)^{-1}$ is the maximum of $f(\bm\psi)$ over probability vectors $\bm\psi$. Here, we instead define $f$ for all non-negative vectors, and proceed by varying the total budget of measurement effort available $\sum_{a \in [K]} \psi_{a}$.
$f$ is non-decreasing in $\psi_i$ for all $i$. $f$ is homogeneous of degree $1$. That is $f(c \bm\psi)=c f(\bm\psi)$ for all $c \geq 1$. For each $c_{1}$, $c_{2}>0$ define
$$
g\left(c_{1}, c_{2}\right)=\max \left\{f(\bm \psi) \mid \bm\psi \in \mathbb{R}_+^K, \: \psi_{i^\star(\bm F)}=c_{1}, \sum_{i \neq i^\star(\bm F)} \psi_{i} \leq c_{2}, \right\}
$$
The function $g$ inherits key properties of $f$; it is also non-decreasing and homogeneous of degree 1 . We have
$$
\begin{aligned}
T^\star_{\beta}(\bm F)^{-1} &=\max \left\{f(\bm \psi)\mid \bm\psi \in \mathbb{R}_+^K, \: \psi_{i^\star(\bm F)}=\beta, \sum_{i \in [K]} \psi_{i}=1 \right\} \\
&=\max \left\{f(\bm \psi)\mid \bm\psi \in \mathbb{R}_+^K, \: \psi_{i^\star(\bm F)}=\beta, \sum_{i \neq i^\star(\bm F)} \psi_{i} \leq 1-\beta \right\} \\
&=g(\beta, 1-\beta)
\end{aligned}
$$
where the second equality uses that $f$ is non-decreasing. Similarly, $T^\star(\bm F)^{-1}=g\left(\beta^{\star}, 1-\beta^{\star}\right)$ where $\beta^{\star} = w^\star_{i^\star(\bm F)}(\bm F)$. Setting
$$
r:=\max \left\{\frac{\beta^{\star}}{\beta}, \frac{1-\beta^{\star}}{1-\beta}\right\}
$$
implies $r \beta \geq \beta^{\star}$ and $r(1-\beta) \geq 1-\beta^{\star}$. Therefore
$$
r T^\star_{\beta}(\bm F)^{-1} =r g(\beta, 1-\beta)=g(r \beta, r(1-\beta)) \geq g\left(\beta^{\star}, 1-\beta^{\star}\right)=T^\star(\bm F)^{-1} \: .
$$
Taking $\beta = \frac{1}{2}$, yields that $T^\star(\bm F)^{-1} \leq 2\max\{\beta^\star, 1-\beta^\star\}T^\star_{1/2}(\bm F)^{-1} \leq 2 T^\star_{1/2}(\bm F)^{-1}$.
\end{proof}

\begin{lemma} \label{lem:funny_property_for_optimal_allocation_algorithm}
	Let $G_{i}(x) = \inf_{u \in [0,B]} \left( \Kinf^{-}(F_{i^\star} , u) + x \Kinf^{+}(F_i , u) \right)$ for $x \in [0,+\infty)$ and $i \neq i^\star$. Then,
	\[
	\lim_{x \to +\infty} G_{i}(x) = \Kinf^{-}(F_{i^\star} , m(F_{i})) \: .
	\]
\end{lemma}
\begin{proof}
	Let $u_{i}(x) \in \argmin_{u \in [0,B]} \left( \Kinf^{-}(F_{i^\star} , u) + x \Kinf^{+}(F_i , u) \right)$.
	It is easy to see that $G_{i}(0) = 0$ and $u_{i}(0) = m(F_1)$.
	Likewise, we have $u_i(x) =_{+ \infty} m(F_i) + o(1)$ by considering $(w_{i^\star},w_{i})$ instead of $x_i = \frac{w_i}{w_{i^\star}}$.
	By continuity of $u \mapsto \Kinf^{-}(F_{i^\star} , u)$ and using the definition of $u_i(x)$
\[
	\lim_{x \to +\infty} G_{i}(x) = \Kinf^{-}(F_{i^\star} , m(F_{i})) + \lim_{x \to +\infty}  x \Kinf^{+}(F_{i} , u_i(x))  \: .
\]

Using the deviations bounds on the $u \mapsto \Kinf^{-}(F,u)$ (e.g. Lemma 6 in \cite{honda2011asymptotically}), we obtain that
\[
	0 < x \Kinf^{-}(F_{i} , u_{i}(x)) \leq x \frac{(m(F_i) - u_i(x))^2}{2} \: .
\]
Therefore, a sufficient condition to obtain $\lim_{x \to +\infty}  x \Kinf^{+}(F_{i} , u_i(x)) = 0$ is to show that $u_i(x) =_{+ \infty} m(F_i) + o\left( \frac{1}{\sqrt{x}}\right)$. The first order condition of optimality on $u_{i}(x)$ can be expressed as
\begin{align*}
	 x \frac{\partial \Kinf^{+}(F_i, u_{i}(x))}{\partial u} = - \frac{\partial \Kinf^{-}(F_{i^\star}, u_{i}(x))}{\partial u} \iff \quad  &  x \lambda_{\star}^{+}(F_i, u_{i}(x)) = \lambda_{\star}^{-}(F_{i^\star}, u_{i}(x)) \: ,
\end{align*}
where we used Lemma~\ref{lem:differentiability_Kinf} for the equivalent formulation.

Using that $u \mapsto \lambda_{\star}^{-}(F, u)$ is decreasing for $u< m(F_{i^\star})$ (Lemma~\ref{lem:strict_convexity_Kinf} and Lemma~\ref{lem:differentiability_Kinf}) yields
\begin{align*}
	x \lambda_{\star}^{+}(F_i, u_{i}(x)) = \lambda_{\star}^{-}(F_{i^\star}, u_{i}(x))  \leq  \lambda_{\star}^{-}(F_{i^\star}, m(F_i)) \leq \frac{1}{m(F_i)} \: .
\end{align*}

Using that $\lambda_{\star}^{+}(F_i, u) \geq \frac{u-m(F_i)}{u(B-u)}$ (Lemma~12 in \cite{HondaTakemura10}) and denoting $y(x) =  u_i(x) - m(F_i) =_{+ \infty} o(1) $, we obtain
\[
	\frac{1}{m(F_i)} \geq x \lambda_{\star}^{+}(F_i, u_i(x)) \geq  \frac{x y(x)}{(m(F_i) + y(x))(B - m(F_i) - y(x))} \: .
\]

Suppose towards contradiction that $y(x)= \cO (\frac{1}{x})$ doesn't hold, i.e. $\lim_{+ \infty} x y(x) = + \infty$. Using that $y(x) =_{+ \infty} o(1)$ and taking the limit in the above inequality yields
\[
	\frac{1}{m(F_i)} \geq \frac{\lim_{+ \infty} x y(x)}{m(F_i)(B-m(F_i))} = + \infty \: ,
\]
which is a direct contradiction. Therefore, we have shown that $y(x)= \cO (\frac{1}{x})$. We showed above that a sufficient condition to conclude was $y(x) =_{+ \infty} o\left( \frac{1}{\sqrt{x}}\right)$. Therefore, we have obtained that $\lim_{+\infty} x \Kinf^{-}(F_{i} , u_{i}(x)) = 0$, which concludes the proof.
\end{proof}

%% file: sections/appendix_bcp.tex

\section{Boundary crossing probability bounds}
\label{app:boundary_crossing_probability_bounds}

In order to analyze the algorithms presented in this paper, we need to quantify probabilities of the form $\mathbb{P}(\theta_1 \ge \theta_2)$ for $\theta_1$ and $\theta_2$ two independent real random variables. We first show how such bounds can be obtained by quantifying the individual deviations $\mathbb{P}(\theta_i \ge u)$ and $\mathbb{P}(\theta_i \le u)$ for all $u \in \mathbb{R}$ (the so-called Boundary Crossing Probabilities). Then we prove upper and lower bounds on those probabilities when $\theta_1$ and $\theta_2$ are obtained from a Dirichlet sampler.

\subsection{From one arm to two}
\label{sub:from_one_arm_to_two}

As remarked in Appendix~\ref{app:sss_how_to_sample}, studying BAI randomized algorithms require to control probability of the form $\mathbb{P}(\theta_1 \ge \theta_2)$ where $\theta_1$ and $\theta_2$ are two independent real random variables.
Thanks to Lemma~\ref{lem:from_bcp_one_to_bcp_two}, it is possible to obtain those by using Boundary Crossing Probability (BCP) bounds, which are extensively studied in the regret minimization literature.
Therefore, while it is based on simple calculations, Lemma~\ref{lem:from_bcp_one_to_bcp_two} is a powerful result of independent interest.

\begin{lemma}\label{lem:from_bcp_one_to_bcp_two}
Let $\theta_1$ and $\theta_2$ be two independent real random variables with cdf $F_1$ and $F_2$. Let $x \in \argmax_{u \in \mathbb{R}}\bP(\theta_2 \geq u) \bP(\theta_1 \leq u)$. Then
\begin{align*}
\bP(\theta_2 \geq x) \bP(\theta_1 \leq x)
\le \mathbb{P}(\theta_2 \ge \theta_1)
&\le g\big(\bP(\theta_2 \geq x) \bP(\theta_1 \leq x)\big) 
\: .
\end{align*}
where $g(u) = u(1-\log(u))$ for all $u \in [0,1]$.
\end{lemma}

\begin{proof} To ease the notation, we introduce the cdfs $F_1(u) = \bP(\theta_1 \leq u)$ and $F_2(u) = \bP(\theta_2\leq u)$. We can suppose that there exists $u \in \mathbb{R}$ with $(1 - F_2(u))F_1(u) > 0$. Otherwise the probability of $\theta_2 \ge \theta_1$ is 0, and both bounds are 0 as well.
We start by proving the upper bound.
\begin{align*}
\mathbb{P}(\theta_2 \ge \theta_1)
&= \int_u \int_v \ind\{u \ge v\} dF_1(v) dF_2(u)
\\
&= \int_{u \le x} \int_{v \le x} \ind\{u \ge v\} dF_1(v) dF_2(u)
	+ \int_{u \le x} \int_{v > x} \ind\{u \ge v\} dF_1(v) dF_2(u)
\\&\quad+ \int_{u > x} \int_{v \le x} \ind\{u \ge v\} dF_1(v) dF_2(u)
	+ \int_{u > x} \int_{v > x} \ind\{u \ge v\} dF_1(v) dF_2(u)
\end{align*}
The second of those four integrals is equal to zero. We now bound integrals 1, 3, and 4.

\emph{1.}
For $x$ such that $F_2(x) < 1$,
\begin{align*}
\int_{u \le x} \int_{v \le x} \ind\{u \ge v\} dF_1(v) dF_2(u)
&= \int_{u \le x} F_1(u) dF_2(u)
\\
&= \int_{u \le x} \frac{1}{1 - F_2(u)} (1 - F_2(u))F_1(u) dF_2(u)
\\
&\le \left( \sup_{u \le x} (1 - F_2(u))F_1(u)\right) \int_{u \le x} \frac{1}{1 - F_2(u)} dF_2(u)
\\
&= - \log (1 - F_2(x)) \sup_{u \le x} (1 - F_2(u))F_1(u)
\: .
\end{align*}

\emph{3.}
\begin{align*}
\int_{u > x} \int_{v \le x} \ind\{u \ge v\} dF_1(v) dF_2(u)
&= F_1(x)(1 - F_2(x))
\: .
\end{align*}

\emph{4.}
For $x$ such that $F_1(x)>0$,
\begin{align*}
\int_{u > x} \int_{v > x} \ind\{u \ge v\} dF_1(v) dF_2(u)
&= \int_{v > x} (1 - F_2(v)) dF_1(v)
\\
&= \int_{v > x} F_1(v)(1 - F_2(v)) \frac{1}{F_1(v)}dF_1(v)
\\
&\le \left( \sup_{v>x} F_1(v)(1 - F_2(v)) \right) \int_{v > x} \frac{1}{F_1(v)}dF_1(v)
\\
&= - \log(F_1(x)) \sup_{v>x} F_1(v)(1 - F_2(v))
\: .
\end{align*}

Putting things together:
\begin{align*}
\mathbb{P}(\theta_2 \ge \theta_1)
&\le - \log (1 - F_2(x)) \sup_{u \le x} (1 - F_2(u))F_1(u)
	+ F_1(x)(1 - F_2(x))
\\&\quad - \log(F_1(x)) \sup_{v>x} F_1(v)(1 - F_2(v))
\: .
\end{align*}

Taking for $x$ the argmax over $\mathbb{R}$ (which verifies $F_1(x)>0$ and $F_2(x)<1$), we get
\begin{align*}
\mathbb{P}(\theta_2 \ge \theta_1) &\le (1 - F_2(x))F_1(x) \left[ 1 - \log((1 - F_2(x))F_1(x)) \right]
\end{align*}

We now prove the lower bound. For $x \in \mathbb{R}$, by independence of $\theta_1$ and $\theta_2$,
\begin{align*}
\mathbb{P}(\theta_2 \ge \theta_1)
&\ge \mathbb{P}(\theta_2 \ge x \ge \theta_1)
= \mathbb{P}(\theta_2 \ge x) \mathbb{P}(\theta_1 \le x)
= (1 - F_2(x)) F_1(x) \: .
\end{align*}
\end{proof}

\subsection{Upper bounds}
\label{sub:upper_bound}

Theorem~\ref{thm:upper_bound_one_arm_bcp_bounded} gives a tight upper bound on the BCP.

\begin{theorem} \label{thm:upper_bound_one_arm_bcp_bounded}
Let $X = (X_1, \ldots, X_n) \in [0,B]^{n}$, let $\hat{F}_n$ be the corresponding empirical distribution and let $\mu \in \mathbb{R}$. Then
\begin{align*}
\mathbb{P}_{L \sim \mathrm{Dir}(1^{n})}(L^\top X \ge \mu)
&\le \exp \left( - n \Kinf^+(\hat{F}_n, \mu) \right)
\: , \\
\mathbb{P}_{L \sim \mathrm{Dir}(1^{n})}(L^\top X \le \mu)
&\le \exp \left( - n \Kinf^-(\hat{F}_n, \mu) \right)
\: .
\end{align*}
\end{theorem}

\begin{proof}
We first prove the bound involving $\Kinf^+$. This proof is extracted from the proof of Lemma 15 of \cite{RiouHonda20}.

Let $R_1, \ldots, R_n$ be independent exponential random variables with parameter 1.
\begin{align*}
\mathbb{P}_{L \sim \mathrm{Dir}(1^{n})}(L^\top X \ge \mu)
&= \mathbb{P}(\sum_{i=1}^n\frac{R_i}{\sum_j R_j}X_i \ge \mu)
= \mathbb{P}(\sum_{i=1}^n R_i(X_i - \mu) \ge 0)
\: .
\end{align*}
For $t \ge 0$, we can compose with exponentials and use Markov's inequality to obtain
\begin{align*}
\mathbb{P}_{L \sim \mathrm{Dir}(1^{n})}(L^\top X \ge \mu)
&= \mathbb{P}(\exp t\sum_{i=1}^n R_i(X_i - \mu) \ge 1)
\le \mathbb{E}e^{t \sum_{i=1}^n R_i(X_i - \mu)} \: .
\end{align*}
By independence, this last expression is equal to $\prod_{i=1}^n \mathbb{E}e^{t(X_i - \mu) R_i}$ . By a simple computation (See \cite{RiouHonda20}) we get, for $t \in [0, \frac{1}{X_i - \mu})$ if $X_i \ge \mu$ and for $t \ge 0$ otherwise,
\begin{align*}
\mathbb{E}e^{t(X_i - \mu) R_i} = \frac{1}{1 - t(X_i - \mu)} \: .
\end{align*}
We have proved that for all $t \in [0, \frac{1}{B - \mu})$,
\begin{align*}
\mathbb{P}_{L \sim \mathrm{Dir}(1^{n})}(L^\top X \ge \mu)
&\le \exp \left( - n \frac{1}{n}\sum_{i=1}^n \log(1 - t (X_i - \mu))\right)
\: .
\end{align*}
This is then also true for $t$ minimizing the right-hand side.

It remains to show that $\sup_{t \in [0, \frac{1}{B - \mu})}\frac{1}{n}\sum_{i=1}^n \log(1 - t (X_i - \mu)) = \Kinf^+(\hat{F}_n, \mu)$~.

From \cite{HondaTakemura10}, Theorem 8, for any distribution $F$ with support in $[0,B]$,
\begin{align*}
\Kinf^+(F, \mu)
&= \sup_{t \in [0, \frac{1}{B - \mu}]} \mathbb{E}_{X \sim F}[\log (1 - t(X - \mu))] \: .
\end{align*}
Applying this to $\hat{F}_n$ gives $\Kinf^+(\hat{F}_n, \mu) = \sup_{t \in [0, \frac{1}{B - \mu}]}\frac{1}{n}\sum_{i=1}^n \log(1 - t (X_i - \mu))$. The only difference with our target is that the supremum is over the closed interval and not the right-open interval, but either the sup is the same by continuity if there is no $X_i$ equal to $B$, or the value at $1/(B - \mu)$ is $-\infty$ and hence not equal to the sup.

We now prove the bound involving $\Kinf^-$. Let $\hat{F}_n^{B-X}$ be the empirical distribution corresponding to $(B-X_1, \ldots, B-X_n)$.
\begin{align*}
\mathbb{P}_{L \sim \mathrm{Dir}(1^{n})}(L^\top X \le \mu)
&= \mathbb{P}_{L \sim \mathrm{Dir}(1^{n})}(L^\top (B - X) \ge B - \mu)
\\
&\le \exp \left( - n \Kinf^+(\hat{F}_n^{B-X}, B - \mu) \right)
\\
&= \exp \left( - n \Kinf^-(\hat{F}_n, \mu) \right)
\: .
\end{align*}
The last equality follows from Lemma~\ref{lem:Kinf+_symm_eq_Kinf-}.
\end{proof}

\begin{corollary} \label{cor:upper_bound_two_arms_bcp_bounded}
Let $X = (X_1, \ldots, X_n) \in [0,B]^{n}$, and let $Y = (Y_1, \ldots, Y_m) \in [0,B]^{m}$. let $\hat{F}_{n,X}$ be empirical distribution corresponding to $X$ (and define $\hat{F}_{m,Y}$ similarly). Then
\begin{align*}
\mathbb{P}_{L_X \sim \mathrm{Dir}(1^{n}), L_Y \sim \mathrm{Dir}(1^{m})}(L_X^\top X \ge L_Y^\top Y)
&\le f \left( - \inf_{\mu \in [0,B]}\left( n \Kinf^+(\hat{F}_{n,X}, \mu) + m \Kinf^-(\hat{F}_{m,Y}, \mu) \right) \right)
\: .
\end{align*}
where $f(x) = (1 + x)e^{-x}$.
\end{corollary}

\begin{proof}
Combine the two bounds of Theorem~\ref{thm:upper_bound_one_arm_bcp_bounded} using Lemma~\ref{lem:from_bcp_one_to_bcp_two}.
\end{proof}

\subsection{Lower bounds}
\label{sub:lower_bound}

Lemma~\ref{lem:coarse_lower_bound_BCP} gives a first, coarse lower bound on the BCP under a Dirichlet sampler.
This result crucially relies on the fact that $\{0,B\}$ have been added to the support.

\begin{lemma} \label{lem:coarse_lower_bound_BCP}
Let $X = (B, 0, X_1 \ldots, X_n) \in [0,B]^{n+2}$ and $u \in (0,B)$. Then,
\begin{align*}
\mathbb{P}_{L \sim \mathrm{Dir}(1^{n+2})}[L^\top X \ge u]
&\ge \left(1 - \frac{u}{B} \right)^{n+1} \quad \text{and} \quad
\mathbb{P}_{L \sim \mathrm{Dir}(1^{n+2})}[L^\top X \le u]
\ge \left(\frac{u}{B} \right)^{n+1} \: .
\end{align*}
\end{lemma}
\begin{proof}
	We consider $\tilde X = (B, 0, 0 \ldots, 0) \in [0,B]^{n+2}$, use that the marginals of Dirichlet are Beta distributions and the Beta-Binomial trick (e.g. \cite{AGAISTAT13}) to obtain
	\begin{align*}
		\mathbb{P}_{L \sim \mathrm{Dir}(1^{n+2})}[L^\top X \ge u] \geq \mathbb{P}_{L \sim \mathrm{Dir}(1^{n+2})}[L^\top \tilde X \ge u] &= \mathbb{P}_{w \sim \mathrm{Beta}(1,n+1)}\left[w  \ge \frac{u}{B} \right] \\
		&= \mathbb{P}_{k \sim \mathrm{Bin}\left(n+1,\frac{u}{B}\right)}[k \le 0] \\
		&= \left(1 - \frac{u}{B} \right)^{n+1}
		\: .
	\end{align*}
	Similarly, considering $\tilde X = (B, 0, B \ldots, B) \in [0,B]^{n+2}$, we obtain
	\begin{align*}
		\mathbb{P}_{L \sim \mathrm{Dir}(1^{n+2})}[L^\top X \le u] \geq \mathbb{P}_{L \sim \mathrm{Dir}(1^{n+2})}[L^\top \tilde X \le u] &= \mathbb{P}_{w \sim \mathrm{Beta}(1, n+1)}\left[w  \ge 1- \frac{u}{B} \right] \\
		&= \mathbb{P}_{k \sim \mathrm{Bin}\left(n+1,1- \frac{u}{B}\right)}[k \le 0] \\
		&= \left(\frac{u}{B} \right)^{n+1}
		\: .
	\end{align*}
\end{proof}

In the rest of this section, we derive a tighter lower bound on the BCP which leads to a tight lower bound on the probability that one Dirichlet sample exceeds another (Theorem~\ref{thm:two_arms_BCP_tight_lower_bound}). These result rely on a discretization argument and on deriving lower bounds for multinomial distributions.

\subsubsection{Multinomial distributions}
\label{ssub:multinomial_distributions}

Theorem~\ref{thm:bcp_lower_bound_multinomial_KL} gives a tight lower bound on the BCP for multinomial distributions.

\begin{theorem}\label{thm:bcp_lower_bound_multinomial_KL}
Let $X_1, \ldots, X_M \in [0,B]$ with $X_M = B$ and let $\beta \in \mathbb{N}^M$ with $\beta_i > 0$ for all $i$. Define $n = \sum_{i=1}^M \beta_i$. For all $\mu \in [0,B]$ and $q \in \triangle_M$ such that $q^\top X \ge \mu$,
\begin{align*}
\mathbb{P}_{L \sim Dir(\beta)}(L^\top X \ge \mu)
\ge \frac{M^{M/2}}{2 (8 \pi)^{\frac{M-1}{2}}}\frac{1}{n^{\frac{M+1}{2}}} \exp\left( -n\KL_{\mathcal M}\left(\frac{\beta}{n}, q\right)\right)
\: .
\end{align*}
where $\KL_{\mathcal M}(p,q)$ is the Kullback-Leibler divergence between multinomial distributions with probability vectors $p$ and $q$.
\end{theorem}

\begin{proof} The proof is strongly inspired by the works of \cite{HondaTakemura10} and \cite{baudry21a} who use lower bound on the BCP for analyzing regret minimization algorithms.

Let $\mathcal S_q = \{p \in \triangle_M \mid \forall i \in [M-1], p_i \le q_i\}$. For $p \in \mathcal S_q$, we necessarily have $p_M \ge q_M$ and $p^\top X \ge q^\top X \ge \mu$. From that inequality we get
\begin{align*}
\mathbb{P}_{L \sim \mathrm{Dir}(\beta)}(L^\top X \ge \mu)
\ge \mathbb{P}_{L \sim \mathrm{Dir}(\beta)}(L \in \mathcal S_q)
\: .
\end{align*}
We now quantify that probability, using the pdf of a Dirichlet distribution,
\begin{align*}
\mathbb{P}_{L \sim \mathrm{Dir}(\beta)}(L \in \mathcal S_q)
&=\frac{\Gamma(n)}{\prod_{i=1}^M \Gamma(\beta_i)} \int_{x \in \mathcal S_q} \prod_{i=1}^M x_i^{\beta_i-1} dx
\\
& \geq \frac{\Gamma(n)}{\prod_{i=1}^M \Gamma(\beta_i)} q_M^{\beta_M-1} \prod_{j=1}^{M-1} \int_{x_j=0}^{q_j}  x_j^{\beta_j-1} dx_j
\\
& = \frac{\Gamma(n)}{\prod_{i=1}^M \Gamma(\beta_i)} q_M^{\beta_M-1} \prod_{j=1}^{M-1} \frac{q_j^{\beta_j}}{\beta_j}
\\
& = \frac{\Gamma(n)}{\prod_{i=1}^M \Gamma(\beta_i)} \frac{\beta_M}{q_M} \prod_{j=1}^{M} \frac{q_j^{\beta_j}}{\beta_j}
\\
&\ge \frac{\Gamma(n)}{\prod_{i=1}^M \Gamma(\beta_i)} \prod_{j=1}^{M} \frac{q_j^{\beta_j}}{\beta_j}
\: .
\end{align*}
The last line uses that since $\beta_M \ge 1$ and $q_M \le 1$, $\beta_M/q_M \ge 1$.
We transform that last expression to exhibit the Kullback-Leibler divergence between multinomial distributions.
\begin{align*}
\mathbb{P}_{L \sim \mathrm{Dir}(\beta)}(L \in \mathcal S_q)
&\ge \frac{\Gamma(n)}{\prod_{i=1}^M \Gamma(\beta_i)} \prod_{j=1}^{M} \frac{q_j^{\beta_j}}{\beta_j}
\\
&= \frac{1}{n^M}\frac{\Gamma(n)}{\prod_{i=1}^M \Gamma(\beta_i)} \prod_{j=1}^{M} \left(\frac{q_j}{\beta_j/n}\right)^{\beta_j} \prod_{j=1}^{M} \left(\frac{\beta_j}{n}\right)^{\beta_j-1}
\\
&= \frac{1}{n^M}\frac{\Gamma(n)}{\prod_{i=1}^M \Gamma(\beta_i)} \prod_{j=1}^{M} \left(\frac{\beta_j}{n}\right)^{\beta_j-1} \exp\left(- n \KL(\frac{\beta}{n}, q)\right)
\: .
\end{align*}
We now simplify the factor in front of the exponential.
\begin{align*}
\frac{1}{n^M}\frac{\Gamma(n)}{\prod_{i=1}^M \Gamma(\beta_i)} \prod_{j=1}^{M} \left(\frac{\beta_j}{n}\right)^{\beta_j-1}
= \frac{1}{n}\frac{n!}{\prod_{i=1}^M \beta_i!} \prod_{j=1}^{M} \left(\frac{\beta_j}{n}\right)^{\beta_j}
&= \frac{1}{n}\frac{n!}{n^n} \prod_{j=1}^{M} \frac{\beta_j^{\beta_j}}{\beta_j!}
\: .
\end{align*}
We use the following bound on the Stirling approximation of the factorial: $\frac{n!}{\sqrt{2 \pi n} (n/e)^n} \in [1, 2]$ for all $n \ge 1$. A tighter approximation is possible, but this one is sufficient for our purpose.
\begin{align*}
\frac{1}{n}\frac{n!}{n^n} \prod_{j=1}^{M} \frac{\beta_j^{\beta_j}}{\beta_j!}
\ge \frac{1}{n}\sqrt{2 \pi n}e^n \prod_{j=1}^{M} \frac{1}{2\sqrt{2 \pi \beta_j} e^{\beta_j}}
&= \frac{1}{n}\sqrt{2 \pi n} \prod_{j=1}^{M} \frac{1}{2\sqrt{2 \pi \beta_j}}
\\
&\ge \frac{1}{n}\sqrt{2 \pi n} \prod_{j=1}^{M} \frac{1}{2\sqrt{2 \pi n/M}}
\\
&= \frac{M^{M/2}}{2 (8 \pi)^{\frac{M-1}{2}}}\frac{1}{n^{\frac{M+1}{2}}}
\: .
\end{align*}

\end{proof}

Lemma~\ref{lem:Kinf_multinomial_eq_inf_KL_multinomial} links $\KL_{\mathcal M}$ with $\Kinf^{\pm}$.

\begin{lemma}\label{lem:Kinf_multinomial_eq_inf_KL_multinomial}
Let $F$ be a multinomial distribution supported on points $X_1, \ldots, X_M \in [0,B]$ and let $p \in \triangle_M$ be the corresponding probability vector.

If there exists $i \in [M]$ with $X_i = B$ and $p_i > 0$, then for all $\mu \in [0,B]$,
\begin{align*}
\Kinf^+(F, \mu)
= \inf_{q \in \triangle_M : q^\top X \ge \mu} \KL_{\mathcal M}(p, q)
\: .
\end{align*}
If there exists $i \in [M]$ with $X_i = 0$ and $p_i > 0$, then for all $\mu \in [0,B]$,
\begin{align*}
\Kinf^-(F, \mu)
= \inf_{q \in \triangle_M  : q^\top X \le \mu} \KL_{\mathcal M}(p, q)
\: .
\end{align*}
\end{lemma}

\begin{proof}
As remarked in \cite{HondaTakemura10}, the probability measure that realizes the infimum (which exists by Lemma~\ref{lem:inf_attained_in_Kinf}) in the $\Kinf^+$ problem for distributions with finite support has mass on the same points and on $B$.
Under the hypothesis that there exists $i \in [M]$ with $X_i = B$ and $p_i > 0$, we get that that infimum is also a multinomial with same support. Hence there exists $q_F$ such that $\Kinf^+(F, \mu) = \KL_{\mathcal M}(p, q_F)$ and we get $\inf_{q \in \triangle_M : q^\top X \ge \mu} \KL_{\mathcal M}(p, q) \le \Kinf^+(F, \mu)$.
The reverse inequality comes from the definition of $\Kinf^+$ as an infimum over all probability distributions (which is a larger set than the multinomial distributions).

The proof for $\Kinf^-$ is similar.
\end{proof}

\begin{theorem}\label{thm:bcp_lower_bound_multinomial_Kinf}
Let $X_1, \ldots, X_M \in [0,B]$ with $X_M = B$ and let $\beta \in \mathbb{N}^M$ with $\beta_M > 0$. Define $n = \sum_{i=1}^M \beta_i$. Let $F_n$ be the multinomial distributions over the $X_i$ with weights $\beta/n$. For all $\mu \in [0,B]$,
\begin{align*}
\mathbb{P}_{L \sim \mathrm{Dir}(\beta)}(L^\top X \ge \mu)
\ge \frac{M^{M/2}}{2 (8 \pi)^{\frac{M-1}{2}}}\frac{1}{n^{\frac{M+1}{2}}} \exp\left( -n\Kinf^+(F_n, \mu)\right)
\: .
\end{align*}
\end{theorem}
\begin{proof}
If $\beta_i > 0$ for all $i \in [M]$, this is the result of a supremum over the lower bounds of Theorem~\ref{thm:bcp_lower_bound_multinomial_KL}, to which we apply the equality of Lemma~\ref{lem:Kinf_multinomial_eq_inf_KL_multinomial}. Now if there are some $i$ for which $\beta_i = 0$, we have for some $M_0 < M$,
\begin{align*}
\mathbb{P}_{L \sim \mathrm{Dir}(\beta)}(L^\top X \ge \mu)
\ge \frac{M_0^{M_0/2}}{2 (8 \pi)^{\frac{M_0-1}{2}}}\frac{1}{n^{\frac{M_0+1}{2}}} \exp\left( -n\Kinf^+(F_n, \mu)\right)
\: .
\end{align*}
But since the leading factor is non-increasing in $M_0$, we recover the result with $M$ instead of $M_0$.
\end{proof}

\begin{corollary}
Let $X_1, \ldots, X_M \in [0,B]$ with $X_M = 0$ and let $\beta \in \mathbb{N}^M$ with $\beta_M > 0$. Define $n = \sum_{i=1}^M \beta_i$. Let $F_n$ be the multinomial distributions over the $X_i$ with weights $\beta/n$. For all $\mu \in [0,B]$,
\begin{align*}
\mathbb{P}_{L \sim \mathrm{Dir}(\beta)}(L^\top X \le \mu)
\ge \frac{M^{M/2}}{2 (8 \pi)^{\frac{M-1}{2}}}\frac{1}{n^{\frac{M+1}{2}}} \exp\left( -n\Kinf^-(F_n, \mu)\right)
\: .
\end{align*}
\end{corollary}

\begin{proof}
Remark that $\mathbb{P}_{L \sim \mathrm{Dir}(\beta)}(L^\top X \le \mu) = \mathbb{P}_{L \sim \mathrm{Dir}(\beta)}(L^\top (B - X) \ge B - \mu)$ and apply Theorem~\ref{thm:bcp_lower_bound_multinomial_Kinf} and Lemma~\ref{lem:Kinf+_symm_eq_Kinf-}.
\end{proof}

%

\subsubsection{Lower bound for bounded distributions}
\label{ssub:lower_bound_for_bounded_distributions}

\begin{lemma}\label{lem:bcp_lower_bound_discretized_one_arm}
Let $X_1, \ldots, X_{n} \in [0,B]$.
Let $\theta = L^\top X$, where $L$ is a Dirichlet random variables, with $L \sim \mathrm{Dir}(1, \ldots, 1)$ ($n$ ones).
Let $Y_1, \ldots, Y_M \in [0,B]$, among which are the values 0 and $B$. For all $i$, let $X_i^{+} = \min \{Y_k \mid k \in [M], Y_k \ge X_i \}$. Let $F_n^{+}$ be the empirical distribution corresponding to points $X_i^{+}$.

If $0 \in \{X_1, \ldots, X_n\}$, then for all $\mu \in [0,B]$,
\begin{align*}
\mathbb{P}(\theta \le \mu)
&\ge \frac{M^{M/2}}{2 (8 \pi)^{\frac{M-1}{2}}}\frac{1}{n^{\frac{M+1}{2}}}
	\exp\left( - n \Kinf^-(F_n^{+}, \mu) \right)
\: .
\end{align*}
If $B \in \{X_1, \ldots, X_n\}$, then for all $\mu \in [0,B]$,
\begin{align*}
\mathbb{P}(\theta \ge \mu)
&\ge \frac{M^{M/2}}{2 (8 \pi)^{\frac{M-1}{2}}}\frac{1}{n^{\frac{M+1}{2}}}
	\exp\left( - n \Kinf^+(F_n^{-}, \mu) \right)
\: .
\end{align*}
\end{lemma}
\begin{proof}
We have $\theta = L^\top X \le L^\top X^{+}$ and $\theta = L^\top X \ge L^\top X^{-}$. Those scalar products can be written as scalar products of Dirichlet random variables with $Y$.
We now apply Theorem~\ref{thm:bcp_lower_bound_multinomial_Kinf} and its corollary.
\end{proof}

%
%

For a probability distribution with cdf $F$ on $[0,B]$ and points $0 = x_0 < x_1 < \ldots < x_M < x_{M+1}=B$ we define two discretized distributions with cdf given by, for $x \in [x_m, x_{m+1})$,
\begin{align*}
F^-(x)
&= \lim_{y \to x_{m+1}, y \le x_{m+1}} F(y)
\: , \\
F^+(x)
&= F(x_m)
\: .
\end{align*}

\begin{lemma}\label{lem:exists_discretization}
For all $\varepsilon > 0$ and all probability distributions $F$ supported on $[0,B]$, there exists a discretization over at most $2 + \lfloor 1/\varepsilon \rfloor$ points (counting points $0$ and $B$) such that $\Vert F^- - F \Vert_\infty \le \varepsilon$ and $\Vert F^+ - F \Vert_\infty \le \varepsilon$.
\end{lemma}
\begin{proof}
Let $M = \lfloor 1/\varepsilon \rfloor$. For $m \in \{0, \ldots, M\}$, let $x_m = \inf\{ x \in [0,B] \mid F(x) \ge m\varepsilon\}$. Let $x_{M+1} = B$.
\begin{align*}
\Vert F^- - F \Vert_\infty
&\le \max_{0 \le m \le M} | \lim_{y \to x_{m+1}, y \le x_{m+1}} F(y) - F(x_m) | \ind\{x_{m+1} \ne x_m\}
\\
&\le \max_{0 \le m \le M} | (m+1) \varepsilon - m \varepsilon |
\le \varepsilon
\: .
\end{align*}
The computation for $F^+$ is similar.
\end{proof}

\begin{lemma}
For all $F,G \in \mathcal{P}(\mathbb{R})$ with support in $[0,B]$ and all finite discretizations,
\begin{align*}
\Vert F^- - G^- \Vert_\infty \le \Vert F - G \Vert_\infty \: ,
\\
\Vert F^+ - G^+ \Vert_\infty \le \Vert F - G \Vert_\infty \:.
\end{align*}
\end{lemma}

\begin{proof}
For all $x \in [x_m, x_{m+1})$, $F^-(x) = \lim_{y \to x_{m+1}, y \le x_{m+1}} F(y)$.
\begin{align*}
\Vert F^- - G^- \Vert_\infty
&\le \max_{0 \le m \le M} | \lim_{y \to x_{m+1}, y \le x_{m+1}} F(y) - \lim_{y \to x_{m+1}, y \le x_{m+1}} G(y) |
\le \Vert F - G \Vert_\infty
\: .
\end{align*}
\end{proof}

\begin{lemma}\label{lem:exists_discretization_close_true_cdf}
Let $F \in \mathcal P(\mathbb{R})$ with support in $[0,B]$.
For all $\varepsilon > 0$, there exists a discretization of $[0,B]$ into $2 + \lfloor 2/\varepsilon \rfloor$ points such that for all $G$ with $\Vert G - F \Vert \le \varepsilon/2$, we have $\Vert G^- - F \Vert_\infty \le \varepsilon$ and $\Vert G^+ - F \Vert_\infty \le \varepsilon$.
\end{lemma}

\begin{proof}
Let $\varepsilon > 0$. For a discretization of $F$ verifying the result of Lemma~\ref{lem:exists_discretization} for $\varepsilon/2$,
\begin{align*}
\Vert G^- - F \Vert_\infty
\le \Vert G^- - F^- \Vert_\infty + \Vert F^- - F \Vert_\infty
\le \Vert G - F \Vert_\infty + \Vert F^- - F \Vert_\infty
\le \varepsilon
\: .
\end{align*}
Same computation for $G^+$.
\end{proof}

Lemma~\ref{lem:bcp_one_arm} gives a tight lower bound on the BCP for bounded distributions.
\begin{lemma}\label{lem:bcp_one_arm}
Let $a > 0$ and $b < B$.
Let $F$ be a probability distribution with support in $[0,B]$.

For points $X, \ldots, X_{n} \in [0,B]$. 
let $\theta = L^\top X$, where $L$ is a Dirichlet random variable, with $L \sim \mathrm{Dir}(1, \ldots, 1)$ ($n$ ones).
Let $F_n$ be the empirical distribution corresponding to points $(X_i)_{i\in [n]}$.

For all $\varepsilon > 0$, there exists $\eta>0$ such that for all such empirical distributions (and in particular for all $n$), if $\Vert F_n - F \Vert_\infty \le \eta$ then for all $u \in [a, b]$,
\begin{align*}
\text{if } B \in \{X_1, \ldots, X_n\} \text{ then } \mathbb{P}(\theta \ge u)
&\ge \frac{M^{M/2}}{2 (8 \pi)^{\frac{M-1}{2}}}\frac{1}{n^{\frac{M+1}{2}}}
	\exp\left( - n \Kinf^+(F, \mu) - n \varepsilon \right)
\:, \\
\text{if } 0 \in \{X_1, \ldots, X_n\} \text{ then } \mathbb{P}(\theta \le u)
&\ge \frac{M^{M/2}}{2 (8 \pi)^{\frac{M-1}{2}}}\frac{1}{n^{\frac{M+1}{2}}}
	\exp\left( - n \Kinf^-(F, \mu) - n \varepsilon \right)
\: .
\end{align*}
with $M = 2 + \lfloor 2/\eta \rfloor$.
\end{lemma}

\begin{proof}
The function $(F,\mu) \mapsto \Kinf^-(F, \mu)$ is continuous on $\mathcal D \times (0,B)$ by Theorem~\ref{thm:Kinf_continuous}. The function $(F,\mu) \mapsto \Kinf^-(F, \mu)$ is then uniformly continuous on $\mathcal D \times [a,b]$ (since $\mathcal D$ is compact).
In particular, there exists $\eta > 0$ such that if $\Vert G - F \Vert_\infty \le \eta$ then for all $\mu \in [a,b]$,
\begin{align*}
\Kinf^-(G, \mu)
\le \Kinf^-(F, \mu) + \varepsilon
\: .
\end{align*}
We have a similar property for $\Kinf^+$.

We now build a discretization such that $F_n^{+}$ and $F_n^{-}$ verify that condition under the hypothesis $\Vert F_n^{(1)} - F^{(1)} \Vert_\infty \le \eta$ and $\Vert F_n - F \Vert_\infty \le \eta$, using Lemma~\ref{lem:exists_discretization_close_true_cdf}. The result is a combination of this continuity inequality and Lemma~\ref{lem:bcp_lower_bound_discretized_one_arm}.
\end{proof}

Theorem~\ref{thm:two_arms_BCP_tight_lower_bound} gives a tight lower bound on probabilities of the form $\bP(\theta_2 \ge \theta_1)$ for bounded distributions.

\begin{theorem} \label{thm:two_arms_BCP_tight_lower_bound}
Let $F^{(1)}$ and $F^{(2)}$ be two probability distributions with means in $(0,B)$.

For points $X_1^{(1)}, \ldots, X_{n_1}^{(1)}, X_1^{(2)}, \ldots, X_{n_2}^{(2)} \in [0,B]$ such that $X_1^{(1)} = 0, X_{n_2}^{(2)} = B$,
let $\theta_1 = (L^{(1)})^\top X^{(1)}$ and $\theta_2 = (L^{(2)})^\top X^{(2)}$, where $L^{(1)}$ and $L^{(2)}$ are independent Dirichlet random variables, with $L^{(1)} \sim \mathrm{Dir}(1, \ldots, 1)$ ($n_1$ ones) and $L^{(2)} \sim \mathrm{Dir}(1, \ldots, 1)$ ($n_2$ ones).
Let $F_n^{(1)}$ be the empirical distribution corresponding to points $X_i^{(1)}$ and $F_n^{(2)}$ be the empirical distribution corresponding to points $X_i^{(2)}$.

For all $\varepsilon > 0$, there exists $\eta>0$ such that for all such empirical distributions (and in particular for all $n_1$ and $n_2$), if $\Vert F_n^{(1)} - F^{(1)} \Vert_\infty \le \eta$ and $\Vert F_n^{(2)} - F^{(2)} \Vert_\infty \le \eta$ then
\begin{align*}
&\mathbb{P}(\theta_2 \ge \theta_1)\\
&\ge \frac{M^{M}}{4 (8 \pi)^{M-1}}\frac{1}{(n_1 n_2)^{\frac{M+1}{2}}}
	\exp\left( -\inf_{\mu \in [0,B]} \left(n_1 \Kinf^-(F^{(1)}, \mu) + n_2 \Kinf^+(F^{(2)}, \mu)\right) - (n_1 + n_2) \varepsilon\right)
\end{align*}
with $M = 2 + \lfloor 2/\eta \rfloor$.
\end{theorem}

\begin{proof}
By Lemma~\ref{lem:from_bcp_one_to_bcp_two},
\begin{align*}
\mathbb{P}(\theta_2 \ge \theta_1)
\ge \sup_{\mu \in [\mu_2, \mu_1]} \mathbb{P}(\theta_2 \ge \mu) \mathbb{P}(\theta_1 \le \mu)
\end{align*}
We now use Lemma~\ref{lem:bcp_one_arm} for both probabilities. This is valid since by hypothesis the interval $[\mu_2, \mu_1]$ is a subset of $(0,B)$. We get the wanted result, except that the infimum is over $\mu \in [\mu_2, \mu_1]$ instead of $\mu \in [0,B]$. But by the monotonicity of $\Kinf^\pm$ in $\mu$ and the fact that $\Kinf^+(F^{(2)},\mu)$ (resp. $\Kinf^-(F^{(1)},\mu)$) is 0 for $\mu \le \mu_2$ (resp. for $\mu \ge \mu_1$), we get that
\begin{align*}
\inf_{\mu \in [0,B]} \left(n_1 \Kinf^-(F^{(1)}, \mu) + n_2 \Kinf^+(F^{(2)}, \mu)\right)
= \inf_{\mu \in [\mu_2, \mu_1]} \left(n_1 \Kinf^-(F^{(1)}, \mu) + n_2 \Kinf^+(F^{(2)}, \mu)\right)
\end{align*}
and the theorem is proved.
\end{proof}

%% file: sections/spef.tex

\section{Single parameter exponential families}
\label{app:spef}

In this section, we explain how our analysis can be used to analyze Top Two algorithms for Single Parameter Exponential Families (SPEF). More precisely, our results apply to SPEF of sub-exponential distributions. We recall below the definition of a sub-exponential distribution, which applies to several typical examples of SPEF: Bernoulli distribution, Gaussian distributions with known variances, exponential and Poissson distributions \cite{vershynin18HDP}.

\begin{definition}
A distribution $X$ is sub-exponential with constant $C$ if it satisfies $\mathbb{P}(|X| \ge x) \le 2 e^{-C x}$.
\end{definition}

\paragraph{Preliminaries}

Let $\mathbb{P}^{(0)}$ be a sub-exponential probability distribution and let $\phi$ be the cumulant generating function of $\mathbb{P}^{(0)}$, defined by $\phi(\lambda) = \log \mathbb{E}_{X \sim \mathbb{P}^{(0)}}e^{\lambda X}$. Let $\mathcal I^{(n)} \subseteq \mathbb{R}$ be the open interval on which it is defined (set of natural parameters).

The single parameter exponential family (SPEF) associated to a probability measure $\mathbb{P}^{(0)}$ is the set of probability measures $\{\mathbb{P}^{(\lambda)} \mid \lambda \in \mathcal I^{(n)}\}$, where $\mathbb{P}^{(\lambda)}$ is the probability measure absolutely continuous with respect to $\mathbb{P}^{(0)}$, with density $x \mapsto e^{\lambda x - \phi(\lambda)}$ with respect to $\mathbb{P}^{(0)}$. That is, $\log \frac{d \mathbb{P}^{(\lambda)}}{d \mathbb{P}^{(0)}}(x) = \lambda x - \phi(\lambda)$~. Using that the reference probability measure $\mathbb{P}^{(0)}$ is assumed sub-exponential, we can verify that that for all $\lambda \in \cI^{(n)}$, $\bP^{(\lambda)}$ is also sub-exponential.

$\phi$ is an analytic, strictly convex function on $\mathcal I^{(n)}$. The distribution $\mathbb{P}^{(\lambda)}$ has mean $\phi'(\lambda)$. Let $\mathcal I = \phi'(\mathcal I^{(n)})$ be the open interval of means of the SPEF. Let $\phi^*$ be the convex conjugate of $\phi$, which is also a strictly convex function.
Recall that $(\phi^*)' = (\phi')^{-1}$.
Let $d_\phi$ be the Bregman divergence associated to $\phi$ and $d_{\phi^*}$ be the Bregman divergence associated $\phi^*$. We have, for $\lambda, \eta \in \mathcal I^{(n)}$,
\begin{align*}
d_\phi(\lambda, \eta)
&= \phi(\lambda) - \phi(\eta) - (\lambda - \eta)^\top \phi'(\eta)
= d_{\phi^*}(\phi'(\eta), \phi'(\lambda))
\: .
\end{align*}
We write $\mathbb{P}_m$ for the unique member of the SPEF with mean $m$ (if it exists, that is if $m \in \mathcal I$). It verifies $\mathbb{P}_m = \mathbb{P}^{(\phi'^{-1}(m))}$. For two distributions in the SPEF with means $m_1$ and $m_2$, the Kullback-Leibler divergence between the corresponding distributions $\mathbb{P}_{m_1}$ and $\mathbb{P}_{m_2}$ is
\begin{align*}
\KL(\mathbb{P}_{m_1}, \mathbb{P}_{m_2})
&= d_{\phi^*}(m_1, m_2) \: .
\end{align*}
In the following, we write simply $d(m_1, m_2)$ for $d_{\phi^*}(m_1, m_2)$, the Kullback-Leibler divergence between the distributions in the SPEF with means $m_1$ and $m_2$.

\paragraph{The $\Kinf$ minimization problem for exponential families}

In a SPEF, the quantity $\Kinf^+$, infimum of the KL from a member of the SPEF to the subset of the family with mean larger that $\mu \in \mathcal I$ becomes
\begin{align*}
\Kinf^+(\mathbb{P}_m, \mu)
&\eqdef \inf \{\KL(\mathbb{P}_m, Q) \mid Q \in \{\mathbb{P}_{m'} \mid m' \in \mathcal I\}, \: \mathbb{E}_Q[X] \ge \mu\}
\\
&= \inf \{\KL(\mathbb{P}_m, \mathbb{P}_{m'}) \mid m' \in \mathcal I, \: m' \ge \mu\}
\\
&= \inf \{d(m, m') \mid m' \in \mathcal I, \: m' \ge \mu\}
\\
&= d(m, \max\{m, \mu\})
\: .
\end{align*}
Similarly for all $m, \mu \in \mathcal I$, $\Kinf^-(\mathbb{P}_m, \mu) = d(m, \min\{m, \mu\})$. We abuse notations and write $\Kinf^+(m, \mu) = \Kinf^+(\mathbb{P}_m, \mu)$.

\paragraph{Properties of $\Kinf$ in a SPEF} The following properties are well-known in the bandit literature, see, e.g. \cite{KLUCBJournal}:

\begin{enumerate}
	\item $\mu \mapsto \Kinf^+(m, \mu)$ is differentiable on $\mathcal I \setminus \{m\}$.
	\item $\mu \mapsto \Kinf^+(m, \mu)$ is zero for $\mu \le m$ and finite, increasing and strictly convex on $[m, +\infty) \cap \mathcal I$.
	\item $\lim_{\mu \to \sup \mathcal I} \Kinf^+(m, \mu) = +\infty$.
	\item $(m, \mu) \mapsto \Kinf^+(m, \mu)$ is jointly continuous on $\mathcal I^2$.
\end{enumerate}

The transportation cost is
\begin{align*}
C_{i,j}(\cT(\bm F), w)
&= \inf_{u \in \mathcal I} \left\{ w_i d(m_i, \min\{m_i, u\}) + w_j d(m_j, \max\{m_j, u\})\right\}
\: .
\end{align*}
The infimum can equivalently be taken over $[m_j, m_i]$.
That transportation cost is jointly continuous on $\mathcal I^K \times \simplex$ by Berge's Maximum theorem: Property~\ref{prop:joint_continuity_true_tc} is verified.

\paragraph{Concentration results} The following result is the counterpart of Lemma~\ref{lem:subG_sup_N} for sub-exponential distributions.

\begin{lemma}\label{lem:sup_sub_exponentials}
Suppose that $(X_n)_{n \ge 1}$ are sub-exponential random variables with constants $(C_{n})$, such that $c \eqdef \inf_n C_{n} > 0$. Then $\sup_n(X_n/\log (e + n))$ is sub-exponential.
\end{lemma}
\begin{proof}
This is due to a simple union bound:
\begin{align*}
\mathbb{P}(|\sup_n X_n/\log (e + n)| \ge x)
&\le \sum_{n=1}^{+\infty} \mathbb{P}(|X_n| \ge x \log (e + n))
\\
&\le 2\sum_{n=1}^{+\infty} e^{- c x \log (e + n)} \: .
\end{align*}
We now use that for $x \ge 4/c$, we have $cx/2 \ge 2$ and $2\log(e+n) \ge 2$, hence $cx\log(e+n) \ge cx/2 + 2\log(e+n)$: for $x \ge 4/c$,
\begin{align*}
\mathbb{P}(|\sup_n X_n/\log (e + n)| \ge x)
\le 2 \left(\sum_{n=1}^{+\infty} \frac{1}{(e + n)^2} \right)e^{- c x/2} \: .
\end{align*}
This shows that the random variable is sub-exponential.
\end{proof}

\begin{lemma}\label{lem:W_concentration_spef}
There exists a sub-exponential random variable $W_d$ such that almost surely, for all $i \in [K]$ and all $n$ such that $N_{n,i} \ge 1$,
\begin{align*}
N_{n,i} d(\mu_{n,i}, \mu_i) \le W_d \log (e + N_{n,i}) \: .
\end{align*}
There exists a sub-exponential random variable $W_\mu$ such that almost surely, for all $i \in [K]$ and all $n$ such that $N_{n,i} \ge 1$,
\begin{align*}
N_{n,i} |\mu_{n,i} - \mu_i| \le W_\mu \log (e + N_{n,i}) \: .
\end{align*}
In particular, any random variable which is polynomial in $W_d$ or $W_{\mu}$ has a finite expectation.
\end{lemma}

\begin{proof}
We start with the proof of the concentration inequality on the divergence. Since the maximum of a finite number of sub-exponential random variables is sub-exponential, it suffices to show that $\sup_n \frac{N_{n,i} d(\mu_{n,i}, \mu_i)}{\log N_{n,i}}$ is sub-exponential. Let then $i \in [K]$. Let $\hat{\mu}_{n,i}$ be the average of the first $n$ samples from arm $i$. It suffices to show that $\sup_n \frac{n d(\hat{\mu}_{n,i}, \mu_i)}{\log (e + n)}$ is sub-exponential.
By \cite[Lemma 4]{menard2017minimax}, we have that for any $n$,
\begin{align*}
\mathbb{P}(n d(\hat{\mu}_{n,i}, \mu_i) \ge \alpha) \le 2 e^{-\alpha} \: .
\end{align*}
That is, for a fixed $n$, $n d(\hat{\mu}_{n,i}, \mu_i)$ is sub-exponential. We then apply Lemma~\ref{lem:sup_sub_exponentials} to obtain than $\sup_n \frac{n d(\hat{\mu}_{n,i}, \mu_i)}{\log (e + n)}$ is sub-exponential.

By hypothesis, the distribution $F_i$ is sub-exponential. Hence at any $n$, $n |\hat{\mu}_{n,i} - \mu_i|$ is as well. We then apply Lemma~\ref{lem:sup_sub_exponentials} to obtain than $\sup_n n |\hat{\mu}_{n,i} - \mu_i|/\log(e+n)$ is sub-exponential. We finally obtain that the maximum over the finitely many arms has the same property.
\end{proof}

\subsection{Proof of the leader and challenger properties for EB, TC and TCI}

We prove, in the case of a sub-exponential SPEF, that Properties~\ref{prop:leader_cdt_suff_explo} and~\ref{prop:leader_cdt_convergence} hold for the EB leader and that Properties~\ref{prop:challenger_cdt_suff_explo} and~\ref{prop:challenger_cdt_convergence} hold for the TC and TCI challengers.

For SPEF, Property~\ref{prop:joint_continuity_true_tc} is a known result from the literature \cite{Russo2016TTTS}.
For sub-exponential SPEF, Property~\ref{prop:rate_convergence_mean_and_distribution} is shown in Lemma~\ref{lem:W_concentration_spef}.
In \cite{KK18Mixtures}, stopping threshold have been derived for general SPEF.
Thanks to their dependency in $(n,\delta)$, those thresholds are also asymptotically tight.
Therefore, applying Corollary~\ref{cor:asymptotic_optimality_top_two_algorithms} yields that $\beta$-EB-TC and $\beta$-EB-TCI are asymptotically $\beta$-optimal algorithms for sub-exponential SPEF with the corresponding threshold.

\paragraph{Rates for empirical transportation costs}

\begin{lemma}\label{lem:SPEF_W_lower_bound}
Let $\bm F \in \cF^K$ with $m(F_j) < m(F_i)$.
There exists $L$ with $\mathbb{E}[|L|^\alpha] < +\infty$ for all $\alpha>0$ and $D_{\bm F} > 0$ such that for $N_{n,i} \ge L$ and $N_{n,j} \ge L$, $W_{n}(i, j) > L D_{\bm F}$ .
\end{lemma}
\begin{proof}
Suppose that $N_{n,i} \ge L$ and $N_{n,j} \ge L$, for some $L$ to be determined. First we get
\begin{align*}
W_{n}(i, j) &= \inf_{u \in \mathcal I} \left\{ N_{n,i} \Kinf^{-}( \cT (F_{n,i}), u) + N_{n,j} \Kinf^{+}(\cT(F_{n,j}), u) \right\}	\\
&\geq L \inf_{u \in \mathcal I} \left\{ \Kinf^{-}( \cT (F_{n,i}), u) + \Kinf^{+}(\cT(F_{n,j}), u)  \right\}  \: ,
\end{align*}
where $\cT (F_{n,i}) = \mu_{n,i}$ is simply the mean.

For any compact interval $\cI_C \subseteq \cI$, the function defined by $\cT(\bm F) \mapsto \inf_{u \in \mathcal I_C} \left\{ \Kinf^{-}(\cT(F_i) , u) + \Kinf^{+}(\cT(F_j) , u) \right\}$ is continuous on $\cT(\cF^K)$.

For $L$ greater than some $L_1$ with finite moments, the means $\mu_{n,i}$ and $\mu_{n,j}$ belong to $[\mu_j - \varepsilon, \mu_i + \varepsilon] \subseteq \cI$ for some $\varepsilon >0$. Furthermore, for $L$ greater than some $L_2$ with finite moments, $\cT(F_{n,i})$ is $\varepsilon$-close to $F_i$ (and same thing for $F_j$). The continuity then gives that there exists $L$ with finite moments such that
\begin{align*}
\inf_{u \in \mathcal I} \left\{ \Kinf^{-}( \cT (F_{n,i}), u) + \Kinf^{+}(\cT(F_{n,j}), u)  \right\}
&\ge \frac{1}{2} \inf_{u \in \mathcal I_C} \left\{ \Kinf^{-}( \cT (F_{i}), u) + \Kinf^{+}(\cT(F_{j}), u)  \right\}
\: .
\end{align*}
This is positive since $m(F_j) < m(F_i)$ by an analogue of Lemma~\ref{lem:inf_Kinf_restricted_to_open_mean_interval}, which holds for exponential families due to the continuity and strict convexity properties detailed earlier.
\end{proof}

\begin{lemma}
Let $S_{n}^{L}$ and $\cI_n^\star$ as in (\ref{eq:def_sampled_enough_sets}).
There exists $L_4$ with $\bE_{\bm F}[(L_4)^{\alpha}] < +\infty$ for all $\alpha > 0$ such that if $L \ge L_4$, for all $n$ such that $S_{n}^{L} \neq \emptyset$,
\[
\forall (i, j) \in \cI_n^\star \times \left(S_{n}^{L} \setminus  \cI_n^\star \right), \quad	W_{n}(i, j) \geq  L D_{\bm F} \: ,
\]
where $ D_{\bm F} > 0$ is a problem dependent constant.
\end{lemma}

\begin{proof}
Let $S_{n}^{L}$ and $\cI_n^\star$ as in (\ref{eq:def_sampled_enough_sets}).
Assume that $S_{n}^{L} \neq \emptyset$.
If $S_{n}^{L} \setminus  \cI_n^\star $ is empty, then the statement is not informative.
Assume $S_{n}^{L} \setminus  \cI_n^\star $ is not empty.
Let $(i, j) \in \cI_n^\star \times \left(S_{n}^{L} \setminus  \cI_n^\star \right)$.
We can now use $\{i, j\} \subseteq S_{n}^{L}$ and Lemma~\ref{lem:SPEF_W_lower_bound}.
\end{proof}

Lemma~\ref{lem:small_transportation_cost_undersampled_arms'} gives an upper bound on the transportation costs between a sampled enough arm and an under-sampled one.
\begin{lemma} \label{lem:small_transportation_cost_undersampled_arms'}
	Let $S_{n}^{L}$ as in (\ref{eq:def_sampled_enough_sets}). There exists $L_5$ with $\bE_{\bm F}[(L_5)^{\alpha}] < +\infty$ for all $\alpha > 0$ such that for all $L \geq L_5$ and all $n \in \N$,
	\[
	\forall  (i,j) \in  S_{n}^{L} \times \overline{S_{n}^{L}} , \quad 	W_{n}( i, j) \leq  L (2W_d + D_1 + D_2 W_\mu)  \: ,
	\]
where $D_1 > 0$ and $D_2 > 0$ are problem dependent constants and $W_d, W_\mu$ are the random variables defined in Lemma~\ref{lem:W_concentration_spef}.
\end{lemma}
\begin{proof}

Let $(i, j) \in S_{n}^{L} \times \overline{S_{n}^{L}}$ ($i$ is sampled more than $L$ times, $j$ is not). Taking $u = \mu_{n,i}$ yields
\begin{align*}
	W_{n}(i, j) &= \inf_{u \in \mathbb{R}} \left\{ N_{n,i} \Kinf^{-}( \cT(F_{n,i}), u) + N_{n,j} \Kinf^{+}( \cT(F_{n,j}), u) \right\}	\\
	&\leq N_{n,j} \Kinf^{+}( \cT(F_{n,j}), \mu_{n,i})
	\leq L \Kinf^{+}( \cT(F_{n,j}), \mu_{n,i})
	\: ,
\end{align*}
where we used that $j \in \overline{S_{n}^{L}}$. 

By definition of $W_d$ and $W_\mu$, we have
\begin{align*}
\mu_{n,j} &\le \mu_j + W_\mu \log(e+N_{n,j})/N_{n,j} \: ,
\\
d(\mu_{n,j}, \mu_j) &\le W_d \log(e+N_{n,j})/N_{n,j} \: .
\end{align*}
The same is true for $i$.
Then the $\Kinf$ is bounded by
\begin{align*}
& \Kinf^{+}(\cT(F_{n,j}), \mu_{n,i}) = d(\mu_{n,j}, \mu_{n,i}) \le d\left(\mu_{n,j}, \mu_i + W_\mu \frac{\log(e+N_{n,i})}{N_{n,i}}\right)
\\
&= d(\mu_{n,j}, \mu_j) + d\left(\mu_j, \mu_i + W_\mu \frac{\log(e+N_{n,i})}{N_{n,i}}\right)
\\&\quad + \left| (\mu_{n,j} - \mu_j)\left((\phi')^{-1}(\mu_j) - (\phi')^{-1}\left(\mu_i + W_\mu \frac{\log(e+N_{n,i})}{N_{n,i}}\right)\right) \right|
\\
&\le W_d \frac{\log(e+N_{n,j})}{N_{n,j}} + d\left(\mu_j, \mu_i + W_\mu \frac{\log(e+N_{n,i})}{N_{n,i}}\right)
\\&\quad + W_\mu \frac{\log(e+N_{n,j})}{N_{n,j}} \left\vert (\phi')^{-1}(\mu_j) - (\phi')^{-1}\left(\mu_i + W_\mu \frac{\log(e+N_{n,i})}{N_{n,i}}\right) \right\vert
\: .
\end{align*}
Since $x \mapsto \frac{\ln(e+x)}{x}$ is decreasing on $\R^{\star}_{+}$, we have $\frac{\log(e+N_{n,j})}{N_{n,j}} \le 2$ for $N_{n,j} \ge 1$ and $\frac{\log(e+N_{n,i})}{N_{n,i}} \le \frac{\ln(e+L)}{L}$ for $N_{n,i} \ge L$. For $\varepsilon > 0$ and $L \ge L_{\varepsilon}$ where $ L_{\varepsilon}  \ge W_\mu  \ln(e+L_{\varepsilon}) /\varepsilon$, we have
\begin{align*}
\Kinf^{+}(\cT(F_{n,j}), \mu_{n,i})
&\le2 W_d + d(\mu_j, \mu_i + \varepsilon) +  2 W_\mu \vert (\phi')^{-1}(\mu_j) - (\phi')^{-1}(\mu_i + \varepsilon) \vert
\: .
\end{align*}
Since the means belong to the interior of the interval $\cI$, there exists a $\varepsilon>0$ such that $d(\mu_j, \mu_i + \varepsilon)$ and $|(\phi')^{-1}(\mu_j) - (\phi')^{-1}(\mu_i + \varepsilon)|$ are finite. We take the corresponding $L$, which is sub-exponential, and obtain the result.
\end{proof}

\subsubsection{EB leader}

The proof of Properties~\ref{prop:leader_cdt_suff_explo} and \ref{prop:leader_cdt_convergence} is almost identical to that for bounded distributions. Only the lower bound on $W_n(i,j)$ is used, which has the same form for SPEFs and bounded distributions.

\subsubsection{TC challenger}

Conditioned on $\cF_n$ and given a leader $B_{n+1}$, the Transportation Cost (TC) challenger is defined in \eqref{eq:def_tc_based_challenger} as the arm with smallest transportation cost compared to the leader
\begin{equation*}
	C_{n+1}^{\text{TC}} \in \argmin_{j \neq   B_{n+1}} W_{n}(  B_{n+1},j)  \quad \text{,} \quad  \bP_{\mid n}[C_{n+1}^{\text{TC}} = j| B_{n+1} = i] = \frac{\indi{j \in \argmin_{k \neq i} W_{n}(i,k)}}{|\argmin_{k \neq i} W_{n}(i,k)|}  \: ,
\end{equation*}
and $\widehat C_{n+1}^{\text{TC}} \in \argmin_{j \neq \widehat  B_{n+1}} W_{n}(\widehat  B_{n+1},j)$.

\paragraph{Property~\ref{prop:challenger_cdt_suff_explo}}
We prove Property~\ref{prop:challenger_cdt_suff_explo} for $C_{n+1}^{\text{TC}}$ in Lemma~\ref{lem:TC_ensures_suff_explo'} by comparing the rates at which $W_n$ increases.

\begin{lemma} \label{lem:TC_ensures_suff_explo'}
Let $B_{n+1}$ be a leader satisfying Property~\ref{prop:leader_cdt_suff_explo}.
Let $(C_{n+1}^{\text{TC}}, \widehat C_{n+1}^{\text{TC}})$ as in (\ref{eq:def_tc_based_challenger}).
Let $U_n^L$ and $V_n^L$ as in (\ref{eq:def_undersampled_sets}) and $\mathcal J_n^\star = \argmax_{ i \in \overline{V_{n}^{L}}} \mu_{i}$.
There exists $L_6$ with $\bE_{\bm F}[L_6] < + \infty$ such that if $L \ge L_6$, for all $n$ such that $U_n^L \neq \emptyset$, $\widehat B_{n+1} \notin V_{n}^{L}$ implies $\widehat C_{n+1}^{\text{TC}} \in V_{n}^{L} \cup \left( \mathcal J_n^\star \setminus \left\{\widehat B_{n+1} \right\} \right)$.
\end{lemma}

\begin{proof}
The proof proceeds similarly to the one of Lemma~\ref{lem:TC_ensures_suff_explo}. The difference is the bounds on $W_n(i,j)$ that we get.
For all $L$ bigger than some random variable with finite expectation,
	\begin{align*}
		&\widehat B_{n+1} \in \mathcal J_n^\star \: , \\
		\forall (i,j) \in \mathcal J_n^\star \times \left(\overline{V_n^L} \setminus \mathcal J_n^\star\right), \quad &W_{n}(i, j) \geq  L^{3/4} D_{\bm F}   \: , \\
		\forall (i,j) \in \overline{U_n^L} \times U_n^L, \quad 	&W_{n}(i,j) \leq  \sqrt{L} (2W_d + D_1 + D_2 W_\mu) \: .
	\end{align*}
For all $L \ge L_7 \eqdef \left(2(2W_d + D_1 + D_2 W_\mu)/D_{\bm F}\right)^4$,
\begin{align*}
L^{3/4} D_{\bm F} > \sqrt{L} (2W_d + D_1 + D_2 W_\mu) \: .
\end{align*}
We now conclude that at least one under-sampled arm has transportation cost lower than all the ones that are much sampled, and proceed as in the proof of Lemma~\ref{lem:TC_ensures_suff_explo}.
\end{proof}

\paragraph{Property~\ref{prop:challenger_cdt_convergence}}
Lemma~\ref{lem:TC_ensures_convergence'} shows that the Property~\ref{prop:challenger_cdt_convergence} is satisfied by $C_{n+1}^{\text{TC}}$.
\begin{lemma} \label{lem:TC_ensures_convergence'}
		Assume Property~\ref{prop:suff_exploration} holds.
		Let $\epsilon > 0$.
		Let $B_{n+1}$ be a leader satisfying Property~\ref{prop:leader_cdt_convergence} and $C_{n+1}^{\text{TC}}$ as in (\ref{eq:def_tc_based_challenger}).
		There exists $N_7$ with $\bE_{\bm F}[N_7] < + \infty$ such that for all $n \geq N_7$ and all $i \neq i^\star(\bm F)$,
		\begin{equation} \label{eq:overshooting_implies_not_sampled_anymore_tc'}
			\frac{\Psi_{n,i}}{n} \geq w_{i}^{\beta} + \epsilon  \quad \implies \quad \bP_{\mid n}[C_{n+1}^{\text{TC}} = i \mid B_{n+1} = i^\star(\bm F)] = 0 \: .
		\end{equation}
\end{lemma}
\begin{proof}
This proof proceeds very similarly to the proof of lemma~\ref{lem:TC_ensures_convergence}.
Let $\epsilon > 0$ and $i^\star = i^\star(\bm F)$.
Using the definition of $C_{n+1}^{\text{TC}}$ in (\ref{eq:def_tc_based_challenger}), we have
\begin{align*}
	\bP_{\mid n}[C_{n+1}^{\text{TC}} = i \mid B_{n+1} = i^\star] = 0
	\quad & \iff \frac{1}{n}\left(W_{n}(i^\star,i) - \min_{j\neq i^\star}W_{n}(i^\star,j) \right) > 0 \: .
\end{align*}

Let $N_1$ as in Property~\ref{prop:suff_exploration}, then $N_{n,i} \geq \sqrt{\frac{n}{K}}$ for all $n \ge N_1$.
Since $i^\star$ is unique, we have $\Delta' \eqdef \min_{j\neq i^\star}|\mu_{i^\star} - \mu_{j}| > 0$. Let $u_0$ be the minimal distance from any mean $\mu_i$ to an end of the interval of means $\mathcal I$ and let $\Delta = \min\{\Delta', u_0\}$.
Lemma~\ref{lem:W_concentration_spef} yields that there exists $N_8 = Poly(W_\mu)$ such that for all $n \ge \max\{N_1, N_8\}$ and all $i \in [K]$, we have $|\mu_{n,i} - \mu_{i}| \leq \frac{\Delta}{4}$.
Therefore, for all $n \ge \max\{N_1, N_8\}$, $\argmax_{i \in [K]} \mu_{n,i} = \argmax_{i \in [K]} \mu_{i} = i^\star$ and for all $i \in [K]$, $\mu_{n,i} \in \cI$.

Let $\xi >0$. Since Property~\ref{prop:suff_exploration} holds and $B_{n+1}$ satisfies Property~\ref{prop:leader_cdt_convergence}, we can use the results from Lemma~\ref{lem:convergence_towards_optimal_allocation_best_arm}. Let $N_4$ defined in Lemma~\ref{lem:convergence_towards_optimal_allocation_best_arm}. We have $\left| \frac{N_{n,i^\star}}{n} - \beta \right| \leq \xi$ for all $n \ge \max \{ N_1, N_4\}$.

Let $i \neq i^\star$ such that $\frac{\Psi_{n,i}}{n} \geq w_{i}^{\beta} + \epsilon$.
Using Lemma~\ref{lem:subG_alloc}, there exists $N_9 = Poly(W_1) $, such that for all $n \ge \max\{N_1, N_9\}$, we have $\frac{N_{n,i}}{n} \geq w_{i}^{\beta} + \frac{\epsilon}{2}$.
Therefore, for all $n \ge \max \{N_1, N_4, N_8, N_9\}$, as in the proof of Lemma~\ref{lem:TC_ensures_convergence},
\begin{align*}
\frac{1}{n}\left(W_{n}(i^\star,i) - \min_{j\neq i^\star}W_{n}(i^\star,j) \right)
&\geq \inf_{\tilde \beta : \left|\tilde \beta - \beta \right| \leq \xi} G_{i}(\cT(\bm F_{n}), \tilde \beta)
\end{align*}
where
\begin{align*}
	G_{i}(\bm m, \tilde \beta) & = \inf_{u \in [0,B]} \left\{  \tilde \beta \Kinf^{-}( m_{i^\star}, u)  + \left( w_{i}^{\beta} + \frac{\epsilon}{2} \right) \Kinf^{+}( m_{i}, u)  \right\}  \\
	&\quad - \sup_{w \in \simplex: w_{i^\star} = \tilde \beta} \min_{j\neq i^\star} \inf_{u \in  \cI} \left\{  w_{i^\star}  \Kinf^{-}( m_{i^\star}, u)  + w_{j} \Kinf^{+}( m_{j}, u)  \right\} \: .
\end{align*}

Since all the means belong to a compact subset of $\cI$ for $n \ge \max \{N_1, N_4, N_8, N_9\}$, we can prove continuity of the functions $(\bm m, \tilde \beta) \mapsto G_{i}(\bm m, \tilde \beta)$ and $\bm m \mapsto \inf_{\tilde \beta : \left|\tilde \beta - \beta \right| \leq \xi} G_{i}(\bm m, \tilde \beta)$ in a similar way as was done for bounded distributions in Lemma~\ref{lem:continuity_results_for_analysis}.
Therefore, there exists $N_{10} = Poly(W_2)$ and $\xi_0$ such that for $n \ge N_{7} \eqdef \{N_1, N_4, N_8, N_9, N_{10}\}$ and all $\xi \leq \xi_0$,
\begin{align*}
	\inf_{\tilde \beta : \left|\tilde \beta - \beta \right| \leq \xi} G_{i}(\cT(\bm F_{n}), \tilde \beta) \geq \frac{1}{2}  \inf_{\tilde \beta : \left|\tilde \beta - \beta \right| \leq \xi} G_{i}(\cT(\bm F), \tilde \beta) \geq  \frac{1}{4} G_{i}(\cT(\bm F), \beta) \: .
\end{align*}

At the $\beta$-equilibrium all transportation costs are equal.
Therefore, by definition of $w^{\beta}$,
\begin{align*}
	&\sup_{w \in \simplex : w_{i^\star} = \beta} \min_{j \neq i^\star}  \inf_{u  \in [0,B]}\left\{w_{i^\star}\Kinf^-( F_{i^\star}, u) + w_j \Kinf^+(F_j, u) \right\} \\
	&\qquad= \min_{j \neq i^\star}  \inf_{u  \in [0,B]}\left\{\beta \Kinf^-(F_{i^\star}, u) + w^\beta_j \Kinf^+(F_j, u) \right\} \\
	&\qquad= \inf_{u  \in [0,B]}\left\{\beta \Kinf^-( F_{i^\star}, u) + w^\beta_i \Kinf^+(F_i, u) \right\} \\
	&\qquad< \inf_{u  \in [0,B]}\left\{\beta \Kinf^-( F_{i^\star}, u) + \left( w^\beta_i + \frac{\epsilon}{2} \right) \Kinf^+(F_i, u) \right\}
\end{align*}
where the strict inequality is obtained because the transportation costs are increasing in their allocation arguments (proved in a similar way as Lemma~\ref{lem:inf_Kinf_increasing_in_w}).
Therefore, we have $G_{i}(\cT(\bm F), \beta) > 0$. This yields that $W_{n}(i^\star,i) > \min_{j\neq i^\star} W_{n}(i^\star,j)$.
As all moments of $W_1$ and $\lambda W_\mu$ are finite, we have $\bE_{\bm F}[N_i] < + \infty$ for $i \in \{8,9,10\}$. Hence this is also the case for $N_7$.
\end{proof}

\subsubsection{TCI challenger}

Showing Properties~\ref{prop:challenger_cdt_suff_explo} and~\ref{prop:challenger_cdt_convergence} for the TCI challenger uses the same arguments as for the proof of Lemma~\ref{lem:TC_ensures_suff_explo'} and~\ref{lem:TC_ensures_convergence'}.
Coping for the penalization term $\ln N_{n,j}$ is done similarly as when we obtained Lemmas~\ref{lem:TCI_ensures_suff_explo} and~\ref{lem:TCI_ensures_convergence} by adapting the proof of Lemmas~\ref{lem:TC_ensures_suff_explo} and~\ref{lem:TC_ensures_convergence}.
Since there is no new arguments, we omit the proof.

%% file: sections/appendix_implementation_add_experiments.tex

\section{Implementation details and additional experiments}
\label{app:additional_experiments}

After presenting the implementations details in Appendix~\ref{app:ss_implementation_details}, we display supplementary experiments in Appendix~\ref{app:ss_supplementary_experiments}.

\subsection{Implementation details}
\label{app:ss_implementation_details}

In non-parametric settings, the algorithms are inherently more costly than their counterpart in parametric settings.
First, the memory cost is linear in time as we need to maintain the whole history $\cF_n$ in memory.
This a direct consequence of the lack of sufficient statistics to summarize $\cF_{n}$.
In contrast, the memory cost is constant for single-parameter exponential families settings.
Second, the computational cost per iteration of many algorithms is at least linear in time: a good algorithm should leverage all the observations to make a decision.

We detail below the most relevant implementation details regarding the sampling rules and discuss their computational cost.
As mentioned above, the implemented algorithms for the bounded setting are computational expensive by nature.
However, we aim at promoting the algorithm(s) achieving good empirical performance in terms of empirical stopping time at a reasonable computational cost.

\paragraph{Stopping-Recommendation pair}
The recommendation rule $\hat \imath_n \in \argmax_{i \in [K]} \mu_{n,i}$ has a $\cO(K)$ computational cost.
This is achieved by simply maintaining the cumulative sum of the observations $\sum_{t \leq n} \indi{i  = I_t} X_{t,i}$ for each arm.

In contrast to the recommendation rule, the stopping rule defined in (\ref{eq:def_stopping_time}) is computationally expensive.
At each time $n$, we need to compute $K-1$ transportation costs $W_{n}(\hat \imath_n, j)$ for $j \neq \hat \imath_n$.
While each one can be evaluated efficiently for single-parameter exponential families (see below), this is not the case for bounded distributions where
\begin{align*}
		& W_n(\hat \imath_n , j)  = \inf_{x \in [\mu_{n,j}, \mu_{n, \hat \imath_n}]} g_n(\hat \imath_n , j, x) \: , \\
		& g_n(\hat \imath_n , j, x)  = N_{n,\hat \imath_n}\Kinf^-(F_{n,\hat \imath_n},x) + N_{n,j}\Kinf^+(F_{n,j},x) \: .
\end{align*}

Using Lemmas~\ref{lem:strict_convex_sum_Kinf} and~\ref{lem:inf_Kinf_restricted_to_open_mean_interval}, the function $x \mapsto g_n(\hat \imath_n , j, x)$ is strictly convex and admit a unique minimizer in $[\mu_{n,j}, \mu_{n, \hat \imath_n}]$.
Lemma~\ref{lem:differentiability_Kinf} gives a formula for the derivatives of $x \mapsto \Kinf^{\pm}(F, x)$.
Unfortunately, $\lambda^{\pm}_{\star}(F,x)$ are often defined implicitly (Lemma~\ref{lem:explicit_implicit_solutions}), hence we can't leverage this knowledge and use first-order optimization methods.
Therefore, in order to compute $W_n(\hat \imath_n , j)$, we rely on a zero-order optimization algorithm designed to minimize a univariate function on a bounded interval.
In practice, we use Brent's method, which is implemented in the \texttt{Optim.jl} package under \texttt{Julia 1.7.2}.

To obtain $g_n(\hat \imath_n , j, x)$ for a given $x$, we need to compute $\Kinf^-(F_{n,\hat \imath_n},x)$ and $\Kinf^+(F_{n,j},x)$.
This is made tractable thanks to their dual formulation.
Taking $\Kinf^+(F_{n,j},x)$ as an example, Theorem~\ref{thm:Kinf_duality} yields
\[
N_{n,j} \Kinf^+(F_{n,j},x)= \sup_{\lambda \in [0, 1]} \sum_{ k \in  [N_{n,i}]} \ln \left( 1 - \lambda \frac{X_{k,i} - x}{B - x}\right) \: ,
\]
where $(X_{k,i})_{k \in [N_{n,i}]}$ denotes the samples collected from arm $i$ at time $n$.
As the function is strictly concave (Lemma~\ref{lem:H_strict_concave}), we will also use Brent's method to compute it's maximum.
Each computation requires to sum over the $N_{n,i}$ samples collected by arm $i$.
Therefore, the computational cost is at least linear in time.
Since we can compute the derivative, it would be possible to use first-order optimization algorithms.
While this will improve the number of iterations required to reach convergence, it is not clear that the overall computational cost will be reduced since those gradient computations are also linear in time.

\paragraph{Top Two sampling rules} We discuss the computational cost of the EB and TS leader, as well as the TC and RS challenger.

The EB leader has virtually no computational cost since it uses the candidate answer $\hat \imath_n$.
Likewise, the TC challenger can also leverage the computations from the stopping rule (\ref{eq:def_stopping_time}).
When $B_{n+1} = \hat \imath_n$, $C_{n+1}^{\text{TC}} \in \argmin_{j \neq \hat \imath_n} W_n(\hat \imath_n , j)$ was already computed in (\ref{eq:def_stopping_time}).
When $B_{n+1} \neq \hat \imath_n$, we have $C_{n+1}^{\text{TC}} \in \argmin_{j \neq B_{n+1}} W_n(B_{n+1}, j) = \left\{ j \neq B_{n+1} \mid \mu_{n,j} \ge \mu_{n,B_{n+1}}\right\}$, which contains at least $\hat \imath_n$.
Therefore, the TC challenger has virtually no computational cost when paired with (\ref{eq:def_stopping_time}).

The TS leader and the RS challenger use a sampler $\Pi_n$.
In single-parameter exponential families, $\Pi_n$ is a posterior distribution which can be computed in constant time by updating the posterior $\cO(K)$ parameters.
However, for bounded distributions, the Dirichlet sampler relies on the whole history $\cF_{n}$.
Therefore, for each arm $i$, we need to define and sample from a Dirichlet distribution with $N_{n,i} + 2$ parameters.
This has linear computational cost per iteration.
The TS leader requires only one Dirichlet observation per arm, hence it has constant cost once the Dirichlet distributions are defined (which is computationally expensive).

For the RS challenger, we re-sample until $B_{n+1}$ is not an arm with highest mean in the corresponding vector of observations. The computational cost is proportional to the number of re-sampling steps which is on average $1/(1-a_{n,B_{n+1}})$. Hence the computational cost can be very high when $\Pi_n$ has converged, that is when $a_{n,i}\approx 0$ for all $i \neq i^\star$. The analysis of the TS leader and the RS challenger reveals that this convergence is exponential, with a rate close to $T^\star_{\beta}(\bm F)$. Therefore, when $B_{n+1} = i^\star$, it is highly unlikely to observe a vector $\theta$ for which $i^\star$ is not the best arm, and the average number of re-sampling steps is exponential.

Based on extensive experiments, we also have empirical evidence that the RS challenger has higher computational cost than the TC challenger.
For Bernoulli instances, when using the stopping threshold defined in (\ref{eq:def_kinf_threshold_glr}), the maximum number of re-sampling steps (set arbitrarily to $10^{6}$) was always reached for large time $n$.
As a direct consequence, the computational cost of the RS challenger explodes in those cases, e.g. $10^4$ times slower than the TC challenger.
As we use a uniform sampling when the maximum number of re-sampling steps is reached, the achieved empirical stopping time is also higher than for the variant using the TC challenger.
In Appendix~\ref{app:sss_expe_rs_challenger}, we perform experiments with the RS challenger for a heuristic stopping threshold defined in (\ref{eq:def_GK16_stopping_threshold}).
This yields four times smaller empirical stopping time compared to using (\ref{eq:def_kinf_threshold_glr}), see Appendix~\ref{app:ss_supplementary_experiments}.


\paragraph{Other sampling rules}
In LUCB-based sampling rules, we need to compute upper and lower confidence bounds based on the inversion of a distance function.
For KL-LUCB, it requires inverting the KL divergence of Bernoulli distributions.
This can be done efficiently by using a binary search algorithm.
For $\Kinf$-LUCB, we need to inverse $\Kinf^{\pm}$, also by using a binary search.
As explained above, computing $\Kinf^{\pm}$ for the empirical cdfs yields a linear computation cost.
Therefore, $\Kinf$-LUCB will be significantly worse than KL-LUCB in terms of computational cost.
However, $\Kinf$-LUCB yields order of magnitude smaller empirical stopping time in the DSSAT instances compared to KL-LUCB.

The sampling rule $\Kinf$-DKM is inspired by DKM \citep{Degenne19GameBAI}, with only one learner on $\simplex$ instead of $K$ learners.
We replace the KL divergences by $\Kinf^{\pm}$ as we are in the bounded setting, hence it will be more costly than DKM for single-parameter exponential families.
Given the allocation $w_{n}$ returned by the learner (e.g. AdaHedge), computing the most confusing alternative parameter has the same computational cost as evaluating the stopping rule (\ref{eq:def_stopping_time}).
For bounded distributions, it is not clear how to define the optimistic bonuses.
Therefore, we replace it by forced exploration, which yields worse empirical stopping times.

\paragraph{Adaptive choice of $\beta$}
Based on the theoretical lower bound, Top Two algorithms with a fixed allocation can be at best asymptotically $\beta$-optimal, not asymptotically optimal (meaning reaching $T^\star(\bm F)$).
To achieve asymptotic optimality, the fixed allocation should match the optimal allocation $\beta^\star = \argmin_{\beta \in (0,1)} T^{\star}(\bm F)$.
As $\beta^\star$ is unknown, it should be learned from observations.
Therefore, this desired adaptive Top Two algorithm should use an adaptive choice of $\beta$ which converges towards $\beta^\star$.
Proving optimality for adaptive Top Two algorithms is an interesting open problem, which is still unsolved even for Gaussian bandits.
The very recent paper \cite{wang_2022_OptimalityConditions} proposes an update mechanism for $\beta$, but they study it only empirically and they don't provide any theoretical guaranty for that scheme.



\paragraph{Optimal allocation oracles}
In the following, we consider bounded distributions $\bm F$ having a discrete support.
Computing the optimal allocation $w^\star(\bm F)$ is computationally very expensive for bounded setting.
Even for single-parameter exponential families, this can be very demanding.
For more complex structure such as top-$k$ identification or combinatorial bandits \citep{jourdan_2021_EfficientPureExploration}, advanced saddle-point algorithms are needed to obtain efficient implementation, even for Gaussian distributions.
Therefore, Track-and-Stop algorithms computing $w^\star(\bm F_n)$ at each time $n$ should not be used to tackle the bounded distributions setting.
While it is costly, we can still compute $w^\star(\bm F)$ for the true distribution $\bm F$ once.
First, it allows to obtain a lower bound on the empirical stopping in $T^\star(\bm F) \kl (\delta, 1-\delta)$.
Second, we can implement the oracle algorithm (referred to as ``fixed'' in the experiments) which tracks the true optimal allocation $w^\star(\bm F)$.

The strategy to compute $w^\star(\bm F)$ is similar as done for single-parameter exponential families \cite{GK16}.
In the following, we describe the heuristic algorithm mimicking the behavior of the oracle in \cite{GK16}.
Let $i^\star = i^\star(\bm F)$.
Let $G_{j}(x)$ defined in Lemma~\ref{lem:properties_characteristic_times}. Recall that $x_j(y) = G_{j}^{-1}(y)$ and $u_j(x)$ is defined as the minimizer yielding $G_{j}(x)$.
If $x_j$ and $u_j$ can be differentiated, we could directly derive (\ref{eq:reformulation_optimization_with_y}) in Lemma~\ref{lem:properties_characteristic_times}.
Additional manipulations \cite{GK16} yield the reformulation of the optimization problem defining $T^\star(\bm F)$ as solving $F(y)=1$ where
\begin{equation} \label{eq:heutistic_optimization_problem_defining_optimal allocation}
\forall y \in \left[ 0, \min_{j \neq i^\star} \Kinf^{-}(F_{i^\star}, \mu_{j})\right), \quad F(y) = \sum_{j \neq i^\star} \frac{\Kinf^{-}(F_{i^\star}, u_j(x_j(y)))}{\Kinf^{+}(F_j, u_j(x_j(y)))}
\end{equation}
is a strictly increasing increasing function such that $F(0) = 0$ and $\lim_{y \to \min_{j \neq i^\star} \Kinf^{-}(F_{i^\star}, \mu_{j})} F(y) = + \infty$.
We use nested binary searches to solve $F(y)=1$.
The outer binary search is done on $y \in \left[ 0, \min_{j \neq i^\star} \Kinf^{-}(F_{i^\star}, \mu_{j})\right)$ (see Lemma~\ref{lem:funny_property_for_optimal_allocation_algorithm}).
The inner binary searches are done to compute $x_j(y)$ for all $j \neq i^\star$.
To compute $u_j(x)$, we use the same procedure than described in the stopping rule.
While proving that $x_j$ and $u_j$ are differentiable still eludes us, we conjecture it to be true. Therefore, this heuristic optimal allocation algorithm gives a good estimate of $T^\star(\bm F)$ and $w^\star(\bm F)$.

When $\bm F$ has a non-discrete support, computing numerically $\Kinf^{\pm}$ by using the dual formulation would require having access to an oracle outputting $\bE_{F_i}[\log(1 - \lambda(X - u))]$ (resp. $\bE_{F_i}[\log(1 + \lambda(X - u))]$) for all $i \in [K]$, $u > \mu_i$ and $\lambda \in \left[ 0, \frac{1}{B - u}\right]$ (resp. $\lambda \in \left[ 0, \frac{1}{u}\right]$).
For continuous distribution, those integrals could be computed by numerical integration.
Instead we adopt a Monte-Carlo approach and use a discrete distribution $\bm{\widehat{F}}$ sampled from $\bm F$.
By Lemma~\ref{lem:subG_cdf}, taking a sufficiently large number of samples ensures that $\max_{i\in [K]}\| \widehat{F}_{i} - \bm F_{i}\|_{\infty}$ is small.
Intuitively, this should ensures that $w^\star(\bm{\widehat{F}})$ and $T^\star(\bm{\widehat{F}})$ are a good approximation of $w^\star(\bm F)$ and $T^\star(\bm F)$.
Formalizing this intuition theoretically requires having access to a Lipschitz constant for the $\| \cdot \|_{\infty}$.
To our knowledge, proving that $\bm F \mapsto w^\star(\bm F)$ and $\bm F \mapsto T^\star(\bm F)$ are Lipschitz is still an open problem, even for Gaussian distributions.

\paragraph{Efficient implementation for Bernoulli}
For Bernoulli distributions, the computational cost is greatly reduced.
In the following $\kl$ denotes the KL divergence for Bernoulli distributions.

First, the stopping rule can be computed in $\cO(K)$ since we have have closed form formulas for $\Kinf^{\pm}$, i.e. $\Kinf^+(F_{n,i}, u) = \kl(\mu_{n,i}, \max\{\mu_{n,i}, u\})$ and $\Kinf^-(F_{n,i},u) = \kl(\mu_{n,i}, \min\{\mu_{n,i}, u\})$, and for the closest alternative parameter, i.e.
\[
\argmin_{x \in [\mu_{n,j}, \mu_{n, \hat \imath_n}]} g_n(\hat \imath_n , j, x) = \frac{N_{n, \hat \imath_n} \mu_{n, \hat \imath_n} + N_{n, j} \mu_{n,j}}{N_{n, \hat \imath_n} + N_{n, j}} \: .
\]
By the same arguments, we have a more efficient computation of the optimal allocation $w^\star(\bm F)$.
As the differentiability of $x_j$ and $u_j$ holds in this setting, the optimal allocation oracle is theoretically validated.

Second, the sampler $\Pi_n$ can be rewritten as a Beta distribution with parameters $(c_{n,i} + 1, N_{n,i} - c_{n,i} +1)$, where $c_{n,i} = \left| \left\{ t \in [n] \mid I_{t} =i, X_{t,i} = B \right\} \right|$.
This leverages the fact that we can group the observations into two values $\{0, B\}$, and a classic results on marginals of Dirichlet distributions.
While this reduces the cost of sampling observations from $\Pi_n$, the computational cost of the re-sampling procedure (discussed above) still remains.

For Bernoulli distributions, $\beta$-TS-TC coincide with T3C \cite{Shang20TTTS} and $\beta$-TS-RS coincide with TTTS \cite{Russo2016TTTS}.
While the algorithms were already known in this setting, we are the first to prove they achieve asymptotic $\beta$-optimality for Bernoulli distributions.
In \cite{Shang20TTTS}, the authors only provide a proof for Gaussian distributions.

\paragraph{Decision Support System for Agrotechnology Transfer (DSSAT)}
\href{https://dssat.net/}{DSSAT}\footnote{DSSAT is an Open-Source project maintained by the DSSAT	Foundation.} \cite{hoogenboom2019dssat} is a crop modeling software that has been developed (mainly) to help agricultural production in developing countries.
This simulator provides a standardized way to generate realistic crop yield for different plants and soil conditions, harnessing more than 30 years of historical field data on 42 different crops.
Simulations are based on complex biophysical models, and take many parameters into account: local soil conditions, genetics, and crop management policy (e.g planting date, fertilization policy).
Our experiment is inspired by the one proposed in \cite{baudry21a}: we consider maize fields, and fixed challenging soil conditions (poor water retention and fertility), that are close to the conditions endured by small-holder farmers in Sub-Saharan Africa.
As the biophysical models are fixed and the weather is sampled by the environment, in this example the learner can play on human decisions such as the planting date and fertilization policy.
To simplify, we only consider the planting date, which already provides a difficult problem as the distributions represented in Figure~\ref{fig:dsat_instances}(b) and Figure~\ref{fig:dsat_instances_gk16} show.
In those figures, each arm corresponds to a yield distribution with all parameters fixed, except for the plantings that are $\sim 20$ days apart from each other, ranging on two months.
Our objective is to perform \textit{in sillico} experiments to compare the performance of different Best Arm Identification bandit algorithms, to help a potential group of farmers to choose an algorithm to use for future real-world experiments.

As calling the simulator is computationally intensive and as we want to perform Monte-Carlo simulations we used the code provided in \cite{baudry21a} to generate $10^6$ empirical data from each distribution and store these points in a csv file that is provided with the code of this paper.
We further re-scale the distributions in $[0, 1]$, which is equivalent as setting the known upper bound as the maximum value sampled by the simulator in our data collection process.
Then, a call to an arm simply consists in sampling one of these points uniformly at random.
We think this approximation is sufficient to reflect the difficulty of the problem, while being less demanding in terms of computation time.

\begin{remark}
	While this setting is simplified over a real-world experiment, it is an interesting and highly non-trivial first step to build more realistic algorithms taking in account contextual information, batch feedback, risk-aversion of farmers at a group and individual level (see \cite{baudry21a}), \dots
	Furthermore, if the asymptotic guarantees of bandit algorithms make them non-realistic for a single farmer (one data point every 3-6 months), they may exhibit tremendous progress over uniform sampling for a group of farmers (typically a few hundreds of data points every 3-6 months) conducting an experiment for several years.
\end{remark}

\paragraph{Reproducibility}
Our code is implemented in \texttt{Julia 1.7.2}, and the plots are generated with the \texttt{StatsPlots.jl} package.
Optimizations are performed based on the Brent's method available in the \texttt{Optim.jl} package.
Other dependencies are listed in the \texttt{Readme.md}.
The \texttt{Readme.md} file also provides detailed julia instructions to reproduce our experiments, and we provide a \texttt{script.sh} to run them all at once.
The general structure of the code (and some functions) is taken from the \href{https://bitbucket.org/wmkoolen/tidnabbil}{tidnabbil} library.\footnote{This library was created by \cite{Degenne19GameBAI}, see https://bitbucket.org/wmkoolen/tidnabbil. No license were available on the repository, but we obtained the authorization from the authors.}

\subsection{Supplementary experiments}
\label{app:ss_supplementary_experiments}

As in Section~\ref{sec:experiments}, we consider a moderate confidence regime in which $\delta = 0.1$ and Top Two algorithm with $\beta =0.5$.

\paragraph{Heuristic GK16 threshold}
While the stopping threshold defined in (\ref{eq:def_kinf_threshold_glr}) ensures $\delta$-correctness of the stopping rule (\ref{eq:def_stopping_time}), it is conservative in practice.
We denote it as TT (Theoretical Threshold).
Aiming at running large scale experiments, we consider the GK16 heuristic threshold defined in \cite{GK16},
\begin{equation} \label{eq:def_GK16_stopping_threshold}
	\beta^{\text{GK16}}(n,\delta) = \ln \left( \frac{1+\ln(n)}{\delta} \right) \: .
\end{equation}
Using GK16 yields an empirical error lower than $\delta$, even if it has no $\delta$-correctness guaranty.
This threshold was extensively used in the BAI literature to conduct experiments.
This idealized dependency in $(n,\delta)$ can be achieved for single-parameter exponential families \citep{KK18Mixtures}.
In this work, we show that $\ln n$ can be achieved for bounded distributions.
Whether it is possible to achieve $\ln \ln n$ for bounded distributions is an interesting open research question.
As it would require more sophisticated results on martingales to obtain such concentrations results for $\Kinf$, we leave it for future work.

\begin{figure}[ht]
	\centering
	\includegraphics[width=0.32\linewidth]{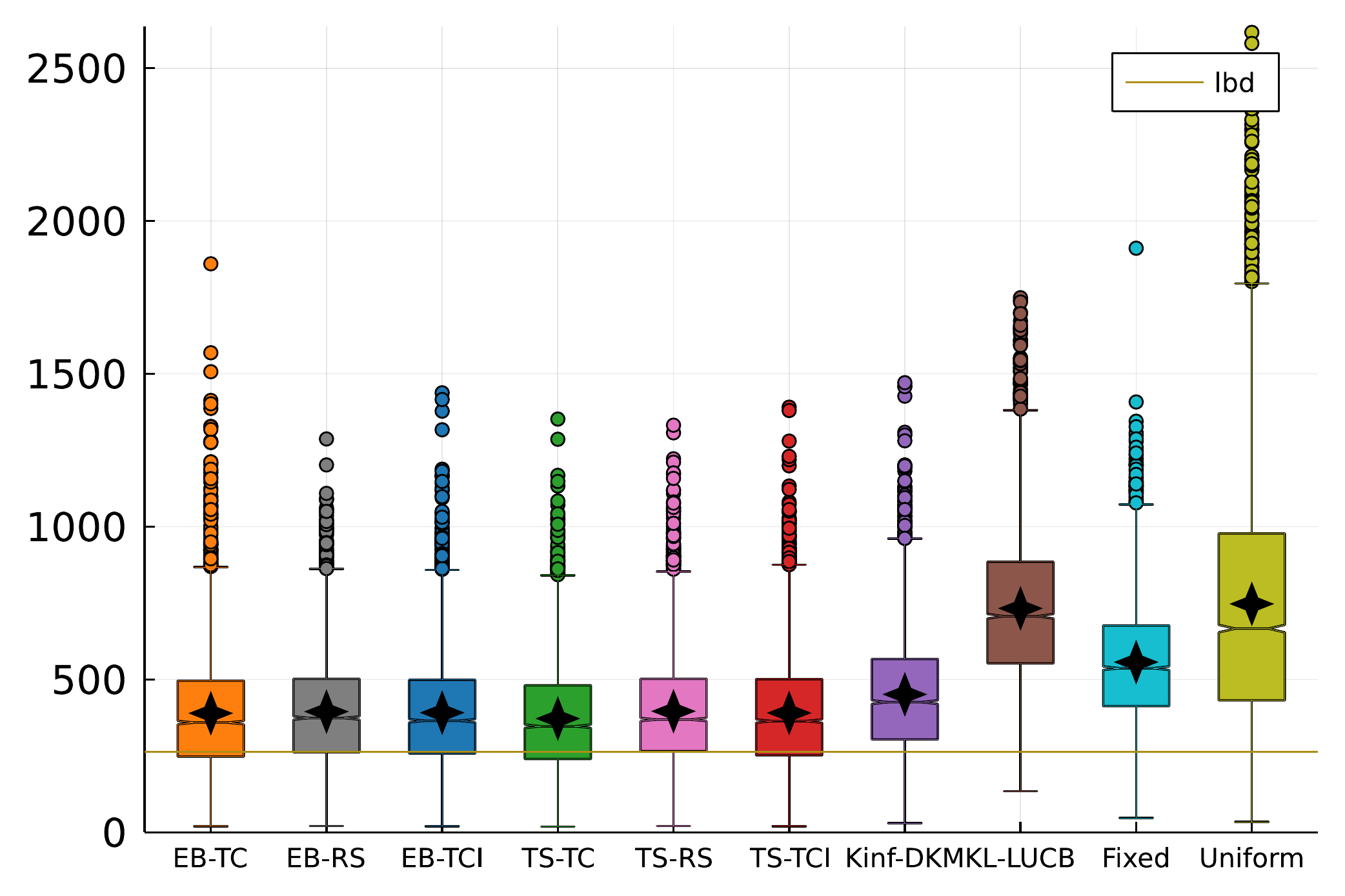}
	\includegraphics[width=0.32\linewidth]{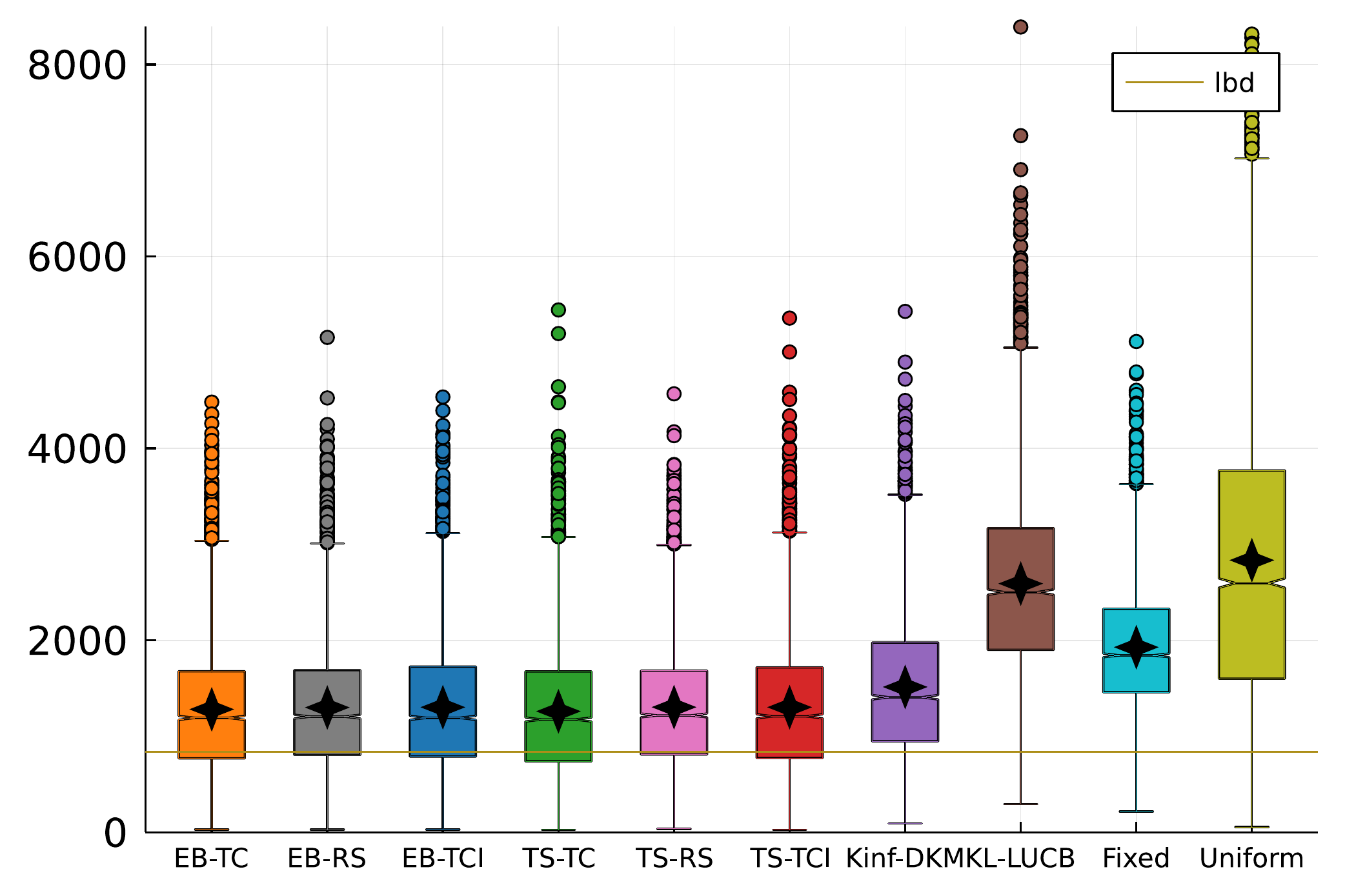}
	\includegraphics[width=0.32\linewidth]{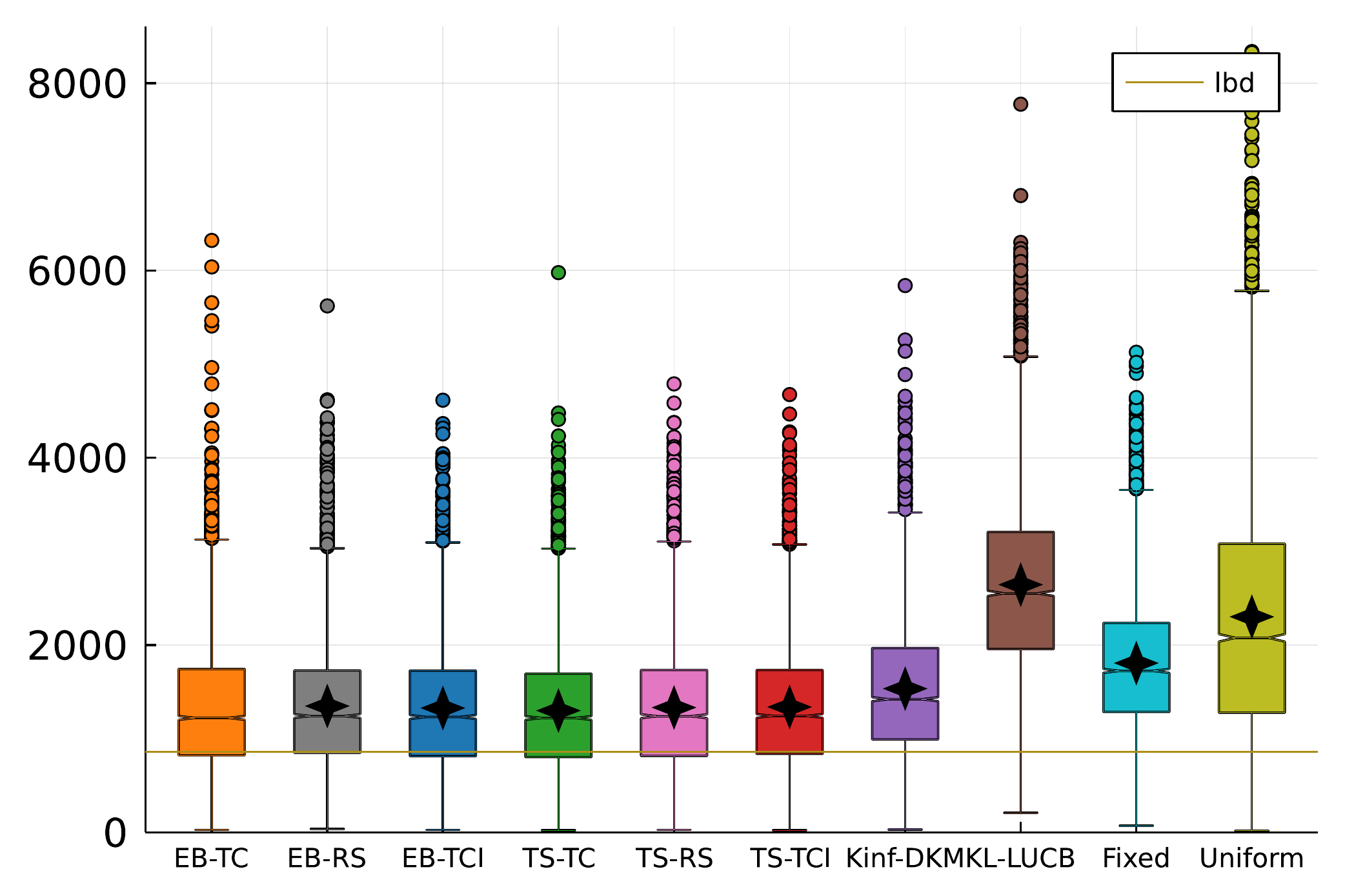} \\
	\includegraphics[width=0.32\linewidth]{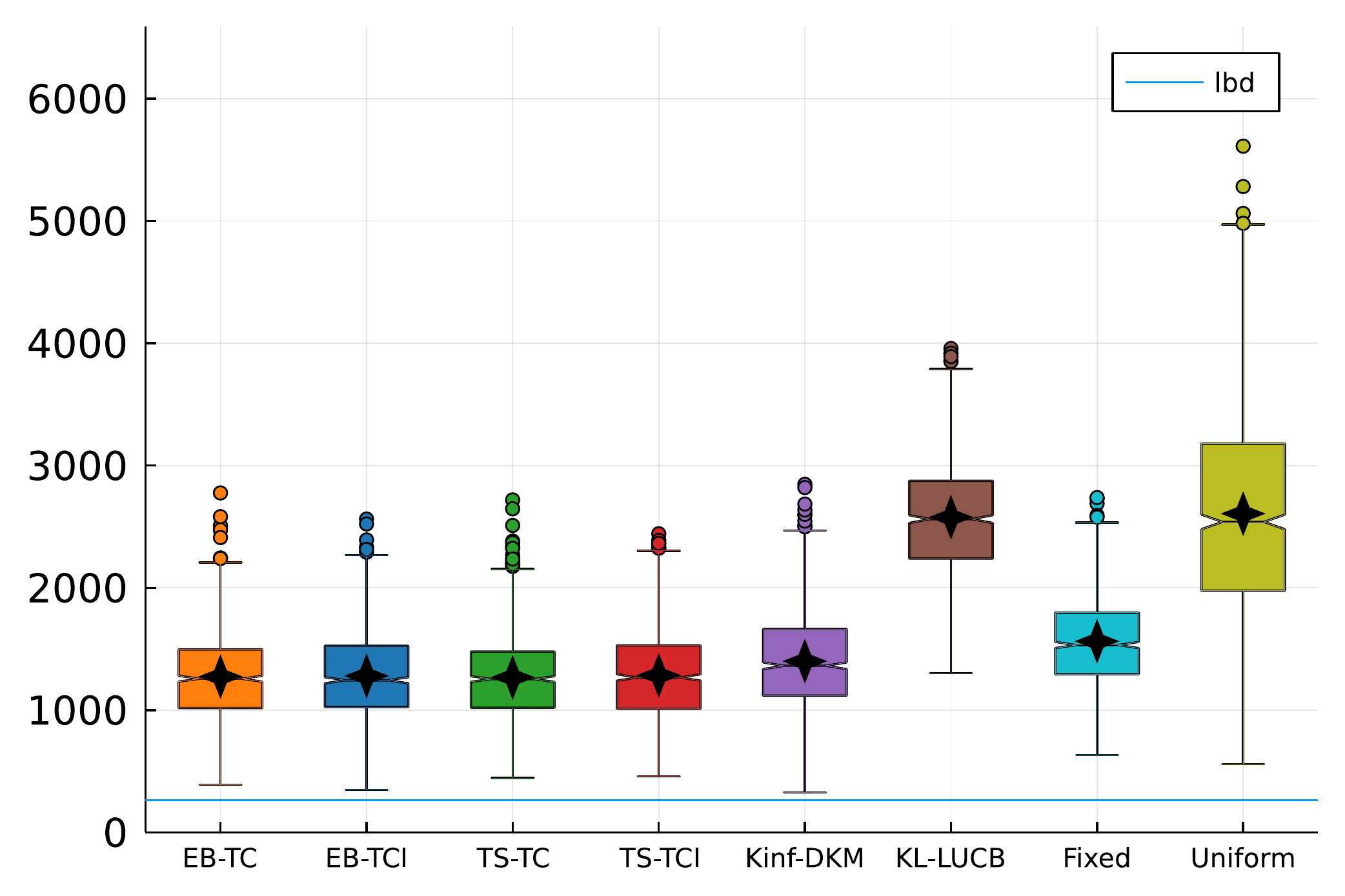}
	\includegraphics[width=0.32\linewidth]{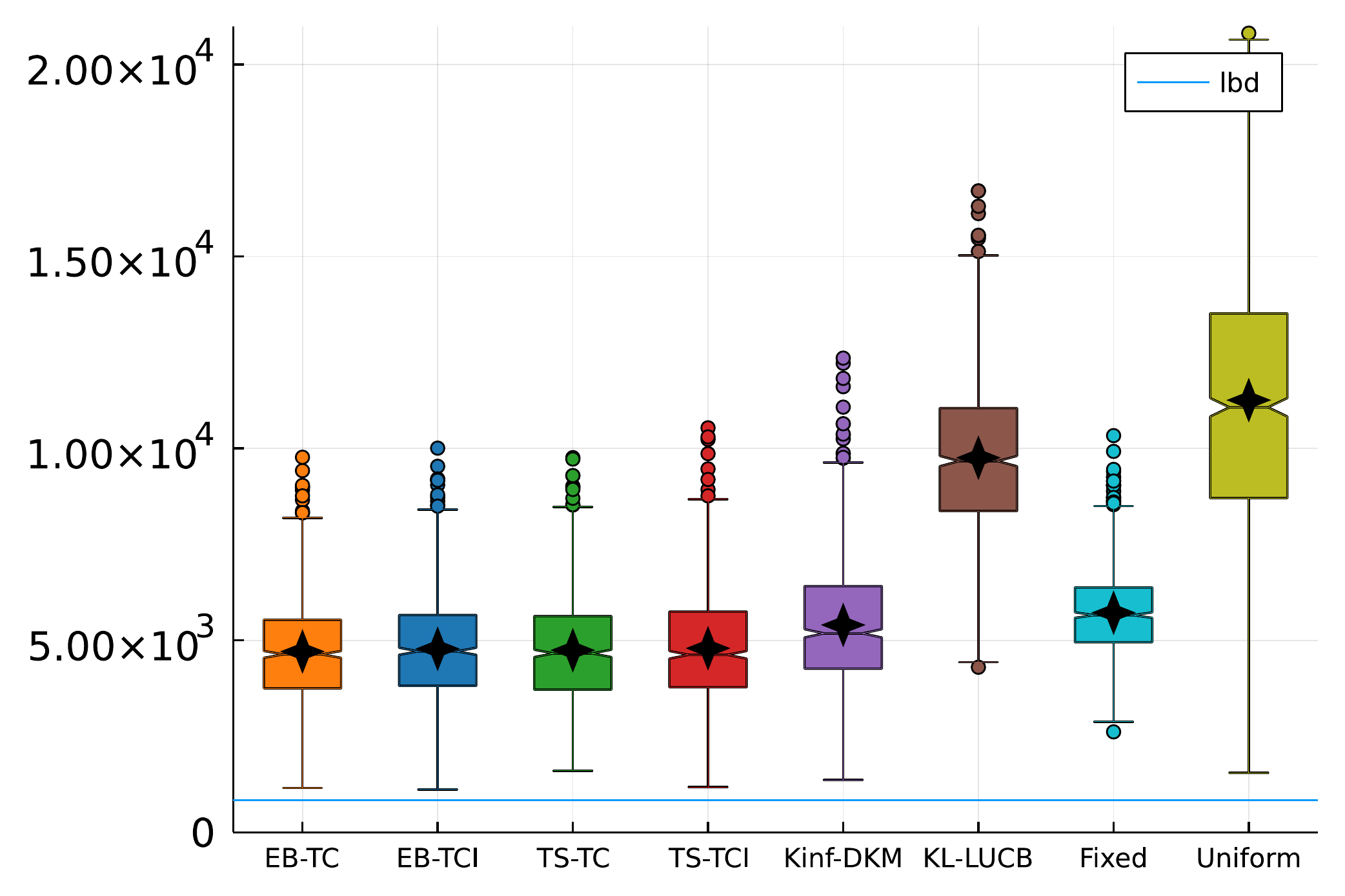}
	\includegraphics[width=0.32\linewidth]{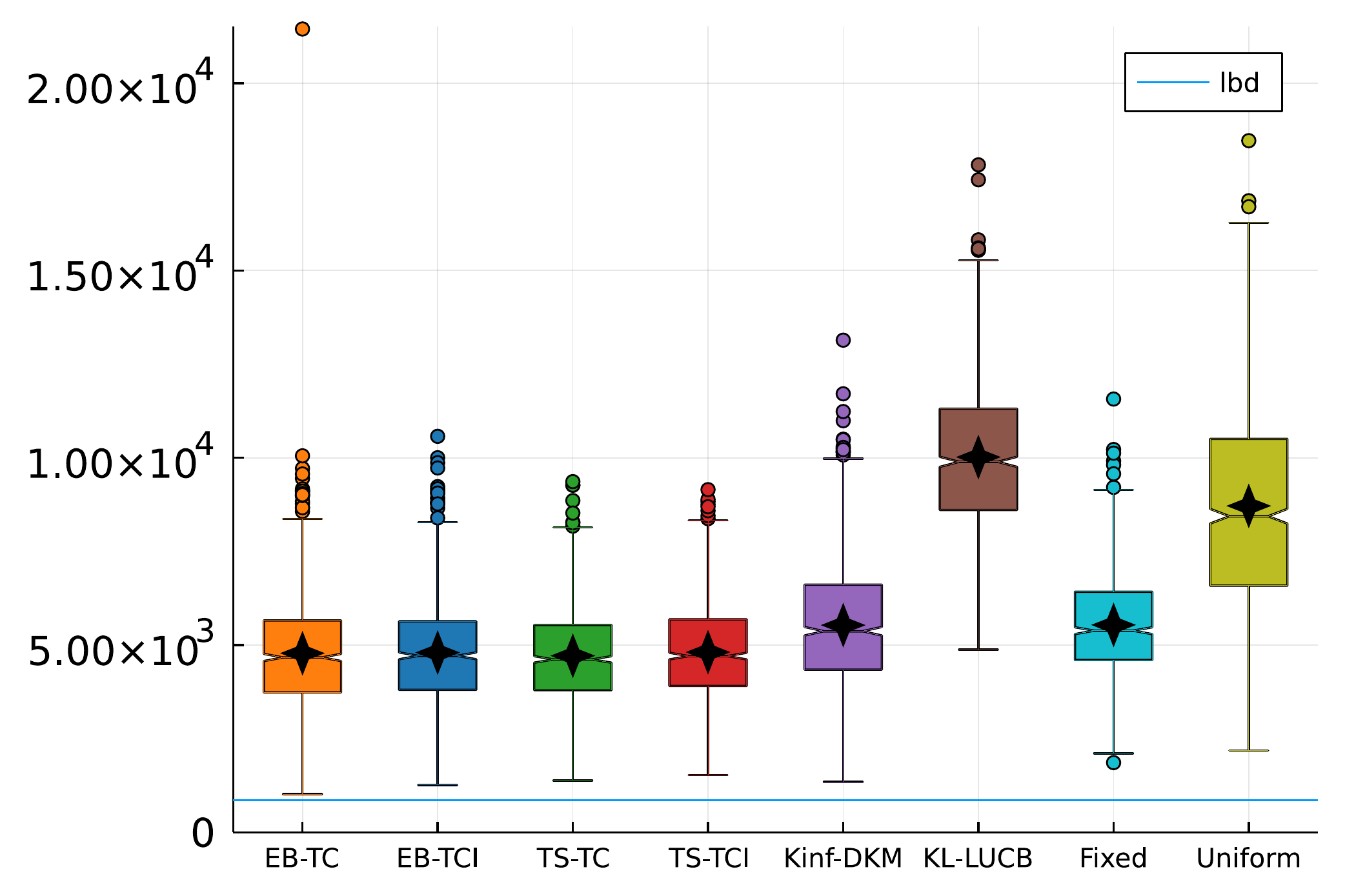}
	\caption{Empirical stopping time on the easy, hard and $3^{\text{rd}}$-equal instances (left to right) for GK16 (top) or TT (bottom) thresholds.}
	\label{fig:simple_bernoulli_instances_gk16}
\end{figure}

On simple Bernoulli instances, we compare the performance of the algorithms from Section~\ref{sec:experiments} for the stopping rule (\ref{eq:def_stopping_time}) using GK16 or TT.
We consider three instances: the \textit{easy} instance with $\mu = \left(0.7, 0.5, 0.4, 0.3, 0.2\right)$, the \textit{hard} instance with $\mu = \left(0.7, 0.6, 0.5, 0.4, 0.3\right)$ and the \textit{$3^{\text{rd}}$-equal} instance $\mu = \left(0.7, 0.6, 0.5, 0.5\right)$.
While we average our results over $5000$ runs for GK16, we only perform $1000$ runs for TT.

In Figure~\ref{fig:simple_bernoulli_instances_gk16}, the empirical stopping time when using GK16 is on average four times lower than when using TT.
As we discussed above, the computational cost of each iteration increases with the time $n$.
Therefore, this speed-up in stopping time naturally yields a speed-up in the averaged computational time per iteration.
In some of our experiments, the averaged computational time per iteration was divided by $10$.

In the following, all the experiments will be conducted with the GK16 heuristic threshold instead of the TT threshold.

\subsubsection{RS challenger}
\label{app:sss_expe_rs_challenger}

In addition of the Top Two algorithms from Section~\ref{sec:experiments}, we assess the empirical performance of instances using the RS challenger. Moreover, we detail more on the lack of robustness of $\beta$-EB-TC (large outliers), explained in Appendix~\ref{app:ss_beyond_all_distinct_means}.
As explained in Appendix~\ref{app:ss_implementation_details}, the RS challenger is computationally intractable for large $n$.
The experiments with RS could be ran only because we used GK16 instead of TT, hence dividing the empirical stopping time by four on average.
Due to their respective flaws, we do not recommend to use those algorithms in practice, even though they enjoy the same theoretical guarantees as $\beta$-EB-TCI, $\beta$-TS-TC and $\beta$-TS-TCI.

\begin{figure}[ht]
	\centering
	\includegraphics[width=0.495\linewidth]{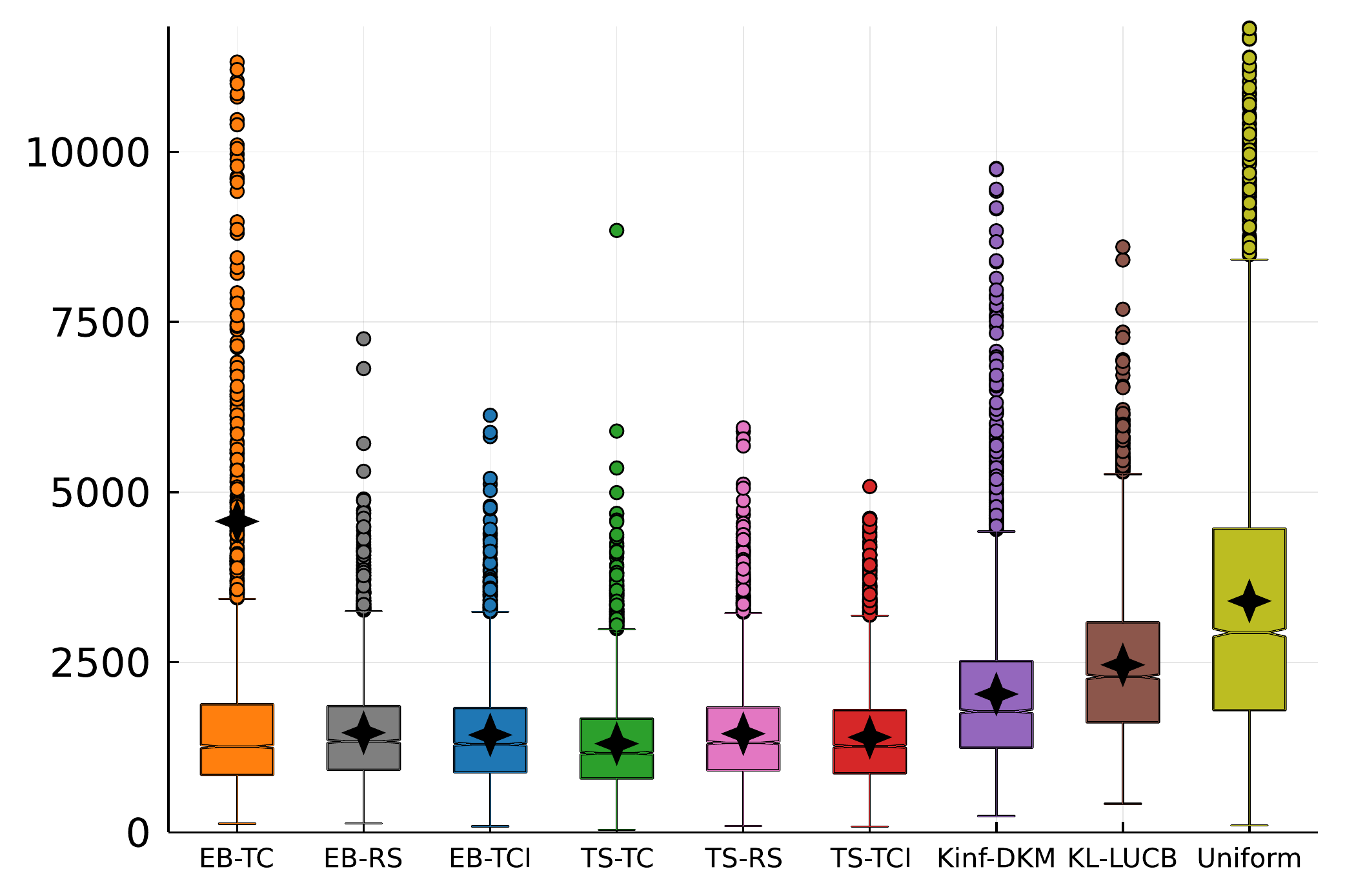}
	\includegraphics[width=0.495\linewidth]{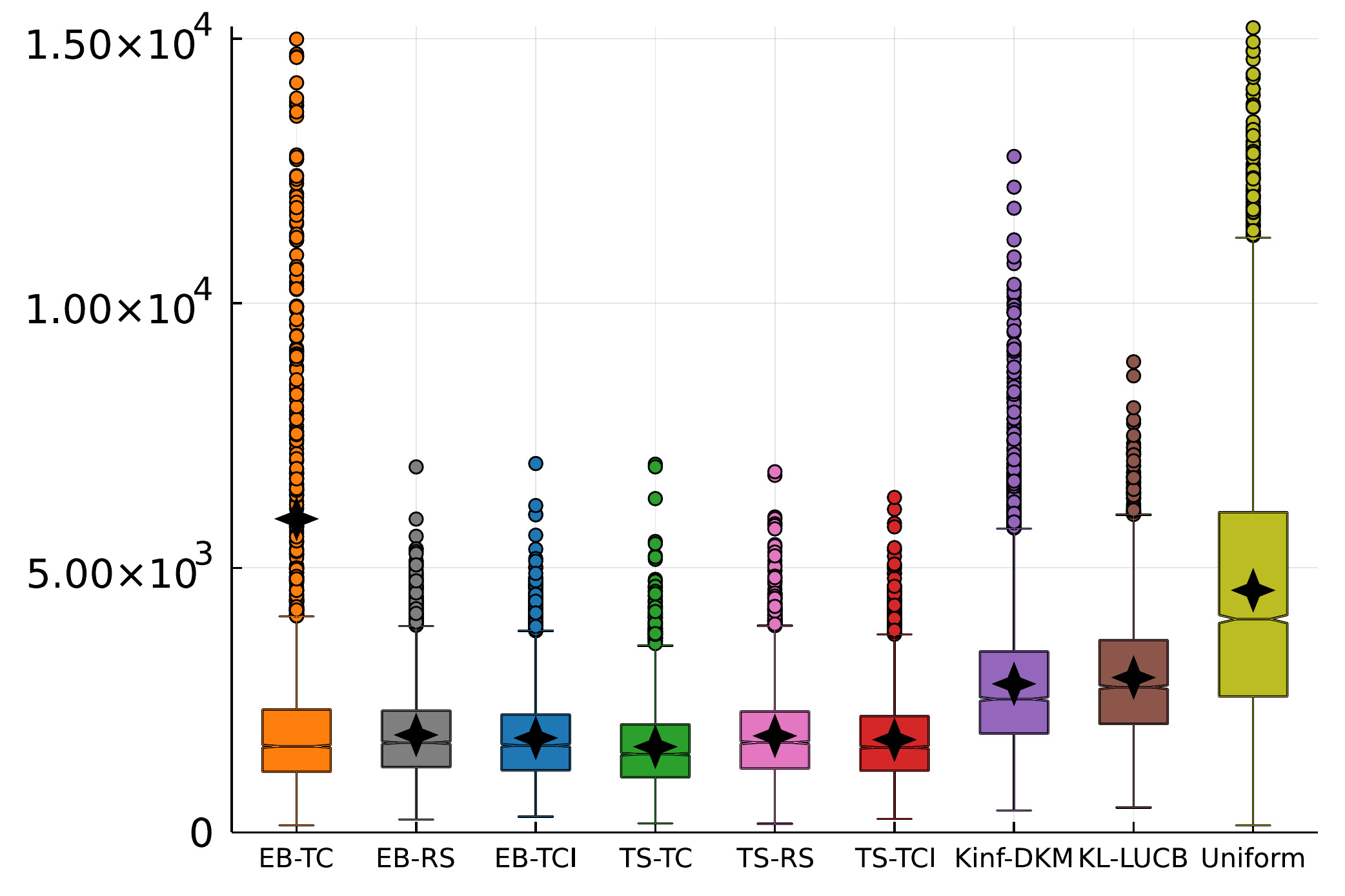}
	\caption{Empirical stopping time on random Bernoulli instances for $K \in (8, 10)$ (left to right).}
	\label{fig:random_bernoulli_instances_gk16}
\end{figure}

\paragraph{Random Bernoulli instances}
For $K \in \{6, 8, 10\}$, we sample $5000$ Bernoulli instances such that $\mu_{1} = 0.6$ and $\mu_{i} \sim \mathcal U ([0.2, 0.5])$ for all $i \neq 1$, where we enforce that $\Delta_{\min} \ge 0.01$.
For all other algorithms, Figure~\ref{fig:random_bernoulli_instances_gk16} delivers the same messages as Figure~\ref{fig:bernoulli_instances}.
Therefore, we refer the reader to Section~\ref{sec:experiments} for the corresponding comments.

Figure~\ref{fig:random_bernoulli_instances_gk16} confirms our theoretical intuition (Appendix~\ref{app:ss_beyond_all_distinct_means}) hinting that $\beta$-EB-TC is not an empirically robust algorithm, even for $\Delta_{\min} > 0$.
This is visible with the large number of outliers, which shift the mean empirical stopping time away from the median empirical stopping time.
Note that the $y$-axis was cut to provided visibility, hence it hides the largest outliers observed for $\beta$-EB-TC.

In Figure~\ref{fig:random_bernoulli_instances_gk16}, we see that $\beta$-EB-RS and $\beta$-TS-RS perform on par with $\beta$-EB-TCI, $\beta$-TS-TC and $\beta$-TS-TCI, while having few outliers.
This confirms our theoretical intuitions on the effect of randomization in the leader and/or challenger (see Appendix~\ref{app:ss_beyond_all_distinct_means}).

\begin{figure}[ht]
	\centering
	\includegraphics[width=0.495\linewidth]{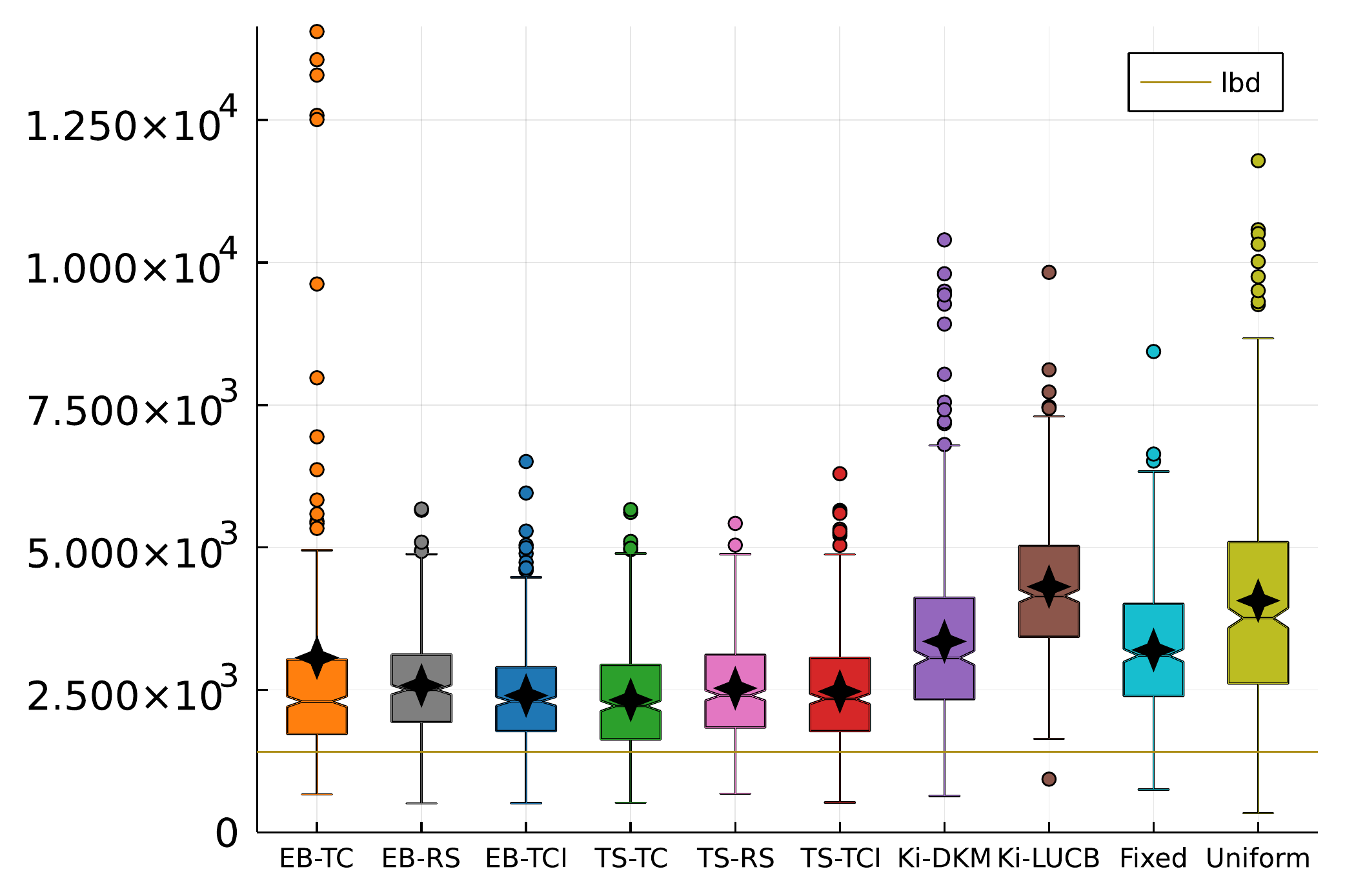}
	\includegraphics[width=0.495\linewidth]{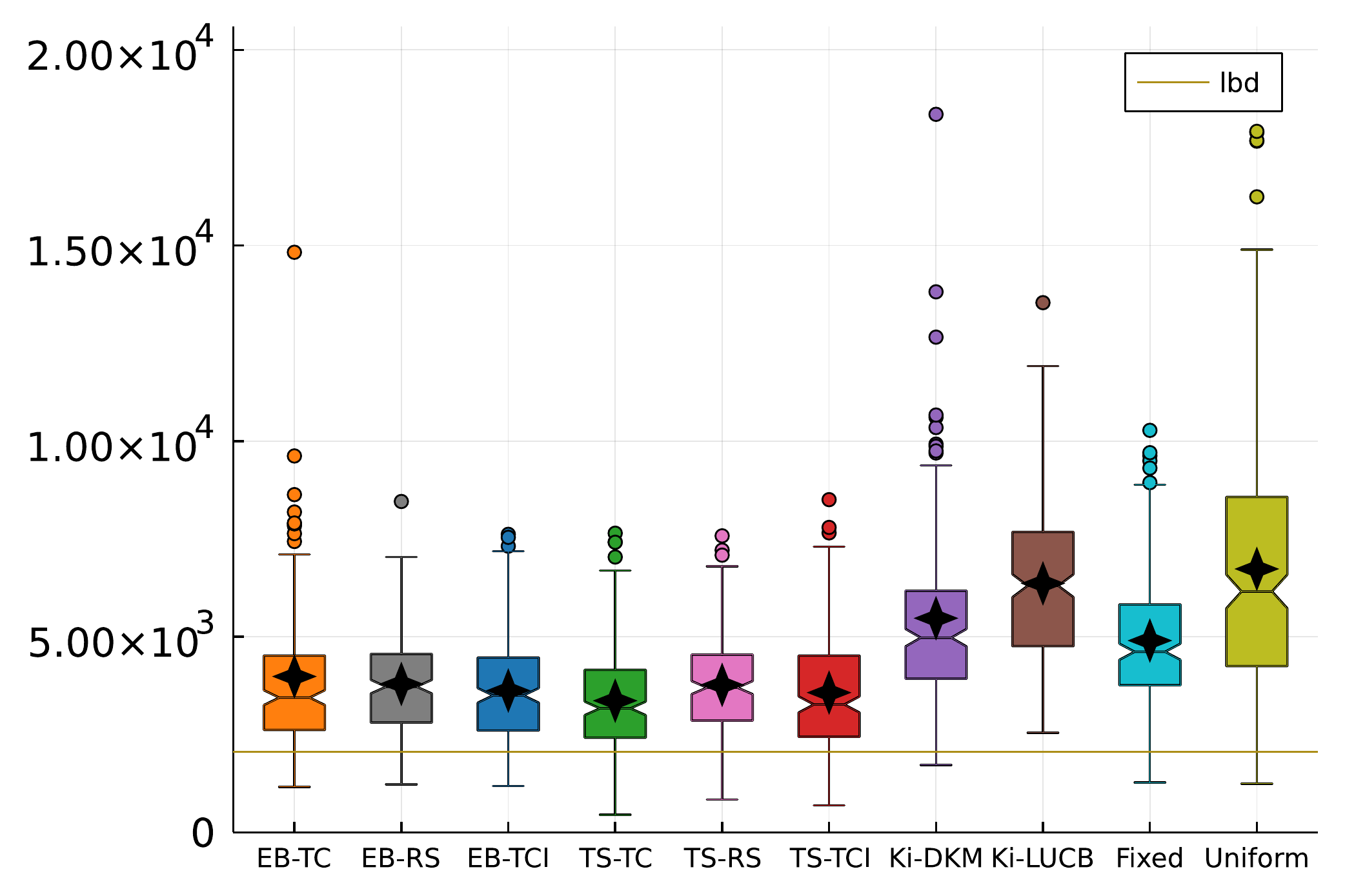}
	\\
	\includegraphics[width=0.495\linewidth]{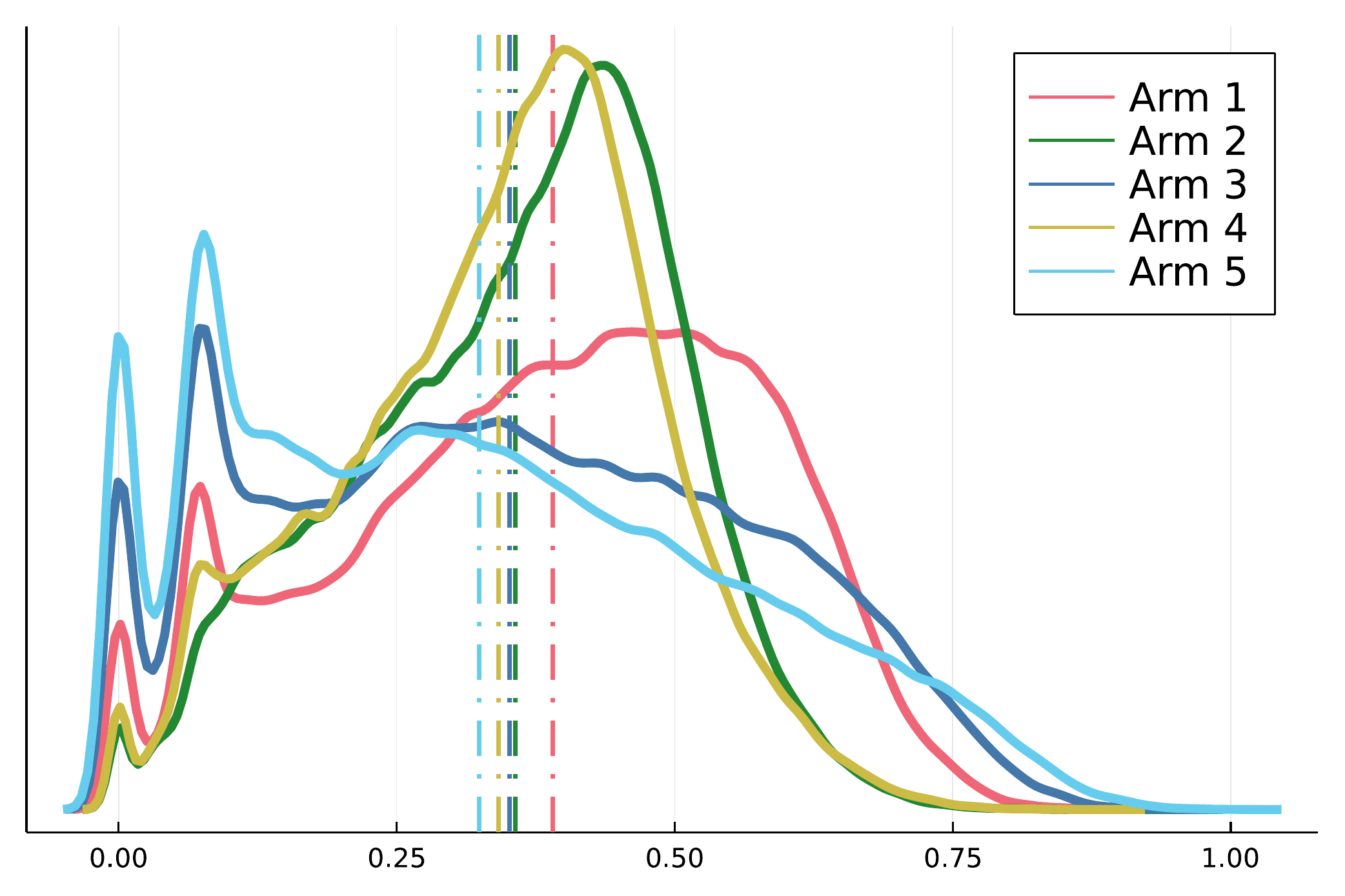}
	\includegraphics[width=0.495\linewidth]{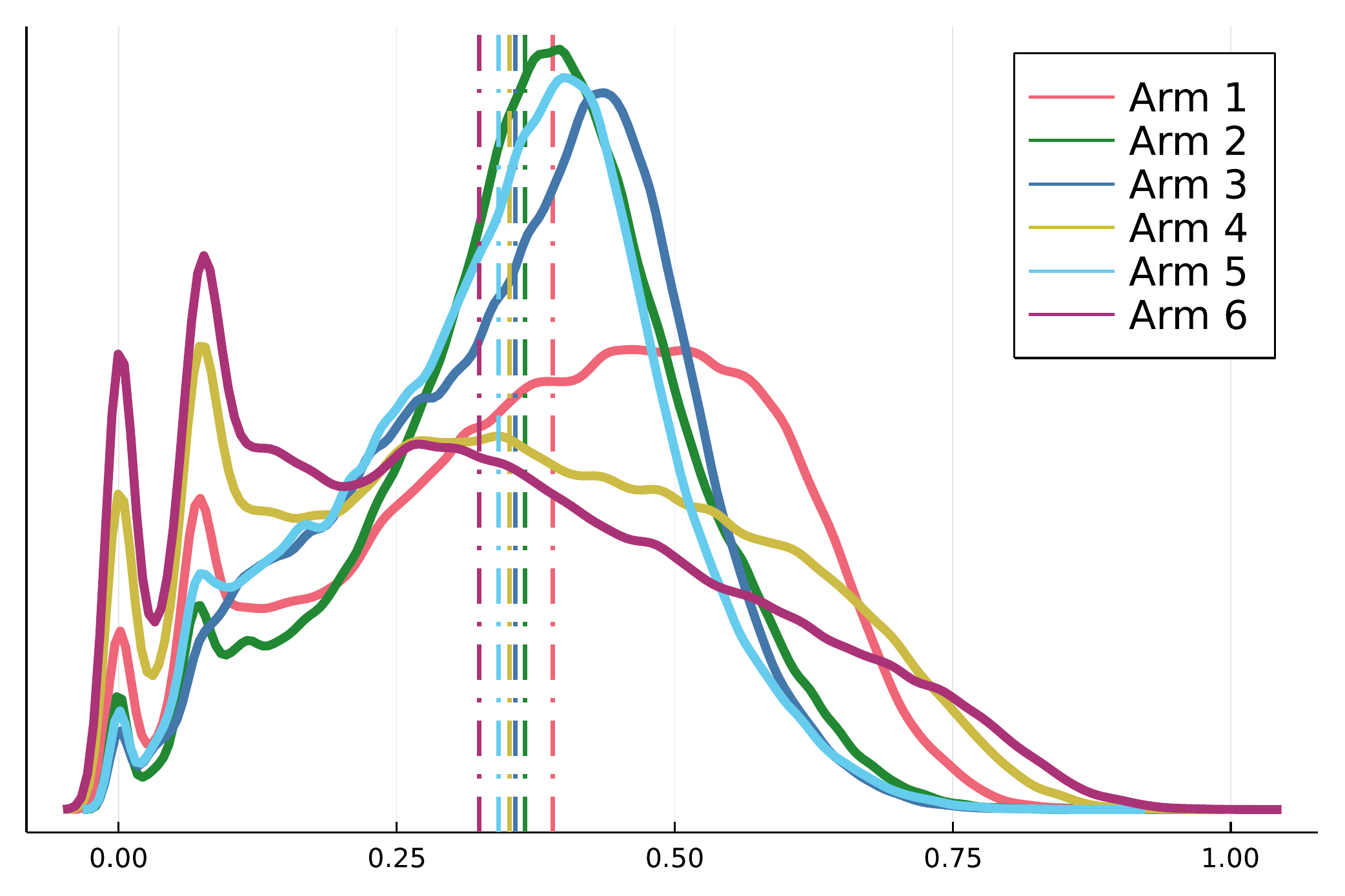}
	\caption{Empirical stopping time (top) on scaled DSSAT instances for $K \in (5,6)$ (left to right) with their density (bottom), where ``Ki'' stands for ``Kinf''.}
	\label{fig:dsat_instances_gk16}
\end{figure}

\paragraph{DSSAT instances}
We use the DSSAT real-world data for $K \in (5,6)$, which we scale by the overall maximum so that $B=1$.
Their histograms define instances with bounded distributions, see the bottom plots of Figure~\ref{fig:dsat_instances_gk16}, on which we can sample.
We average our results over $500$ runs for $K=5$ and $250$ runs for $K=6$.
For all other algorithms, Figure~\ref{fig:dsat_instances_gk16} delivers the same messages as Figure~\ref{fig:dsat_instances}.
Therefore, we refer the reader to Section~\ref{sec:experiments} for the corresponding comments.

In Figure~\ref{fig:dsat_instances_gk16}, we see that $\beta$-EB-RS and $\beta$-TS-RS perform on par with $\beta$-EB-TCI, $\beta$-TS-TC and $\beta$-TS-TCI.
For $K=5$, we observe large outliers on $\beta$-EB-TC.
This is a symptom of its empirical lack of robustness, which would be more striking if more runs had been performed.

In Section~\ref{sec:experiments}, we mentioned that KL-LUCB was performing ten times worse than $\Kinf$-LUCB on DSSAT instances.
With the same setting at Figure~\ref{fig:dsat_instances_gk16}, KL-LUCB used on average (standard deviation) a number of samples equal to $34635$ ($2860$) for $K=5$ and $57820$ ($ 5725$) for $K=6$.

\subsubsection{On the distinct means assumption}
\label{app:sss_expe_on_distinct_means}

In Figure~\ref{fig:simple_bernoulli_instances_gk16}, the considered Top Two algorithms have good empirical performance on the $3^{\text{rd}}$-equal instance.
This instance violates Assumption~\ref{ass:all_arms_distinct_bounded_mean}, i.e. $\Delta_{\min} = \min_{i\neq j} |\mu_i - \mu_j| > 0$, under which we can prove sufficient exploration for the Top Two algorithms.
We empirically study instances where $\Delta_{\min} = 0$ in order to confirm our intuition that some Top Two algorithms have good guarantees in this case (Appendix~\ref{app:ss_beyond_all_distinct_means}).

The most difficult instances having $\Delta_{\min} = 0$ are the ones where the arms having the same mean are in second position.
We consider four toy Bernoulli instances with three arms $\mu_{i} = (\mu_1, \mu_1 - \Delta_{i}, \mu_1 - \Delta_{i})$, where $\Delta_{i} > 0$.
The smaller $\Delta_{i}$, the harder the identification problem is.
Therefore, with smaller values of $\Delta_{i}$, it will be easier to see if Top Two algorithms are failing when $\Delta_{\min} = 0$.
In the experiments below, we consider $\mu_1 = 0.5$ and $\Delta_{i} \in \{0.1, 0.075, 0.05\}$ and average our results over $5000$ runs.

Since we aim at observing whether the algorithms get stuck (at least momentarily) without paying an infinite computational cost, we set a maximum number of iterations $T_{\max}$ for each run.
In our plots, the quantity $T^\star(\mu) \log\left( 1/\delta \right)$ acts as a lower bound.
Therefore, we set $T_{\max} = 15 T^\star(\mu) \log\left( 1/\delta \right)$.
Having an algorithm using more than $T_{\max}$ samples is a symptom of being stuck (at least momentarily).

We observe that Table~\ref{tab:percentage_achieving_Tmax_and_error}(a) and Table~\ref{tab:percentage_achieving_Tmax_and_error}(b) are very similar.
Reaching $T_{\max}$ is a symptom of failing to identifying the best arm $i^\star$.
As there is no $\delta$-correctness guaranty when reaching $T_{\max}$, it is expected that the algorithm recommend an arm different from $i^\star$.
In Table~\ref{tab:percentage_achieving_Tmax_and_error}(a), we see that $\beta$-EB-TC is the only algorithm reaching $T_{\max}$ over the $5000$ runs.
Empirically, this is the only Top Two algorithms that seem to fail on instances with $\Delta_{\min} = 0$.
In Table~\ref{tab:percentage_achieving_Tmax_and_error}(b), all algorithms (except $\beta$-EB-TC) have an empirical error lower than $\delta = 1\%$.

\begin{figure}[ht]
	\centering
	\includegraphics[width=0.32\linewidth]{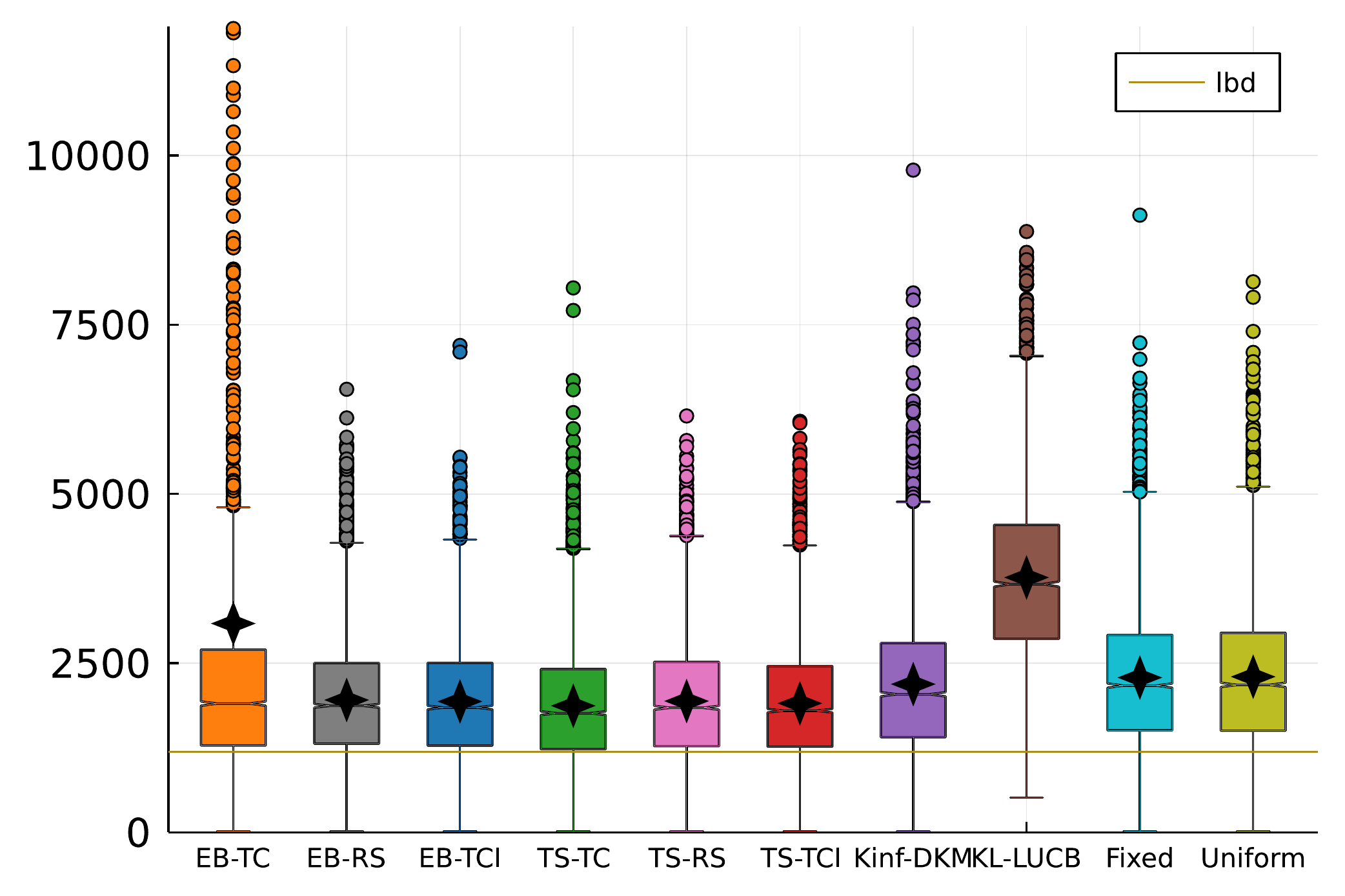}
	\includegraphics[width=0.32\linewidth]{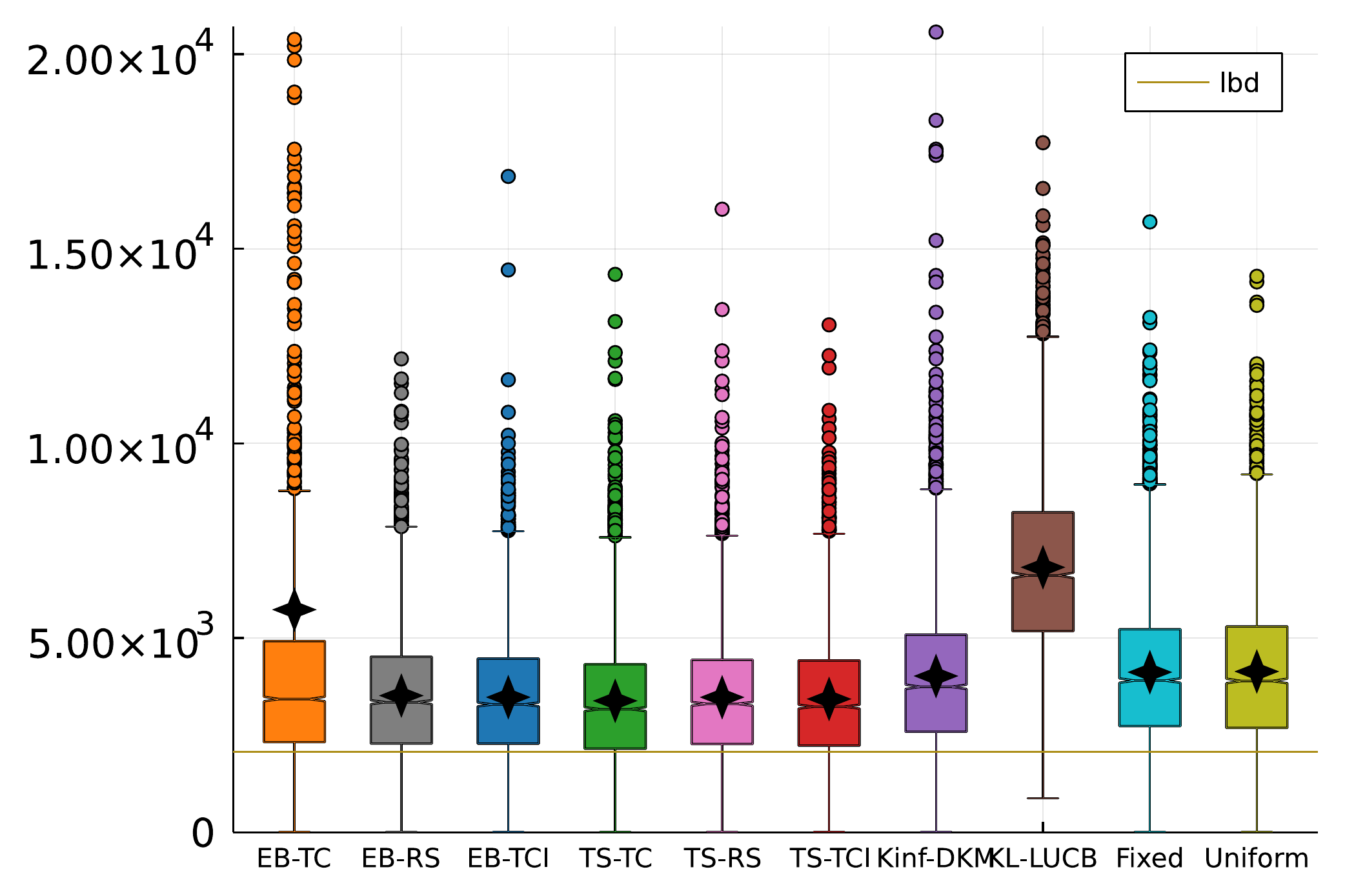}
	\includegraphics[width=0.32\linewidth]{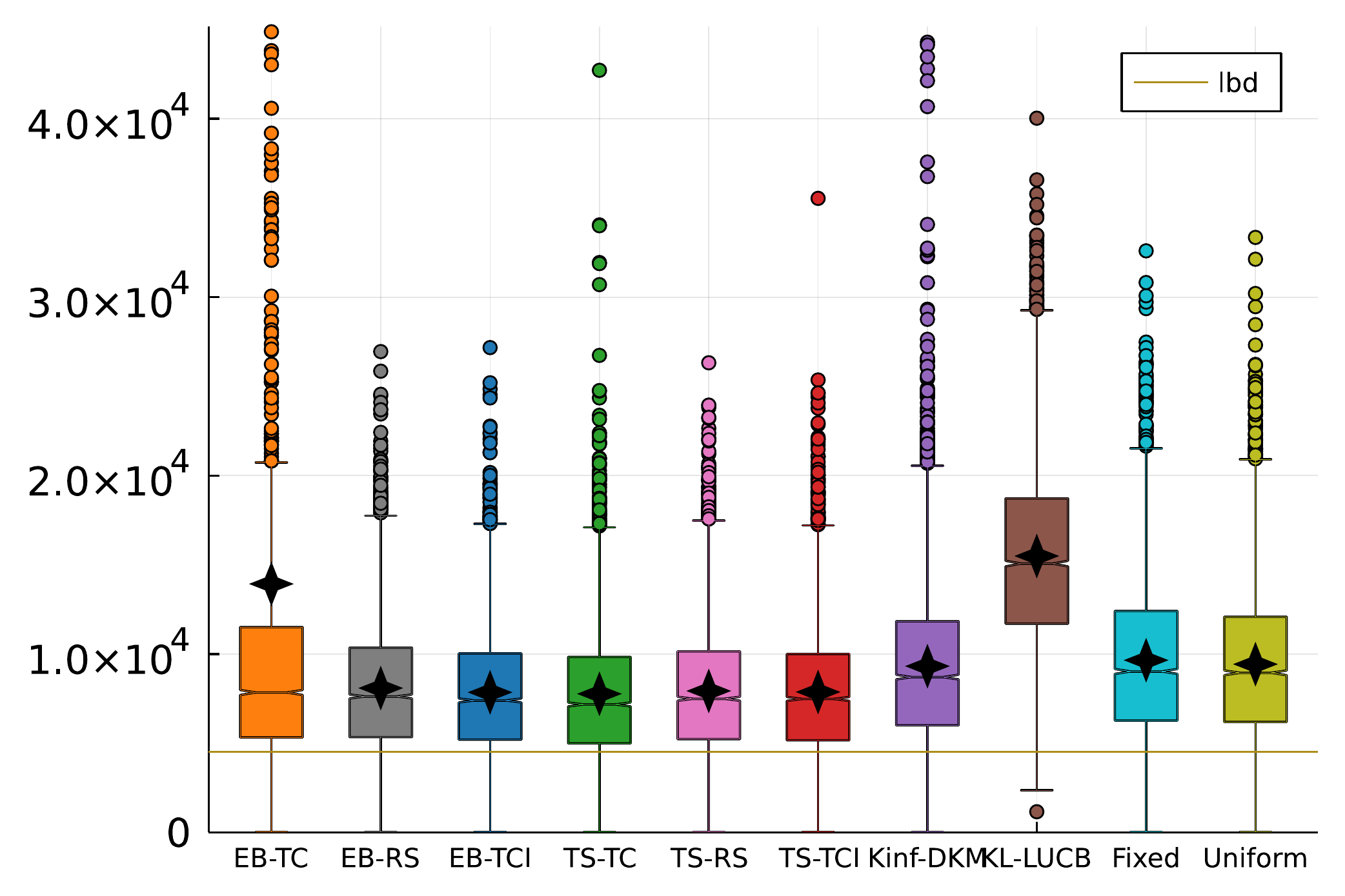}
	\caption{Empirical stopping time on Bernoulli with $\mu_{i} = (\mu_1, \mu_1 - \Delta_{i}, \mu_1 - \Delta_{i})$ for $\mu_1 = 0.5$ and $\Delta_{i} \in (0.1, 0.075, 0.05)$ (left to right).}
	\label{fig:equality_2nd_means_gk16}
\end{figure}

\begin{table}[ht]
\caption{Percentage (in $\%$) of runs (a) achieving $T_{\max}$ and (b) failing at identifying $i^\star = 1$ on Bernoulli with $\mu_{i} = (\mu_1, \mu_1 - \Delta_{i}, \mu_1 - \Delta_{i})$ for $\mu_1 = 0.5$ and $\Delta_{i} \in (0.1, 0.075, 0.05)$ for $\delta = 1\%$.}
\label{tab:percentage_achieving_Tmax_and_error}
\begin{center}
\begin{tabular}{c l l l}
	\toprule
  &  $\Delta_1$ & $\Delta_2$ & $\Delta_3$  \\
 \cmidrule(l){2-4}
$\beta$-EB-TC   & $6.66$ & $7.62$ & $9.56$  \\
$\beta$-EB-RS   & $0$ & $0$ & $0$  \\
$\beta$-EB-TCI   & $0$ & $0$ & $0$   \\
$\beta$-TS-TC   & $0$ & $0$ & $0$   \\
$\beta$-TS-RS   & $0$ & $0$ & $0$   \\
$\beta$-TS-TCI   & $0$ & $0$ & $0$  \\
$\Kinf$-DKM   & $0$ & $0$ & $0$  \\
KL-LUCB   & $0$ & $0$ & $0$  \\
Fixed   & $0$ & $0$ & $0$ \\
Uniform   & $0$ & $0$ & $0$  \\
\bottomrule
\end{tabular}
\hspace{0.5cm}
\begin{tabular}{c l l l}
	\toprule
 &  $\Delta_1$ & $\Delta_2$ & $\Delta_3$  \\
\cmidrule(l){2-4}
EB-TC   & $6.74$ & $7.66$ & $9.54$   \\
EB-RS   & $0.02$ & $0.02$ & $0.02$  \\
EB-TCI   & $0.04$ & $0.04$ & $0.06$  \\
TS-TC   & $0.04$ & $0.06$ & $0.1$    \\
TS-RS   & $0.02$ & $0.04$ & $0.04$    \\
TS-TCI   & $0.04$ & $0.08$ & $0.08$    \\
$\Kinf$-DKM   & $0.06$ & $0.12$ & $0.1$  \\
KL-LUCB   & $0.02$ & $0.02$ & $0.06$   \\
Fixed   & $0.02$ & $0.02$ & $0.06$    \\
Uniform   & $0.04$ & $0.02$ & $0.02$   \\
\bottomrule
\end{tabular}
\end{center}
\end{table}

\paragraph{Lack of robustness of $\beta$-EB-TC}
Figure~\ref{fig:equality_2nd_means_gk16} confirms our theoretical intuition (Appendix~\ref{app:ss_beyond_all_distinct_means}) hinting that $\beta$-EB-TC is not empirically robust and can fail when $\Delta_{\min}=0$.
This is visible with the large number of outliers, which shift the mean empirical stopping time away from the median empirical stopping time.
In Table~\ref{tab:percentage_achieving_Tmax_and_error}, we see observe that $\beta$-EB-TC is the only algorithm reaching $T_{\max}$ and it does it frequently, i.e. between $6\%$ and $10\%$ in our experiments.

\paragraph{TCI challenger}
Figure~\ref{fig:equality_2nd_means_gk16} confirms our theoretical intuition (Appendix~\ref{app:ss_beyond_all_distinct_means})
that $\beta$-EB-TCI copes for the limitations of $\beta$-EB-TC, and that it should work when $\Delta_{\min}=0$.
Based on the comparison of $\beta$-EB-TC and $\beta$-EB-TCI in Figure~\ref{fig:equality_2nd_means_gk16} and Table~\ref{tab:percentage_achieving_Tmax_and_error}, we see that adding the $\ln(N_{n,j})$ term has a stabilization effect, hence reducing the number of outliers.

The difference between $\beta$-TS-TC and $\beta$-TS-TCI is milder.
However, based on Figure~\ref{fig:equality_2nd_means_gk16}, it seems that adding the $\ln(N_{n,j})$ term slightly reduces the number of large outliers.
This effect is less visible than when comparing $\beta$-EB-TC and $\beta$-EB-TCI due to the stabilization effect ensured by the randomization in the TS leader.

\paragraph{Randomized leader or challenger}
Figure~\ref{fig:equality_2nd_means_gk16} confirms our theoretical intuition (Appendix~\ref{app:ss_beyond_all_distinct_means}) 
that randomized mechanisms have a stabilization effect, and that they should work when $\Delta_{\min}=0$.
Based on Figure~\ref{fig:equality_2nd_means_gk16} and Table~\ref{tab:percentage_achieving_Tmax_and_error}, the TS leader or the RS challenger appear to prevent large outliers.
For the TS leader, this effect is particularly striking when comparing the performance of $\beta$-EB-TC and $\beta$-TS-TC.
For the RS challenger, the effect is striking when comparing $\beta$-EB-TC and $\beta$-EB-RS, and milder between $\beta$-TS-TC and $\beta$-TS-RS.

\paragraph{Symmetric instances}
When the two sub-optimal arms have the same mean, we have $w^\star(\mu)_{2} = w^\star(\mu)_{3} = (1 - w^\star(\mu)_{1})/2$ by symmetry of the characteristic time $T^\star(\mu)$.
Therefore, the optimal allocation is close to the uniform allocation $(1/3,1/3,1/3)$.
Experimental results (Figure~\ref{fig:bernoulli_instances}(b) and Figure~\ref{fig:equality_2nd_means_gk16}) show that the uniform sampling performs on par with the fixed oracle algorithm tracking $w^\star(\mu)$.
Therefore, it is not surprising that KL-LUCB performs worse than uniform sampling.

\subsubsection{On larger sets of arms}
\label{app:sss_expe_larger_sets_arms}

Another interesting question that arises as regards our algorithms is to assess whether their performance scales with the number of arms.
We consider the three problem scenarios from \citet{Jamieson14Survey} with varying size of arms.
The underlying distributions are Gaussian with mean $\mu \in \R^{K}$ and hardness $H_1 \eqdef \sum_{i \neq i^\star} \frac{1}{(\mu_{i^\star} - \mu_i)^2}$.
The ``$1$-sparse'' scenario sets $\mu_1 = 1/4$ and $\mu_{i} = 0$ for all $i \in [K] \setminus \{1\}$, resulting in an hardness $H_1  \approx 4 K$.
The ``$\alpha = 0.3$'' and ``$\alpha = 0.6$'' scenarios consider $\mu_i = 1 - \left(\frac{i-1}{K-1}\right)^{\alpha}$ for all $i \in [K]$, with respective hardness $H_1  \approx 3K/2$ and $H_1  \approx 6 K^{1.2}$.
We only consider algorithms whose computational cost scales nicely with the number of arms, namely $\beta$-EB-TCI, $\beta$-TS-TC, LUCB and uniform sampling.
We choose $\delta=0.01$, $\beta = 1/2$ and average our results over $100$ runs.

\begin{figure}[ht]
	\centering
	\includegraphics[width=0.32\linewidth]{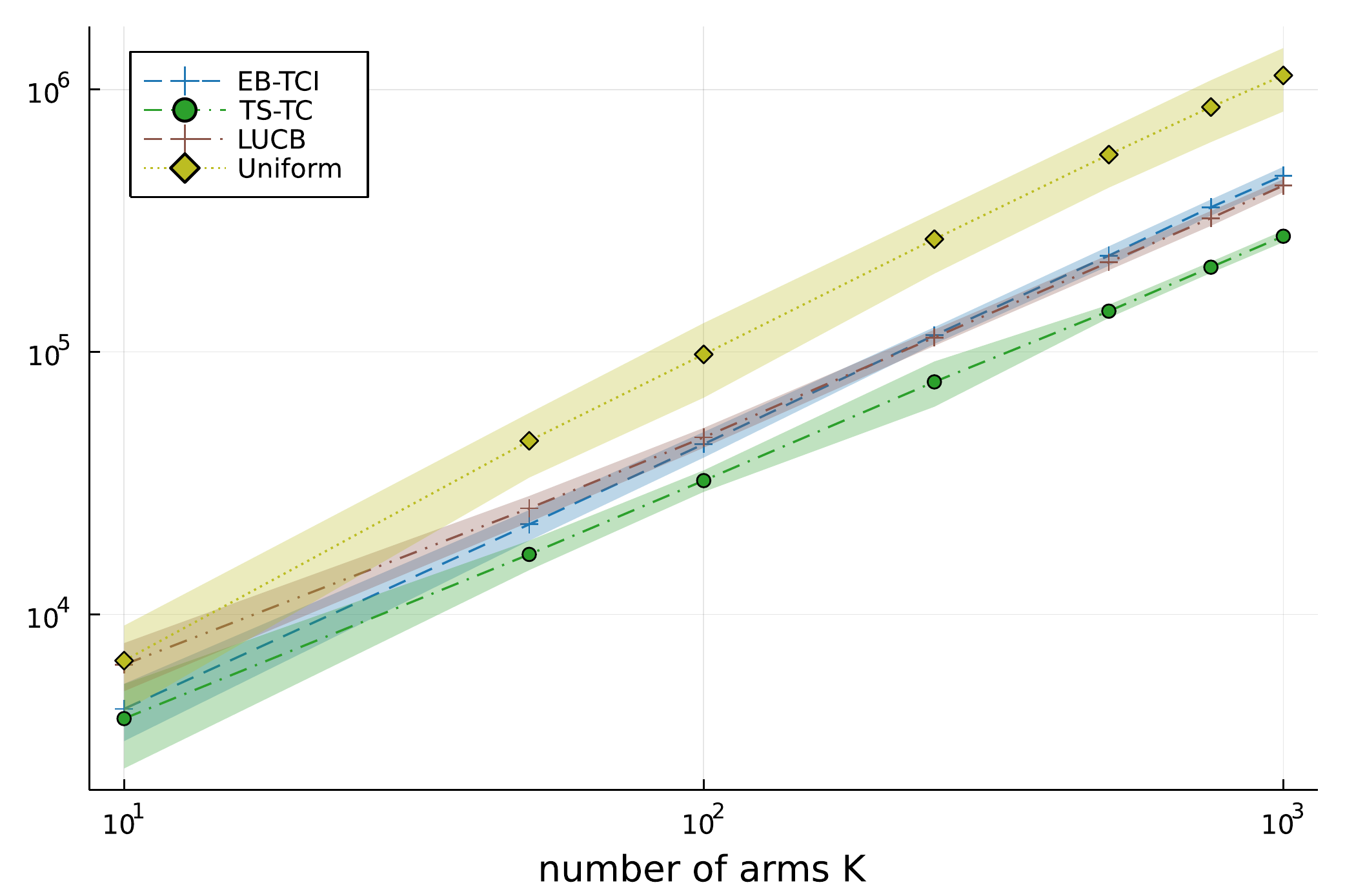}
	\includegraphics[width=0.32\linewidth]{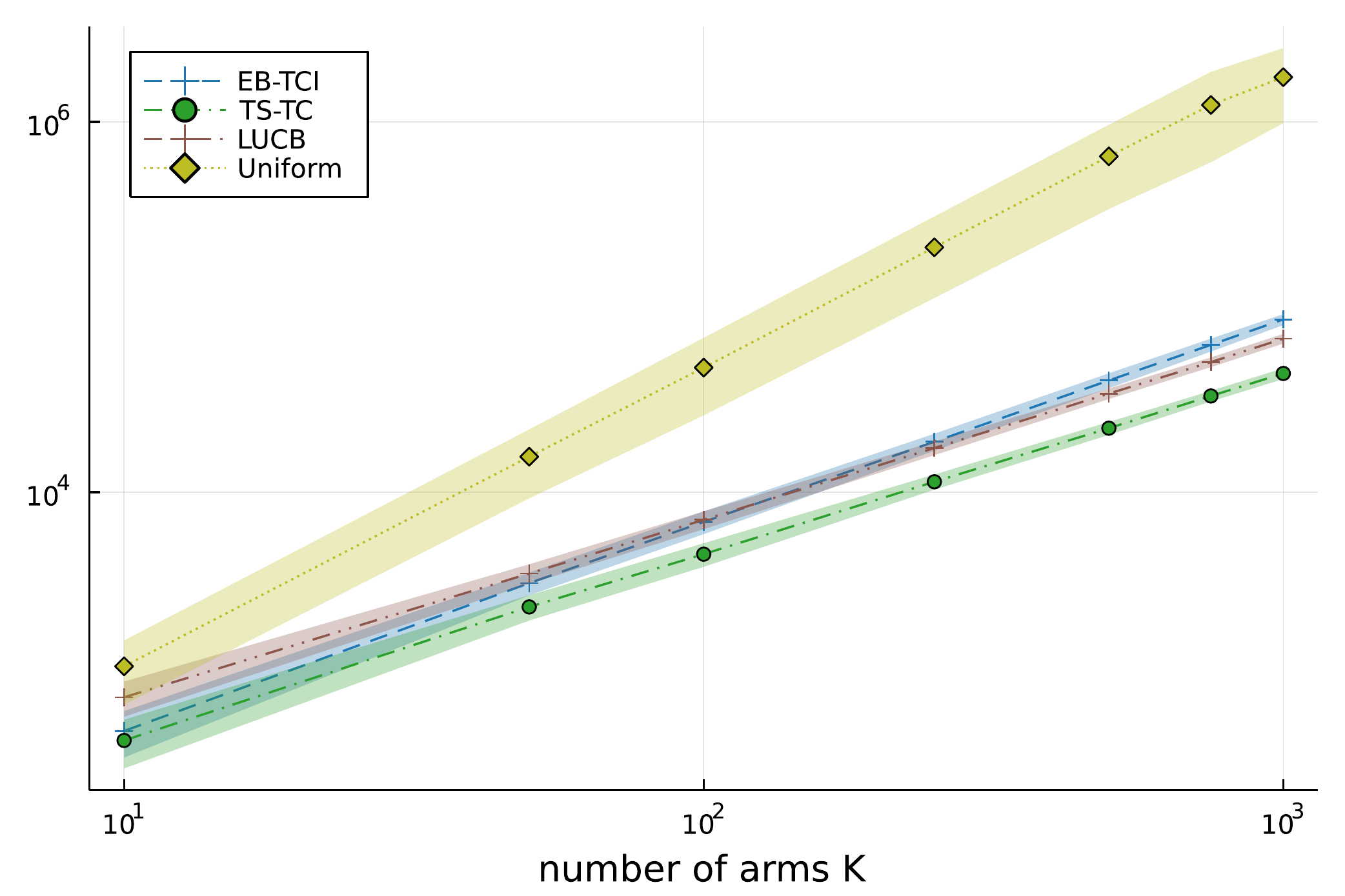}
	\includegraphics[width=0.32\linewidth]{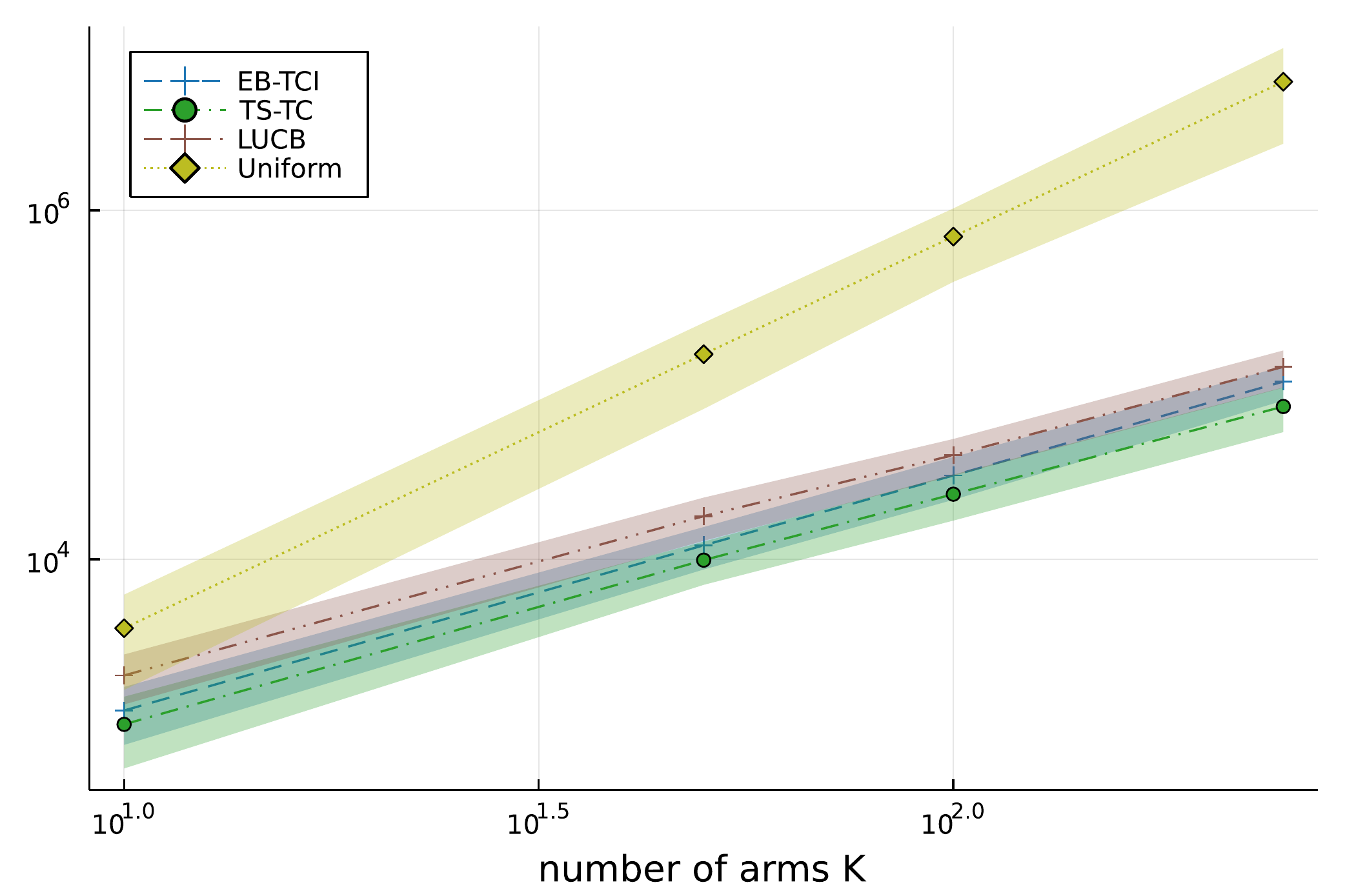}
	\caption{Empirical stopping time for the Gaussian benchmarks (left) ``$1$-sparse'', (middle) ``$\alpha = 0.3$'' and (right) ``$\alpha = 0.6$''.}
	\label{fig:larger_sets_arms}
\end{figure}

In Figure~\ref{fig:larger_sets_arms}, we observe that the performances of $\beta$-EB-TCI, $\beta$-TS-TC and LUCB scale proportionally to the hardness $H_1$ when the number of arms increase, while it is worsening for the uniform sampling.
Surprisingly, the performance gap between $\beta$-EB-TCI and LUCB is diminishing with the number of arms. For larger experiments, LUCB seems to slightly outperform $\beta$-EB-TCI.
Finally, $\beta$-TS-TC significantly outperforms all the other algorithms when the number of arms increase.

%% file: npbai.bbl
\begin{thebibliography}{43}
\providecommand{\natexlab}[1]{#1}
\providecommand{\url}[1]{\texttt{#1}}
\expandafter\ifx\csname urlstyle\endcsname\relax
  \providecommand{\doi}[1]{doi: #1}\else
  \providecommand{\doi}{doi: \begingroup \urlstyle{rm}\Url}\fi

\bibitem[Agrawal and Goyal(2012)]{AGCOLT12}
S.~Agrawal and N.~Goyal.
\newblock {Analysis of Thompson Sampling for the multi-armed bandit problem}.
\newblock In \emph{{Proceedings of the 25th Conference On Learning Theory}},
  2012.

\bibitem[Agrawal and Goyal(2013)]{AGAISTAT13}
S.~Agrawal and N.~Goyal.
\newblock {Further Optimal Regret Bounds for Thompson Sampling}.
\newblock In \emph{{Proceedings of the 16th Conference on Artificial
  Intelligence and Statistics}}, 2013.

\bibitem[Agrawal et~al.(2020)Agrawal, Juneja, and Glynn]{Agrawal20GeneBAI}
S.~Agrawal, S.~Juneja, and P.~W. Glynn.
\newblock Optimal {$\delta$}-correct best-arm selection for heavy-tailed
  distributions.
\newblock In \emph{Algorithmic Learning Theory (ALT)}, 2020.

\bibitem[Agrawal et~al.(2021{\natexlab{a}})Agrawal, Juneja, and
  Koolen]{Agrawal21Regret}
S.~Agrawal, S.~K. Juneja, and W.~M. Koolen.
\newblock Regret minimization in heavy-tailed bandits.
\newblock In \emph{Conference on Learning Theory}, pages 26--62. PMLR,
  2021{\natexlab{a}}.

\bibitem[Agrawal et~al.(2021{\natexlab{b}})Agrawal, Koolen, and
  Juneja]{agrawal2021optimal}
S.~Agrawal, W.~M. Koolen, and S.~Juneja.
\newblock Optimal best-arm identification methods for tail-risk measures.
\newblock \emph{Advances in Neural Information Processing Systems}, 34,
  2021{\natexlab{b}}.

\bibitem[Audibert et~al.(2010)Audibert, Bubeck, and Munos]{Bubeck10BestArm}
J.-Y. Audibert, S.~Bubeck, and R.~Munos.
\newblock {Best Arm Identification in Multi-armed Bandits}.
\newblock In \emph{{Proceedings of the 23rd Conference on Learning Theory}},
  2010.

\bibitem[Baudry et~al.(2021{\natexlab{a}})Baudry, Gautron, Kaufmann, and
  Maillard]{baudry21a}
D.~Baudry, R.~Gautron, E.~Kaufmann, and O.~Maillard.
\newblock Optimal thompson sampling strategies for support-aware cvar bandits.
\newblock In \emph{Proceedings of the 38th International Conference on Machine
  Learning}, 2021{\natexlab{a}}.

\bibitem[Baudry et~al.(2021{\natexlab{b}})Baudry, Saux, and
  Maillard]{Baudry21DS}
D.~Baudry, P.~Saux, and O.~Maillard.
\newblock From optimality to robustness: Dirichlet sampling strategies in
  stochastic bandits.
\newblock In \emph{Advances in Neural Information Processing Systems
  (NeurIPS)}, 2021{\natexlab{b}}.

\bibitem[Berge(1997)]{berge1997topological}
C.~Berge.
\newblock \emph{Topological Spaces: including a treatment of multi-valued
  functions, vector spaces, and convexity}.
\newblock Courier Corporation, 1997.

\bibitem[Brent(2013)]{brent2013algorithms}
R.~P. Brent.
\newblock \emph{Algorithms for minimization without derivatives}.
\newblock Courier Corporation, 2013.

\bibitem[Capp{\'e} et~al.(2013)Capp{\'e}, Garivier, Maillard, Munos, and
  Stoltz]{KLUCBJournal}
O.~Capp{\'e}, A.~Garivier, O.-A. Maillard, R.~Munos, and G.~Stoltz.
\newblock {{K}ullback-{L}eibler upper confidence bounds for optimal sequential
  allocation}.
\newblock \emph{Annals of Statistics}, 41(3):\penalty0 1516--1541, 2013.

\bibitem[Degenne and Koolen(2019)]{Degenne19Multiple}
R.~Degenne and W.~M. Koolen.
\newblock Pure exploration with multiple correct answers.
\newblock In \emph{Advances in Neural Information Processing Systems
  (NeurIPS)}, 2019.

\bibitem[Degenne et~al.(2019)Degenne, Koolen, and
  M{\'{e}}nard]{Degenne19GameBAI}
R.~Degenne, W.~M. Koolen, and P.~M{\'{e}}nard.
\newblock Non-asymptotic pure exploration by solving games.
\newblock In \emph{Advances in Neural Information Processing Systems
  (NeurIPS)}, 2019.

\bibitem[Even-Dar et~al.(2006)Even-Dar, Mannor, and Mansour]{EvenDaral06}
E.~Even-Dar, S.~Mannor, and Y.~Mansour.
\newblock {Action Elimination and Stopping Conditions for the Multi-Armed
  Bandit and Reinforcement Learning Problems}.
\newblock \emph{Journal of Machine Learning Research}, 7:\penalty0 1079--1105,
  2006.

\bibitem[Gabillon et~al.(2012)Gabillon, Ghavamzadeh, and
  Lazaric]{Gabillon12UGapE}
V.~Gabillon, M.~Ghavamzadeh, and A.~Lazaric.
\newblock {Best Arm Identification: A Unified Approach to Fixed Budget and
  Fixed Confidence}.
\newblock In \emph{{Advances in Neural Information Processing Systems}}, 2012.

\bibitem[Garivier and Kaufmann(2016)]{GK16}
A.~Garivier and E.~Kaufmann.
\newblock Optimal best arm identification with fixed confidence.
\newblock In \emph{Proceedings of the 29th Conference On Learning Theory},
  2016.

\bibitem[Garivier et~al.(2018)Garivier, Hadiji, Menard, and
  Stoltz]{garivier2018kl}
A.~Garivier, H.~Hadiji, P.~Menard, and G.~Stoltz.
\newblock Kl-ucb-switch: optimal regret bounds for stochastic bandits from both
  a distribution-dependent and a distribution-free viewpoints.
\newblock \emph{arXiv preprint arXiv:1805.05071}, 2018.

\bibitem[Honda and Takemura(2010)]{HondaTakemura10}
J.~Honda and A.~Takemura.
\newblock {An Asymptotically Optimal Bandit Algorithm for Bounded Support
  Models}.
\newblock In \emph{{Proceedings of the 23rd Conference on Learning Theory}},
  2010.

\bibitem[Honda and Takemura(2011)]{honda2011asymptotically}
J.~Honda and A.~Takemura.
\newblock An asymptotically optimal policy for finite support models in the
  multiarmed bandit problem.
\newblock \emph{Machine Learning}, 85\penalty0 (3):\penalty0 361--391, 2011.

\bibitem[Honda and Takemura(2015)]{Honda15IMED}
J.~Honda and A.~Takemura.
\newblock Non-asymptotic analysis of a new bandit algorithm for semi-bounded
  rewards.
\newblock \emph{Journal of Machine Learning Research}, 16:\penalty0 3721--3756,
  2015.

\bibitem[Hong et~al.(2021)Hong, Fan, and Luo]{Hong21Survey}
L.~Hong, W.~Fan, and J.~Luo.
\newblock Review on ranking and selection: A new perspective.
\newblock \emph{Frontiers of Engineering Management}, 8:\penalty0 321–343,
  2021.

\bibitem[Hoogenboom et~al.(2019)Hoogenboom, Porter, Boote, Shelia, Wilkens,
  Singh, White, Asseng, Lizaso, Moreno, et~al.]{hoogenboom2019dssat}
G.~Hoogenboom, C.~Porter, K.~Boote, V.~Shelia, P.~Wilkens, U.~Singh, J.~White,
  S.~Asseng, J.~Lizaso, L.~Moreno, et~al.
\newblock The dssat crop modeling ecosystem.
\newblock \emph{Advances in crop modelling for a sustainable agriculture},
  pages 173--216, 2019.

\bibitem[Jamieson et~al.(2014)Jamieson, Malloy, Nowak, and
  Bubeck]{Jamiesonal14LILUCB}
K.~Jamieson, M.~Malloy, R.~Nowak, and S.~Bubeck.
\newblock {lil'{UCB}: an Optimal Exploration Algorithm for Multi-Armed
  Bandits}.
\newblock In \emph{{Proceedings of the 27th Conference on Learning Theory}},
  2014.

\bibitem[Jamieson and Nowak(2014)]{Jamieson14Survey}
K.~G. Jamieson and R.~D. Nowak.
\newblock Best-arm identification algorithms for multi-armed bandits in the
  fixed confidence setting.
\newblock In \emph{{CISS}}, pages 1--6. {IEEE}, 2014.

\bibitem[Jourdan and Degenne(2022)]{jourdan2022choosing}
M.~Jourdan and R.~Degenne.
\newblock Choosing answers in $\varepsilon$-best-answer identification for
  linear bandits.
\newblock In \emph{Proceedings of the 39th International Conference on Machine
  Learning}, volume 162, 2022.

\bibitem[Jourdan et~al.(2021)Jourdan, Mutn{\`y}, Kirschner, and
  Krause]{jourdan_2021_EfficientPureExploration}
M.~Jourdan, M.~Mutn{\`y}, J.~Kirschner, and A.~Krause.
\newblock Efficient pure exploration for combinatorial bandits with semi-bandit
  feedback.
\newblock In \emph{Algorithmic Learning Theory (ALT)}, 2021.

\bibitem[Kalyanakrishnan et~al.(2012)Kalyanakrishnan, Tewari, Auer, and
  Stone]{Shivaramal12}
S.~Kalyanakrishnan, A.~Tewari, P.~Auer, and P.~Stone.
\newblock {{PAC} subset selection in stochastic multi-armed bandits}.
\newblock In \emph{{International Conference on Machine Learning (ICML)}},
  2012.

\bibitem[Kaufmann and Kalyanakrishnan(2013)]{COLT13}
E.~Kaufmann and S.~Kalyanakrishnan.
\newblock {Information complexity in bandit subset selection}.
\newblock In \emph{{Proceeding of the 26th Conference On Learning Theory.}},
  2013.

\bibitem[Kaufmann and Koolen(2021)]{KK18Mixtures}
E.~Kaufmann and W.~M. Koolen.
\newblock Mixture martingales revisited with applications to sequential tests
  and confidence intervals.
\newblock \emph{Journal of Machine Learning Research}, 22\penalty0
  (246):\penalty0 1--44, 2021.

\bibitem[Kaufmann et~al.(2014)Kaufmann, Capp{\'e}, and Garivier]{COLT14}
E.~Kaufmann, O.~Capp{\'e}, and A.~Garivier.
\newblock {On the Complexity of A/B Testing}.
\newblock In \emph{{Proceedings of the 27th Conference On Learning Theory}},
  2014.

\bibitem[Massart(1990)]{massart1990}
P.~Massart.
\newblock The tight constant in the dvoretzky-kiefer-wolfowitz inequality.
\newblock \emph{Annals of Probability}, 18, 1990.

\bibitem[M{\'e}nard and Garivier(2017)]{menard2017minimax}
P.~M{\'e}nard and A.~Garivier.
\newblock A minimax and asymptotically optimal algorithm for stochastic
  bandits.
\newblock In \emph{International Conference on Algorithmic Learning Theory},
  pages 223--237. PMLR, 2017.

\bibitem[Mukherjee and Tajer(2022)]{mukherjee_2022_SPRTBAI}
A.~Mukherjee and A.~Tajer.
\newblock Sprt-based efficient best arm identification in stochastic bandits.
\newblock 2022.

\bibitem[Posner(1975)]{posner1975random}
E.~Posner.
\newblock Random coding strategies for minimum entropy.
\newblock \emph{IEEE Transactions on Information Theory}, 21\penalty0
  (4):\penalty0 388--391, 1975.

\bibitem[Qin et~al.(2017)Qin, Klabjan, and Russo]{Qin2017TTEI}
C.~Qin, D.~Klabjan, and D.~Russo.
\newblock {Improving the expected improvement algorithm}.
\newblock In \emph{Advances in Neural Information Processing Systems 30
  (NIPS)}, 2017.

\bibitem[Riou and Honda(2020)]{RiouHonda20}
C.~Riou and J.~Honda.
\newblock Bandit algorithms based on thompson sampling for bounded reward
  distributions.
\newblock In \emph{Algorithmic Learning Theory (ALT)}, 2020.

\bibitem[Russac et~al.(2021)Russac, Katsimerou, Bohle, Capp{\'{e}}, Garivier,
  and Koolen]{Russac21ABn}
Y.~Russac, C.~Katsimerou, D.~Bohle, O.~Capp{\'{e}}, A.~Garivier, and W.~M.
  Koolen.
\newblock A/b/n testing with control in the presence of subpopulations.
\newblock In \emph{Advances in Neural Information Processing Systems
  (NeurIPS)}, 2021.

\bibitem[Russo(2016)]{Russo2016TTTS}
D.~Russo.
\newblock {Simple Bayesian algorithms for best arm identification}.
\newblock In \emph{Proceedings of the 29th Conference on Learning Theory
  (COLT)}, 2016.

\bibitem[Shang et~al.(2020)Shang, de~Heide, Kaufmann, M{\'{e}}nard, and
  Valko]{Shang20TTTS}
X.~Shang, R.~de~Heide, E.~Kaufmann, P.~M{\'{e}}nard, and M.~Valko.
\newblock Fixed-confidence guarantees for bayesian best-arm identification.
\newblock In \emph{International Conference on Artificial Intelligence and
  Statistics (AISTATS)}, 2020.

\bibitem[Sundaram et~al.(1996)]{sundaram1996first}
R.~K. Sundaram et~al.
\newblock \emph{A first course in optimization theory}.
\newblock Cambridge university press, 1996.

\bibitem[Thompson(1933)]{Thompson33}
W.~Thompson.
\newblock {On the likelihood that one unknown probability exceeds another in
  view of the evidence of two samples}.
\newblock \emph{Biometrika}, 25:\penalty0 285--294, 1933.

\bibitem[Vershynin(2018)]{vershynin18HDP}
R.~Vershynin.
\newblock \emph{High-Dimensional Probability: An Introduction with Applications
  in Data Science}.
\newblock Cambridge University Press, 2018.

\bibitem[Wang et~al.(2022)Wang, Yang, and You]{wang_2022_OptimalityConditions}
Z.~Wang, S.~Yang, and W.~You.
\newblock Optimality conditions and algorithms for top-k arm identification.
\newblock \emph{arXiv preprint arXiv:2205.12086}, 2022.

\end{thebibliography}
